\documentclass[twoside,11pt]{article}

\usepackage{tikz}

\usepackage{pgfplots}

\pgfplotsset{compat=1.10}
\usepgfplotslibrary{fillbetween}
\usetikzlibrary{patterns}

\usepackage{anyfontsize}
\usepackage{hhline}
\usepackage{caption}
\usepackage{graphicx, subfigure}
\usepackage{bm} %
\usepackage{appendix}%
\usepackage{epsfig}
\usepackage{amsmath}
\usepackage{amsthm}
\usepackage{amsfonts}
\usepackage{mathrsfs}
\usepackage{bbm}
\usepackage{extarrows}  %
\usepackage{amssymb} %
\usepackage{verbatim}%
\usepackage{longtable} %
\usepackage{framed} %
\usepackage{color}
\usepackage{hyperref}

\usepackage{float}  
\usepackage{enumerate}%
\usepackage{fancyhdr}

\usepackage[top=2.5cm, bottom=2.5cm, left=3cm, right=3cm]{geometry}

\usepackage{graphicx}

\usepackage{arydshln}
\allowdisplaybreaks[4]

\newcommand{\fdnn}{\mc{F}^{\mathbf{FNN}}}

\newcommand{\FNN}{\mathbf{FNN}}

\newcommand{\e}{\varepsilon}

\newcommand{\mbR}{\mathbb{R}}
\newcommand{\idf}{\mathbbm{1}}  %
\newcommand{\mc}{\mathcal}

\newcommand{\ea}{ \end{aligned}}
\newcommand{\ba}{ \begin{aligned}}
\newcommand{\fltl}[1]{\left\{\ba{}#1\ea\right.}
\newcommand{\eeq}{\end{aligned}\end{equation}}
\newcommand{\beq}{\begin{equation}\begin{aligned}}
\newcommand{\mb}{\mathbb}
\newcommand{\mbe}{\mb{E}}
\newcommand{\mr}{\mathrm}

\newcommand{\me}{\mr{e}}

\newcommand{\abs}[1]{\left|#1\right|}
\newcommand{\ceil}[1]{\left \lceil#1\right\rceil} %
\newcommand{\flr}[1]{\left \lfloor#1\right\rfloor} %
\newcommand{\norm}[1]{\left\|#1\right\|}
\newcommand{\ykh}[1]{\left(#1\right)}
\newcommand{\zkh}[1]{\left[#1\right]}
\newcommand{\hkh}[1]{\left\{#1\right\}}

\newcommand{\qd}{\vee}  %
\newcommand{\qx}{\wedge} %

\definecolor{darkgreen}{rgb}{0,0.5,0.1}

\definecolor{lightorange}{rgb}{0.8,0.6,0.1}

\newcommand{\hz}[1]{{\color{red} #1}}%

\newcommand{\setl}[2]{\hkh{\left.#1\right|#2}}
\newcommand{\setr}[2]{\hkh{#1\left|#2\right.}}

\newtheorem{thm}{Theorem}[section]

\newtheorem{lem}{Lemma}[section]
\newtheorem{cor}{Corollary}[section]

\newtheorem{assumption}{Assumption}[section]

\newcommand{\tf}{\tilde{f}}
\newcommand{\te}{\tilde{\eta}}
\newcommand{\tll}{\tilde{l}}

\def\begeqn{\begin{equation}}
	\def\endeqn{\end{equation}}
\def\begth{\begin{theorem}}
	\def\endth{\end{theorem}}
\def\begprop{\begin{proposition}}
	\def\endprop{\end{proposition}}
\def\begdef{\begin{definition}}
	\def\enddef{\end{definition}}
\def\beglemm{\begin{lemma}}
	\def\endlemm{\end{lemma}}
\def\begexm{\begin{example}}
	\def\endexm{\end{example}}
\def\begrem{\begin{remark}}
	\def\endrem{\end{remark}}
\def\begassum{\begin{assumption}}
	\def\endassum{\end{assumption}}

\usepackage{longtable}

\numberwithin{equation}{section}
\numberwithin{figure}{section}

\newcommand{\mybookmark}[2]{\hypersetup{bookmarksdepth}\pdfbookmark[-1]{#1}{#2}\hypersetup{bookmarksdepth=-2}}

\hypersetup{bookmarksdepth} %

\hypersetup{bookmarksdepth=-2} %

\usepackage{lastpage}

\newcommand{\citep}[1]{{\text{$\ykh{{{\cite{#1}}}}$}}}

\begin{document}
	
	\usetikzlibrary{math}

\newcommand{\betikz}{\begin{tikzpicture}}
\newcommand{\eetikz}{\end{tikzpicture}}

\title{Classification with Deep Neural Networks and Logistic Loss
	$^\dag$\footnotetext{\dag~The work described in this paper is supported partially by InnoHK initiative, the Government of the HKSAR, Laboratory for AI-Powered
		Financial Technologies, the Research Grants Council of Hong Kong (Projects No. CityU 11308121, 11306220, 11308020), the Germany/Hong Kong
		Joint Research Scheme (Project No. G-CityU101/20), the NSFC/RGC Joint Research Scheme (Project No. 12061160462 and N CityU102/20). Lei Shi
		is also supported by Shanghai Science and Technology Program (Project No. 21JC1400600). The first version of the paper was written and submitted
		in January 2022 when Ding-Xuan Zhou worked at City University of Hong Kong. Email addresses: zihanzhang19@fudan.edu.cn (Z. Zhang), leishi@fudan.edu.cn (L.
		Shi), dingxuan.zhou@sydney.edu.au (D-X. Zhou). The corresponding author is Lei Shi.}
}

\author{Zihan Zhang$^{1,3}$, Lei Shi$^2$  and Ding-Xuan Zhou$^4$\\
	\small $^1$ Shanghai Center for Mathematical Sciences, Fudan University, Shanghai, China \\
	\small $^2$ School of Mathematical Sciences and Shanghai Key Laboratory for Contemporary\\
	\small Applied Mathematics, Fudan University, Shanghai, China\\
	\small $^3$ School of Data Science, City University of Hong Kong, Kowloon, Hong Kong\\
	\small $^4$ School of Mathematics and Statistics, 
	University of Sydney, Sydney, Australia}

\maketitle

\begin{abstract}%
Deep neural networks (DNNs) trained with the logistic loss (also known as the cross entropy loss) have made impressive advancements in various binary classification tasks. Despite the considerable success in practice, generalization analysis for binary  classification with deep neural networks and the logistic loss remains scarce. The unboundedness of the target function for
the logistic loss in binary  classification is the main obstacle to deriving satisfactory generalization bounds. In this paper, we aim to fill this gap by developing a novel theoretical analysis and using it to establish tight generalization bounds for training fully connected ReLU DNNs with logistic loss in binary  classification. Our generalization analysis is based on an elegant oracle-type inequality which enables us to deal with the boundedness restriction of the target function.
Using this oracle-type inequality, we establish  generalization bounds for fully connected ReLU DNN  classifiers $\hat{f}^{\FNN}_n$ trained by  empirical logistic risk minimization with respect to  i.i.d. samples of size $n$, which lead to sharp rates of convergence as $n\to\infty$.  In particular,  we obtain optimal  convergence rates for $\hat{f}^{\FNN}_n$ (up to some logarithmic factor) only requiring the H\"{o}lder smoothness of the  conditional class probability $\eta$ of data. Moreover, we consider a compositional assumption that  requires $\eta$ to be  the composition of several vector-valued multivariate functions of which each component function is either a maximum value function or a H\"older smooth function only depending on a small number of its input variables. Under this assumption, we can even derive optimal convergence rates for $\hat{f}^{\FNN}_n$ (up to some logarithmic factor) which are independent of the input dimension of data. This result explains why in practice DNN classifiers can overcome the curse of dimensionality and perform well in high-dimensional  classification problems. Furthermore, we  establish dimension-free rates of convergence under other circumstances such as when the decision boundary is piecewise smooth and the input data are bounded away from it. Besides the novel oracle-type inequality, the sharp  convergence rates presented in our paper also owe to a tight error bound for approximating the natural logarithm function near zero (where it is unbounded) by ReLU DNNs. In addition, we justify our claims for the optimality of rates by proving corresponding minimax lower bounds.  All these results are new in the literature and will deepen our theoretical understanding of classification with deep neural networks. 
\end{abstract}

{\bf Keywords and phrases:} 
  deep learning; deep neural networks; binary classification; logistic loss; generalization analysis

\section{Introduction}\label{section: introduction}

In this paper, we study the  binary classification problem using deep neural networks (DNNs) with the rectified linear unit (ReLU) activation function. Deep learning based on DNNs has recently achieved remarkable success in a wide range of classification tasks including text categorization \citep{iyyer2015deep}, image classification \citep{krizhevsky2012imagenet}, and speech recognition \citep{hinton2012deep}, which has become a cutting-edge learning method. ReLU is one of the most popular activation functions, as scalable computing and stochastic optimization techniques can facilitate the training of ReLU DNNs \citep{han2015learning,kingma2014adam}. Given a positive integer $d$, consider the binary classification problem where we regard $[0,1]^d$ as the input space and $\{-1,1\}$ as the output space representing the two labels of input data. Let $P$ be a Borel probability measure on $[0,1]^d\times\{-1,1\}$, regarded as the data distribution (i.e., the joint distribution of the input and output data).  The goal of classification is to learn a real-valued  function  from a \emph{hypothesis space} $\mc{F}$ (i.e., a set of  candidate functions) based on the sample of the distribution $P$. The predictive performance of any (deterministic) real-valued function $f$ which has a Borel measurable restriction to $[0,1]^d$ (i.e., the domain of $f$ contains $[0,1]^d$, and $[0,1]^d\ni x\mapsto f(x)\in\mbR$ is Borel measurable) is measured by the \emph{misclassification error} of $f$ with respect to $P$, given by
\beq\ba\label{2212070301}
\mc{R}_P(f)&:={P}\ykh{\setl{(x,y)\in[0,1]^d\times\{-1,1\}}{y\neq\mr{sgn}(f(x))}},
\ea\eeq
or equivalently, the \emph{excess misclassification error} 
\beq \label{2212070302}
\mc{E}_P(f):=\mc{R}_P(f)-\inf \setr{\mc{R}_P(g) }{\textrm{$g:[0,1]^d\to\mbR$ is Borel measurable}}.
\eeq
Here $\mr{sgn}(\cdot)$ denotes the sign function which is defined as $\mr{sgn}(t)=1$ if $t\geq 0$ and $\mr{sgn}(t)=-1$ otherwise. The misclassification error  $\mc R_P(f)$ characterizes the probability that the binary classifier $\mr{sgn}\circ f$ makes a wrong prediction, where $\circ$ means function composition, and by a \emph{binary classifier} (or \emph{classifier} for short) we mean a $\hkh{-1,1}$-valued function whose domain contains the input space $[0,1]^d$. Since any real-valued function $f$ with its domain containing $[0,1]^d$ determines a classifier $\mr{sgn}\circ f$, we in this paper may call such a function $f$ a \emph{classifier} as well.

Note that the function we learn in a classification problem is based on the sample, meaning that it is not deterministic but a random function. Thus we take the expectation to measure its efficiency using the (excess) misclassification error. More specifically, let $\{(X_i,Y_i)\}_{i=1}^n$ be an independent and identically distributed (i.i.d.) sample of the  distribution $P$ and the hypothesis space $\mc F$ be a set of real-valued functions which have a Borel measurable restriction to $[0,1]^d$. We desire to construct an $\mc{F}$-valued statistic $\hat{f}_n$ from the sample $\{(X_i,Y_i)\}_{i=1}^n$ and the classification performance of  $\hat{f}_n$ can be characterized by upper bounds for the expectation of the excess misclassification error  $\mbe\zkh{\mc{E}_P(\hat{f}_n)}$.   One possible way to produce $\hat{f}_n$ is the \emph{empirical risk minimization} with some \emph{loss function} 
$\phi:\mbR\to[0,\infty)$, which is given by
\beq\label{ERM}
\hat{f}_n\in\mathop{\arg\min}_{f\in\mc{F}}\frac{1}{n}\sum_{i=1}^n\phi\ykh{Y_if(X_i)}.
\eeq If $\hat f_n$ satisfies \eqref{ERM}, then we will call $\hat f_n$ an \emph{empirical $\phi$-risk minimizer} (ERM with respect to $\phi$, or $\phi$-ERM) over $\mc F$.  
For any real-valued function $f$ which has a Borel measurable restriction to $[0,1]^d$,  the \emph{$\phi$-risk} and \emph{excess $\phi$-risk} of $f$ with respect to $P$, denoted by $\mc{R}_P^\phi(f)$ and $\mc{E}_P^\phi(f)$ respectively, are defined as
\beq\label{phirisk}
\mc{R}_P^\phi(f):=\int_{[0,1]^d\times\{-1,1\}}\phi(yf(x))\mr{d}P(x,y)
\eeq
and 
\beq\label{2212070303}
\mc{E}_P^\phi(f):=\mc{R}_P^\phi(f)-\inf \setl{\mc{R}^\phi_P(g)}{ \textrm{$g:[0,1]^d\to\mbR$ is Borel measurable}}.
\eeq 
To derive upper bounds  for $\mbe\zkh{\mc{E}_P(\hat{f}_n)}$, we can first establish upper bounds for $\mbe\zkh{\mc{E}_P^\phi(\hat{f}_n)}$, which are typically  controlled by two parts, namely the sample error and the approximation error (e.g., cf. Chapter 2 of \cite{cucker2007learning}). Then we are able to  bound  $\mbe\zkh{\mc{E}_P(\hat{f}_n)}$ by $\mbe\zkh{\mc{E}_P^\phi(\hat{f}_n)}$ through the so-called \emph{calibration inequality} (also known as \emph{Comparison Theorem}, see, e.g., Theorem 10.5 of \cite{cucker2007learning} and Theorem 3.22 of \cite{steinwart2008support}). In this paper, we will call any upper bound for $\mbe\zkh{\mc E_P(\hat f_n)}$ or $\mbe\zkh{\mc E_P^\phi(\hat f_n)}$ a \emph{generalization bound}.  

Note that $\lim\limits_{n\to\infty}\frac{1}{n}\sum_{i=1}^n\phi\ykh{Y_if(X_i)}=\mc R^\phi_P(f)$ almost surely for all measurable $f$. Therefore,  the empirical $\phi$-risk minimizer $\hat f_n$ defined in \eqref{ERM} can be regarded as an estimation of the so-called  \emph{target function} which minimizes the $\phi$-risk $\mc R^\phi_P$ over all Borel measurable functions $f$.  The target function can be defined pointwise.  Rigorously, we say a measurable function $f^*:[0,1]^d\to[-\infty,\infty]$ is a target function of the $\phi$-risk under the distribution $P$ if  for $P_X$-almost all $x\in[0,1]^d$ the value of $f^*$ at $x$ minimizes $\int_{\hkh{-1,1}}\phi(yz)\mr{d}P(y|x)$ over all  $z\in[-\infty,\infty]$, i.e.,  \beq\label{2302211603} f^*(x)\in\mathop{\arg\min}_{z\in[-\infty,\infty]}\int_{\hkh{-1,1}}\phi(yz)\mr{d}P(y|x)\text{ for $P_X$-almost all $x\in[0,1]^d$,}
\eeq where $\phi(yz):=\varlimsup\limits_{t\to yz}\phi(t)$ if $z\in\hkh{-\infty,\infty}$,  $P_X$ is the marginal distribution of $P$ on $[0,1]^d$, and $P(\cdot|x)$ is the \emph{regular  conditional distribution} of $P$ on $\hkh{-1,1}$ given $x$ (cf. Lemma A.3.16 in \cite{steinwart2008support}). In this paper, we will use $f^*_{\phi,P}$ to denote the target function of the $\phi$-risk under $P$. Note that  $f^*_{\phi,P}$ may take values in $\hkh{-\infty,\infty}$, and $f^*_{\phi,P}$ minimizes $\mc R^\phi_P$ in the sense that
\beq\label{2302081455}
&\mc R^\phi_P(f^*_{\phi,P}):=\int_{[0,1]^d\times\{-1,1\}}\phi(yf^*_{\phi,P}(x))\mr{d}P(x,y)\\
&=\inf \setl{\mc{R}^\phi_P(g)}{ \textrm{$g:[0,1]^d\to\mbR$ is Borel measurable}}, 
\eeq where  $\phi(yf^*_{\phi,P}(x)):=\varlimsup\limits_{t\to yf^*_{\phi,P}(x)}\phi(t)$ if  $yf^*_{\phi,P}(x)\in\hkh{-\infty,\infty}$ (cf. Lemma \ref{2302062347}).

In practice, the choice of the loss function $\phi$ varies, depending on the classification method used. For neural network   classification, although other loss functions have been investigated,  the \emph{logistic loss} $\phi(t)=\log(1+\me^{-t})$, also known as the \emph{cross entropy loss}, is most commonly used (see, e.g., \cite{janocha2016loss,hui2020evaluation,hu2021understanding}).  We now explain why the logistic loss is related to cross entropy. Let $\mc X$ be an arbitrary  nonempty countable set  equipped with the sigma algebra consisting of all its subsets.  For any two probability measures $Q_0$ and $Q$ on $\mc X$, the \emph{cross entropy} of $Q$ relative to $Q_0$ is defined as $\mr H(Q_0,Q):=-\sum_{z\in\mc X}Q_0(\hkh z)\cdot\log Q(\hkh z)$, where $\log 0:=-\infty$ and $0\cdot (-\infty):=0$ (cf. (2.112) of \cite{murphy2012machine}). One can show that $\mr H(Q_0,Q)\geq\mr H(Q_0,Q_0)\geq 0$ and  \[\hkh{Q_0}=\mathop{\arg\min}_{Q}\mr H(Q_0,Q)\text{ if }\mr H(Q_0,Q_0)<\infty. 
\] Therefore, roughly speaking,  the cross entropy $\mr H(Q_0,Q)$ characterizes how close $Q$  is to $Q_0$.  For any $a\in[0,1]$, let $\mathscr{M}_{a}$ denote the probability measure on $\hkh{-1,1}$ with $\mathscr{M}_{a}(\hkh{1})=a$ and $\mathscr{M}_{a}(\hkh{-1})=1-a$. Recall that any real-valued Borel measurable function $f$ defined on the input space $[0,1]^d$ can induce a classifier $\mr{sgn}\circ f$. We can interpret the construction of the classifier $\mr{sgn}\circ f$ from $f$ as follows. Consider the \emph{logistic  function} \beq\label{2302211628}\bar l:\mbR\to(0,1),\;z\mapsto \frac{1}{1+\me^{-z}},\eeq which is strictly increasing.  For each $x\in[0,1]^d$, $f$ induces a probability measure  $\mathscr M_{\bar l(f(x))}$ on $\hkh{-1,1}$ via $\bar l$, which we regard as a prediction made by $f$ of the distribution of the output data (i.e., the two labels $+1$ and $-1$) given the input data $x$. Observe that the larger $f(x)$ is, the closer the number $\bar l(f(x))$ gets to $1$, and the more likely the event $\hkh{1}$ occurs under the distribution $\mathscr M_{\bar l(f(x))}$. If $\mathscr M_{\bar l(f(x))}(\hkh{+1})\geq \mathscr M_{\bar l(f(x))}(\hkh{-1})$, then $+1$ is more likely to appear given the input data $x$ and we thereby think of $f$ as classifying the input $x$ as class $+1$. Otherwise, when $\mathscr M_{\bar l(f(x))}(\hkh{+1})<\mathscr M_{\bar l(f(x))}(\hkh{-1})$, $x$ is classified as $-1$. In this way, $f$ induces a classifier given by
\beq\label{23020301}
x\mapsto\begin{cases}
+1,&\text{ if }\mathscr M_{\bar l(f(x))}(\hkh{1})\geq\mathscr M_{\bar l(f(x))}(\hkh{-1}),\\
-1,&\text{ if }\mathscr M_{\bar l(f(x))}(\hkh{1})<\mathscr M_{\bar l(f(x))}(\hkh{-1}). 
\end{cases}
\eeq Indeed, the classifier in \eqref{23020301} is exactly $\mr{sgn}\circ f$. Thus we can also measure the predictive performance of $f$ in terms of $\mathscr M_{\bar l(f(\cdot))}$ (instead of $\mr{sgn}\circ f$). To this end,  one natural way is to compute the average ``extent'' of how close $\mathscr M_{\bar l(f(x))}$ is to the true conditional distribution of the output given the input $x$. If we use the cross entropy to characterize this ``extent'', then its average,  which measures the classification  performance of $f$, will be  $\int_{[0,1]^d}\mr H(\mathscr Y_x,\mathscr M_{\bar l(f(x))})\mr d\mathscr X(x)$, where $\mathscr X$ is the distribution of the input data, and $\mathscr Y_x$ is the conditional distribution of the output data given the input $x$. However, one can show that this quantity is just the logistic risk of $f$. Indeed,  
\begin{align*}
&\int_{[0,1]^d}\mr H(\mathscr Y_x,\mathscr M_{\bar l(f(x))})\mr d\mathscr X(x)\\&=\int_{[0,1]^d}\ykh{-\mathscr Y_x(\hkh{1})\cdot \log(\mathscr M_{\bar l(f(x))}(\hkh{1}))-\mathscr Y_x(\hkh{-1})\log(\mathscr M_{\bar l(f(x))}(\hkh{-1}))}\mr d\mathscr X(x)\\
&=\int_{[0,1]^d}\ykh{-\mathscr Y_x(\hkh{1})\cdot \log(\bar l(f(x)))-\mathscr Y_x(\hkh{-1})\log(1-\bar l(f(x)))}\mr d\mathscr X(x)\\
&=\int_{[0,1]^d}\ykh{\mathscr Y_x(\hkh{1})\cdot \log(1+\me^{-f(x)})+\mathscr Y_x(\hkh{-1})\log(1+\me^{f(x)})}\mr d\mathscr X(x)\\
&=\int_{[0,1]^d}\ykh{\mathscr Y_x(\hkh{1})\cdot\phi(f(x))+\mathscr Y_x(\hkh{-1})\phi(-f(x))}\mr d\mathscr X(x)\\&=\int_{[0,1]^d}\int_{\hkh{-1,1}}\phi(yf(x))\mr d\mathscr Y_x(y)\mr d\mathscr X(x)=\int_{[0,1]^d\times\hkh{-1,1}}\phi(yf(x))\mr{d}P(x,y)=\mc R^\phi_P(f),
\end{align*} where $\phi$ is the logistic loss and  $P$ is the joint distribution of the input and output data, i.e., $\mr dP(x,y)=\mr d\mathscr Y_x(y)\mr d\mathscr X(x)$. Therefore, the average cross entropy of the distribution $\mathscr M_{\bar l(f(x))}$ induced by $f$ to the true conditional distribution of the output data given the input data $x$ is equal to the logistic risk of $f$ with respect to the joint distribution of the input and output data, which explains why the logistic loss is also called the cross entropy loss. Compared with the misclassification error $\mc R_P(f)$ which measures the performance of the  classifier  $f(x)$ in correctly generating the class label $\mr{sgn}(f(x))$  that equals the most probable  class label of the input data $x$ (i.e., the label $y_x\in\hkh{-1,+1}$ such that $\mathscr Y_x(\hkh{y_x})\geq\mathscr Y_x(\hkh{-y_x})$), the logistic risk $\mc R_P^\phi(f)$ measures how close the induced distribution  $\mathscr M_{\bar l(f(x))}$ is to the true conditional  distribution $\mathscr Y_x$. %
  Consequently, in comparison with the (excess) misclassification error, the (excess) logistic risk is also a reasonable{\label{23072601}} quantity for characterizing the performance of classifiers but from a different angle.  When classifying with the logistic loss, we are essentially learning the conditional distribution $\mathscr Y_x$ through the cross entropy and the logistic function $\bar l$.  Moreover,  for any classifier ${\hat f_n:[0,1]^d\to\mbR}$ trained with logistic loss, the  composite function ${\bar l\circ \hat f_n(x)}=\mathscr M_{\bar l\circ \hat f_n(x)}(\hkh{1})$ yields an estimation of the \emph{conditional class  probability function} $\eta(x):=P(\hkh{1}|x)=\mathscr Y_x(\hkh{1})$.  Therefore, classifiers trained with logistic loss essentially  capture more information about  the exact value of the conditional class probability function $\eta(x)$ than we actually need to minimize the misclassification error $\mc R_P(\cdot)$, since the knowledge of the sign of $2\eta(x)-1$ is already   sufficient for minimizing $\mc R_P(\cdot)$ (see \eqref{23072501}).   In addition, we  point out  that the excess logistic risk $\mc E^\phi_P(f)$ is actually the average  \emph{Kullback-Leibler divergence} (\emph{KL divergence}) from $\mathscr M_{\bar l(f(x))}$ to $\mathscr Y_x$. Here for any two probability measures $Q_0$ and $Q$ on some countable set $\mc X$, the KL divergence from $Q$ to $Q_0$ is defined as $\mr{KL}(Q_0||Q):=\sum_{z\in\mc X}Q_0(\hkh{z})\cdot\log\frac{Q_0(\hkh{z})}{Q(\hkh{z})}$ , where $Q_0(\hkh{z})\cdot\log\frac{Q_0(\hkh{z})}{Q(\hkh{z})}:=0$ if $Q_0(\hkh{z})=0$ and $Q_0(\hkh{z})\cdot\log\frac{Q_0(\hkh{z})}{Q(\hkh{z})}:=\infty$ if $Q_0(\hkh{z})>0=Q(\hkh{z})$ (cf. (2.111) of \cite{murphy2012machine} or  Definition 2.5 of  \cite{tsybakov2009introduction}).

In this work, we focus on the generalization analysis of binary classification with  empirical risk
minimization over ReLU DNNs. That is, the classifiers under consideration are produced by algorithm \eqref{ERM} in which the hypothesis space $\mc{F}$ is generated by deep ReLU networks. Based on recent studies in complexity and approximation theory of DNNs (e.g., \cite{bartlett2019nearly,petersen2018optimal,yarotsky2017error}), several researchers have derived generalization bounds for $\phi$-ERMs over DNNs in  binary classification problems \citep{farrell2021,kim2021fast,shen2021non}. However, to the best of our knowledge, the existing literature fails to establish satisfactory generalization analysis if the target function $f^*_{\phi,P}$ is unbounded. In particular, take $\phi$ to be the logistic loss, i.e., $\phi(t)=\log(1+\me^{-t})$. The target function is then explicitly given by $f^*_{\phi,P}\xlongequal{P_X\text{-a.s.}}\log\frac{\eta}{1-\eta}$ with $\eta(x):=P(\hkh{1}|x)$ $(x\in[0,1]^d)$ being the {conditional class  probability function} of $P$ (cf. Lemma \ref{2302281501}),  where recall that $P(\cdot|x)$ denotes the 
	conditional probability of $P$ on $\hkh{-1,1}$ given $x$.  Hence $f^*_{\phi,P}$ is unbounded if $\eta$ can be arbitrarily close to $0$ or $1$, which happens in many practical problems (see Section \ref{section: related work} for more details). For instance,  we have $\eta(x)=0$ or $\eta(x)=1$ for a noise-free distribution $P$, implying $f^*_{\phi ,P}(x)=\infty$ for $P_X$-almost all $x\in[0,1]^d$, where $P_X$ is the marginal distribution of $P$ on $[0,1]^d$. DNNs trained with the logistic loss perform efficiently in various image recognition applications as the smoothness of the loss function can further simplify the optimization procedure \citep{goodfellow2016deep,krizhevsky2012imagenet,simonyan2014very}. However, due to the unboundedness of $f^*_{\phi,P}$, the existing generalization analysis for classification with DNNs and the logistic loss either results in slow rates of convergence (e.g., the logarithmic rate in  \cite{shen2021non}) or   can only be conducted  under very restrictive conditions (e.g.,  \cite{kim2021fast, farrell2021}) (cf. the discussions in Section \ref{section: related work}).  The unboundedness of the target function brings several technical difficulties to the generalization analysis. Indeed, if $f^*_{\phi,P}$ is unbounded, it cannot be approximated uniformly by continuous functions on $[0,1]^d$, which poses extra challenges for bounding the approximation error. Besides, previous sample error estimates based on concentration techniques are no longer valid because these estimates usually require involved random variables to be bounded or to satisfy strong tail conditions (cf. Chapter 2 of \cite{wainwright2019high}). Therefore, in contrast to empirical studies, the previous strategies for generalization analysis could not demonstrate the efficiency of classification with DNNs and the logistic loss.

To fill this gap, in this paper we develop a novel theoretical analysis to establish tight generalization bounds for training DNNs with ReLU activation function and logistic loss in binary classification. Our main contributions are summarized as follows.
\begin{itemize}
	\item For $\phi$ being the logistic loss, we establish an oracle-type inequality to bound the excess $\phi$-risk  without using the explicit form of the target function $f^*_{\phi,P}$. Through constructing a suitable bivariate function $\psi:[0,1]^d\times\{-1,1\}\to\mbR$, generalization analysis based on this oracle-type inequality can remove the boundedness restriction of the target function. Similar results hold even for the more general case when $\phi$ is merely Lipschitz continuous (see Theorem \ref{thm2.1} and related discussions in Section \ref{section: main results}).
	
	\item By using our oracle-type inequality, we establish tight  generalization bounds for fully connected ReLU DNN classifiers $	\hat{f}^{\FNN}_n$ trained by  empirical logistic risk minimization (see \eqref{FCNNestimator}) and obtain sharp  convergence rates in various settings: 
	\begin{itemize}
		\item[$\circ$] We establish optimal convergence rates for the excess logistic risk of 	$\hat{f}^{\FNN}_n$ only requiring the H\"{o}lder smoothness of the conditional probability function $\eta$ of the data distribution. Specifically, for H\"older-$\beta$ smooth $\eta$, we show that the convergence rates of the excess logistic risk of $	\hat{f}^{\FNN}_n$ can achieve $\mc{O}(\ykh{\frac{\ykh{\log n}^{5}}{n}}^{\frac{\beta}{\beta+d}})$, which is optimal up to the logarithmic term $(\log n)^{\frac{5\beta}{\beta+d}}$. From this we obtain the convergence rate $\mc{O}(\ykh{\frac{\ykh{\log n}^{5}}{n}}^{\frac{\beta}{2\beta+2d}})$ of the excess misclassification error of $\hat{f}^{\FNN}_n$, which is  very close to the optimal rate, by using the  calibration inequality 
		 (see Theorem \ref{thm2.2}).  As a by-product, we also derive a new tight error bound for the approximation of the natural logarithm function (which is unbounded near zero) by ReLU DNNs (see Theorem \ref{thm2.3}). This bound plays a key role in establishing the aforementioned  optimal rates of convergence. 
		\item[$\circ$] We consider a compositional assumption which requires the conditional probability function $\eta$ to be the composition $h_q\circ h_{q-1}\circ\cdots \circ h_1\circ h_0$ of several vector-valued multivariate functions $h_i$, satisfying that each component function of  $h_i$  is either a H\"older-$\beta$ smooth function only depending on (a small number) $d_*$  of its input variables or  the maximum value function among  some  of its input variables.  We show that under this compositional assumption the convergence rate of the excess logistic risk of $	\hat{f}^{\FNN}_n$ can achieve $\mc{O}(\ykh{\frac{(\log n)^5}{n}}^{\frac{\beta\cdot(1\qx\beta)^q}{d_*+\beta\cdot(1\qx\beta)^q}})$, which is optimal up to the logarithmic term $(\log n)^{\frac{5\beta\cdot(1\qx\beta)^q}{d_*+\beta\cdot(1\qx\beta)^q}}$. We then use the calibration inequality to obtain the convergence rate $\mc{O}(\ykh{\frac{(\log n)^5}{n}}^{\frac{\beta\cdot(1\qx\beta)^q}{2d_*+2\beta\cdot(1\qx\beta)^q}})$ of the excess misclassification error of $\hat{f}^{\FNN}_n$ (see Theorem \ref{thm2.2p}). Note that the derived convergence rates  $\mc{O}(\ykh{\frac{(\log n)^5}{n}}^{\frac{\beta\cdot(1\qx\beta)^q}{d_*+\beta\cdot(1\qx\beta)^q}})$ and $\mc{O}(\ykh{\frac{(\log n)^5}{n}}^{\frac{\beta\cdot(1\qx\beta)^q}{2d_*+2\beta\cdot(1\qx\beta)^q}})$ are independent of the input dimension $d$, thereby circumventing the well-known  curse of dimensionality. It can be shown that the above compositional assumption is likely to be satisfied in practice (see comments before Theorem \ref{thm2.2p}). Thus this result helps to explain the huge success of DNNs in practical classification problems, especially high-dimensional ones. 
		\item[$\circ$]     We derive convergence rates of the excess misclassification error of $	\hat{f}^{\FNN}_n$ under the  piecewise smooth  decision boundary condition combining with the noise and margin conditions (see Theorem \ref{thm2.4}). As a special case of this result, we show that when the input data  are  bounded away from  the decision boundary almost surely, the derived rates can also be dimension-free.	\end{itemize}\item We demonstrate   the optimality of the convergence rates  stated   above by presenting corresponding minimax lower bounds (see Theorem \ref{thm2.6p} and Corollary \ref{thm2.6}). 
\end{itemize}

The rest of this paper is organized as follows. In the remainder of this section, we first introduce some conventions and notations that will be used in this paper. Then we describe the mathematical modeling of fully connected ReLU neural networks which defines the hypothesis spaces in our setting. At the end of this section, we provide a symbol glossary for the convenience of readers. In Section \ref{main22103301}, we present our main results in this paper,  including the oracle-type inequality, several generalization bounds for classifiers obtained from  empirical logistic risk minimization over fully connected ReLU DNNs, and two minimax lower bounds.  Section \ref{section: related work} provides 
discussions and comparisons with related works and Section \ref{section: conclusion} concludes the paper.  In  Appendix \ref{section: appendix A1} and Appendix \ref{A.3}, we  present covering number bounds and some approximation bounds for the   space of fully connected ReLU DNNs respectively.  Finally, in Appendix \ref{appendixC240411044938}, we  give detailed proofs of results in the main body of this paper.

\subsection{Conventions and Notations}\label{section: conventions and notations}

Throughout this paper, we follow the conventions that $0^0:=1$, $1^\infty:=1$, $\frac{z}{0}:=\infty=:\infty^{c}$, $\log(\infty):=\infty$,  $\log 0:=-\infty$, $0\cdot w:=0=:w\cdot 0$ and $\frac{a}{\infty}:=0=:b^{\infty}$ for any $a\in\mbR, b\in[0,1), c\in(0,\infty)$, $z\in[0,\infty]$, $w\in[-\infty,\infty]$ where we denote by $\log$ the natural logarithm function (i.e. the base-$\me$ logarithm function). The terminology ``measurable'' means ``Borel measurable'' unless otherwise specified. Any Borel subset of some Euclidean space  $\mbR^m$ is equipped with the Borel sigma algebra by default.  Let $\mc G$ be an arbitrary measurable space and $n$ be a positive integer. We call any sequence of $\mc G$-valued random variables $\hkh{Z_i}_{i=1}^n$ a \emph{sample} in $\mc G$ of size $n$.  Furthermore, for any measurable space $\mc F$ and any sample $\hkh{Z_i}_{i=1}^n$ in $\mc G$,  an $\mc F$-valued statistic on $\mc G^n$ from the sample $\hkh{Z_i}_{i=1}^n$ is a random variable $\hat\theta$ together with a measurable map $\mc T:\mc G^n\to\mc F$  such that $\hat\theta=\mc T(Z_1,\ldots,Z_n)$, where $\mc T$ is called \emph{the map associated with} the statistic $\hat\theta$. Let $\hat\theta$ be an arbitrary $\mc F$-valued statistic from some sample $\hkh{Z_i}_{i=1}^n$ and $\mc T$ is the map associated with $\hat\theta$. Then for any measurable space $\mc D$ and any measurable map $\mc T_0:\mc F\to\mc D$, $\mc T_0(\hat\theta)=\mc T_0(\mc T(Z_1,\ldots,Z_n))$ is a $\mc D$-valued statistic from the sample $\hkh{Z_i}_{i=1}^n$, and  $\mc T_0\circ \mc T$ is the map associated with $\mc T_0(\hat\theta)$.  %

Next we will introduce some notations used in this paper.  We denote by $\mathbb{N}$ the set of all  positive integers $\{1,2,3,4,\ldots \}$. For $d\in\mb N$, we use $\mc F_d$ to denote the set of all Borel measurable functions from $[0,1]^d$ to $(-\infty,\infty)$, and use $\mc H_0^d$ to denote the set of all Borel probability measures on $[0,1]^d\times\hkh{-1,1}$.   For any set ${A}$, the indicator function of ${A}$ is given by \beq\label{23032601}\idf_{{A}}(x):= \begin{cases}0,&\text{if $x\notin A$},\\1,&\text{if $x\in A$}, \end{cases}\eeq and the number of elements of $A$ is denoted by $\#(A)$. For any finite dimensional vector $v$ and any positive integer $l$ less than or equal to the dimension of $v$, we denote by $(v)_l$ the $l$-th component of $v$. More generally, for any nonempty subset $I=\hkh{i_1,i_2,\ldots,i_m}$ of $\mb N$ with $1\leq i_1<i_2<\cdots<i_m\leq$ the dimension of $v$, we denote $(v)_I:=\big((v)_{i_1},(v)_{i_2},\ldots, (v)_{i_m}\big)$, which is a $\#(I)$-dimensional vector.   For any function $f$, we use ${\mathbf{{dom}}}(f)$ to denote the domain of $f$, and use $\mathbf{ran}(f)$ to denote the range of $f$, that is, $\mathbf{ran}(f):=\hkh{f(x)\big|x\in{\mathbf{{dom}}}(f)}$.    If $f$ is a $[-\infty,\infty]^m$-valued function for some $m\in\mb N$ with ${\mathbf{{dom}}}(f)$ containing a nonempty set $\Omega$, then the uniform norm of $f$ on $\Omega$ is given by \beq\label{23051401}
\|f\|_{\Omega}:=\sup\hkh{\abs{\big(f(x)\big)_i}\Big|x\in \Omega,\,i\in\hkh{1,2,\ldots,m}}.\eeq    For integer $m\geq 2$ and real numbers $a_1, \cdots, a_m$, define $a_1\qd a_2\qd\cdots\qd a_m=\max\{a_1,a_2,\cdots a_m\}$ and $
a_1\qx a_2\qx\cdots\qx a_m=\min\{a_1,a_2,\cdots a_m\}.$ Given a real matrix
${\bm A}= (a_{i,j})_{i=1,\ldots,m,j=1,\ldots ,l}$ and $t\in[0,\infty]$,  the ${\ell}^t$-norm of ${\bm A}$ is defined by 
\beq\label{23012223}
\|{\bm A}\|_t:=\left\{
\ba
&\sum_{i=1}^m\sum_{j=1}^l\idf_{(0,\infty)}(\abs{a_{i,j}}),&&\text{if }t=0, \\
&\abs{\sum_{i=1}^m\sum_{j=1}^l\abs{a_{i,j}}^t}^{1/t}, &&\text{if }0<t<\infty,\\
&\sup\hkh{\abs{a_{i,j}}\big{|}i\in\{1,\cdots,m\},j\in\{1,\cdots,l\}},&&\text{if }t=\infty.\\
\ea\right.
\eeq Note that a vector is exactly a matrix with only one column or one row. Consequently,  \eqref{23012223} with $l=1$ or $m=1$ actually defines the $\ell^t$-norm  of a real vector $\bm A$.  Let $\mc G$ be a measurable space, $\hkh{Z_i}_{i=1}^n$ be a sample in $\mc G$ of size $n$, $\mc P_n$ be a probability measure on $\mc G^n$, and $\hat\theta$ be a $[-\infty,\infty]$-valued statistic on $\mc G^n$ from the sample $\hkh{Z_i}_{i=1}^n$.  Then we denote  \beq\label{2302221721}
	\bm E_{\mc P_n}[\hat\theta]:=\int \mc T\mr d \mc P_n\eeq provided that the integral $\int \mc T\mr d \mc P_n$  exists, where $\mc T$ is the map associated with $\hat\theta$.  Therefore, 
	\[\bm E_{\mc P_n}[\hat\theta]=\mbe\zkh{\mc T(Z_1,\ldots,Z_n)}=\mbe[\hat\theta]\] if the joint distribution of $(Z_1,\ldots,Z_n)$ is exactly $\mc P_n$.   Let $P$ be a Borel probability measure on $[0,1]^d\times\hkh{-1,1}$ and  $x\in[0,1]^d$. We use $P(\cdot|x)$ to denote the regular conditional distribution of $P$ on $\hkh{-1,1}$ given $x$, and $P_X$ to denote the marginal distribution of $P$ on $[0,1]^d$. For short, we will call the function $[0,1]^d\ni x\mapsto P(\hkh{1}|x)\in[0,1]$ the \emph{conditional probability function} (instead of the  \emph{conditional class  probability function}) of $P$.  For any probability measure $\mathscr Q$ defined on some measurable space $(\Omega,\mc F)$ and any $n\in\mb N$, we use  ${\mathscr Q}^{\otimes n}$ to denote the product measure $\underbrace{\mathscr Q\times \mathscr Q\times\cdots \mathscr Q}_n$ defined on the product measurable space $(\underbrace{\Omega\times \Omega\times\cdots \Omega}_n,\; \underbrace{\mc F\otimes \mc F\otimes\cdots \mc F}_n)$.

\subsection{Spaces  of Fully Connected Neural Networks}\label{section: FCNN}

In this paper, we restrict ourselves to neural networks with the ReLU activation function. Consequently,  hereinafter,  for simplicity, we sometimes  omit the word ``ReLU'' and the terminology ``neural networks'' will always refer to ``ReLU neural networks''.

The ReLU function is given by $\sigma:\mbR\to[0,\infty),\;t\mapsto\max\hkh{t,0}$. For any vector $v\in \mbR^m$ with $m$ being some positive integer,  the{\label{23072304}} $v$-shifted ReLU function is defined as $\sigma_{v}:\mbR^m\to[0,\infty)^m,\;x\mapsto \sigma(x- v)$, where the function $\sigma$ is applied componentwise. 

Neural networks considered in this paper can be expressed as a family of real-valued functions which take the form
\beq\label{neuralnetwok}
f:\mathbb{R}^d\to
\mbR,\quad x\mapsto {\bm W}_L\sigma_{v_L}{\bm W}_{L-1}\sigma_{v_{L-1}}\cdots {\bm W}_1\sigma_{v_1}{\bm W}_0x,
\eeq
where the depth $L$ denotes the number of hidden layers, $\mr m_k$ is the width of $k$-th layer, ${\bm W}_k$ is an $\mr m_{k+1}\times \mr m_{k}$ weight matrix with $\mr m_0=d$ and $\mr m_{L+1}=1$, and the shift vector $v_{k}\in\mbR^{\mr m_k}$ is called a bias. The architecture of a neural network is parameterized by weight matrices $\{{\bm W}_k\}_{k=0}^L$ and biases $\{v_k\}_{k=1}^L$, which will be estimated from data.  Throughout the paper, whenever we talk about a neural network, we will explicitly associate it with a function $f$ of the form \eqref{neuralnetwok} generated by $\{{\bm W}_k\}_{k=0}^L$ and 
$\{v_k\}_{k=1}^L$.

The  space  of fully connected neural networks is characterized by their depth and width, as well as the number of nonzero parameters in weight matrices and bias vectors. In addition, the complexity of this space is also determined by the $\|\cdot\|_{\infty}$-bounds of neural network parameters and $\|\cdot\|_{[0,1]^d}$-bounds of associated functions in form \eqref{neuralnetwok}. Concretely, let $(G,N)\in [0,\infty)^2$ and $(S,B,F)\in [0,\infty]^3$, the  space of fully connected neural networks is defined as 
\beq\label{spaceofFCNN}
\fdnn_d(G,N,S,B,F):=
\setr{f:\mathbb{R}^d\to \mbR}{\ba
	&\textrm{$f$ is defined in (\ref{neuralnetwok}) satisfying that}\\
	&L\leq G,\ \mr m_1\qd \mr m_2\qd\cdots\qd \mr m_L \leq N,\\
	&\ykh{\sum_{k=0}^L\|{\bm W}_k\|_0}+\ykh{\sum_{k=1}^L\|v_k\|_0}\leq S,\\
	& \sup_{k=0,1,\cdots,L}\|{\bm W}_k\|_\infty\qd \sup_{k=1,\cdots,L}\|v_k\|_\infty \leq B,\\
	& \textrm{and } \norm{f}_{[0,1]^d}\leq F
	\ea}.
\eeq In this definition,  the freedom in choosing the position of nonzero entries of $\bm W_k$ reflects the fully
connected nature between consecutive layers of the neural network $f$. It should be noticed that $B$ and $F$ in the definition \eqref{spaceofFCNN} above can be $\infty$, meaning that there is no restriction on the upper bounds of $\|{\bm W}_k\|_\infty$ and $\|v_k\|_\infty$, or $\|f\|_{[0,1]^d}$. The parameter {\label{23072202}}$S$ in \eqref{spaceofFCNN} can also be $\infty$, leading to a structure without sparsity. The space $\fdnn_d(G,N,S,B,F)$ incorporates all the essential features of fully connected neural network architectures and has been adopted to study the generalization properties of fully connected neural network models in regression and classification \citep{kim2021fast,schmidt2020nonparametric}.   

   \subsection{Glossary}
    
   At the end of this section, we provide a glossary of frequently used symbols in this paper for the convenience of readers. 
\begin{longtable}{lp{0.56\textwidth}l}
	\label{table22120607}\\	
	\hline
	\hline
	\textbf{Symbol} &\textbf{Meaning} & \textbf{Definition}\\ \hline
	$\mb Z$ &The set of  integers. &\\\hdashline
	$\mb N$ &The set of positive integers. &\\\hdashline
	$\mb R$ &The set of real numbers. &\\\hdashline	
	$\qd$ &Taking the maximum, e.g., $a_1\qd a_2\qd a_3\qd a_4$ is equal to the maximum of $a_1,\ldots a_4$. &\\\hdashline
	$\qx$ &Taking the minimum, e.g., $a_1\qx a_2\qx a_3\qx a_4$ is equal to the minimum of $a_1,\ldots a_4$. &\\	\hdashline
$\circ$ &Function composition, e.g., for $f:\mbR\to\mbR$ and $g:\mbR\to\mbR$, $g\circ f$ denotes the map $\mb R\ni x\mapsto g(f(x))\in\mbR$. &\\ \hdashline
${\mathbf{{dom}}}(f)$ &The domain of a function $f$.  &Below Eq. \eqref{23032601}\\\hdashline
${\mathbf{{ran}}}(f)$ &The range of a function $f$.  &Below Eq. \eqref{23032601}\\\hdashline
$\bm A^\top$ &The transpose of a  matrix $\bm A$.  &\\\hdashline
$\#(A)$ &The number of elements of a set $A$.  &\\\hdashline
$\flr{\;}$ &The floor function, which is defined as  $\flr{x}:=\sup\setl{z\in\mb Z}{z\leq x}$. &\\\hdashline
$\ceil{\;}$ &The ceiling function, which is  defined as $\ceil{x}:=\inf\setl{z\in\mb Z}{z\geq x}$. &\\\hdashline
$\idf_A$&The indicator function of a set $A$. &Eq.  \eqref{23032601}\\\hdashline
$(v)_l$&The $l$-th component of a vector $v$. &Below Eq. \eqref{23032601}\\\hdashline
$(v)_I$&The $\#(I)$-dimensional vector whose components are exactly $\hkh{(v)_i}_{i\in I}$. &Below Eq. \eqref{23032601}\\\hdashline
$\norm{\cdot}_\Omega$ &The uniform norm on a set $\Omega$. &Eq.  \eqref{23051401}\\\hdashline
$\norm{\cdot}_t$ &The $\ell^t$-norm. &Eq. \eqref{23012223}  \\\hdashline
$\norm{\cdot}_{\mc{C}^{k,\lambda}(\Omega)}$&The H\"older norm. &Eq. \eqref{23022701}\\\hdashline
$\mr{sgn}$ &The sign function. &Below Eq.  \eqref{2212070302}\\\hdashline
$\sigma$ &The ReLU function, that is,  $\mbR\ni t\mapsto \max\hkh{0,t}\in[0,\infty)$.&Above Eq.  \eqref{neuralnetwok}\\\hdashline
$\sigma_{v}$ &The $v$-shifted ReLU function. &Above Eq.  \eqref{neuralnetwok}\\
\hdashline
$\mathscr{M}_{a}$&The probability measure on $\hkh{-1,1}$ with $\mathscr{M}_{a}(\hkh{1})=a$.&Above Eq. \eqref{2302211628}\\\hdashline
$P_X$&The marginal distribution of $P$ on $[0,1]^d$.&Below Eq. \eqref{2302211603}\\\hdashline
$P(\cdot|x)$&The regular  conditional distribution of $P$ on $\hkh{-1,1}$ given $x\in[0,1]^d$. &Below Eq. \eqref{2302211603}\\\hdashline
$P_{\eta,\mathscr Q}$&The probability on $[0,1]^d\times\hkh{-1,1}$ of which the marginal distribution on $[0,1]^d$ is $\mathscr Q$ and the conditional probability function is ${\eta}$.&Eq. \eqref{221031221401}\\\hdashline
$P_{\eta}$&The probability on $[0,1]^d\times\hkh{-1,1}$ of which the marginal distribution on $[0,1]^d$ is the Lebesgue measure and the conditional probability function is ${\eta}$.&Below Eq. \eqref{221031221401}\\\hdashline
$\bm E_{\mc P_n}[\hat\theta]$&The expectation of a statistic $\hat\theta$ when the joint distribution of the sample on which $\hat\theta$ depends is $\mc P_n$. &Eq. \eqref{2302221721}\\\hdashline
${\mathscr Q}^{\otimes n}$&The product measure $\underbrace{\mathscr Q\times \mathscr Q\times\cdots \mathscr Q}_n$. &Below Eq. \eqref{2302221721}\\
\hdashline
$\mc R_P(f)$ &The misclassification error of $f$ with respect to $P$. &Eq. \eqref{2212070301}\\\hdashline
$\mc E_P(f)$ &The excess misclassification error of $f$ with respect to $P$. &Eq. \eqref{2212070302}\\\hdashline
$\mc R^\phi_P(f)$ &The $\phi$-risk of $f$ with respect to $P$. &Eq. \eqref{phirisk}\\\hdashline
$\mc E^\phi_P(f)$ &The excess  $\phi$-risk of $f$ with respect to $P$. &Eq.  \eqref{2212070303}\\\hdashline
$f^*_{\phi,P}$&The target function of the $\phi$-risk under some distribution $P$. &Eq. \eqref{2302211603}\\
 \hdashline
$\mc{N}(\mc{F},\gamma)$&The covering number of a class of real-valued functions $\mc F$ with radius $\gamma$ in the uniform norm. &Eq.  \eqref{2302131620}\\
 \hdashline
$\mc{B}^{\beta}_r\ykh{\Omega}$ &The closed ball of radius $r$ centered at the origin in the H\"older space of order $\beta$ on $\Omega$. &Eq.  \eqref{holderball}\\\hdashline
$\mc G_d^{\mathbf{M}}(d_\star)$ &The set of all functions from $[0,1]^d$ to $\mbR$ which compute the maximum value of  up to $d_\star$ components of their input vectors.&Eq. \eqref{23061101} \\\hdashline
$\mc G_d^{\mathbf{H}}(d_*, \beta,r)$&The set of all functions in $\mc B^\beta_r([0,1]^d)$ whose output values depend on exactly   $d_*$ components of their input vectors.  &Eq. \eqref{23061102}\\\hdashline
$\mc G_\infty^{\mathbf{M}}(d_\star)$&$\mc G_\infty^{\mathbf{M}}(d_\star):=\bigcup_{d=1}^\infty \mc G_d^{\mathbf{M}}(d_\star)$&Above Eq. \eqref{23060601}\\\hdashline
$\mc G_\infty^{\mathbf{H}}(d_*, \beta,r)$&$\mc G_\infty^{\mathbf{H}}(d_*, \beta,r):=\bigcup_{d=1}^\infty \mc G_d^{\mathbf{H}}(d_*, \beta,r)$&Above Eq. \eqref{23060601}\\\hdashline
$\mc G_d^{\mathbf{CH}}(\cdots)$&$\mc G_d^{\mathbf{CH}}(q, K, d_*, \beta,r)$ consists of compositional functions $h_q\circ \cdots \circ h_0$ satisfying  that each component function of $h_i$ belongs to  $\mc G_\infty^{\mathbf{H}}(d_*, \beta,r)$.&Eq. \eqref{23061201}\\\hdashline
$\mc G_d^{\mathbf{CHOM}}(\cdots)$&$\mc G_d^{\mathbf{CHOM}}(q, K,d_\star, d_*, \beta,r)$ consists of compositional functions $h_q\circ \cdots \circ h_0$ satisfying  that each component function of $h_i$ belongs to $\mc G_\infty^{\mathbf{H}}(d_*, \beta,r)\cup \mc G_\infty^{\mathbf{M}}(d_\star)$.&Eq. \eqref{23051301}\\\hdashline
$\mc{C}^{d,\beta,r,I,\Theta}$&The set of binary classifiers $\texttt C:[0,1]^d\to\hkh{-1,+1}$ such that  $\setl{x\in[0,1]^d}{\texttt C(x)=+1}$ is the union of some disjoint closed regions with  piecewise H\"older smooth boundary.  &Eq. \eqref{2301170025}\\\hdashline
$\Delta_{\texttt C}(x)$&The distance from some point $x\in[0,1]^d$ to the decision boundary of some  classifier $\texttt C\in\mc{C}^{d,\beta,r,I,\Theta}$.  &Eq. \eqref{23022216}\\\hdashline
$\mc F_d$&The set of all Borel measurable functions from $[0,1]^d$ to $(-\infty,\infty)$.  &Above Eq.  \eqref{23032601}\\\hdashline
$\fdnn_d(\cdots)$&The class of  ReLU neural networks defined on $\mbR^d$. &Eq.  \eqref{spaceofFCNN}\\\hdashline
$\mc H_0^d$&The set  of all Borel probability measures on $[0,1]^d\times\hkh{-1,1}$.  &Above Eq.  \eqref{23032601}\\\hdashline
 $\mc{H}^{d,\beta,r}_{1}$&The set of all probability measures $P\in\mc H_0^d$ whose conditional probability function coincides with some function in $\mc{B}^{\beta}_r\ykh{[0,1]^d}$ $P_X$-a.s.. &Eq.  \eqref{2302132212}\\
\hdashline
$\mc H^{d,\beta,r}_{2,s_1,c_1,t_1}$&The set of all probability measures $P$ in $\mc{H}^{d,\beta,r}_{1}$ satisfying the noise condition \eqref{Tsybakovnoisecondition}.  &Eq.  \eqref{2302132219}\\
\hdashline
$\mc{H}^{d,\beta,r}_{3,A}$&The set of all probability measures $P\in\mc H_0^d$ whose marginal distribution on $[0,1]^d$ is the Lebesgue measure and whose conditional probability function is in  $\mc{B}^{\beta}_r\ykh{[0,1]^d}$ and  bounded away from $\frac{1}{2}$ almost surely. &Eq. \eqref{2302132224}\\\hdashline
$\mc{H}^{d,\beta,r}_{4,q,K,d_\star,d_*}$&The set of all probability measures $P\in\mc H_0^d$ whose conditional probability function coincides with some function in $\mc G_d^{\mathbf{CHOM}}(q, K,d_\star, d_*, \beta,r)$ $P_X$-a.s..&Eq. \eqref{2302132212p}\\\hdashline
$\mc{H}^{d,\beta,r}_{5,A,q,K,d_*}$&The set of all probability measures $P\in\mc H_0^d$ whose marginal distribution on $[0,1]^d$ is the Lebesgue measure and whose conditional probability function is in  $\mc G_d^{\mathbf{CH}}(q, K, d_*, \beta,r)$ and  bounded away from $\frac{1}{2}$ almost surely. &Eq. \eqref{2302132224}\\\hdashline
$\mc H^{d,\beta,r,I,\Theta,s_1,s_2}_{6,t_1,c_1,t_2,c_2}$&The set of all probability measures $P\in\mc H_0^d$ which  satisfy the piecewise smooth decision boundary condition \eqref{230225012},  the noise condition \eqref{Tsybakovnoisecondition} and the margin condition \eqref{margincondition} for some $\texttt C\in \mc{C}^{d,\beta,r,I,\Theta}$. &Eq. \eqref{2301151903}\\\hdashline
$\mc H^{d,\beta}_{7}$&The set of all probability measures $P\in\mc H_0^d$ such that the target function of the logistic risk under $P$  belongs to $\mc{B}^{\beta}_1\ykh{[0,1]^d}$.&Above Eq. \eqref{bound 3.2}\\
\hdashline
$\hat{f}^{\FNN}_n$&The DNN estimator obtained from empirical logistic risk minimization over  the space of fully connected ReLU DNNs. &Eq. \eqref{FCNNestimator}\\
\hline
	\caption{Glossary of frequently used symbols in this paper}
\end{longtable}

\section{Main Results}\label{main22103301}

In this section, we give our main results, consisting of upper bounds presented in Subsection \ref{section: main results} and lower bounds presented in Subsection \ref{section: main lower}. 

\subsection{Main Upper Bounds}\label{section: main results}

In this subsection, we state our main results about upper bounds for the  (excess) logistic risk or (excess) misclassification error of empirical logistic risk minimizers. The first result, given in Theorem \ref{thm2.1}, is an oracle-type inequality which provides upper bounds for the logistic risk of empirical logistic risk minimizers. Oracle-type inequalities have been extensively studied in the literature of nonparametric statistics (see \cite{johnstone1998oracle} and references therein).  As one of the main contributions in this paper, this inequality deserves special attention in its own right, allowing us to establish a novel strategy for generalization analysis. Before we state Theorem \ref{thm2.1}, we introduce some notations.  For any pseudometric space $(\mc{F},\rho)$ (cf. Section 10.5 of \cite{bartlett2009})  and $\gamma\in(0,\infty)$, the covering number of $(\mc{F},\rho)$ with radius $\gamma$ is defined as
\[
\mc{N}\ykh{\ykh{\mc{F},\rho},\gamma}:=\inf\hkh{\#\ykh{\mc{A}}\left|
	\ba
	&\mc{A}\subset\mc{F}\textrm{, and for any  $f\in\mc{F}$ there }\\&\textrm{exists $g\in \mc{A}$ such that  $\rho(f,g)\leq\gamma$}\ea\right.},
\]
where we recall that $\#\ykh{\mc{A}}$ denotes the number of elements of the set $\mc A$. When the pseudometric $\rho$ on $\mc{F}$ is clear and no confusion arises, we write $\mc{N}(\mc{F},\gamma)$ instead of $\mc{N}\ykh{(\mc{F},\rho),\gamma}$ for simplicity. In particular, if $\mc{F}$ consists of real-valued functions  which are bounded on $[0,1]^d$, we will use $\mc{N}(\mc{F},\gamma)$ to denote
\beq\label{2302131620}
\mc{N}\ykh{\Big(\mc{F}, \rho: (f,g)\mapsto \sup_{x\in [0,1]^d}\abs{f(x)-g(x)}\Big),\gamma}
\eeq unless otherwise specified. Recall that the $\phi$-risk of a measurable function  $f:[0,1]^d\to \mathbb{R}$ with respect to a distribution $P$ on $[0,1]^d\times\hkh{-1,1}$ is denoted by $\mc{R}_P^\phi(f)$ and defined in \eqref{phirisk}.

\begin{thm}\label{thm2.1} Let $\{(X_i,Y_i)\}_{i=1}^n$ be an i.i.d. sample of a probability distribution $P$ on $[0,1]^d\times\{-1,1\}$, $\mc{F}$ be a nonempty class of uniformly bounded real-valued  functions defined on $[0,1]^d$, and $\hat{f}_n$ be an  ERM  with respect to the logistic loss $\phi(t)=\log(1+\me^{-t})$ over $\mc{F}$, i.e., 
	\beq\label{2212271526}
	\hat{f}_n\in\mathop{{\arg\min}}_{f\in\mc{F}}\frac{1}{n}\sum_{i=1}^n\phi\ykh{Y_if(X_i)}.
	\eeq If there exists a measurable function $\psi:[0,1]^d\times\{-1,1\}\to\mbR$ and a constant triple $(M,\Gamma,\gamma)\in(0,\infty)^3$ such that
	\beq
			\int_{[0,1]^d\times\{-1,1\}}{\psi(x,y)}\mr{d}P(x,y)\leq \inf_{f\in\mc{F}}\int_{[0,1]^d\times\{-1,1\}}{\phi(yf(x))}\mr{d}P(x,y), \label{ineq 2.10}
			\eeq
			\beq
			\sup\hkh{\phi(t)\left|\abs{t}\leq\sup_{f\in\mc{F}}\|f\|_{[0,1]^d}\right.} \qd \sup\hkh{\abs{\psi(x,y)}\left|(x,y)\in [0,1]^d\times\{-1,1\} \right.} \leq M, \label{ineq 2.11}\eeq
			\beq
			&\int_{[0,1]^d\times\{-1,1\}}{\ykh{\phi(yf(x))-\psi\ykh{x,y}}^2}\mr{d}P(x,y)\\ &\leq\Gamma \cdot {\int_{[0,1]^d\times\{-1,1\}}\ykh{\phi(yf(x))-\psi(x,y)}\mr{d}P(x,y)},\;\forall\;f\in\mc{F}, \label{ineq 2.12}
	\eeq and 
	\[W:=\max\hkh{3,\;\mc{N}\ykh{\mc{F},\gamma}}<\infty.\]
	Then for any $\e\in(0,1)$, there holds 
	\beq\label{bound 2.10}
	&\mbe\zkh{\mc{R}_P^{\phi}\ykh{\hat{f}_n}-\int_{[0,1]^d\times\{-1,1\}}{\psi(x,y)}\mr{d}P(x,y)}\\
	&\;\leq 80\cdot\frac{(1+\e)^2}{\e}\cdot \frac{\Gamma\log W}{n}+(20+20\e)\cdot \frac{M\log W}{n}+(20+20\e)\cdot\sqrt{\gamma}\cdot\sqrt{\frac{\Gamma\log W}{n}}\\
	&\;\;\;\;\;\;+4\gamma+(1+\e)\cdot\inf_{f\in\mc{F}}\ykh{\mc{R}_P^{\phi}(f)-\int_{[0,1]^d\times\{-1,1\}}{\psi(x,y)}\mr{d}P(x,y)}.
	\eeq
\end{thm}

According to its proof in Appendix  \ref{section: proof of thm2.1}, Theorem \ref{thm2.1} remains true when the logistic loss is replaced by any nonnegative function $\phi$ satisfying
\[\label{240410212610}
\abs{\phi(t)-\phi(t')}\leq\abs{t-t'},\;\forall\;t,t'\in\zkh{-\sup_{f\in\mc{F}}\|f\|_{[0,1]^d},\sup_{f\in\mc{F}}\|f\|_{[0,1]^d}}. 
\] 
Then by rescaling, Theorem \ref{thm2.1} can be further generalized to the case when $\phi$ is any nonnegative locally Lipschitz  continuous loss function such as  the exponential loss or the LUM (large-margin unified machine) loss (cf. \cite{liu2011hard}). Generalization analysis for classification with these loss functions based on oracle-type inequalities similar to Theorem \ref{thm2.1} has been studied in our coming work \cite{zhang202304cnn}. 

Let us give some comments on conditions \eqref{ineq 2.10} and \eqref{ineq 2.12} of Theorem \ref{thm2.1}. To our knowledge, these two conditions are introduced for the first time in this paper, and will play pivotal roles in our estimates.  Let $\phi$ be the logistic loss and $P$ be a probability measure on $[0,1]^d\times\hkh{-1,1}$.  Recall that $f^*_{\phi,P}$ denotes the target function of the logistic risk. If 
\beq\label{eq 2.13} 
\int_{[0,1]^d\times\{-1,1\}}\psi(x,y)\mr{d}P(x,y)=\inf \setl{\mc{R}_P^\phi(f)}{  \textrm{$f:[0,1]^d\to\mbR$ is measurable}},
\eeq then condition \eqref{ineq 2.10} is satisfied and the left hand side of (\ref{bound 2.10}) is exactly $\mbe\zkh{\mc{E}_P^\phi\ykh{\hat{f}_n}}$. Therefore, Theorem \ref{thm2.1} can be used to establish excess $\phi$-risk bounds for the $\phi$-ERM  $\hat{f}_n$. In particular, one can take $\psi(x,y)$ to be $\phi(yf^*_{\phi,P}(x))$ to ensure the equality \eqref{eq 2.13} (recalling \eqref{2302081455}). It should be pointed out that if $\psi(x,y)=\phi(yf^*_{\phi,P}(x))$, inequality (\ref{ineq 2.12}) is of the same form as the following inequality with $\tau=1$, which asserts that there exist $\tau \in [0,1]$ and $\Gamma>0$ such that
\beq\ba\label{ineq 2.14}
\int_{[0,1]^d\times\{-1,1\}}\Big(\phi(yf(x))-\phi\ykh{yf_\phi^*(x)}\Big)^2\mr{d}P(x,y)\leq \Gamma\cdot  \Big(\mc{E}_P^\phi(f)\Big)^\tau,\;\forall\; f\in\mc{F}. 
\ea\eeq
This inequality appears naturally when bounding the sample error by using concentration inequalities, which is of great importance in previous generalization analysis for binary classification (cf. condition (A4) in \cite{kim2021fast} and  Definition 10.15 in \cite{cucker2007learning}). In \cite{farrell2021}, the authors actually prove that if the target function $f^*_{\phi,P}$ is bounded and the functions in $\mc{F}$ are uniformly bounded by some $F>0$, the inequality (\ref{ineq 2.12}) holds with $\psi(x,y)=\phi(yf^*_{\phi,P}(x))$ and \[\Gamma=\frac{2}{\inf\hkh{\phi''(t)\left|t\in\mbR,\;{\abs{t}\leq\max\hkh{F,\norm{f^*_{\phi,P}}_{[0,1]^d}}}\right.}}.\] Here $\phi''(t)$ denotes the second order derivative of $\phi(t)=\log(1+\me^{-t})$ which is given by $\phi''(t)=\frac{\me^t}{(1+\me^t)^2}$. The boundedness of $f^*_{\phi,P}$ is a key ingredient leading to the main results in \cite{farrell2021} (see Section \ref{section: related work} for more details). However, $f^*_{\phi,P}$ is explicitly given by $\log\frac{\eta}{1-\eta}$ with $\eta(x)=P(\hkh{1}|x)$, which tends to infinity when $\eta$ approaches to $0$ or $1$. In some cases, the uniformly boundedness assumption on $f^*_{\phi,P}$ is too restrictive. When $f^*_{\phi,P}$ is unbounded, i.e., $\|f^*_{\phi,P}\|_{[0,1]^d}=\infty$,  condition (\ref{ineq 2.12}) will not be satisfied by simply taking $\psi(x,y)=\phi(yf^*_{\phi,P}(x))$. Since in this case we have $\inf_{t \in (-\infty, +\infty)} \phi''(t)=0$, one cannot find a finite constant $\Gamma$ to guarantee the validity of (\ref{ineq 2.12}), i.e., the inequality \eqref{ineq 2.14} cannot hold for $\tau=1$, which means the previous strategy for generalization analysis in \cite{farrell2021} fails to work. In Theorem \ref{thm2.1}, the requirement for $\psi(x,y)$ is much more flexible, we don't require $\psi(x,y)$ to be $\phi(yf^*_{\phi,P}(x))$ or even to satisfy (\ref{eq 2.13}). In this paper, by recurring to Theorem \ref{thm2.1}, we carefully construct $\psi$ to avoid using $f^*_{\phi,P}$ directly in the following estimates. Based on this strategy, under some mild regularity conditions on $\eta$, we can develop a more general analysis to demonstrate the performance of neural network classifiers trained with the logistic loss regardless of the  unboundedness of $f^*_{\phi,P}$. The derived generalization bounds and rates of convergence are stated in Theorem \ref{thm2.2}, Theorem \ref{thm2.2p},  and  Theorem \ref{thm2.4},   which are new in the literature and constitute the main contributions of this paper. It is worth noticing that in Theorem \ref{thm2.2}  and  Theorem \ref{thm2.2p}, we use Theorem \ref{thm2.1} to obtain optimal rates of convergence (up to some logarithmic factor), which demonstrates the tightness and power of the inequality \eqref{bound 2.10} in Theorem \ref{thm2.1}. To obtain these optimal rates from Theorem \ref{thm2.1}, a delicate construction of $\psi$  which allows small constants $M$ and $\Gamma$ in  \eqref{ineq 2.11} and \eqref{ineq 2.12} is necessary. One frequently used form of $\psi$ in this paper is 
\beq\label{2212280209}
\psi:&[0,1]^d\times\{-1,1\}\to\mbR,\\ &(x,y)\mapsto \left\{
\ba
&\phi\ykh{y\log\frac{\eta(x)}{1-\eta(x)}},&&\eta(x)\in [\delta_1,1-\delta_1],\\
&0,&&\eta(x)\in\{0,1\},\\
&\eta(x)\log\frac{1}{\eta(x)}+(1-\eta(x))\log\frac{1}{1-\eta(x)},&&\eta(x)\in (0, \delta_1)\cup(1-\delta_1,1), 
\ea
\right.
\eeq  which can be regarded as a truncated version of $\phi(yf^*_{\phi,P}(x))=\phi\ykh{y\log\frac{\eta(x)}{1-\eta(x)}}$, where $\delta_1$ is some suitable constant in $(0,1/2]$. However, in Theorem \ref{thm2.4} we use a different form of $\psi$, which will be specified later. 

The proof of Theorem \ref{thm2.1} is based on the following error decomposition 
\beq\label{2212262229}
&\mbe\zkh{\mc{R}_P^\phi\ykh{\hat{f}_n}-\Psi}\leq  \mr T_{\e,\psi,n}+(1+\e)\cdot \inf_{g\in\mc{F}}\ykh{\mc{R}_P^\phi(g)-\Psi},	\;\forall\;\e\in[0,1),
\eeq where $\mr T_{\e,\psi,n}:=\mbe\zkh{\mc{R}_P^\phi\ykh{\hat{f}_n}-\Psi-(1+\e)\cdot\frac{1}{n}\sum_{i=1}^n\ykh{\phi\ykh{Y_i\hat{f}_n(X_i)}-\psi(X_i,Y_i)}}$ and  $\Psi=\int_{[0,1]^d\times\hkh{-1,1}}\psi(x,y)\mr d P(x,y)$ (see \eqref{ineq 5.11}). Although \eqref{2212262229} is true for $\e=0$,  it's better to take $\e>0$ in \eqref{2212262229} to obtain sharp rates of convergence. This is because bounding the term $\mr T_{\e,\psi,n}$ with $\e\in(0,1)$ is easier than bounding $\mr T_{0,\psi,n}$.  To see this, note that for $\e\in(0,1)$ we have
\begin{align*}
\mr T_{\e,\psi,n}=(1+\e)\cdot \mr T_{0,\psi,n}-\e\cdot \mbe\zkh{\mc{R}_P^\phi\ykh{\hat{f}_n}-\Psi}\leq (1+\e)\cdot \mr T_{0,\psi,n}, 
\end{align*} meaning that we can always establish tighter upper bounds for $\mr T_{\e,\psi,n}$ than for $\mr T_{0,\psi,n}$ (up to the constant factor $1+\e< 2$).    Indeed, $\e>0$ is necessary in establishing Theorem \ref{thm2.1}, as indicated in its proof in Appendix \ref{section: proof of thm2.1}. We also point out that, setting $\e=0$ and $\psi\equiv 0$ (hence $\Psi=0$) in \eqref{2212262229}, and subtracting $\inf\setl{\mc R_P^\phi(g)}{g:[0,1]^d\to\mbR\text{ measurable}}$ from both sides, we will obtain a simpler error decomposition
\beq\label{2212270051}
&\mbe\zkh{\mc{E}_P^\phi\ykh{\hat{f}_n}}\leq  \mbe\zkh{\mc{R}_P^\phi\ykh{\hat{f}_n}-\frac{1}{n}\sum_{i=1}^n\ykh{\phi\ykh{Y_i\hat{f}_n(X_i)}}}+ \inf_{g\in\mc{F}}\mc{E}_P^\phi(g)\\
&\leq\mbe\zkh{\sup_{g\in\mc F}\abs{\mc{R}_P^\phi\ykh{g}-\frac{1}{n}\sum_{i=1}^n\ykh{\phi\ykh{Y_ig(X_i)}}}}+ \inf_{g\in\mc{F}}\mc{E}_P^\phi(g),
\eeq which is frequently used in the literature (see e.g.,  Lemma 2 in \cite{kohler2020statistical} and the proof of Proposition 4.1 in \cite{FML}). Note that \eqref{2212270051} does not require the explicit form of $f^*_{\phi,P}$, which means that we can also use this  error decomposition to establish rates of convergence for $\mbe\zkh{\mc E^\phi_P(\hat f_n)}$ regardless of the unboundedness of $f^*_{\phi,P}$.  However, in comparison with Theorem \ref{thm2.1}, using \eqref{2212270051} may result in slow rates of convergence because of the absence of the  positive parameter $\e$ and a carefully constructed  function $\psi$.

We now state Theorem \ref{thm2.2} which establishes generalization bounds for  empirical logistic risk minimizers over DNNs. In order to  present this result, we need the definition of H\"older spaces \citep{evans1998partial}. The H\"older space $\mc{C}^{k,\lambda}(\Omega)$, where $\Omega\subset \mbR^d$ is a closed domain, $k\in\mb N\cup\hkh{0}$ and $\lambda\in (0,1]$, consists of all those functions from $\Omega$ to $\mbR$  which have continuous derivatives up to order $k$ and whose $k$-th partial derivatives are H\"older-$\lambda$ continuous on $\Omega$. Here we say a function $g:\Omega\to\mbR$ is H\"older-$\lambda$ continuous on $\Omega$, if
\[\abs{g}_{\mc{C}^{0,\lambda}(\Omega)}:=\sup_{\Omega\ni x\neq z\in\Omega}\frac{\abs{g(x)-g(z)}}{\norm{x-z}_2^{\lambda}}<\infty.\]
Then the {H\"older spaces} $\mc{C}^{k,\lambda}(\Omega)$ can be assigned the norm
\beq\label{23022701}
\|f\|_{\mc{C}^{k,\lambda}(\Omega)}:=\max_{\norm{{\bm m}}_1\leq k}\|\mr D^{\bm m} f\|_{{\Omega}}+\max_{\norm{{\bm m}}_1=k}\abs{\mr D^{\bm m} f}_{\mc{C}^{0,\lambda}(\Omega)},\eeq where $\bm m=(m_1,\cdots,m_d) \in \left(\mb{N}\cup\{0\}\right)^d$ ranges over multi-indices (hence $\norm{\bm m}_1=\sum_{i=1}^dm_i$) and $\mr D^{\bm m}f(x_1,\ldots,x_d)=\frac{\partial^{{{ m}}_1}}{\partial x_1^{m_1}}\cdots\frac{\partial^{{{ m}}_d}}{\partial x_d^{m_d}}f(x_1,\ldots,x_d)$. Given $\beta\in (0,\infty)$, we say a function $f:\Omega\to\mbR$ is H\"older-$\beta$ smooth if $f \in \mc{C}^{k,\lambda}(\Omega)$ with $k=\ceil{\beta}-1$ and $\lambda=\beta-\ceil{\beta}+1$, where $\ceil{\beta}$ denotes the smallest integer larger than or equal to $\beta$. For any $\beta\in(0,\infty)$ and any $r\in(0,\infty)$, let 
\begin{equation}\label{holderball}
	\mc{B}^{\beta}_r\ykh{\Omega}:=\setr{f: \Omega \to \mb R }{\begin{minipage}{0.39\textwidth}{$f \in \mc{C}^{k,\lambda}(\Omega)$ and $\norm{f}_{\mc{C}^{k,\lambda}(\Omega)}\leq r$ for $k=-1+\ceil{\beta}$ and  $\lambda=\beta-\ceil{\beta}+1$ }
	\end{minipage}}
\end{equation} denote the closed ball of radius $r$ centered at the origin in the H\"older space of order $\beta$ on $\Omega$. Recall that the  space $\fdnn_d(G,N,S,B,F)$ generated by fully connected neural networks is given in \eqref{spaceofFCNN}, which is parameterized by the depth and width of neural networks (bounded by $G$ and $N$), the number of nonzero entries in weight matrices and bias vectors (bounded by $S$), and the upper bounds of neural network parameters and associated functions of form \eqref{neuralnetwok} (denoted by $B$ and $F$). In the following theorem, we show that to ensure the rate of convergence as the sample size $n$ becomes large, all these parameters should be taken within certain ranges scaling with $n$. For two positive sequences $\{\lambda_n\}_{n\geq 1}$ and $\{\nu_n\}_{n\geq 1}$, we say $\lambda_n \lesssim \nu_n$ holds if there exist $n_0\in \mb N$ and a positive constant $c$ independent of $n$ such that $\lambda_n \leq c \nu_n, \forall \;n\geq n_0$. In addition, we write  $\lambda_n \asymp \nu_n$ if and only if  $\lambda_n \lesssim \nu_n$ and $\nu_n \lesssim \lambda_n$. Recall that the excess misclassification error of $f:\mbR^d \to \mathbb{R}$ with respect to some distribution $P$ on $[0,1]^d\times\hkh{-1,1}$ is defined as 
\[
\mc{E}_P(f)=\mc{R}_P(f)-\inf \setr{\mc{R}_P(g)}{ \textrm{$g:[0,1]^d\to\mbR$ is Borel measurable}},
\] where $\mc{R}_P(f)$ denotes the misclassification error of $f$ given by 
\begin{equation*}
	\mc{R}_P(f)=P\ykh{\setl{(x,y)\in[0,1]^d\times\{-1,1\}}{ y\neq\mr{sgn}(f(x))}}.
\end{equation*}

\begin{thm}\label{thm2.2} Let $d\in\mb N$, $(\beta,r)\in(0,\infty)^2$, $n\in\mb N$, $\nu\in[0,\infty)$,  $\{(X_i,Y_i)\}_{i=1}^n$ be an i.i.d. sample in  $[0,1]^d\times\{-1,1\}$ and $\hat{f}^{\FNN}_n$ be an ERM with respect to the  logistic loss $\phi(t)=\log\ykh{1+\me^{-t}}$ over  $\fdnn_d(G,N,S,B,F)$, i.e., 
	\beq\label{FCNNestimator}
	\hat{f}^{\FNN}_n\in\mathop{{\arg\min}}_{f \in \fdnn_d(G, N,S,B,F)}\frac{1}{n}\sum_{i=1}^n\phi\ykh{Y_i f(X_i)}.
	\eeq Define
\beq\label{2302132212}
\mc{H}^{d,\beta,r}_{1}:=\setr{P\in\mc H_0^d}{\begin{minipage}{0.44\textwidth} $P_X(\setl{z\in[0,1]^d}{P(\hkh{1}|z)=\hat\eta(z)})=1$ for some $\hat\eta\in\mc{B}^{\beta}_r\ykh{[0,1]^d}$\end{minipage}}.
\eeq Then there exists a constant $\mr c\in(0,\infty)$ only depending on $(d,\beta,r)$, such that the estimator $\hat{f}^{\FNN}_n$ defined by \eqref{FCNNestimator} with 
	\beq
	&\mr c\log n\leq G \lesssim \log n, \ N \asymp \ykh{\frac{(\log n)^5}{n}}^{\frac{-d}{d+\beta}}, 
	\ S \asymp \ykh{\frac{(\log n)^5}{n}}^{\frac{-d}{d+\beta}} \cdot\log n,\\
	& 1\leq B \lesssim  n^{\nu}, \textrm{ and } \ \frac{\beta}{d+\beta}\cdot\log n\leq F\lesssim\log n
	\eeq
	satisfies 
	\beq\ba\label{bound 2.21x}
	\sup_{P\in\mc H^{d,\beta,r}_1}\bm E_{P^{\otimes n}}\zkh{\mc{E}_P^\phi\ykh{\hat{f}^{\FNN}_n}}\lesssim \ykh{\frac{\ykh{\log n}^{5}}{n}}^{\frac{\beta}{\beta+d}}
	\ea\eeq and 
	\beq\ba\label{bound 2.21}
	\sup_{P\in\mc H^{d,\beta,r}_1}\bm E_{P^{\otimes n}}\zkh{\mc{E}_{P}\ykh{\hat{f}^{\FNN}_n}}\lesssim \ykh{\frac{\ykh{\log n}^{5}}{n}}^{\frac{\beta}{2\beta+2d}}.
	\ea\eeq
\end{thm}

Theorem \ref{thm2.2} will be proved in Appendix \ref{section: proof of thm2.2}. As far as we know, for classification with neural networks and the  logistic loss $\phi$, generalization bounds presented in \eqref{bound 2.21x} and  \eqref{bound 2.21} establish  fastest  rates of convergence among the existing literature under the  H\"{o}lder smoothness condition on the conditional probability function $\eta$ of the data distribution $P$. Note that to obtain  such generalization bounds in \eqref{bound 2.21x} and  \eqref{bound 2.21} we do not require any assumption on the marginal distribution $P_X$ of the distribution $P$. For example, we dot not require that $P_X$ is absolutely continuous with respect to the Lebesgue measure.   The rate $\mc O(\ykh{\frac{\ykh{\log n}^{5}}{n}}^{\frac{\beta}{\beta+d}})$ in \eqref{bound 2.21x} for the convergence of excess $\phi$-risk is indeed optimal (up to some logarithmic factor) in the minimax sense (see Corollary \ref{thm2.6} and comments therein). However, the rate $\mc O(\ykh{\frac{\ykh{\log n}^{5}}{n}}^{\frac{\beta}{2\beta+2d}})$ in \eqref{bound 2.21} for the convergence of excess misclassification error is not optimal. According to Theorem 4.1, Theorem 4.2, Theorem 4.3 and their proofs in   \cite{audibert2007fast}, there holds
\beq\label{2301052346}
\inf_{\hat{f}_n}\sup_{P\in \mc H^{d,\beta,r}_{1}}{\bm E}_{P^{\otimes n}}\zkh{\mc{E}_P(\hat{f}_n)}\asymp n^{-\frac{\beta}{2\beta+d}}, 
\eeq  where the infimum is taken over all  $\mc F_d$-valued statistics from the sample $\hkh{(X_i,Y_i)}_{i=1}^n$. Therefore, the rate $\mc O(\ykh{\frac{\ykh{\log n}^{5}}{n}}^{\frac{\beta}{2\beta+2d}})$ in \eqref{bound 2.21} does not match the minimax optimal rate $\mc O(\ykh{\frac{1}{n}}^{\frac{\beta}{2\beta+d}})$. Despite suboptimality, the rate  $\mc O(\ykh{\frac{\ykh{\log n}^{5}}{n}}^{\frac{\beta}{2\beta+2d}})$  in \eqref{bound 2.21} is fairly close to the optimal rate $\mc O(\ykh{\frac{1}{n}}^{\frac{\beta}{2\beta+d}})$, especially when $\beta>>d$ because the exponents satisfy
\[
\lim_{\beta\to+\infty}\frac{\beta}{2\beta+2d}=\frac{1}{2}=\lim_{\beta\to+\infty}\frac{\beta}{2\beta+d}. 
\]

In our proof of Theorem \ref{thm2.2}, the rate $\mc O(\ykh{\frac{\ykh{\log n}^{5}}{n}}^{\frac{\beta}{2\beta+2d}})$ in \eqref{bound 2.21} is derived directly from the rate $\ykh{\frac{\ykh{\log n}^{5}}{n}}^{\frac{\beta}{\beta+d}}$ in \eqref{bound 2.21x} via the so-called calibration inequality which takes the form  
\beq\label{2301062317}
&\mc E_P(f)\leq c\cdot \ykh{\mc E_P^\phi(f)}^{\frac{1}{2}}\text{ for any } f\in \mc F_d\text{ and} \text{ any  $P\in\mc H_0^d$}
\eeq with $c$ being a constant independent of $P$ and $f$ (see  \eqref{23032701}).  Indeed, it follows from  Theorem 8.29 of \cite{steinwart2008support}  that
\beq\ba\label{calibrationineq}
\mc{E}_P\ykh{f} \leq 2\sqrt{2}\cdot \ykh{\mc{E}_P^{\phi}\ykh{f}}^{\frac{1}{2}} \text{ for any } f\in \mc F_d\text{ and} \text{ any  $P\in\mc H_0^d$}.
\ea\eeq In other words, \eqref{2301062317} holds when  $c=2\sqrt{2}$. Interestingly, we can use Theorem \ref{thm2.2} to obtain that the  inequality \eqref{2301062317} is optimal in the sense that the exponent $\frac{1}{2}$ cannot be replaced by a larger one. Specifically, by using \eqref{bound 2.21x} of  our Theorem \ref{thm2.2} together with \eqref{2301052346}, we can prove that $\frac{1}{2}$ is the largest number $s$ such that there holds
\beq\label{2301060034}
\mc E_P(f)\leq c\cdot \ykh{\mc E_P^\phi(f)}^{s}\text{ for any } f\in \mc F_d\text{ and} \text{ any  $P\in\mc H_0^d$}
\eeq for some constant $c$ independent of $P$ or $f$. We now demonstrate this by contradiction. Fix $d\in\mb N$. Suppose there exists an $s\in(1/2,\infty)$ and a $c\in(0,\infty)$ such that \eqref{2301060034} holds. Since \[\lim_{\beta\to+\infty}\frac{(\frac{2}{3}\qx s)\cdot\beta}{d+\beta}=\frac{2}{3}\qx s>1/2=\lim_{\beta\to+\infty}\frac{\beta}{2\beta+d},\] we can choose $\beta$ large enough such that $\frac{(\frac{2}{3}\qx s)\cdot\beta}{d+\beta}>\frac{\beta}{2\beta+d}$. Besides, it follows from $\mc E_P(f)\leq 1$ and \eqref{2301060034} that 
\beq\label{2301060036}
\mc E_P(f)\leq \abs{\mc E_P(f)}^{\frac{\frac{2}{3}\qx s}{s}}\leq \abs{c\cdot \ykh{\mc E_P^\phi(f)}^{s}}^{\frac{\frac{2}{3}\qx s}{s}}\leq (1+c)\cdot \ykh{\mc E_P^\phi(f)}^{(\frac{2}{3}\qx s)}
\eeq for any  $f\in\mc F_d$ and any $P\in\mc H_0^d$ . Let $r=3$ and $\hat{f}^{\FNN}_n$ be the estimator in Theorem \ref{thm2.2}. Then it follows from  \eqref{bound 2.21x}, \eqref{2301052346}, \eqref{2301060036} and H\"older's inequality that \begin{align*}
&n^{-\frac{\beta}{2\beta+d}}\asymp\inf_{\hat{f}_n}\sup_{P\in \mc H^{d,\beta,r}_{1}}{\bm E}_{P^{\otimes n}}\zkh{\mc{E}_P(\hat{f}_n)}\leq \sup_{P\in\mc H^{d,\beta,r}_1}\bm E_{P^{\otimes n}}\zkh{\mc{E}_P\ykh{\hat{f}^{\FNN}_n}}\\
&\leq \sup_{P\in\mc H^{d,\beta,r}_1}\bm E_{P^{\otimes n}}\zkh{(1+c)\cdot \ykh{\mc E_P^\phi(\hat{f}^{\FNN}_n)}^{(\frac{2}{3}\qx s)}}\\&\leq (1+c)\cdot \sup_{P\in\mc H^{d,\beta,r}_1}\ykh{\bm E_{P^{\otimes n}}\zkh{\mc E_P^\phi(\hat{f}^{\FNN}_n)}}^{(\frac{2}{3}\qx s)}\\
&\leq (1+c)\cdot \ykh{\sup_{P\in\mc H^{d,\beta,r}_1}\bm E_{P^{\otimes n}}\zkh{\mc E_P^\phi(\hat{f}^{\FNN}_n)}}^{(\frac{2}{3}\qx s)}\\&\lesssim \ykh{\ykh{\frac{\ykh{\log n}^{5}}{n}}^{\frac{\beta}{\beta+d}}}^{(\frac{2}{3}\qx s)}=\ykh{\frac{\ykh{\log n}^{5}}{n}}^{\frac{(\frac{2}{3}\qx s)\cdot\beta}{\beta+d}}.
\end{align*}
 Hence $n^{-\frac{\beta}{2\beta+d}}\lesssim \ykh{\frac{\ykh{\log n}^{5}}{n}}^{\frac{(\frac{2}{3}\qx s)\cdot\beta}{\beta+d}}$, which contradicts the fact that $\frac{(\frac{2}{3}\qx s)\cdot\beta}{d+\beta}>\frac{\beta}{2\beta+d}$. This proves the desired result. Due to the optimality of \eqref{2301062317} and the minimax  lower bound $\mc{O}(n^{-\frac{\beta}{d+\beta}})$ for rates of convergence of the excess $\phi$-risk stated in Corollary \ref{thm2.6},  we deduce that rates of convergence of the excess misclassification error obtained directly from those of the excess $\phi$-risk and the calibration inequality which takes the form of \eqref{2301060034}  can never be faster than $\mc O(n^{-\frac{\beta}{2d+2\beta}})$. Therefore, the convergence rate $\mc O(\ykh{\frac{\ykh{\log n}^{5}}{n}}^{\frac{\beta}{2\beta+2d}})$ of the excess misclassification error in \eqref{bound 2.21} is the fastest one (up to the logarithmic term $(\log n)^{\frac{5\beta}{2\beta+2d}}$) among all those that are derived directly from the convergence rates of the excess $\phi$-risk and the calibration inequality of the form \eqref{2301060034}, which justifies the tightness of \eqref{bound 2.21}.

It should be pointed out that the rate $\mc O(\ykh{\frac{\ykh{\log n}^{5}}{n}}^{\frac{\beta}{2\beta+2d}})$ in \eqref{bound 2.21} can be further improved if we assume the following noise condition (cf. \cite{tsybakov2004optimal}) on $P$: there exist $c_1>0$, $t_1>0$ and $s_1\in [0,\infty]$ such that
\begin{equation}\label{Tsybakovnoisecondition}
	P_X\ykh{\setl{x\in[0,1]^d }{ \big|2\cdot P(\hkh{1}|x)-1\big|\leq t}} \leq c_1 t^{s_1}, \quad \forall\;0<t
	\leq t_1.
\end{equation}  This condition  measures the size of high-noisy points and reflects the noise level through the exponent $s_1 \in [0,\infty]$. Obviously, every distribution satisfies condition \eqref{Tsybakovnoisecondition} with $s_1=0$ and $c_1=1$, whereas $s_1=\infty$ implies that we have a low amount of noise in labeling $x$, i.e., the conditional probability function $P(\hkh{1}|x)$ is bounded away from $1/2$ for $P_X$-almost all $x\in [0,1]^d$. Under the noise condition \eqref{Tsybakovnoisecondition}, the calibration inequality for logistic loss $\phi$ can be refined as
\beq\label{20221026230601}
\mc{E}_P\ykh{f} \leq \bar c \cdot \ykh{\mc{E}_P^{\phi}\ykh{f}}^{\frac{s_1+1}{s_1+2}} \text{ for all  $f\in\mc F_d$},
\eeq where $\bar c\in(0,\infty)$ is a constant only depending on $(s_1,c_1,t_1)$,  and $\frac{s_1+1}{s_1+2}:=1$ if $s_1=\infty$ (cf. Theorem 8.29 in \cite{steinwart2008support} and Theorem 1.1 in \cite{xiang2011classification}). Combining this calibration inequality \eqref{20221026230601} and \eqref{bound 2.21x}, we can obtain an improved generalization bound given by 
\[\sup_{P\in\mc H^{d,\beta,r}_{2,s_1,c_1,t_1}}\bm E_{P^{\otimes n}}\zkh{\mc{E}_{P}\ykh{\hat{f}^{\FNN}_n}}\lesssim \ykh{\frac{\ykh{\log n}^{5}}{n}}^{\frac{(s_1+1)\beta}{(s_1+2)(\beta+d)}},\]
where \beq\label{2302132219}
\mc H^{d,\beta,r}_{2,s_1,c_1,t_1}:=\setl{P\in\mc H^{d,\beta,r}_1}{\text{$P$ satisfies \eqref{Tsybakovnoisecondition}}}.\eeq

One can refer to Section \ref{section: related work} for more discussions about comparisons between Theorem \ref{thm2.2} and other related results.

In our Theorem \ref{thm2.2},  the rates $\ykh{\frac{\ykh{\log n}^{5}}{n}}^{\frac{\beta}{\beta+d}}$ and $\ykh{\frac{\ykh{\log n}^{5}}{n}}^{\frac{\beta}{2\beta+2d}}$ become slow when the dimension $d$ is large. This phenomenon, known as the curse of dimensionality, arises in our Theorem \ref{thm2.2} because our assumption on the data distribution $P$ is very mild and general. Except for the H\"older smoothness condition on the conditional probability function $\eta$ of $P$, we do not require any other assumptions in our Theorem \ref{thm2.2}.  The curse of dimensionality cannot be circumvented under such general assumption on $P$, as shown in Corollary \ref{thm2.6} and \eqref{2301052346}. Therefore, to overcome the curse of dimensionality, we need other assumptions. In our next theorem, we assume that $\eta$ is the composition of several  multivariate vector-valued functions $h_q\circ\cdots\circ h_1\circ h_0$ such that each component function of $h_i$ is either a H\"older smooth function whose output values  only depend on a small number of  its input variables, or the function computing the maximum value of some of its input variables (see \eqref{23051301} and \eqref{2302132212p}).  Under this assumption, the curse of dimensionality is circumvented because each component function of $h_i$ is either essentially defined on a low-dimensional space or a very simple maximum value function.  Our hierarchical composition  assumption on the conditional probability function is convincing and likely to be met in practice because many phenomena in natural sciences can be ``described well by processes that  take place 
at a sequence of increasing scales and are local at each scale, in
the sense that they can be described well by neighbor-to-neighbor
interactions'' (Appendix 2 of \cite{poggio2016and}). Similar compositional assumptions have  been adopted in many works such as \cite{schmidt2020nonparametric,kohler2020statistical,kohler2022rate}. One may refer to \cite{poggio2015theory,poggio2016and, poggio2017and,kohler2022rate} for more discussions about the reasonableness of such compositional assumptions.

 In our compositional assumption mentioned above, we allow the component function of  $h_i$ to be the maximum value function, which is not H\"older-$\beta$ smooth when $\beta>1$. The maximum value function is incorporated  because taking the maximum value is an important  operation to pass key information from lower scale levels to higher ones, which appears naturally  in the compositional structure of the conditional probability function $\eta$ in practical classification problems.  To see this, let us consider the following example. Suppose the classification problem is to determine whether  an input image contains a cat.  We assume the data is perfectly classified, in the sense that  the conditional probability function $\eta$  is equal to zero or one  almost surely. It should be noted that the assumption ``$\eta=0\text{ or }1$ almost surely'' does not conflict with the continuity of $\eta$ because the support of the distribution of the input data may be unconnected. This classification task can be done by human beings through considering each subpart of the input image and determining whether each subpart contains a cat.  Mathematically, let $\mc V$ be a family of subset of $\hkh{1,2,\ldots,d}$ which consists of all  the index sets of  those (considered) subparts of the input image $x\in[0,1]^d$. $\mc V$ should satisfy 
 \[
 \bigcup_{J\in\mc{V}}J=\hkh{1,2,\ldots,d}
 \] because the union of all the subparts should cover the input image itself. For each $J\in\mc V$, let 
 \[\eta_J((x)_J)=\begin{cases}
 1,&\text{ if the subpart $(x)_J$ of the input image $x$ contains a cat,}	\\
 0,&\text{ if the subpart $(x)_J$ of the input image $x$ doesn't contains a cat}. 	
 \end{cases}
 \] Then we will have $\eta(x)=\max_{J\in\mc V}\hkh{\eta_J((x)_J)}$ a.s. because 
 \begin{align*}&\eta(x)=1\xLeftrightarrow{\;\text{a.s.}\;}\text{$x$ contains a cat}\Leftrightarrow \text{at least one of the subpart $(x)_J$ contains a cat}\\
 &\Leftrightarrow\eta_J((x)_J)=1 \text{ for at least one }J\in\mc V\Leftrightarrow\max_{J\in\mc V}\hkh{\eta_J((x)_J)}=1. \end{align*} Hence the maximum value function emerges naturally in the expression of $\eta$. 

We now give the specific mathematical definition of our compositional model. For any $(d,d_\star,d_*,\beta,r)\in \mb N\times\mb N\times\mb N\times(0,\infty)\times(0,\infty)$, define
\beq\label{23061101}
\mc G_d^{\mathbf{M}}(d_\star):=\setr{f:[0,1]^d\to\mbR}{\begin{minipage}{0.47\textwidth}
		$\exists\; I\subset \hkh{1,2,\ldots,d}$ such that $1\leq \#(I)\leq d_\star$ and $ f(x)=\max\setr{(x)_i}{i\in I},\,\forall\,x\in[0,1]^d$ 
\end{minipage}},
\eeq
and 
\beq\label{23061102}
&\mc G_d^{\mathbf{H}}(d_*, \beta,r)\\&:=\setr{f:[0,1]^d\to\mbR}{\begin{minipage}{0.54\textwidth}
		$\exists\;I\subset \hkh{1,2,\ldots,d}$ and $g\in \mc{B}^{\beta}_r\ykh{[0,1]^{d_*}}$ such that $\#(I)=d_*$ and $f(x)=g\ykh{(x)_I}$ for all $x\in[0,1]^d$ 
\end{minipage}}.
\eeq Thus $\mc G_d^{\mathbf{M}}(d_\star)$ consists of  all functions from $[0,1]^d$ to $\mbR$ which compute the maximum value of at most $d_\star$ components of their input vectors, and  $\mc G_d^{\mathbf{H}}(d_*, \beta,r)$ consists of all functions from $[0,1]^d$ to $\mbR$ which only depend on $d_*$ components of the input vector and are H\"older-$\beta$ smooth with corresponding H\"older-$\beta$ norm less than or equal to $r$.     Obviously, \beq\label{23060501}
\mc G_d^{\mathbf{H}}(d_*, \beta,r)=\varnothing,\;\forall\;(d,d_*,\beta,r)\in\mb N\times\mb N\times(0,\infty)\times(0,\infty) \text{ with }d<d_*.\eeq Next, for any $(d_\star, d_*,\beta,r)\in \mb N\times\mb N\times(0,\infty)\times(0,\infty)$, define  $\mc G_\infty^{\mathbf{H}}(d_*, \beta,r):=\bigcup_{d=1}^\infty \mc G_d^{\mathbf{H}}(d_*, \beta,r)$ and  $\mc G_\infty^{\mathbf{M}}(d_\star):=\bigcup_{d=1}^\infty \mc G_d^{\mathbf{M}}(d_\star)$. Finally,  for any $q\in \mb N\cup\hkh{0}$,  any $(\beta,r)\in(0,\infty)^2$ and any $(d,d_\star, d_*,K)\in\mb N^4$ with 
\beq\label{23060601}
d_*\leq \min\hkh{d,K+\idf_{\hkh{0}}(q)\cdot(d-K)},
\eeq define  \beq\label{23061201}
&\mc G_d^{\mathbf{CH}}(q, K, d_*, \beta,r)\\&:=\setr{h_q\circ\cdots \circ h_1\circ h_0}{\begin{minipage}{0.45\textwidth}
		$h_0,h_1,\ldots, h_{q-1}, h_q$ are functions satisfying the following conditions: 
		\begin{enumerate}[(i)]
			\item   ${\mathbf{{dom}}}(h_i)=[0,1]^K$ for $0<i\leq q$ and ${\mathbf{{dom}}}(h_0)=[0,1]^d$;
			\item $\mathbf{ran}(h_i)\subset[0,1]^K$ for $0 \leq i<q$ and $\mathbf{ran}(h_q)\subset\mbR$;
			\item $h_q\in \mc G_\infty^{\mathbf{H}}(d_*, \beta,r)$;
			\item For $0\leq i< q$ and $1\leq j\leq K$, the $j$-th coordinate function  of $h_i$ given by  $\mathbf{dom}(h_i)\ni x\mapsto (h_i(x))_j\in\mbR$ belongs to $\mc G_\infty^{\mathbf{H}}(d_*, \beta,r)$
\end{enumerate}\end{minipage}}
\eeq and 
\beq\label{23051301}
&\mc G_d^{\mathbf{CHOM}}(q, K,d_\star, d_*, \beta,r)\\&:=\setr{h_q\circ\cdots \circ h_1\circ h_0}{\begin{minipage}{0.45\textwidth}
		$h_0,h_1,\ldots, h_{q-1}, h_q$ are functions satisfying the following conditions: 
		\begin{enumerate}[(i)]
			\item   ${\mathbf{{dom}}}(h_i)=[0,1]^K$ for $0<i\leq q$ and ${\mathbf{{dom}}}(h_0)=[0,1]^d$;
			\item $\mathbf{ran}(h_i)\subset[0,1]^K$ for $0 \leq i<q$ and $\mathbf{ran}(h_q)\subset\mbR$;
			\item $h_q\in \mc G_\infty^{\mathbf{H}}(d_*, \beta,r)\cup \mc G_\infty^{\mathbf{M}}(d_\star)$;
			\item For $0\leq i< q$ and $1\leq j\leq K$, the $j$-th coordinate function  of $h_i$ given by  $\mathbf{dom}(h_i)\ni x\mapsto (h_i(x))_j\in\mbR$ belongs to $\mc G_\infty^{\mathbf{H}}(d_*, \beta,r)\cup \mc G_\infty^{\mathbf{M}}(d_\star)$
\end{enumerate}\end{minipage}}.
\eeq Obviously, we always have that $\mc G_d^{\mathbf{CH}}(q, K, d_*, \beta,r)\subset \mc G_d^{\mathbf{CHOM}}(q, K,d_\star, d_*, \beta,r)$. The condition \eqref{23060601}, which is equivalent to
\[
d_*\leq \begin{cases}
	d,&\text{ if }q=0,\\
	d\qx K,&\text{ if }q>0,
\end{cases}
\] is required in the above definitions because it follows from \eqref{23060501} that 
\[
\mc G_d^{\mathbf{CH}}(q, K, d_*, \beta,r)=\varnothing \text{ if }d_*> \min\hkh{d,K+\idf_{\hkh{0}}(q)\cdot(d-K)}.
\] Thus we impose the condition \eqref{23060601}  simply to avoid the trivial empty set.  The space  $\mc G_d^{\mathbf{CH}}(q, K, d_*, \beta,r)$ consists of composite functions $h_q\circ\cdots h_1\circ h_0$ satisfying that each component function of $h_i$   only depends on $d_*$ components of its input vector and is H\"older-$\beta$ smooth with corresponding H\"older-$\beta$ norm less than or equal to $r$. For example, the function $[0,1]^4\ni x\mapsto \sum\limits_{1\leq i<j\leq 4}{(x)_i\cdot(x)_j}\in\mbR$ belongs to $\mc G_4^{\mathbf{CH}}(2, 4, 2, 2,8)$ (cf. Figure \ref{fig1y}). \begin{figure}[htbp]
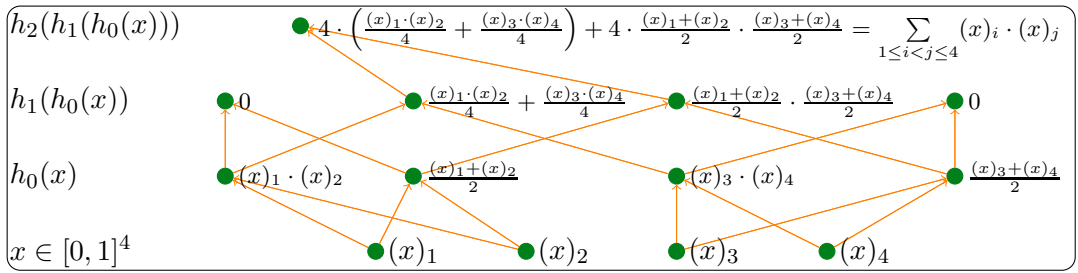

	\centering
	\betikz
	\tikzset{ %
		mypt/.style ={
			circle, %
			minimum width =0pt, %
			minimum height =0pt, %
			inner sep=0pt, %
			draw=none, %
		}
	}
	\tikzset{
		mybox/.style ={
			rectangle, %
			rounded corners =5pt, %
			minimum width =40pt, %
			minimum height =100pt, %
			inner sep=5pt, %
			draw=blue, %
			fill=cyan
	}}
	\tikzset{
		bigbox/.style ={
			rectangle, %
			rounded corners =5pt, %
			minimum width =404pt, %
			minimum height =100pt, %
			inner sep=0pt, %
			draw=black, %
			fill=none
	}}
	\tikzset{
		tinycircle/.style ={
			circle, %
			minimum width =6pt, %
			minimum height =6pt, %
			inner sep=0pt, %
			draw=darkgreen, %
			fill=darkgreen %
		}
	}
	
	\node[bigbox] at (-0.8,-1.995) {};

	\node[tinycircle]  (01) at(-3,-3.5) {};
	\node[tinycircle]  (02) at(-1,-3.5) {};
	\node[tinycircle]  (03) at(1,-3.5) {};
	\node[tinycircle]  (04) at(3,-3.5) {};

		\node[tinycircle]  (11) at(-5,-2.5) {};
			\node[tinycircle]  (12) at(-2.5,-2.5) {};
			\node[tinycircle]  (13) at(1,-2.5) {};
			\node[tinycircle]  (14) at(4.7,-2.5) {};

			\node[tinycircle]  (21) at(-5,-1.5) {};
			\node[tinycircle]  (22) at(-2.5,-1.5) {};
			\node[tinycircle]  (23) at(1,-1.5) {};
			\node[tinycircle]  (24) at(4.7,-1.5) {};
			
		\filldraw[->, orange] (01)--(11);
		\filldraw[->, orange] (01)--(12);
		\filldraw[->, orange] (02)--(11);
		\filldraw[->, orange] (02)--(12);
		
		\filldraw[->, orange] (03)--(13);
		\filldraw[->, orange] (03)--(14);
		\filldraw[->, orange] (04)--(13);
		\filldraw[->, orange] (04)--(14);

\filldraw[->, orange] (11)--(22);
\filldraw[->, orange] (13)--(22);
\filldraw[->, orange] (12)--(23);
\filldraw[->, orange] (14)--(23);
\filldraw[->, orange] (11)--(21);
\filldraw[->, orange] (12)--(21);
\filldraw[->, orange] (13)--(24);
\filldraw[->, orange] (14)--(24);
\node[tinycircle]  (31) at(-4,-0.5) {};

\filldraw[->, orange] (23)--(31);
\filldraw[->, orange] (22)--(31);

\node[mypt] (x01) at (-2.5,-3.5) {
	$(x)_1$};
	\node[mypt] (x02) at (-0.5,-3.5) { $(x)_2$};
	\node[mypt] (x03) at (1.5,-3.5) { $(x)_3$};
	\node[mypt] (x04) at (3.5,-3.5) {{$(x)_4$}};

	\node[right] (y0) at (-8,-3.5) {{$x\in[0,1]^4$}};
	\node[right] (y1) at (-8,-2.5) {{$h_0(x)$}};
	\node[right] (y2) at (-8,-1.5) {{$h_1(h_0(x))$}};
	\node[right] (y3) at (-8,-0.5) {{$h_2(h_1(h_0(x)))$}};

\node[right] (x11) at (-4.965,-2.51) {%
{\fontsize{9}{11}\selectfont
$(x)_1\cdot (x)_2$}};
\node[right] (x12) at (-2.465,-2.51) {%
{\fontsize{9}{11}\selectfont
$\frac{(x)_1+ (x)_2}{2}$}};
\node[right] (x13) at (1.035,-2.51) {%
{\fontsize{9}{11}\selectfont
${(x)_3\cdot (x)_4}$}};
\node[right] (x14) at (4.735,-2.51) {%
{\fontsize{9}{11}\selectfont
$\frac{(x)_3+ (x)_4}{2}$}};

\node[right] (x21) at (-4.955,-1.51) {%
{\fontsize{9}{11}\selectfont
$0$}};
\node[right] (x22) at (-2.465,-1.51) {%
{\fontsize{9}{11}\selectfont
$\frac{(x)_1\cdot (x)_2}{4}+\frac{(x)_3\cdot (x)_4}{4}$}};
\node[right] (x23) at (1.035,-1.51) {%
{\fontsize{9}{11}\selectfont
$\frac{(x)_1+ (x)_2}{2}\cdot\frac{(x)_3+ (x)_4}{2}$}};
\node[right] (x24) at (4.745,-1.51) {%
{\fontsize{9}{11}\selectfont
$0$}};

\node[right] (x31) at (-3.9,-0.63) {%
{\fontsize{9}{11}\selectfont
$4\cdot\ykh{\frac{(x)_1\cdot (x)_2}{4}+\frac{(x)_3\cdot (x)_4}{4}}+4\cdot\frac{(x)_1+ (x)_2}{2}\cdot\frac{(x)_3+ (x)_4}{2}=\sum\limits_{1\leq i<j\leq 4}(x)_i\cdot(x)_j$}};

	\eetikz
	\captionsetup{justification=centering}
	\caption{An illustration of the function $[0,1]^4\ni x\mapsto \sum\limits_{1\leq i<j\leq 4}{(x)_i\cdot(x)_j}\in\mbR$, \\which belongs to $\mc G_4^{\mathbf{CH}}(2, 4, 2, 2,8)$.  }
	\label{fig1y}
\end{figure} The definition of $\mc G_d^{\mathbf{CHOM}}(q, K,d_\star, d_*, \beta,r)$ is similar to that of  $\mc G_d^{\mathbf{CH}}(q, K, d_*, \beta,r)$. The only difference is that, in comparison to $\mc G_d^{\mathbf{CH}}(q, K, d_*, \beta,r)$,  we in the definition of $\mc G_d^{\mathbf{CHOM}}(q, K,d_\star, d_*, \beta,r)$ additionally allow the component function of $h_i$ to be the function which computes the maximum value of at most $d_\star$ components of its input vector. For example, the function $[0,1]^4\ni x\mapsto \max\limits_{1\leq i<j\leq 4}{(x)_i\cdot(x)_j}\in\mbR$ belongs to $\mc G_4^{\mathbf{CHOM}}(2, 6,3, 2, 2,2)$ (cf. Figure \ref{fig1x}). From the above description of the spaces $\mc G_d^{\mathbf{CH}}(q, K, d_*, \beta,r)$ and $\mc G_d^{\mathbf{CHOM}}(q, K,d_\star, d_*, \beta,r)$, we see that the condition (2.30) is very natural because it merely requires the essential input dimension $d_*$ of the H\"older-$\beta$ smooth  component function of $h_i$  to be less than or equal to its actual input dimension, which is $d$ (if $i=0$) or $K$ (if $i>0$).   At last, we point out that the space $\mc G_d^{\mathbf{CH}}(q, K, d_*, \beta,r)$ reduces to the H\"older ball $\mc{B}^{\beta}_r([0,1]^{d})$ when $q=0$ and $d_*=d$. Indeed, we have  that
\beq\label{23052201}
& \mc{B}^{\beta}_r([0,1]^{d})=\mc G_d^{\mathbf{H}}(d, \beta,r)=\mc G_d^{\mathbf{CH}}(0, K, d, \beta,r)\\&\subset \mc G_d^{\mathbf{CHOM}}(0, K,d_\star, d, \beta,r),\;\forall\;K\in\mb N, \;d\in\mb N,\;d_\star\in\mb N,\;\beta\in(0,\infty),r\in(0,\infty).
\eeq 
\begin{figure}[htbp]
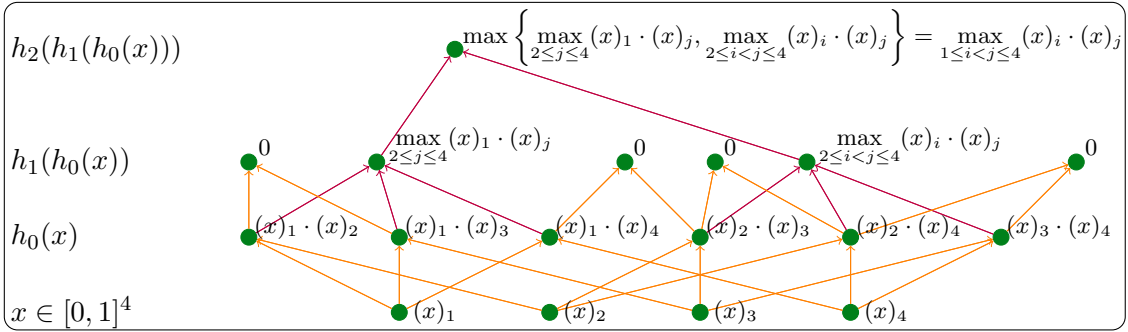

	\centering
	\betikz
	\tikzset{ %
		mypt/.style ={
			circle, %
			minimum width =0pt, %
			minimum height =0pt, %
			inner sep=0pt, %
			draw=none, %
		}
	}
	\tikzset{
		mybox/.style ={
			rectangle, %
			rounded corners =5pt, %
			minimum width =40pt, %
			minimum height =100pt, %
			inner sep=5pt, %
			draw=blue, %
			fill=cyan
	}}
	\tikzset{
		bigbox/.style ={
			rectangle, %
			rounded corners =5pt, %
			minimum width =424pt, %
			minimum height =124pt, %
			inner sep=0pt, %
			draw=black, %
			fill=none
	}}
	\tikzset{
		tinycircle/.style ={
			circle, %
			minimum width =6pt, %
			minimum height =6pt, %
			inner sep=0pt, %
			draw=darkgreen, %
			fill=darkgreen %
		}
	}
	
	\node[bigbox] at (-0.8,-1.555) {};

	\node[tinycircle]  (01) at(-3,-3.5) {};
	\node[tinycircle]  (02) at(-1,-3.5) {};
	\node[tinycircle]  (03) at(1,-3.5) {};
		\node[tinycircle]  (04) at(3,-3.5) {};

		\node[tinycircle]  (11) at(-3,-2.5) {};
				\node[tinycircle]  (12) at(-1,-2.5) {};
				\node[tinycircle]  (13) at(1,-2.5) {};
				\node[tinycircle]  (14) at(3,-2.5) {};
		\node[tinycircle]  (15) at(-5,-2.5) {};
		\node[tinycircle]  (16) at(5,-2.5) {};
		
			\node[tinycircle]  (21) at(-3.3,-1.5) {};
		\node[tinycircle]  (22) at(0,-1.5) {};
		\node[tinycircle]  (23) at(1.2,-1.5) {};
		\node[tinycircle]  (24) at(2.42,-1.5) {};
		\node[tinycircle]  (25) at(-5,-1.5) {};
		\node[tinycircle]  (26) at(6,-1.5) {};
		
		\node[tinycircle]  (31) at(-2.26,0) {};

\filldraw[->, orange] (01)--(11);
\filldraw[->, orange] (02)--(15);
\filldraw[->, orange] (01)--(12);
\filldraw[->, orange] (03)--(11);
\filldraw[->, orange] (01)--(15);
\filldraw[->, orange] (04)--(12);

\filldraw[->, orange] (02)--(13);
\filldraw[->, orange] (03)--(13);
\filldraw[->, orange] (02)--(14);
\filldraw[->, orange] (03)--(16);
\filldraw[->, orange] (04)--(16);
\filldraw[->, orange] (04)--(14);

\filldraw[->, purple] (11)--(21);
\filldraw[->, purple] (15)--(21);
\filldraw[->, purple] (12)--(21);

\filldraw[->, purple] (13)--(24);
\filldraw[->, purple] (16)--(24);
\filldraw[->, purple] (14)--(24);

\filldraw[->, orange] (15)--(25);
\filldraw[->, orange] (11)--(25);

\filldraw[->, orange] (14)--(26);
\filldraw[->, orange] (16)--(26);

\filldraw[->, orange] (13)--(22);
\filldraw[->, orange] (12)--(22);

\filldraw[->, orange] (13)--(23);
\filldraw[->, orange] (14)--(23);

\filldraw[->, purple] (21)--(31);
\filldraw[->, purple] (24)--(31);

\node[mypt] (x01) at (-2.53,-3.5) {\fontsize{9}{11}\selectfont $(x)_1$};
\node[mypt] (x02) at (-0.53,-3.5) {\fontsize{9}{11}\selectfont $(x)_2$};
\node[mypt] (x03) at (1.47,-3.5) {\fontsize{9}{11}\selectfont $(x)_3$};
\node[mypt] (x04) at (3.47,-3.5) { \fontsize{9}{11}\selectfont {$(x)_4$}};

\node[mypt] (x11) at (-4.23,-2.366) {%
{\fontsize{9}{11}\selectfont
$(x)_1\cdot(x)_2$}};
\node[mypt] (x12) at (-2.23,-2.366) {%
{\fontsize{9}{11}\selectfont
$(x)_1\cdot(x)_3$}};
\node[mypt] (x13) at (-0.23,-2.366) {%
{\fontsize{9}{11}\selectfont
$(x)_1\cdot(x)_4$}};
\node[mypt] (x14) at (1.77,-2.366) {%
{\fontsize{9}{11}\selectfont
$(x)_2\cdot(x)_3$}};
\node[mypt] (x15) at (3.77,-2.366) {%
{\fontsize{9}{11}\selectfont
$(x)_2\cdot(x)_4$}};
\node[mypt] (x16) at (5.77,-2.366) {%
{\fontsize{9}{11}\selectfont
$(x)_3\cdot(x)_4$}};

\node[mypt] (x21) at (-2.05,-1.3) {%
{\fontsize{9}{11}\selectfont
$\max\limits_{2\leq j\leq 4}(x)_1\cdot(x)_j$}};

\node[mypt] (x24) at (3.8,-1.3) {%
{\fontsize{9}{11}\selectfont
$\max\limits_{2\leq i< j\leq 4}(x)_i\cdot(x)_j$}};

\node[mypt] (x23) at (1.4,-1.3) {%
{\fontsize{9}{11}\selectfont
$0$}};
\node[mypt] (x25) at (6.2,-1.3) {%
{\fontsize{9}{11}\selectfont
$0$}};
\node[mypt] (x26) at (0.2,-1.3) {%
{\fontsize{9}{11}\selectfont
$0$}};
\node[mypt] (x22) at (-4.79,-1.3) {%
{\fontsize{9}{11}\selectfont
$0$}};

\node[right] (x33) at (-2.291,0.16) {%
{\fontsize{9}{11}\selectfont
$\max\hkh{\max\limits_{2\leq j\leq 4}(x)_1\cdot(x)_j,\max\limits_{2\leq i< j\leq 4}(x)_i\cdot(x)_j}=\max\limits_{1\leq i<j\leq 4}(x)_i\cdot(x)_j$}};

\node[right] (y0) at (-8.3,-3.5) {{$x\in[0,1]^4$}};
\node[right] (y1) at (-8.3,-2.5) {{$h_0(x)$}};
\node[right] (y2) at (-8.3,-1.5) {{$h_1(h_0(x))$}};
\node[right] (y3) at (-8.3,-0) {{$h_2(h_1(h_0(x)))$}};

	\eetikz
	 \captionsetup{justification=centering}
	\caption{An illustration of the function $[0,1]^4\ni x\mapsto \max\limits_{1\leq i<j\leq 4}{(x)_i\cdot(x)_j}\in\mbR$, \\which belongs to $\mc G_4^{\mathbf{CHOM}}(2, 6,3, 2, 2,2)$.  }
	\label{fig1x}
\end{figure}

Now we are in a position to state our Theorem \ref{thm2.2p}, where we establish sharp convergence rates, which are free from the input dimension $d$,  for fully connected DNN classifiers trained with the logistic loss under the assumption that the conditional probability function $\eta$ of the data distribution belongs to  $\mc G_d^{\mathbf{CHOM}}(q, K,d_\star, d_*, \beta,r)$. In particular, it can be shown the convergence rate of the excess logistic risk stated in \eqref{bound 2.21xp} in Theorem \ref{thm2.2p} is optimal (up to some logarithmic term).  Since  $\mc G_d^{\mathbf{CH}}(q, K, d_*, \beta,r)\subset \mc G_d^{\mathbf{CHOM}}(q, K,d_\star, d_*, \beta,r)$, the same convergences rates as in Theorem \ref{thm2.2p} can also be achieved under the slightly narrower assumption that $
\eta$ belongs to $\mc G_d^{\mathbf{CH}}(q, K, d_*, \beta,r)$.  The results of Theorem \ref{thm2.2p} break the curse of dimensionality and help explain why deep neural networks perform well, 
especially in high-dimensional problems.

\begin{thm}\label{thm2.2p} Let $q\in\mb N\cup\hkh{0}$,   $(d,d_\star, d_*, K)\in\mb N^4$ with  $d_*\leq \min\hkh{d,K+\idf_{\hkh{0}}(q)\cdot(d-K)}$,  $(\beta,r)\in(0,\infty)^2$, $n\in\mb N$,  $\nu\in[0,\infty)$,  $\{(X_i,Y_i)\}_{i=1}^n$ be an i.i.d. sample in  $[0,1]^d\times\{-1,1\}$ and $\hat{f}^{\FNN}_n$ be an ERM with respect to the  logistic loss $\phi(t)=\log\ykh{1+\me^{-t}}$ over the space   $\fdnn_d(G,N,S,B,F)$, which is given by \eqref{FCNNestimator}.  Define
	\beq\label{2302132212p}
	\mc{H}^{d,\beta,r}_{4,q,K,d_\star,d_*}:=\setr{P\in\mc H_0^d}{\begin{minipage}{0.44\textwidth} $P_X(\setl{z\in[0,1]^d}{P(\hkh{1}|z)=\hat\eta(z)})=1$ for some $\hat\eta\in\mc G_d^{\mathbf{CHOM}}(q, K,d_\star, d_*, \beta,r)$\end{minipage}}.
	\eeq Then there exists a constant $\mr c\in(0,\infty)$ only depending on  $(d_\star,d_*,\beta,r,q)$, such that the estimator $\hat{f}^{\FNN}_n$ defined by \eqref{FCNNestimator} with 
	\beq
	&\mr c\log n\leq G \lesssim \log n, \ N \asymp \ykh{\frac{(\log n)^5}{n}}^{\frac{-d_*}{d_*+\beta\cdot(1\qx\beta)^q}}, 
	\ S \asymp \ykh{\frac{(\log n)^5}{n}}^{\frac{-d_*}{d_*+\beta\cdot(1\qx\beta)^q}} \cdot\log n,\\
	& 1\leq B \lesssim  n^{\nu}, \textrm{ and } \ \frac{\beta\cdot(1\qx\beta)^q}{d_*+\beta\cdot(1\qx\beta)^q}\cdot\log n\leq F\lesssim\log n
	\eeq
	satisfies 
	\beq\ba\label{bound 2.21xp}
	\sup_{P\in\mc{H}^{d,\beta,r}_{4,q,K,d_\star,d_*}}\bm E_{P^{\otimes n}}\zkh{\mc{E}_P^\phi\ykh{\hat{f}^{\FNN}_n}}\lesssim \ykh{\frac{(\log n)^5}{n}}^{\frac{\beta\cdot(1\qx\beta)^q}{d_*+\beta\cdot(1\qx\beta)^q}}
	\ea\eeq and 
	\beq\ba\label{bound 2.21p}
	\sup_{P\in\mc{H}^{d,\beta,r}_{4,q,K,d_\star,d_*}}\bm E_{P^{\otimes n}}\zkh{\mc{E}_{P}\ykh{\hat{f}^{\FNN}_n}}\lesssim \ykh{\frac{(\log n)^5}{n}}^{\frac{\beta\cdot(1\qx\beta)^q}{2d_*+2\beta\cdot(1\qx\beta)^q}}.
	\ea\eeq
\end{thm}

The proof of Theorem \ref{thm2.2p} is given in Appendix \ref{section: proof of thm2.2}. Note that Theorem \ref{thm2.2p} directly leads to  Theorem \ref{thm2.2}  because it follows from  \eqref{23052201} that 
\[
\mc H^{d,\beta,r}_1\subset \mc{H}^{d,\beta,r}_{4,q,K,d_\star,d_*}\;\text{ if $q=0$, $d_*=d$ and $d_\star=K=1$. }
\] Consequently, Theorem \ref{thm2.2p} can be regarded as a generalization of Theorem \ref{thm2.2}. Note that both the rates $\mc{O}(\ykh{\frac{(\log n)^5}{n}}^{\frac{\beta\cdot(1\qx\beta)^q}{d_*+\beta\cdot(1\qx\beta)^q}})$ and  $\mc{O}(\ykh{\frac{(\log n)^5}{n}}^{\frac{\beta\cdot(1\qx\beta)^q}{2d_*+2\beta\cdot(1\qx\beta)^q}})$ in \eqref{bound 2.21xp} and \eqref{bound 2.21p} are independent of the input dimension $d$, thereby overcoming the curse of dimensionality.  Moreover,  according to Theorem \ref{thm2.6p} and the comments therein, the rate $\ykh{\frac{(\log n)^5}{n}}^{\frac{\beta\cdot(1\qx\beta)^q}{d_*+\beta\cdot(1\qx\beta)^q}}$ in \eqref{bound 2.21xp} for the convergence of the excess logistic risk is even optimal (up to some logarithmic factor). This justifies the sharpness of  Theorem \ref{thm2.2p}. 

Next, we would like to demonstrate the main idea of the proof of Theorem \ref{thm2.2p}. The strategy we adopted is to apply Theorem \ref{thm2.1} with a suitable $\psi$ satisfying \eqref{eq 2.13}. Let $P$ be an arbitrary probability in $\mc{H}^{d,\beta,r}_{4,q,K,d_\star,d_*}$ and denote by $\eta$ the conditional probability function $P(\hkh{1}|\cdot)$ of $P$.  
According to the previous discussions, we cannot simply take $\psi(x,y)=\phi(yf^*_{\phi,P}(x))$ as the target function  $f^*_{\phi,P}=\log\frac{\eta}{1-\eta}$ is unbounded. Instead, we define $\psi$ by \eqref{2212280209} 
for some carefully selected $\delta_1\in (0,1/2]$. For such $\psi$, we prove 
\beq\label{23072602}
\int_{[0,1]^d\times\{-1,1\}}{\psi\ykh{x,y}}\mr{d}P(x,y)=\inf \setl{\mc{R}_P^\phi(f)}{ \textrm{$f:[0,1]^d\to\mbR$ is measurable}}\eeq in Lemma \ref{23022804}, 
and establish a tight inequality of form  (\ref{ineq 2.12}) with $\Gamma=\mc O( (\log\frac{1}{\delta_1})^2)$ in Lemma \ref{lem5.5}. We then calculate the covering numbers of $\mc F:=\fdnn_d(G,N,S,B,F)$ by Corollary \ref{corollaryA1} and use Lemma \ref{lemma5.6} to estimate the approximation error
\[
\inf_{f\in \mc F}\ykh{\mc{R}_P^\phi(f)-\int_{[0,1]^d\times\{-1,1\}}{\psi(x,y)}\mr{d}P(x,y)}
\] which is essentially $\inf_{f\in \mc F}\mc{E}_P^\phi(f)$. Substituting the above estimations into the right hand side of \eqref{bound 2.10} and taking supremum over $P\in\mc{H}^{d,\beta,r}_{4,q,K,d_\star,d_*}$, we obtain \eqref{bound 2.21xp}. We then derive  \eqref{bound 2.21p} from \eqref{bound 2.21xp} through the calibration inequality \eqref{calibrationineq}.

We would like to point out that the above scheme  for obtaining generalization bounds,  which is built on our novel oracle-type inequality in Theorem \ref{thm2.1} with a carefully constructed $\psi$,  is very general. This scheme can be used to establish generalization bounds for classification in other settings, provided that the estimation for the corresponding approximation error is given. For example, one can expect to establish generalization bounds for convolutional neural network (CNN)  classification with the logistic loss by using Theorem \ref{thm2.1} together with recent results about CNN approximation. CNNs perform convolutions instead of matrix multiplications in at least one of their layers (cf. Chapter 9 of \cite{goodfellow2016deep}).  Approximation  properties of various CNN architectures have been intensively studied recently. For instance, 1D CNN approximation is studied in \cite{zhou2020universality,zdxdown,mao2021theory,fang2020theory}, and 2D CNN approximation is investigated in \cite{kohler2022rate,he2022approximation}. With the help of  these CNN approximation results and classical  concentration techniques, generalization bounds for CNN classification have been established in many works such as {\cite{kohler2020statistical,kohler2022rate,shen2021non,feng2021generalization}{\label{23072306}}}. In our coming work \cite{zhang202304cnn},  we will  derive generalization bounds for CNN classification with logistic loss on spheres under the Sobolev smooth conditional probability  assumption through the novel framework developed in our paper.  %

In our proof of Theorem \ref{thm2.2} and Theorem \ref{thm2.2p}, a tight error bound for neural network  approximation of the  logarithm function $\log(\cdot)$ arises as a by-product. Indeed, for a given data distribution $P$ on $[0,1]^d\times\hkh{-1,1}$,   to estimate the approximation error of $\fdnn_d$, we need to construct neural networks to approximate the target function $f^*_{\phi,P}=\log\frac{\eta}{1-\eta}$, where $\eta$ denotes the conditional probability function of $P$.  Due to the unboundedness of $f^*_{\phi,P}$, one cannot approximate $f^*_{\phi,P}$ directly. To overcome this difficulty, we consider truncating $f^*_{\phi,P}$ to obtain an efficient approximation. We design neural networks $\tilde{\eta}$ and $\tilde{l}$ to approximate $\eta$ on $[0,1]^d$ and $\log(\cdot)$ on $[\delta_n,1-\delta_n]$ respectively, where $\delta_n\in(0,1/4]$ is a carefully selected number which depends on the sample size $n$ and tends to zero as $n\to\infty$. Let $\Pi_{\delta_n}$ denote the clipping function given by $\Pi_{\delta_n}:\mbR\to[\delta_n,1-\delta_n], t \mapsto\mathop{\arg\min}_{t'\in[\delta_n,1-\delta_n]}\abs{t'-t}$. Then $\tilde L:t\mapsto\tilde{l}(\Pi_{\delta_n}(t))-\tilde l(1-\Pi_{\delta_n}(t))$ is a neural network which approximates the function \beq\label{23031201}
\overline L_{\delta_n}:t\mapsto\begin{cases}
	\log\frac{t}{1-t},&\text{ if }t\in[\delta_n,1-\delta_n],\\
	\log\frac{1-\delta_n}{\delta_n},&\text{ if }t>1-\delta_n,\\
		\log\frac{\delta_n}{1-\delta_n},&\text{ if }t<\delta_n,
\end{cases}\eeq meaning that the function  $\tilde L(\tilde\eta(x))=\tilde{l}\ykh{\Pi_{\delta_n}\ykh{\tilde{\eta}(x)}}-\tilde{l}\ykh{\Pi_{\delta_n}\ykh{1-\tilde{\eta}(x)}}$ is a neural network which approximates the truncated $f^*_{\phi,P}$ given by
\[
\overline L_{\delta_n}\circ\eta:x\mapsto \overline L_{\delta_n}(\eta(x))= 
\left\{\ba &f^*_{\phi,P}(x),&&\text{ if } \abs{f^*_{\phi,P}(x)}\leq\log\frac{1-\delta_n}{\delta_n},\\&\mr{sgn}(f^*_{\phi,P}(x)) \log\frac{1-\delta_n}{\delta_n},&&\textrm{otherwise}.\ea\right.\] One can build  $\tilde{\eta}$ by applying some existing results on approximation theory of neural networks (see Appendix \ref{A.3}). However, the construction of $\tilde{l}$ requires more effort. Since the logarithm function $\log(\cdot)$ is unbounded near $0$, which leads to the blow-up of its H\"older norm on $[\delta_n,1-\delta_n]$ when $\delta_n$ is becoming small, existing conclusions, e.g., the results in Appendix \ref{A.3}, cannot yield satisfactory error bounds for neural network approximation of $\log(\cdot)$ on $[\delta_n,1-\delta_n]$. To see this, let us consider using Theorem \ref{thmA3} to estimate the approximation error directly. Note that approximating $\log(\cdot)$ on $[\delta_n,1-\delta_n]$ is equivalent to approximating  $l_{\delta_n}(t):=\log((1-2\delta_n)t+\delta_n)$ on $[0,1]$. For $\beta_1>0$ with $k=\ceil{\beta_1-1}$ and $\lambda=\beta_1-\ceil{\beta_1-1}$, denote by $l^{(k)}_{\delta_n}$ the $k$-th derivative of $l_{\delta_n}$. Then there holds
\begin{align*}
&\norm{l_{\delta_n}}_{\mc{C}^{k,\lambda}([0,1])}\geq\sup_{0\leq t<t'\leq 1}\frac{\abs{l_{\delta_n}^{(k)}(t
		)-l_{\delta_n}^{(k)}(t')}}{\abs{t-t'}^\lambda}\\&\geq  \frac{\abs{l_{\delta_n}^{(k)}(0
		)-l_{\delta_n}^{(k)}\ykh{\frac{\delta_n}{1-2\delta_n}}}}{\abs{0-\frac{\delta_n}{1-2\delta_n}}^\lambda}\geq\inf_{t\in \zkh{0,\frac{\delta_n}{1-2\delta_n}}}\frac{\abs{l_{\delta_n}^{(k+1)}\ykh{t}}\cdot\abs{0-\frac{\delta_n}{1-2\delta_n}}}{\abs{0-\frac{\delta_n}{1-2\delta_n}}^\lambda}\\
&=\inf_{t\in \zkh{0,\frac{\delta_n}{1-2\delta_n}}}\frac{\abs{\frac{k!}{((1-2\delta_n)t+\delta_n)^{k+1}}}\cdot(1-2\delta_n)^{k+1}\cdot\abs{0-\frac{\delta_n}{1-2\delta_n}}}{\abs{0-\frac{\delta_n}{1-2\delta_n}}^\lambda}=\frac{k!}{2^{k+1}}\cdot (1-2\delta_n)^{k+\lambda}\cdot\frac{1}{\delta_n^{k+\lambda}}. 
\end{align*} Hence it follows from $\delta_n\in(0, 1/4]$ that 
\[\ba
\norm{l_{\delta_n}}_{\mc{C}^{k,\lambda}([0,1])}\geq \frac{\ceil{\beta_1-1}!}{2^{\ceil{\beta_1}}}\cdot (1-2\delta_n)^{\beta_1}\cdot\frac{1}{\delta_n^{\beta_1}}\geq \frac{\ceil{\beta_1-1}!}{4^{\ceil{\beta_1}}}\cdot \frac{1}{\delta_n^{\beta_1}}\geq\frac{3}{128}\cdot\frac{1}{\delta_n^{\beta_1}}. 
\ea\] 
By Theorem \ref{thmA3}, for any positive integers $m$ and $M'$ with
\beq\label{bd3.2}
M'\geq \max \hkh{(\beta_1+1), \ykh{\norm{l_{\delta_n}}_{\mc{C}^{k,\lambda}([0,1])}\ceil{\beta_1}+1}\cdot\mr{e}}\geq \norm{l_{\delta_n}}_{\mc{C}^{k,\lambda}([0,1])} \geq \frac{3}{128}\cdot\frac{1}{\delta_n^{\beta_1}}, 
\eeq there exists a neural network
\beq\label{23030901}
\tilde{f}\in\fdnn_1\ykh{14m(2+{\log_2\ykh{1\qd\beta_1}}),6\ykh{1+\ceil{\beta_1}}M',987(2+\beta_1)^{4}M'm,1,\infty}
\eeq
such that
\[\ba
\sup_{x\in[0,1]}\abs{l_{\delta_n}(x)-\tilde{f}(x)}&\leq \norm{l_{\delta_n}}_{\mc{C}^{k,\lambda}([0,1])}\cdot\ceil{\beta_1}\cdot 3^{\beta_1}M'^{-\beta_1}\\&\;\;\;\;\;\;\;\;+\ykh{1+2\norm{l_{\delta_n}}_{\mc{C}^{k,\lambda}([0,1])}\cdot\ceil{\beta_1}}\cdot 6\cdot(2+\beta_1^2)\cdot M'\cdot 2^{-m}.
\ea\]
To make this error less than or equal to a given error threshold $\e_n$ (depending on $n$), there must hold
\[
\e_n\geq \norm{l_{\delta_n}}_{\mc{C}^{k,\lambda}([0,1])}\cdot\ceil{\beta_1}\cdot 3^{\beta_1}M'^{-\beta_1}\geq \norm{l_{\delta_n}}_{\mc{C}^{k,\lambda}([0,1])}\cdot M'^{-\beta_1} \geq M'^{-\beta_1}\cdot\frac{3}{128}\cdot\frac{1}{\delta_n^{\beta_1}}.
\]
This together with (\ref{bd3.2}) gives
\beq\label{bd3.3}
M'\geq\max\hkh{\frac{3}{128}\cdot\frac{1}{\delta_n^{\beta_1}}, \e_n^{-1/\beta_1}\cdot \abs{\frac{3}{128}}^{1/\beta_1}\cdot \frac{1}{\delta_n}}. 
\eeq
Consequently, the width and the number of nonzero parameters of $\tilde{f}$ are greater than or equal to the right hand side of (\ref{bd3.3}), which may be too large when $\delta_n$ is small (recall that $\delta_n\to 0$ as $n\to\infty$). In this paper, we establish a new sharp error bound for approximating the  natural logarithm function $\log(\cdot)$ on $[\delta_n,1-\delta_n]$, which indicates that one can achieve the same approximation error by using a much smaller network. This refined error bound is given in Theorem \ref{thm2.3} which is critical in our proof of Theorem \ref{thm2.2} and also deserves special attention in its own right.
\begin{thm}\label{thm2.3} Given $a\in(0,1/2]$, $b\in(a,1]$, $\alpha\in(0,\infty)$ and $\e\in(0,1/2]$, there exists 
	\[\ba
	\tilde{f}\in\fdnn_1&\left(A_1\log\frac{1}{\e}+139\log\frac{1}{a}\,,\; A_2\abs{\frac{1}{\e}}^{\frac{1}{\alpha}}\cdot{\log\frac{1}{a}}\,,\right.\\&\;\;\;\;\;\;\;\;\;\;\;\;\;\;\;\left. A_3\abs{\frac{1}{\e}}^{{\frac{1}{\alpha}}}\cdot\abs{\log\frac{1}{\e}}\cdot{\log\frac{1}{a}}+65440\abs{\log\frac{1}{a}}^2,1, \infty\right)\\
	\ea\]
	such that
	\[
	\sup_{z\in[a,b]}\abs{\log z-\tilde{f}(z)}\leq \e \textrm{ and }\log a\leq \tilde{f}(t)\leq \log b,\;\forall\;t\in \mbR,
	\]
	where $(A_1,A_2,A_3)\in (0,\infty)^3$ are constants depending only on $\alpha$.
\end{thm} 

In Theorem \ref{thm2.3}, we show that for each fixed $\alpha\in(0,\infty)$ one can construct a neural network to approximate the natural logarithm function $\log(\cdot)$ on $[a,b]$ with error $\e$, where the depth, width and number of nonzero parameters of this neural network are in the same order of magnitude as $\log\frac{1}{\e}+\log\frac{1}{a}$, $\ykh{\frac{1}{\e}}^{\frac{1}{\alpha}}\ykh{\log\frac{1}{a}}$ and $\ykh{\frac{1}{\e}}^{{\frac{1}{\alpha}}}{\ykh{\log\frac{1}{\e}}}\ykh{\log\frac{1}{a}}+\ykh{\log\frac{1}{a}}^2$ respectively. Recall that in our generalization analysis we need to approximate $\log$ on $[\delta_n,1-\delta_n]$, which is equivalent to  approximating $l_{\delta_n}(t)=\log((1-2\delta_n)t+\delta_n)$ on $[0,1]$. Let $\e_n\in(0,1/2]$ denote the desired  accuracy of the approximation of $l_{\delta_n}$ on $[0,1]$, which depends on the sample size $n$ and converges to zero as $n\to\infty$. Using  Theorem \ref{thm2.3} with $\alpha=2\beta_1$, we deduce that   for any $\beta_1>0$ one can approximate $l_{\delta_n}$ on $[0,1]$ with error $\e_n$ by a network of which the width and the number of nonzero parameters are less than $C_{\beta_1}\e_n^{-\frac{1}{2\beta_1}}\abs{\log{\e_n}}\cdot\abs{\log \delta_n}^2$ with some constant $C_{\beta_1}>0$ (depending only on $\beta_1$). The complexity of this neural network is much smaller than that of $\tilde{f}$ defined in \eqref{23030901} with  (\ref{bd3.3}) as $n\to\infty$ since $\abs{\log \delta_n}^2=\mr{o}\ykh{1/\delta_n}$ and $\e_n^{-\frac{1}{2\beta_1}}\abs{\log{\e_n}}=\mr{o}\ykh{\e_n^{-1/\beta_1}}$ as $n\to \infty$. In particular,  when \beq\label{23031301}
\text{$\frac{1}{n^{\theta_2}}\lesssim\e_n\qx\delta_n\leq \e_n\qd\delta_n\lesssim  \frac{1}{n^{\theta_1}}$ for some  $\theta_2\geq\theta_1>0$ independent of $n$ or $\beta_1$,}\eeq which occurs in our generalization analysis (e.g., in our proof of Theorem \ref{thm2.2p}, we essentially take $\e_n=\delta_n\asymp \ykh{\frac{(\log n)^5}{n}}^{\frac{\beta\cdot(1\qx\beta)^q}{d_*+\beta\cdot(1\qx\beta)^q}}$, meaning that $n^{\frac{-\beta\cdot(1\qx\beta)^q}{d_*+\beta\cdot(1\qx\beta)^q}}\lesssim \e_n=\delta_n\lesssim n^{\frac{-\beta\cdot(1\qx\beta)^q}{2d_*+\beta\cdot(1\qx\beta)^q}}$ (cf. \eqref{approximationerror1}, \eqref{ineq 5.63}, \eqref{nnn69} and \eqref{n74}), we will have that the right hand side of (\ref{bd3.3}) grows no slower than $n^{\theta_1+\theta_1/\beta_1}$. Hence, in this case, no matter what $\beta_1$ is,  the width and the number of nonzero parameters of the network $\tilde{f}$, which  approximates $l_{\delta_n}$ on $[0,1]$  with error $\e_n$ and is obtained by using Theorem \ref{thmA3} directly (cf. \eqref{23030901}),  will grow faster than $n^{\theta_1}$  as $n\to \infty$. However, it follows from Theorem \ref{thm2.3} that there exists a network $\overline{f}$ of which the width and the number of nonzero parameters are less than $C_{\beta_1}\e_n^{-\frac{1}{2\beta_1}}\abs{\log{\e_n}}\cdot\abs{\log \delta_n}^2\lesssim n^{\frac{\theta_2}{2\beta_1}}\abs{\log n}^3$ such that it achieves the same approximation error as that of $\tilde{f}$. By taking $\beta_1$ large enough we can make the growth (as $n\to\infty$) of the width and the number of nonzero parameters of $\overline{f}$ slower than ${n^\theta}$ for arbitrary ${\theta\in(0,\theta_1]}$. Therefore, in the usual case when the complexity of $\tilde\eta$ is not too small in the sense that the width and the number of nonzero parameters of $\tilde\eta$ grow faster than $n^{\theta_3}$ as $n\to\infty$ for some $\theta_3\in(0,\infty)$ independent of $n$ or $\beta_1$, we can use Theorem \ref{thm2.3} with a large enough $\alpha=2\beta_1$ to construct the desired network $\tilde l$ of which the complexity is insignificant in comparison to that of $\tilde L\circ\tilde\eta$. In other words, the neural network approximation of logarithmic function based on Theorem \ref{thm2.3} brings little complexity in approximating the target function $f^*_{\phi,P}$. The above discussion demonstrates the  tightness of the inequality in Theorem \ref{thm2.3} and the advantage of Theorem \ref{thm2.3} over those general results on approximation theory of neural networks such as Theorem \ref{thmA3}.

It{\label{23072201}} is worth mentioning that an alternative way to approximate the function $\overline L_{\delta_n}$ defined in \eqref{23031201} is by simply using its piecewise linear interpolation. For example,  in \cite{kohler2020statistical}, the authors express the piecewise linear interpolation of $\overline L_{\delta_n}$ at equidistant points by a neural network $\tilde L$, and construct a CNN $\tilde\eta$ to approximate $\eta$, leading to an approximation of the truncated target function of the logistic risk $\tilde L\circ\tilde\eta$.  It follows from Proposition 3.2.4 of \cite{atkinson2009theoretical} that
\beq
h_n^2\lesssim\norm{\tilde L-\overline L_{\delta_n}}_{[\delta_n,1-\delta_n]}\lesssim \frac{h_n^2}{\delta_n^2}, 
\eeq where $h_n$ denotes the step size of the interpolation. Therefore, to ensure the error bound $\e_n$ for the approximation of $\overline L_{\delta_n}$ by $\tilde L$,  we must have $h_n\lesssim \sqrt{\e_n}$, implying that the number of nonzero parameters of $\tilde L$ will grow no slower than $\frac{1}{h_n}\gtrsim\frac{1}{\sqrt{\e_n}}$ as $n\to\infty$. Consequently,  in the case \eqref{23031301}, we will have that the number of nonzero parameters of $\tilde L$ will grow no slower than $n^{\theta_1/2}$. Therefore, in contrast to using Theorem \ref{thm2.3}, we cannot make the  number of nonzero parameters of the network $\tilde L$ obtained from piecewise linear interpolation grow slower than $n^\theta$ for arbitrarily small $\theta>0$. As a result, using piecewise linear interpolation to approximate $\overline L_{\delta_n}$ may bring extra complexity in establishing the approximation of the target function. However, the advantage of using piecewise linear interpolation is that one can make the depth or width of the network $\tilde L$ which expresses the desired interpolation  bounded as $n\to\infty$ (cf. Lemma 7 in \cite{kohler2020statistical} and its proof therein).   

The proof of Theorem \ref{thm2.3} is in Appendix \ref{section: proof of thm2.3}. The key observation in our proof is the fact that for all $k\in \mb{N}$, the following holds true: 
\beq\label{230401}
\log x=\log (2^k\cdot x)-k\log 2,\;\quad \forall \; x\in(0,\infty).
\eeq Then we can use the values of $\log(\cdot)$ which are taken far  away from zero (i.e., $\log (2^k\cdot x)$ in the right hand side of \eqref{230401}) to determine its values taken near zero, while approximating the former is more efficient as the H\"{o}lder norm of the natural logarithm function on domains far away from zero can be well controlled.

In the next theorem, we show that if the data distribution has a piecewise smooth decision  boundary, then DNN classifiers trained by empirical logistic risk minimization  can also achieve dimension-free rates of  convergence  under the noise condition \eqref{Tsybakovnoisecondition} and a margin condition (see \eqref{margincondition} below).  Before stating this result, we need to introduce this margin  condition and relevant concepts.

We first define the set of (binary) classifiers which have a piecewise H\"older smooth decision boundary. We will adopt similar notations from \cite{kim2021fast} to describe this set. Specifically, 
let $\beta,r\in (0,\infty)$ and $I,\Theta \in \mb{N}$. For $g\in \mc{B}^{\beta}_r\ykh{[0,1]^{d-1}}$ and $j=1,2,\cdots,d$, we define horizon function $\Psi_{g,j}:[0,1]^d \to \{0,1\}$ as $\Psi_{g,j}(x):=\idf_{\hkh{(x)_j\geq g(x_{-j})}}$, where  $x_{-j}:=((x)_1,\cdots,(x)_{j-1},(x)_{j+1},\cdots,(x)_d)\in[0,1]^{d-1}$. For each horizon function, the corresponding basis piece $\Lambda_{g,j}$ is defined as $\Lambda_{g,j}:=\setr{x\in[0,1]^d}{\ba \Psi_{g,j}(x)=1 \ea}$. Note that $\Lambda_{g,j}=\setl{x\in[0,1]^d}{(x)_j\geq \max\hkh{0,g(x_{-j})}}$. Thus  $\Lambda_{g,j}$ is  enclosed by the hypersurface $\mc{S}_{g,j}:=\setl{x\in[0,1]^d}{(x)_j= \max\hkh{0,g(x_{-j})}}$ and (part of) the boundary of $[0,1]^d$.   We then define the set of pieces which are the intersection of $I$ basis pieces as
\[	\mc{A}^{d,\beta,r,I}:=\hkh{{A} \left| {A}=\bigcap_{k=1}^I \Lambda_{g_k, j_k} \textrm{ for some }j_k\in \left\{1,2,\cdots,d\right\} \textrm{ and } g_k \in \mc{B}^{\beta}_r\left([0,1]^{d-1}\right)\textrm{}\right.},
\] and define $\mc{C}^{d,\beta,r,I,\Theta}$ to be a set of binary classifiers as
\beq\label{2301170025}
	&\mc{C}^{d,\beta,r,I,\Theta}\\
	&:=\hkh{\left. \texttt C(x)=2\sum_{i=1}^{\Theta} \idf_{{A}_i}(x)-1:[0,1]^d\to \{-1,1\} \right|\begin{minipage}{0.25\textwidth}
			\textrm{${A}_1, A_2, A_3, \cdots, {A}_\Theta$ are disjoint sets in $\mc{A}^{d,\beta,r,I}$}
		\end{minipage} }.
\eeq Thus $\mc{C}^{d,\beta,r,I,\Theta}$ consists of all binary classifiers which are equal to $+1$ on some disjoint sets $A_1,\ldots,A_\Theta$ in $\mc{A}^{d,\beta,r,I}$ and $-1$ otherwise. Let  $A_t=\cap_{k=1}^I\Lambda_{g_{t,k},j_{t,k}}$ ($t=1,2,\ldots,\Theta$) be arbitrary disjoint sets  in $\mc{A}^{d,\beta,r,I}$, where $j_{t,k}\in\hkh{1,2,\ldots,d}$ and $g_{t,k}\in \mc{B}^{\beta}_r\ykh{[0,1]^{d-1}}$. Then $\texttt C:[0,1]^d\to\hkh{-1,1},x\mapsto2\sum_{i=1}^{\Theta} \idf_{{A}_i}(x)-1$ is a classifier in $\mc{C}^{d,\beta,r,I,\Theta}$. Recall that $\Lambda_{g_{t,k},j_{t,k}}$ is enclosed by $\mc S_{g_{t,k},j_{t,k}}$ and  (part of) the boundary of $[0,1]^d$ for each $t,k$. Hence for each $t$, the region $A_t$ is enclosed by hypersurfaces $\mc S_{g_{t,k},j_{t,k}}$ ($k=1,\ldots,I$) and (part of) the boundary of $[0,1]^d$. We say the piecewise H\"older smooth hypersurface 
\beq\label{23070701}
D^*_{\texttt C}:=\bigcup_{t=1}^{\Theta}\bigcup_{k=1}^I\ykh{\mc S_{g_{t,k},j_{t,k}}\cap A_t}
\eeq is the \emph{decision boundary} of the classifier $\texttt C$ because intuitively, points on different sides of $D^*_{\texttt C}$ are classified into different categories (i.e. $+1$ and $-1$) by $\texttt C$ (cf. {Figure} \ref{fig10}). Denote by $\Delta_{\texttt C}(x)$ the distance from $x\in[0,1]^d$ to the decision boundary $D^*_{\texttt C}$, i.e., \beq\label{23022216}
\Delta_{\texttt C}(x):=\inf\hkh{\norm{x-x'}_2\big|x'\in D^*_{\texttt C}}.
\eeq

	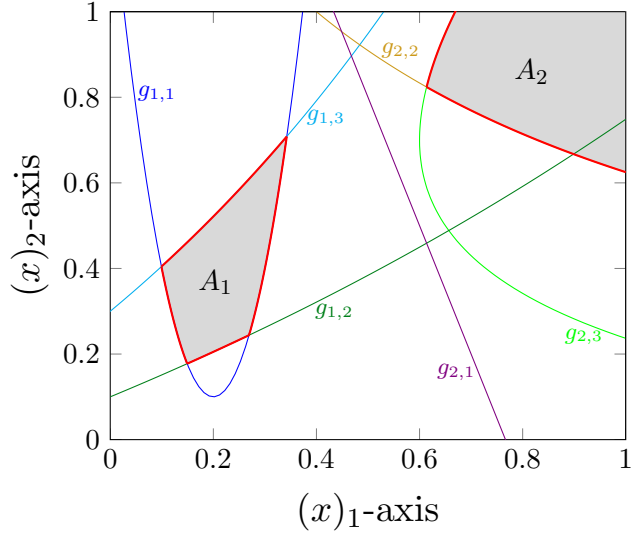
\begin{figure}[htbp]	
	\centering
	\betikz
	\tikzset{ %
		mypt/.style ={
			circle, %
			minimum width =0pt, %
			minimum height =0pt, %
			inner sep=0pt, %
			draw=none, %
		}
	}
	\tikzset{
		mybox/.style ={
			rectangle, %
			rounded corners =5pt, %
			minimum width =30pt, %
			minimum height =30pt, %
			inner sep=5pt, %
			draw=blue, %
			fill=cyan
	}}
	\tikzset{
		bigbox/.style ={
			rectangle, %
			rounded corners =5pt, %
			minimum width =347pt, %
			minimum height =309pt, %
			inner sep=0pt, %
			draw=black, %
			fill=none
	}}
	\tikzset{
		tinycircle/.style ={
			circle, %
			minimum width =6pt, %
			minimum height =6pt, %
			inner sep=0pt, %
			draw=red, %
			fill=red %
		}
	}
 \begin{axis}[xmin=0, xmax=1, ymin=0, ymax=1,
 	x label style={at={(axis description cs:0.5,-0.1)},anchor=north},
 	y label style={at={(axis description cs:-0.1,0.5)},rotate=0,anchor=south},
 	xlabel={\scalebox{1.3}{$(x)_1$-axis}},
 	ylabel={\scalebox{1.3}{$(x)_2$-axis}}]
	\addplot[name path=Atrue, blue, domain=0:1, samples=100] ({x}, {0.1+30*(x-0.2)^2});
	\addplot[name path=A, blue, domain=0.0992744:0.342399, samples=100] ({x}, {max(0.1+30*(x-0.2)^2,0.1+e^(0.5*x)-1)});
	\addplot[name path=Btrue, cyan, domain=0:1, samples=100] ({ln(x+0.7)}, {x});
	\addplot[name path=B, cyan, domain=0.404369:0.7083222, samples=100] ({ln(x+0.7)}, {x});

	\addplot[name path=Ctrue, lightorange, domain=0:1, samples=100] ({x}, {1/(x+0.6)});
		\addplot[name path=C, yellow, domain=0.613669:1, samples=100] ({x}, {1/(x+0.6)});
	\addplot[name path=D, pink, domain=0.823948:2, samples=100] ({x+1/(x+0.3)-1.1}, {x});
	\addplot[name path=Dtrue, green, domain=0:1, samples=100] ({x+1/(x+0.3)-1.1}, {x});
	
	\addplot[black, domain=0:1, samples=100] ({x}, {1});
		\addplot[black, domain=0:1, samples=100] ({1}, {x});
	\addplot[name path=E, darkgreen, domain=0:1, samples=100] ({x}, {0.1+e^(0.5*x)-1});
		\addplot[ violet, domain=0:1, samples=100] ({x}, {-3*x+2.3});
	\addplot[gray!30] fill between[of=A and B];
	\addplot[gray!30] fill between[of=C and D];
	\addplot[thick, red, domain=0.404369:0.7083222, samples=100] ({ln(x+0.7)}, {x});
	\addplot[ thick, red, domain=0.823948:2, samples=100] ({x+1/(x+0.3)-1.1}, {x});
	\addplot[thick, red, domain=0.613669:1, samples=100] ({x}, {1/(x+0.6)});
		\addplot[thick,red, domain=0.0992744:0.342399, samples=100] ({x}, {max(0.1+30*(x-0.2)^2,0.1+e^(0.5*x)-1)});		
\end{axis}
\node[mypt] (xx) at (1.4,2.1) {$A_1$};
\node[mypt] (xx1) at (5.6,4.9) {$A_2$};
\node[mypt] (xx1) at (4.6,0.9) {{\color{violet}%
{\fontsize{9}{11}\selectfont
$g_{2,1}$}}};
\node[mypt] (xx2) at (6.3,1.37) {{\color{green}%
{\fontsize{9}{11}\selectfont
$g_{2,3}$}}};
\node[mypt] (xx9) at (2.98,1.7) {{\color{darkgreen}%
{\fontsize{9}{11}\selectfont
$g_{1,2}$}}};
\node[mypt] (xx7) at (0.62,4.6) {{\color{blue}%
{\fontsize{9}{11}\selectfont
$g_{1,1}$}}};
\node[mypt] (xx3) at (3.86,5.17) {{\color{lightorange}%
{\fontsize{9}{11}\selectfont
$g_{2,2}$}}};
\node[mypt] (xx4) at (2.87,4.28) {{\color{cyan}%
{\fontsize{9}{11}\selectfont
$g_{1,3}$}}};
	\eetikz	
	 \captionsetup{justification=centering}
	\caption{Illustration of the sets $A_1,\ldots A_\Theta$ when $d=2$, $\Theta=2$, $I=3$, $j_{2,1}=j_{2,2}=j_{1,1}=j_{1,2}=2$ and  $j_{1,3}=j_{2,3}=1$. The classifier $\texttt C(x)=2\sum_{t=1}^\Theta\idf_{A_t}(x)-1$ is equal to $+1$ on $A_1\cup A_2$ and $-1$ otherwise.  The decision boundary $D^*_{\texttt C}$ of $\texttt C$ is marked red.}
	\label{fig10}
\end{figure}

We then describe the margin condition mentioned above. Let $P$ be a  probability measure on $[0,1]^d\times\hkh{-1,1}$, which we regard as the joint distribution of the input and output data,  and $\eta(\cdot)=P(\hkh{1}|\cdot)$ is the conditional probability function of $P$. The corresponding  Bayes classifier is the sign of $2\eta-1$  which minimizes the misclassification error over all measurable functions, i.e., 
\beq\label{23072501}
\mc{R}_P(\mr{sgn}(2\eta-1))=\mc{R}_P(2\eta-1)=\inf \setr{\mc{R}_P(f) 
}{\textrm{$f:[0,1]^d\to\mbR$ is measurable}}.\eeq We say the distribution $P$ has a piecewise smooth decision boundary if
\[\exists \;\texttt C\in \mc{C}^{d,\beta,r,I,\Theta} \;\;\mr{s.t.}\;\; \mr{sgn}(2\eta-1)\xlongequal{P_X\text{-a.s.}}\texttt C,  \] that is,  \beq\label{230225012}
P_X \ykh{\setl{x\in[0,1]^d}{\mr{sgn}(2\cdot P(\hkh{1}|x)-1)=\texttt C(x)}}=1  \eeq for some $\texttt C\in \mc{C}^{d,\beta,r,I,\Theta}$. Suppose $\texttt C\in \mc{C}^{d,\beta,r,I,\Theta}$ and \eqref{230225012} holds. We call  $D^*_{\texttt C}$ the \emph{decision boundary} of $P$, and for $c_2\in(0,\infty)$, $t_2\in(0,\infty)$, $s_2\in[0,\infty]$, we use the following condition
\begin{equation}\label{margincondition}
P_X\ykh{\setl{x\in[0,1]^d }{\Delta_{\texttt C}(x)\leq t}} \leq c_2 t^{s_2}, \quad \forall \; 0<t \leq t_2,
\end{equation} which we call the \emph{margin condition}, to measure the concentration of the input distribution $P_X$ near the decision boundary $D^*_{\texttt C}$ of $P$. In particular, when the input data are bounded away from the decision boundary $D^*_{\texttt C}$ of $P$ ($P_X$-a.s.), \eqref{margincondition} will hold for  $s_2=\infty$. 

 Now we are ready to give our next main theorem.

\begin{thm}\label{thm2.4}Let $d\in\mb N\cap[2,\infty)$,  $(n,I,\Theta)\in\mb N^3$,  $(\beta,r,t_1,t_2,c_1,c_2)\in(0,\infty)^6$, $(s_1,s_2)\in [0,\infty]^2$,  $\{(X_i,Y_i)\}_{i=1}^n$ be a sample in  $[0,1]^d\times\{-1,1\}$ and $\hat{f}^{\FNN}_n$ be an ERM with respect to the  logistic loss $\phi(t)=\log\ykh{1+\me^{-t}}$ over  $\fdnn_d(G,N,S,B,F)$ which is given by \eqref{FCNNestimator}. Define 
\beq\label{2301151903}
\mc H^{d,\beta,r,I,\Theta,s_1,s_2}_{6,t_1,c_1,t_2,c_2}:=\setr{P\in\mc H_0^d}{\begin{minipage}{0.3\textwidth}   \eqref{Tsybakovnoisecondition}, \eqref{230225012} and \eqref{margincondition} hold for some $\texttt C\in \mc{C}^{d,\beta,r,I,\Theta}$\end{minipage}}. 
\eeq Then the following statements hold true:
	\begin{itemize}
		\item[{(1)}] For $s_1\in[0,\infty]$ and $s_2=\infty$, the $\phi$-ERM  $\hat{f}^{\FNN}_n$ with 
		\[\ba
		&G = G_0 \log \frac{1}{t_2\qx\frac{1}{2}}, \ N = N_0 \left(\frac{1}{t_2\qx\frac{1}{2}}\right)^{\frac{d-1}{\beta}},
		\ S = S_0\left(\frac{1}{t_2\qx\frac{1}{2}}\right)^{\frac{d-1}{\beta}}\log\left(\frac{1}{t_2\qx\frac{1}{2}}\right),
		\\&B = B_0\left(\frac{1}{t_2\qx\frac{1}{2}}\right), \textrm{ and } \ F \asymp \ykh{\frac{\log n}{n}}^{\frac{1}{s_1+2}}
		\ea\] satisfies 
		\beq\ba\label{bound 2.41}
\sup_{P\in \mc H^{d,\beta,r,I,\Theta,s_1,s_2}_{6,t_1,c_1,t_2,c_2}}\bm{E}_{P^{\otimes n}}\zkh{\mc{E}_P\ykh{\hat{f}^{\FNN}_n}}\lesssim \ykh{\frac{\log n}{n}}^{\frac{s_1}{s_1+2}},
		\ea\eeq where $G_0, N_0,S_0,B_0$ are positive constants only depending on $d,\beta, r, I, \Theta$;
		\item[(2)] For $s_1 =\infty$ and $s_2 \in [0,\infty)$, the $\phi$-ERM  $\hat{f}^{\FNN}_n$ with 
		\[\ba
		& G \asymp  \log {n}, \ N \asymp \left(\frac{n}{(\log n)^3}\right)^{\frac{d-1}{s_2\beta+d-1}},\ S \asymp \left(\frac{n}{(\log n)^3}\right)^{\frac{d-1}{s_2\beta+d-1}}\log {n},\\
		& B \asymp \left(\frac{n}{(\log n)^3}\right)^{\frac{1}{s_2+\frac{d-1}{\beta}}}, \textrm{ and } \ F={t_1\qx\frac{1}{2}}
		\ea\] satisfies 
		\beq\ba\label{bound 2.42}
		\sup_{P\in \mc H^{d,\beta,r,I,\Theta,s_1,s_2}_{6,t_1,c_1,t_2,c_2}}\bm{E}_{P^{\otimes n}}\zkh{\mc{E}_P\ykh{\hat{f}^{\FNN}_n}}\lesssim \ykh{\frac{\left(\log n\right)^3}{n}}^{\frac{1}{1+\frac{d-1}{\beta s_2}}};
		\ea\eeq 
		\item[(3)] For $s_1 \in [0,\infty)$ and $s_2 \in [0,\infty)$, the $\phi$-ERM $\hat{f}^{\FNN}_n$ with 
		\[\ba
		& G \asymp  \log {n}, \scalebox{0.5}{\,} N \asymp \left(\frac{n}{(\log n)^3}\right)^{\frac{(d-1)(s_1+1)}{s_2\beta+(s_1+1)\left(s_2\beta+d-1\right)}},\scalebox{0.5}{\,} S \asymp \left(\frac{n}{(\log n)^3}\right)^{\frac{(d-1)(s_1+1)}{s_2\beta+(s_1+1)\left(s_2\beta+d-1\right)}}\log {n},\\
		& B \asymp \left(\frac{n}{(\log n)^3}\right)^{\frac{s_1+1}{s_2+(s_1+1)\left(s_2+\frac{d-1}{\beta}\right)}}, \textrm{ and } \ F \asymp \ykh{\frac{(\log n)^3}{n}}^{\frac{s_2}{s_2+(s_1+1)\ykh{s_2+\frac{d-1}{\beta}}}}
		\ea\] satisfies 
		\beq\ba\label{bound 2.43}
		\sup_{P\in \mc H^{d,\beta,r,I,\Theta,s_1,s_2}_{6,t_1,c_1,t_2,c_2}}\bm{E}_{P^{\otimes n}}\zkh{\mc{E}_P\ykh{\hat{f}^{\FNN}_n}}\lesssim \ykh{\frac{\left(\log n\right)^3}{n}}^{\frac{s_1}{1+(s_1+1)\left(1+\frac{d-1}{\beta s_2}\right)}}.
		\ea\eeq 
	\end{itemize}
\end{thm} It is worth noting that the rate $\mc O(\ykh{\frac{\log n}{n}}^{\frac{s_1}{s_1+2}})$ established in  \eqref{bound 2.41} does not depend on  the dimension $d$, and dependency of the rates in  \eqref{bound 2.42} and \eqref{bound 2.43} on the dimension $d$ diminishes as $s_2$ increases, which demonstrates that the  condition \eqref{margincondition} with $s_2=\infty$  helps circumvent the curse of dimensionality. In particular, \eqref{bound 2.41} will give a fast dimension-free rate of convergence $\mc O(\frac{\log n}{n})$ if  $s_1=s_2=\infty$.  One may refer to Section \ref{section: related work} for more discussions about the result of Theorem \ref{thm2.4}.  

The proof of Theorem \ref{thm2.4} is in  Appendix \ref{section: proof of thm2.4}. Our proof relies on  Theorem \ref{thm2.1} and the fact that the ReLU networks are good at approximating indicator functions of bounded regions with piecewise smooth boundary \citep{imaizumi2019deep,petersen2018optimal}. Let $P$ be an arbitrary probability in $\mc H^{d,\beta,r,I,\Theta,s_1,s_2}_{6,t_1,c_1,t_2,c_2}$ and denote by $\eta$ the condition probability function $P(\hkh{1}|\cdot)$ of $P$.  To apply Theorem \ref{thm2.1} and make good use of the noise condition \eqref{Tsybakovnoisecondition} and the margin condition \eqref{margincondition}, we define another $\psi$ (which is different from that in  \eqref{2212280209}) as
\[\ba 
\psi:[0,1]^d\times\{-1,1\}\to\mbR,\quad (x,y)\mapsto \left\{
\ba
&\phi\ykh{yF_0\mr{sgn}(2\eta(x)-1)},&&\textrm{ if } \abs{2\eta(x)-1}>\eta_0,\\
&\phi\ykh{y\log \frac{\eta(x)}{1-\eta(x)}},&&\textrm{ if } \abs{2\eta(x)-1}\leq \eta_0
\ea
\right.
\ea\] for some suitable $\eta_0\in (0,1)$ and $F_0 \in \ykh{0, \log \frac{1+\eta_0}{1-\eta_0}}$.  For such $\psi$, Lemma \ref{lemma5.8} guarantees that inequality \eqref{ineq 2.10} holds as
\[
\int_{[0,1]^d\times\{-1,1\}}{\psi\ykh{x,y}}\mr{d}P(x,y)\leq \inf \setl{\mc{R}_P^\phi(f)}{ \textrm{$f:[0,1]^d\to\mbR$ is measurable}},\] and \eqref{ineq 2.11}, \eqref{ineq 2.12} of Theorem \ref{thm2.1} are satisfied with $M=\frac{2}{1-\eta_0}$ and $\Gamma=\frac{8}{1-\eta^2_0}$. Moreover, we use the noise condition \eqref{Tsybakovnoisecondition} and the margin condition \eqref{margincondition} to bound the approximation error \beq\label{23063001}
\inf_{f\in \fdnn_d(G,N,S,B,F)}\ykh{\mc{R}_P^\phi(f)-\int_{[0,1]^d\times\{-1,1\}}{\psi(x,y)}\mr{d}P(x,y)}\eeq (see \eqref{23063002}, \eqref{23063003}, \eqref{23063004}). Then, as in the proof of Theorem \ref{thm2.2}, we combine Theorem \ref{thm2.1} with estimates for the covering number of $\fdnn_d(G,N,S,B,F)$ and the approximation error \eqref{23063001} to obtain an upper bound for $\bm E_{P^{\otimes n}}\zkh{\mc{R}_P^{\phi}\ykh{\hat{f}_n^{\FNN}}-\int{\psi}\mr{d}P}$, which, together with the noise condition \eqref{Tsybakovnoisecondition},  yields an upper bound for  $\bm{E}_{P^{\otimes n}}\zkh{\mc{E}_P(\hat{f}^{\FNN}_n)}$ (see \eqref{ineq 2.41}).  Finally taking the  supremum over all $P\in \mc H^{d,\beta,r,I,\Theta,s_1,s_2}_{6,t_1,c_1,t_2,c_2}$ gives the desired result. The proof of Theorem \ref{thm2.4} along with that of Theorem \ref{thm2.2} and Theorem \ref{thm2.2p} indicates that Theorem \ref{thm2.1} is very flexible in the sense that it can be used in various settings with different choices of $\psi$.

\subsection{Main Lower  Bounds}\label{section: main lower}

In this subsection, we will give our main results on lower bounds for convergence rates of the logistic risk, which will justify the optimality of  our upper bounds established in the  last subsection.  To state these results,  we need some notations. 

Recall that for any $a\in[0,1]$,  $\mathscr{M}_{a}$ denotes  the probability measure on $\hkh{-1,1}$ with $\mathscr{M}_{a}(\hkh{1})=a$ and $\mathscr{M}_{a}(\hkh{-1})=1-a$. For any measurable $\eta:[0,1]^d\to[0,1]$ and any Borel probability measure $\mathscr Q$ on $[0,1]^d$, we denote
\beq\label{221031221401}
P_{\eta,\mathscr Q}:&\hkh{\text{Borel subsets of $[0,1]^d\times\hkh{-1,1}$}}\to [0,1],\\
&{\;\;\;\;\;\;}S\mapsto\int_{[0,1]^d}\int_{\hkh{-1,1}}\idf_{S}(x,y)\mr{d}\mathscr M_{\eta(x)}(y)\mr{d}\mathscr Q(x). 
\eeq Therefore, $P_{\eta,\mathscr Q}$ is the (unique) probability measure on $[0,1]^d\times\hkh{-1,1}$ of which the marginal distribution on $[0,1]^d$ is $\mathscr Q$ and the conditional probability function is ${\eta}$. If $\mathscr Q$ is the Lebesgue measure on $[0,1]^d$, we will write $P_{\eta}$ for $P_{\eta,\mathscr Q}$.

For any  $\beta\in(0,\infty)$, $r\in(0,\infty)$, $A\in[0,1)$, $q\in\mb N\cup\hkh{0}$,   and  $(d, d_*, K)\in\mb N^3$ with  $d_*\leq \min\hkh{d,K+\idf_{\hkh{0}}(q)\cdot(d-K)}$,  define 
\beq\label{2302132224}
&\mc{H}^{d,\beta,r}_{3,A}:=\setr{P_{\eta}}{\begin{minipage}{0.4\textwidth}$\eta\in\mc{B}^{\beta}_r([0,1]^d)$, $\mathbf{ran}(\eta)\subset [0,1]$, and  $\int_{[0,1]^d}\idf_{[0,A]}(\abs{2\eta(x)-1})\mr{d}x=0$\end{minipage}},\\
&\mc{H}^{d,\beta,r}_{5,A,q,K,d_*}:=\setr{P_\eta}{\begin{minipage}{0.43\textwidth}  $\eta\in\mc G_d^{\mathbf{CH}}(q, K, d_*, \beta,r)$, $\mathbf{ran}(\eta)\subset[0,1]$, and $\int_{[0,1]^d}\idf_{[0,A]}(\abs{2\eta(x)-1})\mr{d}x=0$ \end{minipage}}.
\eeq

Now we can state our  Theorem \ref{thm2.6p}.  Recall that $\mc F_d$ is the set of all measurable real-valued functions defined on $[0,1]^d$. 

\begin{thm}\label{thm2.6p}
Let $\phi$ be the logistic loss,   $n\in\mb N$, $\beta\in(0,\infty)$, $r\in(0,\infty)$, $A\in[0,1)$, $q\in\mb N\cup\hkh{0}$,   and  $(d, d_*, K)\in\mb N^3$ with  $d_*\leq \min\hkh{d,K+\idf_{\hkh{0}}(q)\cdot(d-K)}$.  Suppose $\hkh{(X_i,Y_i)}_{i=1}^n$ is a sample in $[0,1]^d\times\hkh{-1,1}$ of size $n$.  Then there exists a constant $\mr c_0\in(0,\infty)$ only depending on $(d_*,\beta,r,q)$, such that
\[
\inf_{\hat{f}_n}\sup_{P\in\mc{H}^{d,\beta,r}_{5,A,q,K,d_*}}{\bm E}_{P^{\otimes n}}\zkh{\mc{E}_P^\phi(\hat{f}_n)}\geq \mr c_0 n^{-\frac{\beta\cdot(1\qx\beta)^q}{d_*+\beta\cdot(1\qx\beta)^q}}\text{ provided that }n>\abs{\frac{7}{1-A}}^{\frac{d_*+\beta\cdot(1\qx\beta)^q}{\beta\cdot(1\qx\beta)^q}},
\] where the infimum is taken over all $\mc F_d$-valued statistics on $([0,1]^d\times\hkh{-1,1})^n$ from the sample $\hkh{(X_i,Y_i)}_{i=1}^n$. 
\end{thm}

Taking $q=0$, $K=1$, and $d_*=d$  in Theorem \ref{thm2.6p}, we immediately obtain the following corollary: 

\begin{cor}\label{thm2.6}Let $\phi$ be the logistic loss,  $d\in\mb N$, $\beta\in(0,\infty)$, $r\in(0,\infty)$,  $A\in[0,1)$, and $n\in\mb N$.   Suppose $\hkh{(X_i,Y_i)}_{i=1}^n$ is a sample in $[0,1]^d\times\hkh{-1,1}$ of size $n$.  Then there exists a constant $\mr c_0\in(0,\infty)$ only depending on $(d,\beta,r)$, such that
	\[
	\inf_{\hat{f}_n}\sup_{P\in \mc H^{d,\beta,r}_{3,A}}{\bm E}_{P^{\otimes n}}\zkh{\mc{E}_P^\phi(\hat{f}_n)}\geq \mr c_0 n^{-\frac{\beta}{d+\beta}}\text{ provided that }n>\abs{\frac{7}{1-A}}^{\frac{d+\beta}{\beta}},
	\] where the infimum is taken over all $\mc F_d$-valued statistics on $([0,1]^d\times\hkh{-1,1})^n$ from the sample $\hkh{(X_i,Y_i)}_{i=1}^n$. 
\end{cor}

Theorem \ref{thm2.6p}, together with Corollary \ref{thm2.6}, is proved  in Appendix \ref{section: proof of thm2.6}. 

Obviously, $\mc{H}^{d,\beta,r}_{5,A,q,K,d_*}\subset \mc{H}^{d,\beta,r}_{4,q,K,d_\star,d_*}$. Therefore, it follows from Theorem \ref{thm2.6p} that 
\[
\inf_{\hat{f}_n}\sup_{P\in \mc{H}^{d,\beta,r}_{4,q,K,d_\star,d_*}}{\bm E}_{P^{\otimes n}}\zkh{\mc{E}_P^\phi(\hat{f}_n)}\geq\inf_{\hat{f}_n}\sup_{P\in\mc{H}^{d,\beta,r}_{5,A,q,K,d_*}}{\bm E}_{P^{\otimes n}}\zkh{\mc{E}_P^\phi(\hat{f}_n)}\gtrsim n^{-\frac{\beta\cdot(1\qx\beta)^q}{d_*+\beta\cdot(1\qx\beta)^q}}. 
\]This justifies that the rate $\mc{O}(\ykh{\frac{(\log n)^5}{n}}^{\frac{\beta\cdot(1\qx\beta)^q}{d_*+\beta\cdot(1\qx\beta)^q}})$ in \eqref{bound 2.21xp} is optimal (up to the logarithmic factor $(\log n)^{\frac{5\beta\cdot(1\qx\beta)^q}{d_*+\beta\cdot(1\qx\beta)^q}}$). Similarly, it follows from $\mc H^{d,\beta,r}_{3,A}\subset\mc H_1^{d,\beta,r}$ and Corollary \ref{thm2.6} that 
\[
\inf_{\hat{f}_n}\sup_{P\in \mc H^{d,\beta,r}_{1}}{\bm E}_{P^{\otimes n}}\zkh{\mc{E}_P^\phi(\hat{f}_n)}\geq \inf_{\hat{f}_n}\sup_{P\in \mc H^{d,\beta,r}_{3,A}}{\bm E}_{P^{\otimes n}}\zkh{\mc{E}_P^\phi(\hat{f}_n)}\gtrsim n^{-\frac{\beta}{d+\beta}},
\]which justifies that the rate $\mc O(\ykh{\frac{\ykh{\log n}^{5}}{n}}^{\frac{\beta}{\beta+d}})$ in \eqref{bound 2.21x} is optimal (up to the logarithmic factor $(\log n)^{\frac{5\beta}{\beta+d}}$). Moreover, note that any probability $P$ in $\mc H^{d,\beta,r}_{3,A}$ must satisfy the noise condition \eqref{Tsybakovnoisecondition} provided that $s_1\in[0,\infty]$, $t_1\in(0,A]$, and  $c_1\in(0,\infty)$. In other words, for any $s_1\in[0,\infty]$, $t_1\in(0,A]$, and  $c_1\in(0,\infty)$, there holds  $\mc H^{d,\beta,r}_{3,A}\subset\mc H^{d,\beta,r}_{2,s_1,c_1,t_1}$, meaning that  
\begin{align*}
&n^{-\frac{\beta}{d+\beta}}\lesssim \inf_{\hat{f}_n}\sup_{P\in \mc H^{d,\beta,r}_{3,A}}{\bm E}_{P^{\otimes n}}\zkh{\mc{E}_P^\phi(\hat{f}_n)}\leq \inf_{\hat{f}_n}\sup_{P\in \mc H^{d,\beta,r}_{2,s_1,c_1,t_1}}{\bm E}_{P^{\otimes n}}\zkh{\mc{E}_P^\phi(\hat{f}_n)}\\
&\leq \inf_{\hat{f}_n}\sup_{P\in \mc H^{d,\beta,r}_{1}}{\bm E}_{P^{\otimes n}}\zkh{\mc{E}_P^\phi(\hat{f}_n)}\leq \sup_{P\in\mc H^{d,\beta,r}_1}\bm E_{P^{\otimes n}}\zkh{\mc{E}_P^\phi\ykh{\hat{f}^{\FNN}_n}}\lesssim \ykh{\frac{\ykh{\log n}^{5}}{n}}^{\frac{\beta}{\beta+d}}
,\end{align*}  where $\hat{f}^{\FNN}_n$ is the estimator defined in Theorem \ref{thm2.2}. From above inequalities we see that the noise condition \eqref{Tsybakovnoisecondition} does little to help improve the convergence rate of the  excess $\phi$-risk in  classification.

The proof of Theorem \ref{thm2.6p} and Corollary \ref{thm2.6} is based on a  general scheme for obtaining lower bounds, which is given in Section 2 of \cite{tsybakov2009introduction}.  However, the scheme in \cite{tsybakov2009introduction} is stated for a class of probabilities $\mc H$ that takes the form $\mc H=\hkh{Q_\theta|\theta\in\Theta}$ with $\Theta$ being some pseudometric space.  In our setting, we do not have such pseudometric space. Instead, we introduce another  quantity 
\beq\label{2310120012}
\inf_{f\in\mc F_d}\abs{\mc{E}^\phi_{P}(f)+\mc{E}^\phi_{Q}(f)}
\eeq	to characterize the difference between  any two probability measures $P$ and $Q$ (see \eqref{082603}).  Estimating lower bounds for the quantity defined in \eqref{2310120012}  plays a key role in our proof of Theorem \ref{thm2.6p} and  Corollary  \ref{thm2.6}.

\section{Discussions on Related Work}\label{section: related work}

In this section,  we compare our results with some existing ones in the literature. We first compare Theorem \ref{thm2.2} and Theorem \ref{thm2.4} with related results about binary classification using fully connected DNNs and logistic loss in \cite{kim2021fast} and \cite{farrell2021}  respectively.  Then we  compare our work with \cite{ji2021early}, in which the authors carry out generalization analysis for estimators obtained from gradient descent algorithms. 

Throughout this section, we will use $\phi$ to denote the logistic loss (i.e., $\phi(t)=\log(1+\me^{-t})$) and $\hkh{(X_i,Y_i)}_{i=1}^n$ to denote an i.i.d.   sample in $[0,1]^d\times\hkh{-1,1}$. The symbols $d$, $\beta$, $r$, $I$, $\Theta$, $t_1$, $c_1$, $t_2$, $c_2$ and $c$  will denote arbitrary numbers in $\mb N$, $(0,\infty)$, $(0,\infty)$, $\mb N$, $\mb N$, $(0,\infty)$, $(0,\infty)$, $(0,\infty)$, $(0,\infty)$ and $[0,\infty)$,  respectively.   The symbol $P$ will always  denote some  probability measure on $[0,1]^d\times\hkh{-1,1}$, regarded as the data distribution, and $\eta$ will denote the corresponding conditional probability function $P(\hkh{1}|\cdot)$ of $P$.

Recall that $\mc{C}^{d,\beta,r,I,\Theta}$, defined  in \eqref{2301170025}, is the space  consisting of classifiers which are equal to $+1$ on the union of some disjoint regions with piecewise H\"older smooth boundary and $-1$ otherwise. In Theorem 4.1 of  \cite{kim2021fast}, the authors  conduct generalization analysis when the data distribution $P$  satisfies the  piecewise smooth  decision boundary condition \eqref{230225012},  the noise condition \eqref{Tsybakovnoisecondition},  and the margin condition \eqref{margincondition} with $s_1=s_2=\infty$ for some $\texttt C\in\mc{C}^{d,\beta,r,I,\Theta}$.   They show that there exist constants $G_0,N_0,S_0,B_0,F_0$ not depending on the sample size $n$ such that the $\phi$-ERM 
\[
\hat{f}^{\FNN}_n\in\mathop{{\arg\min}}_{f \in \fdnn_d(G_0, N_0,S_0,B_0,F_0)}\frac{1}{n}\sum_{i=1}^n\phi\ykh{Y_i f(X_i)}
\] satisfies 
\beq\label{bound 3.1}
\sup_{P\in \mc H^{d,\beta,r,I,\Theta,\infty,\infty}_{6,t_1,c_1,t_2,c_2}}\bm{E}_{P^{\otimes n}}\zkh{\mc{E}_P\ykh{\hat{f}^{\FNN}_n}} \lesssim \frac{(\log n)^{1+\epsilon}}{n}
\eeq for any $\epsilon>0$. Indeed, the noise conditions \eqref{Tsybakovnoisecondition} and the margin condition \eqref{margincondition} with $s_1=s_2=\infty$ are equivalent to the following two conditions: there exist $\eta_0\in(0,1)$ and $\overline\Delta>0$ such that
\[
P_X\ykh{\setl{x\in[0,1]^d}{ \abs{2{\eta(x)}-1}\leq \eta_0 }}=0
\] and
\[
P_X\ykh{\hkh{x\in[0,1]^d \mid  \Delta_{\texttt C}(x)\leq \overline\Delta}}=0 \] (cf. conditions $(
\mr N')$ and $(\mr M')$ in \cite{kim2021fast}). 
Under the two conditions above, combining with the assumption $\mr{sgn}(2\eta-1)\xlongequal{P_X\text{-a.s.}}\texttt C\in \mc{C}^{d,\beta,r,I,\Theta}$, Lemma A.7 of \cite{kim2021fast} asserts that there exists $f_0^*\in \fdnn_d(G_0,N_0,S_0,B_0,F_0)$ such that \[f_0^* \in \mathop{\arg\min}_{f\in \fdnn_d(G_0,N_0,S_0,B_0,F_0)} \mc{R}_P^\phi(f)\] and
$$\mc{R}_P(f^*_0)=\mc{R}_P(2\eta-1)=\inf \setr{\mc{R}_P(f)}{\textrm{$f:[0,1]^d\to\mbR$ is measurable}}.$$ The excess misclassification error of $f:[0,1]^d\to \mathbb{R}$ is then given by $\mc{E}_P(f)=\mc{R}_P(f)-\mc{R}_P(f^*_0)$. Since $f_0^*$ is bounded by $F_0$, the authors in \cite{kim2021fast} can apply classical  concentration techniques developed for bounded random variables (cf. Appendix A.2 of \cite{kim2021fast}) to deal with $f_0^*$ (instead of the target function $f^*_{\phi,P}$), leading to  the generalization bound \eqref{bound 3.1}. In this paper, employing Theorem \ref{thm2.1}, we extend Theorem 4.1 of \cite{kim2021fast} to much less restrictive cases in which the noise exponent $s_1$ and the margin exponent $s_2$ are allowed to be taken from $[0,\infty]$.  The derived generalization bounds are presented in Theorem \ref{thm2.4}. In particular, when $s_1=s_2=\infty$ (i.e., let $s_1=\infty$ in statement $(1)$ of Theorem \ref{thm2.4}), we obtain a refined generalization bound under the same conditions as those of  Theorem 4.1 in \cite{kim2021fast}, which asserts that the $\phi$-ERM $\hat{f}^{\FNN}_n$ over $\fdnn_d(G_0, N_0,S_0,B_0,F_0)$ satisfies 
\beq\label{2307010}
\sup_{P\in \mc H^{d,\beta,r,I,\Theta,\infty,\infty}_{6,t_1,c_1,t_2,c_2}}\bm{E}_{P^{\otimes n}}\zkh{\mc{E}_P\ykh{\hat{f}^{\FNN}_n}} \lesssim \frac{\log n}{n},
\eeq removing the $\epsilon$ in their bound  \eqref{bound 3.1}. The above discussion indicates that Theorem \ref{thm2.1} can lead to sharper estimates in comparison with classical concentration techniques, and can be applied in very general settings. However, we would like to point out that if ${s_1<\infty}$ and $s_2<\infty$, then the convergence rate obtained in  Theorem \ref{thm2.4} (that is, the rate $\mc{O}\Big(\ykh{\frac{\left(\log n\right)^3}{n}}^{\frac{s_1}{1+(s_1+1)\left(1+\frac{d-1}{\beta s_2}\right)}}\Big)$ in \eqref{bound 2.43}) is suboptimal. Indeed, Theorem 3.1 and Theorem 3.4 of \cite{kim2021fast} show that DNN classifier 
 $\hat{f}^{\FNN}_n$ trained with empirical hinge risk minimization can achieve a  convergence rate 
 \begin{equation}\begin{aligned}\label{240410060407}
 	\sup_{P\in \mc H^{d,\beta,r,I,\Theta,s_1,s_2}_{6,t_1,c_1,t_2,c_2}}\bm{E}_{P^{\otimes n}}\zkh{\mc{E}_P\ykh{\hat{f}^{\FNN}_n}}\lesssim \ykh{\frac{\left(\log n\right)^3}{n}}^{\frac{s_1+1}{1+(s_1+1)\left(1+\frac{d-1}{\beta\cdot \ykh{1\qd s_2}}\right)}},
 \end{aligned}\end{equation} which is strictly faster than the rate $\mc{O}\Big(\ykh{\frac{\left(\log n\right)^3}{n}}^{\frac{s_1}{1+(s_1+1)\left(1+\frac{d-1}{\beta s_2}\right)}}\Big)$ in \eqref{bound 2.43}. Moreover, as  mentioned below Theorem 3.1 in \cite{kim2021fast},   even the rate in \eqref{240410060407} is  suboptimal in general.  In \cite{hu2022minimax}, the authors propose a new  DNN classifier  which are constructed  in a divide-and-conquer manner: DNN classifiers  are trained with empirical  $0$-$1$ risk minimization on each local region and then ``aggregated to a global one''.   \cite{hu2022minimax} provides minimax optimal convergence rates for this new DNN classifier under  the assumption  that the data distribution $P\in H^{d,\beta,r,1,1,0,0}_{6,t_1,1,t_2,1}$ (that is, the decision boundary of $P$ is assumed to be H\"older-$\beta$ smooth (rather than just piecewise smooth), but the noise condition \eqref{Tsybakovnoisecondition} and the margin condition \eqref{margincondition} are not required) along with a ``localized version'' of  the noise condition \eqref{Tsybakovnoisecondition} (see assumptions (M1) and (M2) in \cite{hu2022minimax}). It is interesting to further study whether we can  apply  Theorem \ref{thm2.1}  to establish optimal convergence rates for the new   DNN classifiers proposed in  \cite{hu2022minimax} which are locally trained with some  surrogate loss (as we have already pointed out, Theorem \ref{thm2.1} remains true for any locally Lipschitz  continuous loss function $\phi$, see the discussion on page \pageref{240410212610}) such as logistic loss  instead of $0$-$1$ loss.

The recent work \cite{farrell2021} considers estimation and inference using fully connected DNNs and the logistic loss in which their setting can cover both regression and classification. For any probability measure $P$ on $[0,1]^d\times\hkh{-1,1}$ and any measurable function $f:[0,1]^d\to[-\infty,\infty]$,  define $\norm{f}_{\mc{L}^2_{P_X}}:=\ykh{\int_{[0,1]^d}\abs{f(x)}^2\mr dP_X(x)}^{\frac{1}{2}}$. %
 Recall that $\mc{B}^{\beta}_r\ykh{\Omega}$ is defined  in  \eqref{holderball}. Let $\mc H^{d,\beta}_{7}$ be the set of all probability measures $P$ on $[0,1]^d\times\hkh{-1,1}$ such that the target function $f^*_{\phi,P}$ belongs to $\mc{B}^{\beta}_1\ykh{[0,1]^d}$.  In Corollary 1 of \cite{farrell2021}, the authors claimed that if $P\in\mc H^{d,\beta}_7$ and $\beta\in\mb{N}$, then with probability at least $1-\mr{e}^{-\upsilon}$ there holds
\beq \label{bound 3.2}
\norm{\hat{f}^{\FNN}_n-f^*_{\phi,P}}_{\mc{L}^2_{P_X}}^2 \lesssim n^{-\frac{2\beta}{2\beta+d}}\log^4 n+\frac{\log\log n+\upsilon}{n},
\eeq where the estimator $\hat{f}^{\FNN}_n \in \fdnn_d(G,N,S,\infty,F)$ is defined by \eqref{FCNNestimator} with 
\beq\ba\label{bound 3.3}
&G \asymp \log n, \ N \asymp n^{\frac{d}{d+2\beta}}, 
\ S \asymp n^{\frac{d}{d+2\beta}}\log n, \textrm{ and } \ F=2.
\ea\eeq  Note that $f^*_{\phi,P}\in \mc{B}^{\beta}_1\ykh{[0,1]^d}$ implies $\|f^*_{\phi,P}\|_{\infty}\leq 1$. From Lemma 8 of \cite{farrell2021},  bounding the quantity $\norm{\hat{f}^{\FNN}_n-f^*_{\phi,P}}_{\mc{L}^2_{P_X}}^2$ on the left hand side of (\ref{bound 3.2}) is equivalent to bounding $\mc{E}_P^{\phi}(\hat{f}^{\FNN}_n)$, since 
\begin{equation}\label{bound 3.4}
	\frac{1}{2(\me+\me^{-1}+2)}\norm{\hat{f}^{\FNN}_n-f^*_{\phi,P}}_{\mc{L}^2_{P_X}}^2 \leq \mc{E}_P^{\phi}(\hat{f}^{\FNN}_n) \leq \frac{1}{4} \norm{\hat{f}^{\FNN}_n-f^*_{\phi,P}}_{\mc{L}^2_{P_X}}^2.
\end{equation}
Hence (\ref{bound 3.2}) actually establishes the same upper bound (up to a constant independent of $n$ and $P$) for the excess $\phi$-risk of $\hat{f}^{\FNN}_n$, leading to upper bounds for the excess misclassification error  $\mc{E}_P(\hat{f}^{\FNN}_n)$ through the calibration inequality. The authors in \cite{farrell2021} apply concentration techniques based on  (empirical) \emph{Rademacher complexity} (cf. Section A.2 of \cite{farrell2021} or \cite{LocalBartlett,2006Local}) to derive the bound \eqref{bound 3.2}, which allows for removing the restriction of uniformly boundedness on the weights and biases in the neural network models, i.e., the hypothesis space generated by neural networks in their analysis can be of the form $\fdnn_d\ykh{G,N,S,\infty,F}$. In our paper, we employ the covering number to measure the complexity of hypothesis space. Due to the lack of compactness, the covering numbers of $\fdnn_d\ykh{G,N,S,\infty,F}$ are in general equal to infinity. Consequently, in our convergence analysis, we require the neural networks to possess bounded weights and biases. The assumption of bounded parameters may lead to additional optimization constraints in the training process. However, it has been found that the weights and biases of a trained neural network are typically around their initial values (cf. \cite{goodfellow2016deep}). Thus the boundedness assumption matches what is observed in practice and has been adopted by most of the literature (see, e,g., \cite{kim2021fast,schmidt2020nonparametric}).  In particular, the work \cite{schmidt2020nonparametric} considers nonparametric  regression using neural networks with all parameters bounded by one (i.e., $B=1$). This assumption can be realized by projecting the parameters of the neural network onto $[-1,1]$ after each updating. Though the framework developed in this paper would not deliver generalization bounds without restriction of uniformly bounded parameters, we weaken this constraint in Theorem \ref{thm2.2} by allowing the upper bound $B$ to grow polynomially with the sample size $n$, which simply requires $1\leq B \lesssim  n^{\nu}$ for any $\nu>0$. It is worth mentioning that in our coming work \cite{zhang202304cnn}, we actually establish  oracle-type inequalities analogous to Theorem \ref{thm2.1},  with the covering number $\mc{N}\ykh{\mc{F},\gamma}$ replaced by the supremum of some empirical $L_1$-covering numbers.  These  enable us to derive generalization bounds for the empirical $\phi$-risk minimizer $\hat{f}^{\FNN}_n$ over $\fdnn_d\ykh{G,N,S,\infty,F}$ because empirical $L_1$-covering numbers of $\fdnn_d\ykh{G,N,S,\infty,F}$ can be well-controlled, as indicated by Lemma 4 and Lemma 6 of \cite{farrell2021} (see also  Theorem 9.4 of \cite{gyorfi2002distribution} and Theorem 7 of \cite{bartlett2019nearly}). In addition,  note that \eqref{bound 3.2} can lead to probability bounds (i.e., confidence bounds) for the excess $\phi$-risk and misclassification error of $\hat{f}^{\FNN}_n$, while the generalization bounds presented in this paper are only in expectation. Nonetheless, in \cite{zhang202304cnn}, we obtain both probability bounds and expectation bounds for the empirical $\phi$-risk minimizer.

As discussed in Section \ref{section: introduction}, the boundedness assumptions on the target function  $f^*_{\phi,P}$ and its derivatives, i.e., $f^*_{\phi,P}\in \mc{B}^{\beta}_1\ykh{[0,1]^d}$, are too restrictive. This assumption actually requires that there exists some $\delta \in (0,1/2)$ such that the conditional class probability $\eta(x)=P(\{1\}|x)$ satisfies $\delta <\eta(x) <1-\delta$ for $P_X$-almost all $x\in [0,1]^d$, which rules out the case when $\eta$ takes values in $0$ or $1$ with positive probabilities. However, it is believed that the conditional class probability should be determined by the patterns that make the two classes mutually exclusive, implying that $\eta(x)$ should be closed to either $0$ or $1$. This is also observed in many benchmark datasets for image recognition. For example, it is reported in \cite{kim2021fast}, the conditional class probabilities of CIFAR10 data set estimated by neural networks with the logistic loss almost solely concentrate on $0$ or $1$ and very few are around $0.5$ (see Fig.2 in \cite{kim2021fast}). Overall, the boundedness restriction on $f^*_{\phi,P}$ is not expected to hold in binary classification as it would exclude the well classified data. We further point out that the techniques used in \cite{farrell2021} cannot deal with the case when $f^*_{\phi,P}$ is unbounded, or equivalently, when $\eta$ can take values close to $0$ or $1$. Indeed, the authors apply approximation theory of neural networks developed in \cite{yarotsky2017error} to construct uniform approximations of $f^*_{\phi,P}$, which requires $f^*_{\phi,P}\in \mc{B}^{\beta}_1\ykh{[0,1]^d}$ with $\beta\in\mb{N}$. However, if $f^*_{\phi,P}$ is unbounded, uniformly approximating $f^*_{\phi,P}$ by neural networks on $[0,1]^d$ is impossible, which brings the essential difficulty in estimating the approximation error.  
Besides, the authors use Bernstein's inequality to bound the quantity $\frac{1}{n}\sum_{i=1}^n\ykh{\phi(Y_if^*_1(X_i))-\phi(Y_if^*_{\phi,P}(X_i))}$ appearing in the error decomposition for $\norm{\hat{f}^{\FNN}_n-f^*_{\phi,P}}_{\mc{L}^2_{P_X}}^2$ (see (A.1) in \cite{farrell2021}), where $f_1^* \in \arg\min_{f\in \fdnn_d(G,N,S,\infty,2)} \|f-f^*_{\phi,P}\|_{[0,1]^d}.$ We can see that the unboundedness of $f^*_{\phi,P}$ will lead to the unboundedness of the random variable $\ykh{\phi(Yf^*_1(X))-\phi(Yf^*_{\phi,P}(X))}$, which makes Bernstein's inequality invalid to bound its empirical mean by the expectation.  In addition, the boundedness assumption on $f^*_{\phi,P}$ ensures the inequality \eqref{bound 3.4} on which the entire framework of convergence estimates in \cite{farrell2021} is built (cf. Appendix A.1 and A.2 of \cite{farrell2021}). Without this assumption, most of the theoretical arguments in \cite{farrell2021} are not feasible. In contrast, we require $\eta\xlongequal{P_X\text{-a.s.}}\hat\eta$ for some $\hat\eta \in \mc{B}^{\beta}_r\ykh{[0,1]^d}$ and $r\in (0,\infty)$ in Theorem \ref{thm2.2}. This H\"{o}lder smoothness condition on $\eta$ is well adopted in the study of binary classifiers (see \cite{audibert2007fast} and references therein). Note that $f^*_{\phi,P}\in \mc{B}^{\beta}_1\ykh{[0,1]^d}$ indeed implies $\eta\xlongequal{P_X\text{-a.s.}}\hat\eta$ for some $\hat\eta \in \mc{B}^{\beta}_r\ykh{[0,1]^d}$ and $r\in (0,\infty)$ which only depends on $(d,\beta)$. Therefore, the setting considered in Theorem \ref{thm2.2} is more general than that of \cite{farrell2021}. Moreover, the condition $\eta\xlongequal{P_X\text{-a.s.}}\hat\eta\in \mc{B}^{\beta}_r\ykh{[0,1]^d}$ is more nature, allowing $\eta$ to take values close to $0$ and $1$ with positive probabilities. We finally point out that, under the same assumption (i.e., $P\in\mc H^{d,\beta}_7$), one can use Theorem \ref{thm2.1} to establish a convergence rate which is slightly improved compared with \eqref{bound 3.2}. Actually, we can show that there exists a constant $\mr c\in(0,\infty)$ only depending on $(d,\beta)$, such that for any  $\mu\in[1,\infty)$, and $\nu\in[0,\infty)$,  there holds 
\beq \label{bound 3.5}
\sup_{P\in\mc H^{d,\beta}_7}\bm{E}_{P^{\otimes n}}\zkh{\norm{\hat{f}^{\FNN}_n-f^*_{\phi,P}}^2_{\mc{L}^2_{P_X}}} \lesssim \ykh{\frac{\ykh{\log n}^3}{n}}^{\frac{2\beta}{2\beta+d}},
\eeq where the estimator $\hat{f}^{\FNN}_n \in \fdnn_d(G,N,S,B,F)$ is defined by \eqref{FCNNestimator} with 
\beq\label{230119171}
&\mr c\log n\leq G \asymp \log {n}, \ N \asymp \left(\frac{n}{\log^3 n}\right)^{\frac{d}{d+2\beta}}, 
\ S \asymp  \left(\frac{n}{\log^3 n}\right)^{\frac{d}{d+2\beta}}\cdot\log {n},\\&  1\leq B\lesssim n^\nu, \textrm{ and } 1\leq F\leq \mu.
\eeq Though we restrict the weights and biases to be bounded by $B$,  both the convergence rate and the network complexities in the result above refine the previous estimates established in \eqref{bound 3.2} and \eqref{bound 3.3}. In particular, since $\frac{6\beta}{2\beta+d} < 3<4$,  the convergence rate in (\ref{bound 3.5}) is indeed faster than that in (\ref{bound 3.2}) due to a smaller 
power exponent of the term $\log n$. The proof of this claim is in Appendix  \ref{section: proof of bound3.5}. We also remark that the convergence rate in (\ref{bound 3.5}) achieves the minimax optimal rate established in \cite{stone1982} up to log factors (so does the rate in \eqref{bound 3.2}), which confirms that generalization analysis developed in this paper is also rate-optimal for bounded $f^*_{\phi,P}$.

 In our work, we have established generalization bounds for ERMs over  hypothesis spaces consisting of neural networks. However,  such ERMs cannot be obtained in practice because the correspoding optimization problems (e.g., \eqref{2212271526}) cannot be solved explicitly. Instead, practical neural network estimators are obtained from algorithms which  numerically solve the empirical risk minimization problem.  Therefore, it is better to conduct generalization analysis for estimators obtained from such algorithms. One typical work in this direction is \cite{ji2021early}. 

In \cite{ji2021early}, for classification tasks, the authors establish  excess $\phi$-risk bounds to  show that  classifiers  obtained from solving empirical risk minimization with respect to the logistic loss over shallow neural networks using gradient descent with (or without) early stopping are consistent. Note that the setting of \cite{ji2021early} is quite different from ours: We consider deep neural network models in our work, while \cite{ji2021early} considers shallow ones. Besides, we use the  smoothness  of  the conditional probability function $\eta(\cdot)=P(\hkh{1}|\cdot)$ to characterize the regularity (or complexity) of the data distribution $P$. Instead, in \cite{ji2021early}, for each $\overline{U}_\infty:\mbR^d\to\mbR^d$, the authors construct a function \[f(\ \cdot\ ;\overline{U}_\infty):\mbR^d\to\mbR,x\mapsto \int_{\mbR^d}x^\top\overline{U}_\infty(v)\cdot\idf_{[0,\infty)}(v^\top x)\cdot\frac{1}{(2\pi)^{n/2}}\cdot\exp(-\frac{\norm{v}_2^2}{2})\mr d v\]called infinite-width random feature model. Then they use  the norm of $\overline{U}_\infty$ which makes $\mc E_P^\phi(f(\ \cdot\ ;\overline{U}_\infty))$ small to characterize the regularity of data: the data distribution is regarded as simple if there is a $\overline{U}_\infty$ with $\mc E_P^\phi(f(\ \cdot\ ;\overline{U}_\infty))\approx 0$ and moreover has a low norm. More rigorously, the slower the quantity
\beq\label{221229231}
\inf\setr{\norm{\overline{U}_\infty}_{\mbR^d}}{\mc E_P^\phi(f(\ \cdot\ ;\overline{U}_\infty))\leq\e}
\eeq grows as $\e\to 0$, the more regular (simpler) the data distribution $P$ is. In \cite{ji2021early}, the established excess $\phi$-risk bounds depend on the quantity $\mc E_P^\phi(f(\ \cdot\ ;\overline{U}_\infty))$ and the norm $\norm{\overline{U}_\infty}_{\mbR^d}$. Hence by assuming certain growth rates of the quantity in \eqref{221229231} as $\e\to 0$, we can obtain specific rates of convergence from the excess $\phi$-risk bounds in \cite{ji2021early}. It is natural to ask is there any relation between these two characterizations of data regularity, that is, the smoothness of conditional probability function, and the rate of growth of the quantity in \eqref{221229231} as $\e\to 0$.  For example, will H\"older smoothness of the conditional probability function imply certain growth rates of the quantity in \eqref{221229231} as $\e\to 0$? This question is worth considering because once we prove the equivalence of these two characterizations, then the generalization analysis in  \cite{ji2021early} will be able to be used in other settings  requiring smoothness of the conditional probability function and vice versa.  In addition, it is also interesting to study how can we use our new  techniques developed in this paper to  establish generalization bounds for deep neural network  estimators obtained from learning algorithms (e.g., gradient descent) within the settings in this paper.

\section{Conclusion}\label{section: conclusion}
In this paper, we develop a novel generalization analysis for binary classification with DNNs and logistic loss. The unboundedness of the target function in logistic classification poses challenges for the estimates of sample error and approximation error when deriving generalization bounds. To overcome these difficulties, we introduce a bivariate function $\psi:[0,1]^d\times\{-1,1\}\to\mbR$ to establish an elegant oracle-type inequality, aiming to bound the excess risk with respect to the logistic loss. This inequality incorporates the estimation of sample error and enables us to propose a framework for generalization analysis, which avoids using the explicit form of the target function. By properly choosing $\psi$ under this framework, we can eliminate the boundedness restriction of the target function and establish sharp rates of convergence. In particular, for fully connected DNN classifiers trained by minimizing the empirical logistic risk, we obtain an optimal (up to some logarithmic factor) rate of convergence of the excess logistic risk (which further yields a rate of convergence of the excess misclassification error via the calibration inequality) merely  under the H\"{o}lder smoothness assumption on the conditional probability function. If we instead assume that the conditional probability function is the composition of several vector-valued multivariate functions of which each component function is either a maximum value function of some of its input variables or a H\"older smooth function only depending on a small number of its input variables, we can even establish dimension-free optimal (up to some logarithmic factor) convergence rates for the excess logistic risk of fully connected DNN classifiers,  further leading to dimension-free rates of convergence of their excess misclassification error through the calibration inequality. This result  serves to elucidate the remarkable achievements of DNNs in high-dimensional real-world classification tasks.   In other circumstances such as when the data distribution has a piecewise smooth decision boundary  and the input data are bounded away from it (i.e., $s_2=\infty$ in \eqref{margincondition}), dimension-free rates of convergence can also be derived. Besides the novel oracle-type inequality, the sharp estimates presented in our paper also owe to a tight  error bound for approximating the natural logarithm function (which is unbounded near zero) by fully connected DNNs. All the claims for the optimality of rates in our paper are justified by corresponding minimax lower bounds.   As far as we know, all these results are new to the literature, which further enrich the theoretical understanding of classification using deep neural networks. At last, we would like to emphasize that our framework of generalization analysis is very general and can be extended to  many other settings (e.g.,  when the loss function, the hypothesis space, or the assumption on the data distribution is different from that in this current paper). In particular, in our forthcoming research \cite{zhang202304cnn}, we have investigated generalization analysis for CNN classifiers trained with the logistic loss, exponential loss, or LUM loss on spheres under the Sobolev smooth conditional probability assumption.   Motivated by recent work ${\text{\cite{guo2020modeling,guo2017thresholded,lin2018distributed,zdx2018}}}$, we will also study more efficient implementations of deep logistic classification for dealing with big data.

\appendix

\section{Covering Numbers of Spaces of Fully Connected DNNs} %
\label{section: appendix A1}

In this appendix, we provide upper bounds for the covering numbers of  spaces  of  fully connected DNNs.  Recall that if $\mc{F}$ consists of bounded real-valued functions defined on a domain containing $[0,1]^d$, the covering number of $\mc{F}$ with respect to the radius $\gamma$ and the metric  $\mc{F}\times\mc{F}\ni (f,g)\mapsto\sup_{x\in [0,1]^d}\abs{f(x)-g(x)}\in[0,\infty)$ is denoted by $\mc{N}(\mc{F},\gamma)$. For the  space $\fdnn_d\ykh{G,N,S,B,F}$ defined by \eqref{spaceofFCNN}, the covering number $\mc{N}\ykh{\fdnn_d\ykh{G,N,S,B,F},\gamma}$ can be bounded from above in terms of $G, N, S, B$, and the radius of covering $\gamma$. The related results are stated below. 

\begin{thm}\label{thmA1} For $G\in[1,\infty)$, $(N,S,B)\in [0,\infty)^3$, and $\gamma\in(0,1)$, there holds 
	\[\ba
	&\log \Big(\mc{N}\ykh{\fdnn_d\ykh{G,N,S,B,\infty},\gamma}\Big)\\&\leq (S+Gd+1)(2G+5)\cdot\log{\frac{(\max\hkh{N,d}+1) (B\qd 1)(G+1)}{\gamma}}.
	\ea\]
\end{thm}

Theorem \ref{thmA1} can be proved in the same manner as in the proof of Lemma 5 in \cite{schmidt2020nonparametric}. Therefore, we omit the proof here. Similar results are also presented in Proposition A.1 of \cite{kim2021fast} and Lemma 3 of \cite{suzuki2018adaptivity}. Corollary \ref{corollaryA1} follows immediately from Theorem \ref{thmA1} and Lemma 10.6 of \cite{bartlett2009}. 

\begin{cor}\label{corollaryA1} 
	For $G\in[1,\infty)$, $(N,S,B)\in [0,\infty)^3$, $F\in [0,\infty]$ and $\gamma\in(0,1)$, there holds
	\begin{align*}
	&\log\Big(\mc{N}\ykh{\fdnn_d\ykh{G,N,S,B,F},\gamma}\Big)\\&\leq (S+Gd+1)(2G+5)\cdot\log{\frac{(\max\hkh{N,d}+1) (B\qd 1)(2G+2)}{\gamma}}.
	\end{align*}
\end{cor}

\section{Approximation Theory of Fully Connected DNNs}\label{A.3}

Theorem \ref{thmA3} below gives error bounds for approximating  H\"older continuous functions by fully connected DNNs. Since it can be derived  straightforwardly from Theorem 5 of \cite{schmidt2020nonparametric}, we omit its proof. 

\begin{thm} \label{thmA3}  Suppose that $f \in \mc{B}^{\beta}_r\ykh{[0,1]^d}$ with some $(\beta, r)\in (0,\infty)^2$. Then for any positive integers $m$ and $M'$ with $M'\geq \max \hkh{(\beta+1)^d, \ykh{r\sqrt{d}\ceil{\beta}^d+1}\mr{e}^d}$, there exists
	\[
	\tilde{f}\in\fdnn_d\ykh{14m(2+{\log_2\ykh{d\qd\beta}}),6\ykh{d+\ceil{\beta}}M',987(2d+\beta)^{4d}M'm,1,\infty}
	\]
	such that
	\begin{align*}
	&\sup_{x\in[0,1]^d}\abs{f(x)-\tilde{f}(x)}\\&\leq r\sqrt{d}\ceil{\beta}^d\cdot 3^{\beta}M'^{-\beta/d}+\ykh{1+2r\sqrt{d}\ceil{\beta}^d}\cdot 6^d\cdot(1+d^2+\beta^2)\cdot M'\cdot 2^{-m}.
	\end{align*}
\end{thm} 

Corollary \ref{corollaryA2} follows directly from Theorem \ref{thmA3}.  

\mybookmark{Fully connected NN approx of Holder func}{2023?5?16? 00:35:24}

\begin{cor}\label{corollaryA2}
	Suppose that $f \in \mc{B}^{\beta}_r\ykh{[0,1]^d}$ with some $(\beta, r)\in (0,\infty)^2$. Then for any $\e\in(0,1/2]$, there exists
	\[
	\tilde{f}\in\fdnn_d\ykh{D_1\log\frac{1}{\e},D_2\e^{-\frac{d}{\beta}},D_3\e^{-\frac{d}{\beta}}\log\frac{1}{\e},1,\infty}
	\]
	such that
	\[
	\sup_{x\in[0,1]^d}\abs{f(x)-\tilde{f}(x)}\leq\e,
	\] where $(D_1,D_2,D_3)\in (0,\infty)^3$ are constants depending only on $d$, $\beta$ and $r$.
\end{cor}

\begin{proof}
	Let 
\begin{align*}
	&E_1=\max\hkh{(\beta+1)^d, \ykh{r\sqrt{d}\ceil{\beta}^d+1}\mr{e}^d,\ykh{\frac{1}{2r}\cdot 3^{-\beta}\cdot \frac{1}{\sqrt{d}\ceil{\beta}^d}}^{-d/\beta}},\\
	&E_2=3\max\hkh{1+\frac{d}{\beta},\frac{\log\ykh{4E_1\cdot\ykh{1+2r\sqrt{d}\ceil{\beta}^d}(1+d^2+\beta^2)\cdot 6^d}}{\log 2}},
	\end{align*}
	and
\begin{align*}
	&D_1=14\cdot(2+\log_2\ykh{d\qd\beta})\cdot(E_2+2),\\
	&D_2=6\cdot\ykh{d+\ceil{\beta}}\cdot(E_1+1),\\
	&D_3=987\cdot(2d+\beta)^{4d}\cdot (E_1+1)\cdot (E_2+2).
	\end{align*}
	Then $D_1,D_2,D_3$ are constants only depending on $d,\beta,r$. 
	
	For $f \in \mc{B}^{\beta}_r\ykh{[0,1]^d}$ and $\e\in(0,1/2]$, choose $M'=\ceil{E_1\cdot \e^{-d/\beta}}$ and  $m=\ceil{E_2\log(1/\e)}$. Then $m$ and $M'$ are positive integers satisfying that
	\beq\label{51}\ba
	1&\leq\max \hkh{(\beta+1)^d, \ykh{r\sqrt{d}\ceil{\beta}^d+1}\mr{e}^d}\leq E_1\leq E_1\cdot \e^{-d/\beta}\\&\leq M'\leq 1+E_1\cdot \e^{-d/\beta}\leq (E_1+1) \cdot \e^{-d/\beta} ,
	\ea\eeq
	\beq\label{52}
	\ba
	M'^{-\beta/d}&\leq \ykh{E_1\cdot \e^{-d/\beta}}^{-\beta/d}\leq\e\cdot\frac{1}{2r}\cdot 3^{-\beta}\cdot \frac{1}{\sqrt{d}\ceil{\beta}^d},
	\ea
	\eeq
	and
	\beq\ba\label{53}
	m&\leq E_2\log(1/\e)+2\log 2\leq E_2\log(1/\e)+2\log(1/\e)=(2+ E_2)\cdot\log(1/\e).
	\ea\eeq
	Moreover, we have that
	\beq\ba\label{54}
	&2\cdot\ykh{1+2r\sqrt{d}\ceil{\beta}^d}\cdot 6^d\cdot(1+d^2+\beta^2)\cdot M'\cdot \frac{1}{\e}\\
	&\leq 2\cdot\ykh{1+2r\sqrt{d}\ceil{\beta}^d}\cdot 6^d\cdot(1+d^2+\beta^2)\cdot (E_1+1)\cdot \e^{-1-d/\beta}\\
	&\leq 2\cdot\ykh{1+2r\sqrt{{d}}\ceil{\beta}^d}\cdot 6^d\cdot(1+d^2+\beta^2)\cdot 2E_1\cdot \e^{-1-d/\beta}\\
	&\leq 2^{\frac{1}{3}E_2}\cdot \e^{-1-d/\beta} \leq 2^{\frac{1}{3}E_2}\cdot \e^{-\frac{1}{3}E_2}\leq \e^{-\frac{1}{3}E_2}\cdot \e^{-\frac{1}{3}E_2}\\
	&\leq\e^{-E_2\cdot\log 2}=2^{E_2\log(1/ \e)}\leq 2^m.
	\ea\eeq
	Therefore, from (\ref{51}), (\ref{52}), (\ref{53}), (\ref{54}), and Theorem \ref{thmA3}, we conclude that there exists
\begin{align*}
	\tilde{f}&\in\fdnn_d(14m(2+{\log_2\ykh{d\qd\beta}}),6\ykh{d+\ceil{\beta}}M',987(2d+\beta)^{4d}M'm,1,\infty)\\
	&=\fdnn_d\ykh{\frac{D_1}{E_2+2}\cdot m,\frac{D_2}{E_1+1}\cdot M',\frac{D_3}{(E_1+1)\cdot(E_2+2)}\cdot M'm,1,\infty}\\
	&\subset\fdnn_d\ykh{D_1\log\frac{1}{\e},D_2\e^{-\frac{d}{\beta}},D_3\e^{-\frac{d}{\beta}}\log\frac{1}{\e},1,\infty}
	\end{align*}
	such that
	\[\ba
	&\sup_{x\in[0,1]^d}\abs{f(x)-\tilde{f}(x)}\\&\leq r\sqrt{d}\ceil{\beta}^d\cdot 3^{\beta}M'^{-\beta/d}+\ykh{1+2r\sqrt{d}\ceil{\beta}^d}\cdot 6^d\cdot(1+d^2+\beta^2)\cdot M'\cdot 2^{-m}\leq \frac{\e}{2}+\frac{\e}{2}=\e.
	\ea\]
	Thus we complete the proof.
\end{proof}

\section{Proofs of Results in the Main Body}\label{appendixC240411044938}

The proofs in this appendix will be organized in logical order in the sense that each result  in this appendix is proved without relying on results that are presented after it.  

Throughout this appendix, we use \[
C_{\texttt{Parameter}_1, \texttt{Parameter}_2, \cdots, \texttt{Parameter}_m}
\] to denote a positive constant only depending on $\texttt{Parameter}_1$, $\texttt{Parameter}_2$, $\cdots$, $\texttt{Parameter}_m$. For example, we may use $C_{d,\beta}$ to denote a positive constant only depending on $(d,\beta)$. The values of such  constants appearing in the proofs may be different from line to line or even in the same line. Besides, we may use the same symbol with different meanings in different proofs. For example, the symbol $I$ may denote a number in one proof, and denote a set in another proof. To avoid confusion, we will explicitly redefine these symbols in each proof.

\subsection{Proofs of Some Properties of the Target Function}

The following lemma justifies our claim in \eqref{2302081455}.  

\begin{lem}\label{2302062347}Let $d\in\mb N$,  $P$ be a probability measure on $[0,1]^d\times\hkh{-1,1}$, and $\phi:\mbR\to[0,\infty)$ be a measurable function. Define \[
	\overline\phi:[-\infty,\infty]\to[0,\infty],\;z\mapsto\begin{cases}
	\varlimsup\limits_{t\to+\infty}\phi(t),&\text{ if }z=\infty,\\
	\phi(z),&\text{ if }z\in\mbR,\\
	\varlimsup\limits_{t\to-\infty}\phi(t),&\text{ if }z=-\infty,\\	
	\end{cases}
	\] which is an extension of $\phi$ to $[-\infty,\infty]$.  Suppose $f^*:[0,1]^d\to[-\infty,\infty]$ is a measurable function satisfying that 
	\beq\label{2302081657}
	f^*(x)\in\mathop{\arg\min}_{z\in[-\infty,\infty]}\int_{\hkh{-1,1}}\overline\phi(yz)\mr{d}P(y|x)\text{ for $P_X$-almost all $x\in[0,1]^d$.}\eeq Then there holds
\[
\int_{[0,1]^d\times\{-1,1\}}\overline\phi(yf^*(x))\mr{d}P(x,y)=\inf\setl{\mc R^\phi_P(g)}{g:[0,1]^d\to\mbR\text{ is measurable}}. 
\]
\end{lem}
\begin{proof} Let $\Omega_0:=\setl{x\in[0,1]^d}{f^*(x)\in\mbR}\times\hkh{-1,1}$. Then for any $m\in\mb N$ and any $(i,j)\in\hkh{-1,1}^2$,  define 
\[
f_m:[0,1]^d\to\mbR,\;x\mapsto\begin{cases}
m,&\text{ if }f^*(x)=\infty,\\
f^*(x),&\text{ if }f^*(x)\in\mbR,\\
-m,&\text{ if }f^*(x)=-\infty, 
\end{cases}
\] and  $\Omega_{i,j}=\setl{x\in[0,1]^d}{f^*(x)=i\cdot\infty}\times\hkh{j}$. 
Obviously,  $yf^*(x)=ij\cdot\infty$ and $yf_m(x)=ijm$ for any $(i,j)\in\hkh{-1,1}^2$, any $m\in\mb N$, and any $(x,y)\in\Omega_{i,j}$. Therefore,
\beq\label{2302062316}
&\varlimsup_{m\to+\infty}\int_{\Omega_{i,j}}\phi(yf_m(x))\mr{d}P(x,y)\\&=\varlimsup_{m\to+\infty}\int_{\Omega_{i,j}}\phi(ijm)\mr{d}P(x,y)=P(\Omega_{i,j})\cdot\varlimsup_{m\to+\infty}\phi(ijm)\\
&\leq P(\Omega_{i,j})\cdot\varlimsup_{t\to ij\cdot\infty}\phi(t)=P(\Omega_{i,j})\cdot\overline\phi(ij\cdot\infty)= \int_{\Omega_{i,j}}\overline\phi(ij\cdot\infty)\mr{d}P(x,y)\\
&=\int_{\Omega_{i,j}}\overline\phi(yf^*(x))\mr{d}P(x,y),\;\forall\;(i,j)\in\hkh{-1,1}^2. 
\eeq Besides, it is easy to verify that $yf_m(x)=yf^*(x)\in\mbR$ for any $(x,y)\in\Omega_0$ and any $m\in\mb N$, which means that 
\beq\label{2302062315}
\int_{\Omega_{0}}\phi(yf_m(x))\mr{d}P(x,y)=\int_{\Omega_{0}}\overline\phi(yf^*(x))\mr{d}P(x,y),\;\forall\;m\in\mb N. 
\eeq Combining \eqref{2302062316} and \eqref{2302062315}, we obtain
\beq\label{2302081717}
&\inf\setl{\mc R^\phi_P(g)}{g:[0,1]^d\to\mbR\text{ is measurable}}\\&\leq \varlimsup_{m\to+\infty}\mc R^\phi_P(f_m)=\varlimsup_{m\to+\infty}\int_{[0,1]^d\times\hkh{-1,1}}\phi(yf_m(x))\mr{d}P(x,y)\\
&=\varlimsup_{m\to+\infty}\ykh{\int_{\Omega_0}\phi(yf_m(x))\mr{d}P(x,y)+\sum_{i\in\hkh{-1,1}}\sum_{j\in\hkh{-1,1}}\int_{\Omega_{i,j}}\phi(yf_m(x))\mr{d}P(x,y)}\\
&\leq \varlimsup_{m\to+\infty}\int_{\Omega_0}\phi(yf_m(x))\mr{d}P(x,y)+\sum_{i\in\hkh{-1,1}}\sum_{j\in\hkh{-1,1}}\varlimsup_{m\to+\infty}\int_{\Omega_{i,j}}\phi(yf_m(x))\mr{d}P(x,y)\\
&\leq \int_{\Omega_0}\overline\phi(yf^*(x))\mr{d}P(x,y)+\sum_{i\in\hkh{-1,1}}\sum_{j\in\hkh{-1,1}}\int_{\Omega_{i,j}}\overline\phi(yf^*(x))\mr{d}P(x,y)\\&=\int_{[0,1]^d\times\hkh{-1,1}}\overline\phi(yf^*(x))\mr{d}P(x,y). 
\eeq 
On the other hand, for any measurable $g:[0,1]^d\to\mbR$, it follows from \eqref{2302081657} that   
\begin{align*}
&\int_{\hkh{-1,1}}\overline\phi(yf^*(x))\mr{d}P(y|x)= \mathop{\inf}_{z\in[-\infty,\infty]}\int_{\hkh{-1,1}}\overline\phi(yz)\mr{d}P(y|x)\leq \int_{\hkh{-1,1}}\overline\phi(yg(x))\mr{d}P(y|x)\\
&= \int_{\hkh{-1,1}}\phi(yg(x))\mr{d}P(y|x) \text{ for $P_X$-almost all $x\in[0,1]^d$.}	
\end{align*} Integrating both sides, we obtain
\begin{align*}
&\int_{[0,1]^d\times\{-1,1\}}\overline\phi(yf^*(x))\mr{d}P(x,y)=\int_{[0,1]^d}\int_{\hkh{-1,1}}\overline\phi(yf^*(x))\mr{d}P(y|x)\mr{d}P_X(x)\\
&\leq\int_{[0,1]^d}\int_{\hkh{-1,1}}\phi(yg(x))\mr{d}P(y|x) P_X(x)=\int_{[0,1]^d\times\hkh{-1,1}}\phi(yg(x))\mr{d}P(x,y)=\mc R^\phi_P(g). 
\end{align*} Since $g$ is arbitrary, we deduce that
\begin{align*}
\int_{[0,1]^d\times\{-1,1\}}\overline\phi(yf^*(x))\mr{d}P(x,y)\leq  \inf\setl{\mc R^\phi_P(g)}{g:[0,1]^d\to\mbR\text{ is measurable}},
\end{align*} which, together with \eqref{2302081717}, proves the desired result.  \end{proof}

The next lemma gives the explicit form of the target function of the  logistic risk. 

\begin{lem}\label{2302281501}Let $\phi(t)=\log(1+\me^{-t})$ be the logistic loss, $d\in\mb N$, $P$ be a probability measure on $[0,1]^d\times\hkh{-1,1}$, and $\eta$ be the conditional probability function $P(\hkh{1}|\cdot)$ of $P$. Define
\beq\label{23022815}
f^*:[0,1]^d\to[-\infty,\infty],\; x\mapsto \begin{cases}
\infty,&\text{ if }\eta(x)=1,\\
\log\frac{\eta(x)}{1-\eta(x)},&\text{ if }\eta(x)\in(0,1),\\
-\infty,&\text{ if }\eta(x)=0,\\	
\end{cases}
\eeq which is a natural extension of the map
\[ \setl{z\in[0,1]^d}{\eta(z)\in(0,1)}\ni x\mapsto\log\frac{\eta(x)}{1-\eta(x)}\in\mbR
\] to all of $[0,1]^d$.  Then $f^*$ is a target function of the $\phi$-risk under $P$, i.e., \eqref{2302211603} holds. In addition, the target function of the $\phi$-risk under $P$ is unique up to a $P_X$-null set. In other words, for any target function ${f}^\star$ of the $\phi$-risk under $P$, we must have \[P_X\ykh{\setl{x\in[0,1]^d}{f^*(x)\neq f^\star(x)}}=0.\]  \end{lem}
\begin{proof}
Define \beq\label{230228150}
\overline\phi:[-\infty,\infty]\to[0,\infty],\;z\mapsto\begin{cases}0,&\text{ if }z=\infty,\\
\phi(z),&\text{ if }z\in\mbR,\\
\infty,&\text{ if }z=-\infty,\\	
\end{cases}
\eeq which is a natural extension of the logistic loss $\phi$ to $[-\infty,\infty]$, and define \[
V_{a}:[-\infty,\infty]\to[0,\infty],z\mapsto a\overline\phi(z)+(1-a)\overline\phi(-z)
\]for any $a\in[0,1]$. Then we have that \beq\label{230228020}
&\int_{\hkh{-1,1}}\overline\phi(yz)\mr{d}P(y|x)=\eta(x)\overline\phi(z)+(1-\eta(x))\overline\phi(-z)\\&=V_{\eta(x)}(z),\;\forall\;x\in[0,1]^d,\;z\in[-\infty,\infty]. 
\eeq For any $a\in[0,1]$, we have that $V_{a}$ is smooth on $\mbR$, and an elementary calculation gives
\[
V_a''(t)=\frac{1}{2+\me^t+\me^{-t}}>0,\;\forall\;t\in\mbR. 
\]Therefore, $V_a$ is strictly convex on $\mbR$ and 
\beq\label{23022801}
&\mathop{\arg\min}_{z\in\mbR}V_a(z)=\hkh{z\in\mbR\big|V_a'(z)=0}=\setr{z\in\mbR}{a\phi'(z)-(1-a)\phi'(-z)=0}\\
&=\setr{z\in\mbR}{-a+\frac{\me^z}{1+\me^z}=0}=\begin{cases}
\hkh{\log\frac{a}{1-a}},&\text{if }a\in(0,1),\\
\varnothing,&\text{if }	a\in\hkh{0,1}.
\end{cases}
\eeq Besides, it is easy to verify that
\[
V_{a}(z)=\infty,\;\forall\;a\in(0,1),\;\forall\;z\in\hkh{\infty,-\infty},
\] which, together with \eqref{23022801}, yields
\beq\label{23022802}
\mathop{\arg\min}_{z\in[-\infty,\infty]}V_a(z)=\mathop{\arg\min}_{z\in\mbR}V_a(z)=\hkh{\log\frac{a}{1-a}},\;\forall\;a\in(0,1). 
\eeq In addition, it follows from 
\[
\overline\phi(z)>0=\overline\phi(\infty),\;\forall\;z\in[-\infty,\infty) 
\]that 
\beq\label{230228021}
\mathop{\arg\min}_{z\in[-\infty,\infty]}V_1(z)=\mathop{\arg\min}_{z\in[-\infty,\infty]}\overline\phi(z)=\hkh{\infty}
\eeq and 
\beq\label{230228022}
\mathop{\arg\min}_{z\in[-\infty,\infty]}V_{0}(z)=\mathop{\arg\min}_{z\in[-\infty,\infty]}\overline\phi(-z)=\hkh{-\infty}. 
\eeq Combining \eqref{230228020}, \eqref{230228021} and \eqref{230228022}, we obtain
\begin{align*}
&\mathop{\arg\min}_{z\in[-\infty,\infty]}\int_{\hkh{-1,1}}\overline\phi(yz)\mr{d}P(y|x)=\mathop{\arg\min}_{z\in[-\infty,\infty]}V_{\eta(x)}(z)=\begin{cases}
\hkh{+\infty},&\text{if }\eta(x)=1,\\
\hkh{\log\frac{\eta(x)}{1-\eta(x)}},&\text{if }\eta(x)\in(0,1),\\
\hkh{-\infty},&\text{if }\eta(x)=0
\end{cases}\\
&=\hkh{f^*(x)},\;\forall\;x\in[0,1]^d,
\end{align*} which implies  \eqref{2302211603}. Therefore, $f^*$ is a target function of the $\phi$-risk under the distribution $P$. Moreover, the uniqueness of the target function of the $\phi$-risk under $P$  follows immediately from the fact that  for all $x\in[0,1]^d$ the set  \[\mathop{\arg\min}\limits_{z\in[-\infty,\infty]}\int_{\hkh{-1,1}}\overline\phi(yz)\mr{d}P(y|x)=\hkh{f^*(x)}\] contains exactly one point and the uniqueness (up to some $P_X$-null set) of the conditional distribution $P(\cdot|\cdot)$ of $P$. This completes the proof. \end{proof}

The Lemma \ref{23022804} below provides a formula for computing the infimum of the logistic risk over all real-valued measurable functions. 

\begin{lem}\label{23022804}
Let $\phi(t)=\log(1+\me^{-t})$ be the logistic loss,  $\delta\in(0,1/2]$, $d\in\mb N$, $P$ be a probability measure on $[0,1]^d\times\hkh{-1,1}$, $\eta$ be the conditional probability function $P(\hkh{1}|\cdot)$ of $P$, $f^*$ be defined by \eqref{23022815},   $\overline\phi$ be defined by 	\eqref{230228150}, $H$ be defined by \[\ba
H:[0,1]\to[0,\infty),t\mapsto\left\{\ba & t\log\ykh{\frac{1}{t}}+(1-t)\log\ykh{\frac{1}{1-t}},&&\text{ if }t\in(0,1),\\
&0,&&\text{ if }t\in\{0,1\},
\ea\right.
\ea\] and $\psi$ be defined by \[\ba
\psi:\,&[0,1]^d\times\{-1,1\}\to[0,\infty),\\&\;\;\;(x,y)\mapsto\left\{
\ba
&\phi\ykh{y\log\frac{\eta(x)}{1-\eta(x)}},&&\text{if }\eta(x)\in [\delta,1-\delta],\\
&0,&&\text{if }\eta(x)\in\{0,1\},\\
&\eta(x)\log\frac{1}{\eta(x)}+(1-\eta(x))\log\frac{1}{1-\eta(x)},&&\text{if }\eta(x)\in (0, \delta)\cup(1-\delta,1).
\ea
\right.
\ea\] Then there holds
\begin{align*}
&\inf \setl{\mc{R}^\phi_P(g)}{ \textrm{$g:[0,1]^d\to\mbR$ is measurable}}=\int_{[0,1]^d\times\{-1,1\}}\overline\phi(yf^*(x))\mr{d}P(x,y)\\&=\int_{[0,1]^d}H(\eta(x))\mr{d}P_X(x)=\int_{[0,1]^d\times\hkh{-1,1}}\psi(x,y)\mr{d}P(x,y).
\end{align*}
\end{lem}
\begin{proof} According to Lemma \ref{2302281501}, $f^*$ is a target  function of the $\phi$-risk under the distribution $P$, meaning that 	\[
f^*(x)\in\mathop{\arg\min}_{z\in[-\infty,\infty]}\int_{\hkh{-1,1}}\overline\phi(yz)\mr{d}P(y|x)\text{ for $P_X$-almost all $x\in[0,1]^d$.}\] Then it follows from Lemma \ref{2302062347} that
\beq\label{23022818}
&\inf \setl{\mc{R}^\phi_P(g)}{ \textrm{$g:[0,1]^d\to\mbR$ is measurable}}=\int_{[0,1]^d\times\{-1,1\}}\overline\phi(yf^*(x))\mr{d}P(x,y)\\
&=\int_{[0,1]^d}\int_{\hkh{-1,1}}\overline\phi(yf^*(x))\mr{d}P(y|x)\mr d P_X(x)\\&=\int_{[0,1]^d}\Big(\eta(x)\overline\phi(f^*(x))+(1-\eta(x))\overline\phi(-f^*(x))\Big)\mr d P_X(x). 
\eeq For any $x\in[0,1]^d$, if $\eta(x)=1$, then we have 
\begin{align*}
&\eta(x)\overline\phi(f^*(x))+(1-\eta(x))\overline\phi(-f^*(x))=\overline\phi(f^*(x))=\overline\phi(+\infty)=0=H(\eta(x))=0\\&=1\cdot 0+(1-1)\cdot 0=\eta(x)\psi(x,1)+(1-\eta(x))\psi(x,-1)=\int_{\hkh{-1,1}}\psi(x,y)\mr{d}P(y|x);
\end{align*}  If $\eta(x)=0$, then we have 
\begin{align*}
&\eta(x)\overline\phi(f^*(x))+(1-\eta(x))\overline\phi(-f^*(x))=\overline\phi(-f^*(x))=\overline\phi(+\infty)=0=H(\eta(x))=0\\&=0\cdot 0+(1-0)\cdot 0=\eta(x)\psi(x,1)+(1-\eta(x))\psi(x,-1)=\int_{\hkh{-1,1}}\psi(x,y)\mr{d}P(y|x);
\end{align*} If $\eta(x)\in(0,\delta)\cup(1-\delta,1)$, then we have 
\begin{align*}
&\eta(x)\overline\phi\ykh{f^*(x)}+(1-\eta(x))\overline\phi(-f^*(x))\\&=\eta(x)\phi\ykh{\log\frac{\eta(x)}{1-\eta(x)}}+(1-\eta(x))\phi\ykh{-\log\frac{\eta(x)}{1-\eta(x)}}
\\&=\eta(x)\log\ykh{1+\frac{1-\eta(x)}{\eta(x)}}+(1-\eta(x))\log\ykh{1+\frac{\eta(x)}{1-\eta(x)}}\\&=\eta(x)\log\frac{1}{\eta(x)}+(1-\eta(x))\log\frac{1}{1-\eta(x)}\\
&=H(\eta(x))=\int_{\hkh{-1,1}}\ykh{\eta(x)\log\frac{1}{\eta(x)}+(1-\eta(x))\log\frac{1}{1-\eta(x)}}\mr{d}P(y|x)\\&=\int_{\hkh{-1,1}}\psi(x,y)\mr{d}P(y|x);
\end{align*} If $\eta(x)\in[\delta,1-\delta]$, then we have that
\begin{align*}
&\eta(x)\overline\phi\ykh{f^*(x)}+(1-\eta(x))\overline\phi(-f^*(x))\\&=\eta(x)\phi\ykh{\log\frac{\eta(x)}{1-\eta(x)}}+(1-\eta(x))\phi\ykh{-\log\frac{\eta(x)}{1-\eta(x)}}
\\&=\eta(x)\log\ykh{1+\frac{1-\eta(x)}{\eta(x)}}+(1-\eta(x))\log\ykh{1+\frac{\eta(x)}{1-\eta(x)}}\\&=\eta(x)\log\frac{1}{\eta(x)}+(1-\eta(x))\log\frac{1}{1-\eta(x)}\\
&=H(\eta(x))=\eta(x)\phi\ykh{\log\frac{\eta(x)}{1-\eta(x)}}+(1-\eta(x))\phi\ykh{-\log\frac{\eta(x)}{1-\eta(x)}}\\&=\eta(x)\psi(x,1)+(1-\eta(x))\psi(x,-1)=\int_{\hkh{-1,1}}\psi(x,y)\mr{d}P(y|x).
\end{align*} In conclusion, we always have that 
\[
\eta(x)\overline\phi(f^*(x))+(1-\eta(x))\overline\phi(-f^*(x))=H(\eta(x))=\int_{\hkh{-1,1}}\psi(x,y)\mr{d}P(y|x). 
\] Since $x$ is arbitrary, we deduce that 
\begin{align*}
&\int_{[0,1]^d}\Big(\eta(x)\overline\phi(f^*(x))+(1-\eta(x))\overline\phi(-f^*(x))\Big)\mr d P_X(x)=\int_{[0,1]^d}H(\eta(x))\mr{d}P_X(x)\\
&=\int_{[0,1]^d}\int_{\hkh{-1,1}}\psi(x,y)\mr{d}P(y|x)\mr{d}P_X(x)=\int_{[0,1]^d\times\hkh{-1,1}}\psi(x,y)\mr{d}P(x,y), 
\end{align*} which, together with \eqref{23022818}, proves the desired result. 
\end{proof}

\subsection{Proof of Theorem \ref{thm2.1}} \label{section: proof of thm2.1}
Appendix \ref{section: proof of thm2.1} is devoted to the proof of Theorem \ref{thm2.1}.

\begin{proof} [Proof of Theorem \ref{thm2.1}]\newcommand{\tema}{\psi(X_i,Y_i)}\newcommand{\temb}{\psi(X_i',Y_i')}\newcommand{\temc}{\psi(X',Y')} Throughout this proof, we denote \[\Psi:=\int_{[0,1]^d\times\{-1,1\}}\psi(x,y)\mr{d}P(x,y).\]  Then it follows from \eqref{ineq 2.10} and \eqref{ineq 2.11} that $0\leq \mc{R}_P^\phi\ykh{\hat{f}_n}-\Psi\leq 2M<\infty$. Let $\{(X'_k,Y_k')\}_{k=1}^n$ be an i.i.d. sample from distribution $P$ which is independent of $\hkh{(X_k,Y_k)}_{k=1}^n$. By independence, we have 
	\[
	\mbe\zkh{\mc{R}_P^\phi\ykh{\hat{f}_n}-\Psi}=\frac{1}{n}\sum_{i=1}^n\mbe\zkh{\phi\ykh{Y_i'\hat{f}_n(X_i')}-\psi\ykh{X_i',Y_i'}}
	\] with its empirical counterpart given by	
	\[\hat{R}:=\frac{1}{n}\sum_{i=1}^n\mbe\zkh{\phi\ykh{Y_i\hat{f}_n(X_i)}-\psi(X_i,Y_i)}.\] Then we have
	\newcommand{\hef}{\hat{e}\ykh{\hat{f}}}
\begin{align*}
	\hat{R}-\ykh{\mc{R}_P^\phi(g)-\Psi}&=\frac{1}{n}\sum_{i=1}^n\mbe\zkh{\phi\ykh{Y_i\hat{f}_n(X_i)}-\phi(Y_ig(X_i))}\\
	&=\mbe\zkh{\frac{1}{n}\sum_{i=1}^n\phi\ykh{Y_i\hat{f}_n(X_i)}-\frac{1}{n}\sum_{i=1}^n\phi\ykh{Y_ig(X_i)}}\leq 0, \;\forall\;g\in\mc{F},
	\end{align*}
	where the last inequality follows from the fact that $\hat{f}_n$ is an empirical $\phi$-risk minimizer which minimizes  $\frac{1}{n}\sum_{i=1}^n\phi\ykh{Y_ig(X_i)}$ over all $g\in \mc{F}$. Hence $\hat{R}\leq\inf_{g\in\mc{F}}\ykh{\mc{R}_P^\phi(g)-\Psi}$, which means that
	\beq\label{ineq 5.11}
	&\mbe\zkh{\mc{R}_P^\phi\ykh{\hat{f}_n}-\Psi}= \ykh{	\mbe\zkh{\mc{R}_P^\phi\ykh{\hat{f}_n}-\Psi}-(1+\e)\cdot\hat R}+(1+\e)\cdot\hat R\\
	&\leq \ykh{	\mbe\zkh{\mc{R}_P^\phi\ykh{\hat{f}_n}-\Psi}-(1+\e)\cdot\hat R}+(1+\e)\cdot \inf_{g\in\mc{F}}\ykh{\mc{R}_P^\phi(g)-\Psi},\;\forall\;\e\in[0,1). 	\eeq
	
	\newcommand{\eef}{\mbe\zkh{e\ykh{\hat{f}}}} 
	We then establish an upper bound for $	\mbe\zkh{\mc{R}_P^\phi\ykh{\hat{f}_n}-\Psi}-(1+\e)\cdot\hat R$ by using a similar argument to that in the proof of Lemma 4 of \cite{schmidt2020nonparametric}. The desired inequality (\ref{bound 2.10}) will follow from this bound and (\ref{ineq 5.11}). Recall that $W=\max\hkh{3,\;\mc{N}\ykh{\mc{F},\gamma}}$. From the definition of $W$, there exist $f_1,\cdots,f_W\in\mc{F}$ such that for any $f\in\mc{F}$, there exists some $j\in\hkh{1,\cdots,W}$, such that $\|f-f_j\|_\infty\leq \gamma$. Therefore, there holds $\norm{\hat{f}_n-f_{j^*}}_{{[0,1]^d}}\leq \gamma$ where $j^*$ is a $\{1,\cdots,W\}$-valued statistic from the sample $\hkh{(X_i,Y_i)}_{i=1}^n$. Denote \beq\label{defA}
	A:=M\cdot\sqrt{\frac{\log W}{\Gamma n}}.
	\eeq And for $j=1,2,\cdots, W$, let
	\beq\label{defhVr}
	&h_{j,1}:=\mc{R}_P^\phi(f_j)-\Psi,\\
	&h_{j,2}:=\int_{[0,1]^d\times\{-1,1\}}{\ykh{\phi(yf_j(x))-\psi(x,y)}^2\mr{d}P(x,y)},\\ &V_j:=\abs{\sum_{i=1}^n\ykh{\phi\ykh{Y_if_j(X_i)}-\psi\ykh{X_i,Y_i}-\phi\ykh{Y_i'f_j(X_i')}+\psi\ykh{X_i',Y_i'}}},\\
	&r_j:=A\qd\sqrt{h_{j,1}}.
	\eeq Then define \[T:=\max\limits_{j=1,\cdots,W}\frac{V_j}{r_j}.\] Denote by $\mbe\zkh{\left.\cdot \right|\ykh{X_i,Y_i}_{i=1}^n}$ the conditional expectation with respect to $\{\ykh{X_i,Y_i}\}_{i=1}^n$.
	Then  we have that 
\begin{align*}
	r_{j^*}&=A\qd\sqrt{h_{j^*,1}}
	\\&\leq 
	A+\sqrt{h_{j^*,1}}
	\\&=
	A+\sqrt{\mbe\zkh{\phi\ykh{Y' f_{j^*}(X')}-\temc|\ykh{X_i,Y_i}_{i=1}^n}}\\
	&\leq A+\sqrt{\gamma+\mbe\zkh{\left.\phi\ykh{Y'\hat {f}_n(X')}-\temc \right|\ykh{X_i,Y_i}_{i=1}^n}}\\
	&=A+\sqrt{\gamma +\mc{R}_P^\phi\ykh{\hat{f}_n}-\Psi}
	\\&\leq A+\sqrt{\gamma}+\sqrt{\mc{R}_P^\phi\ykh{\hat{f}_n}-\Psi},
	\end{align*}
	where $(X',Y')$ is an i.i.d. copy of $\ykh{X_i,Y_i} (1 \leq i \leq n)$ and the second inequality follows from
	\beq\label{ineq 36}
	\abs{\phi(t_1)-\phi(t_2)}\leq\abs{t_1-t_2},\;\forall\;t_1,t_2\in\mbR
	\eeq and $\norm{f_{j^*}-\hat{f}}_{{[0,1]^d}}\leq \gamma$. Consequently, 
	\beq\ba\label{ineq 5.12}
	&\mbe\zkh{\mc{R}_P^\phi\ykh{\hat{f}_n}-\Psi}-\hat R\leq\abs{\hat{R}-\mbe\zkh{\mc{R}_P^\phi\ykh{\hat{f}_n}-\Psi}}\\&=\frac{1}{n}\abs{\mbe\zkh{\sum_{i=1}^n\ykh{\phi\ykh{Y_i\hat {f}_n(X_i)}-\tema-\phi\ykh{Y_i'\hat{f}_n(X_i')}+\temb}}}\\
	&\leq \frac{1}{n}\mbe\zkh{\abs{\sum_{i=1}^n\ykh{\phi\ykh{Y_i f_{j^*}(X_i)}-\tema-\phi\ykh{Y_i'f_{j^*}(X_i')}+\temb}}}+2\gamma \\
	&= \frac{1}{n}\mbe\zkh{V_{j^*}}+2\gamma\leq \frac{1}{n}\mbe\zkh{T \cdot r_{j^*}}+2\gamma\\&\leq \frac{1}{n}\mbe\zkh{T\cdot\sqrt{\mc{R}_P^\phi\ykh{\hat{f}_n}-\Psi}}+\frac{A+\sqrt{\gamma}}{n}\cdot\mbe\zkh{T}+2\gamma\\
	&\leq
	\frac{1}{n}\sqrt{\mbe\zkh{T^2}}\cdot
	\sqrt{\mbe\zkh{\mc{R}_P^\phi\ykh{\hat{f}_n}-\Psi}}+\frac{A+\sqrt{\gamma}}{n}\cdot\mbe\zkh{T}+2\gamma\\
	&\leq \frac{\e \mbe\zkh{\mc{R}_P^\phi\ykh{\hat{f}_n}-\Psi}}{2+2\e}+\frac{{(1+\e)}{\mbe\zkh{T^2}}}{2\e\cdot n^2}+\frac{A+\sqrt{\gamma}}{n}\mbe\zkh{T}+2\gamma,\;\forall\;\e\in (0,1), 
	\ea\eeq where the last inequality follows from $2 \sqrt{ab}\leq \frac{\epsilon}{1+\epsilon} a + \frac{1+\epsilon}{\epsilon} b$, $\forall a>0,b>0$.   
	We then bound $\mbe\zkh{T}$ and $\mbe\zkh{T^2}$ by Bernstein's inequality (see e.g., Chapter 3.1 of \cite{cucker2007learning} and Chapter 6.2 of \cite{steinwart2008support}).  Indeed, it follows from (\ref{ineq 2.12}) and (\ref{defhVr}) that
	\[
	h_{j,2}\leq \Gamma \cdot h_{j,1}\leq \Gamma\cdot \ykh{r_j}^2,\;\forall\;j\in\{1,\cdots,W\}.
	\] For any $j\in\{1,\cdots,W\}$ and $t\geq 0$, we apply Bernstein's inequality to the zero mean i.i.d. random variables \[\hkh{\phi\ykh{Y_if_j(X_i)}-\tema-\phi\ykh{Y_i'f_j(X_i')}+\temb}_{i=1}^n\] and obtain
\begin{align*}
	&\mb{P}(V_j\geq t)\\
	&=\mb{P}\ykh{\abs{\sum_{i=1}^n\ykh{\phi\ykh{Y_if_j(X_i)}-\tema-\phi\ykh{Y_i'f_j(X_i')}+\temb}}\geq t}\\
	&\leq 2\exp\left(\frac{-t^2/2}{Mt+\sum_{i=1}^n\mbe\zkh{\ykh{\phi\ykh{Y_if_j(X_i)}-\tema-\phi\ykh{Y_i'f_j(X_i')}+\temb}^2}}\right)\\
	&\leq 2\exp\left(\frac{-t^2/2}{Mt+2\sum_{i=1}^n\mbe\zkh{\ykh{\phi\ykh{Y_if_j(X_i)}-\tema}^2+\ykh{\phi\ykh{Y_i'f_j(X_i')}-\temb}^2}}\right)\\
	&=2\exp\left(\frac{-t^2/2}{Mt+4\sum_{i=1}^nh_{j,2}}\right)=2\exp\left(\frac{-t^2}{2Mt+8nh_{j,2}}\right)\leq 2\exp\left(-\frac{t^2}{2Mt+8n\Gamma\cdot \ykh{r_j}^2}\right).
	\end{align*} Hence
\begin{align*}
	\mb{P}(T\geq t)&\leq \sum_{j=1}^W\mb{P}(V_j/r_j\geq t)=\sum_{j=1}^W\mb{P}(V_j\geq tr_j)\\
	&\leq 2\sum_{j=1}^W\exp\left(-\frac{(tr_j)^2}{2Mtr_j+8n\Gamma \cdot r_j^2}\right)=2\sum_{j=1}^W\exp\left(-\frac{t^2}{2Mt/r_j+8n\Gamma}\right)\\
	&\leq 2\sum_{j=1}^W\exp\left(-\frac{t^2}{2Mt/A+8n\Gamma}\right)=2W\exp\left(-\frac{t^2}{2Mt/A+8n\Gamma}\right),\;\forall\;t\in [0,\infty). 
	\end{align*}
	Therefore, for any $\theta\in \{1,2\}$, by taking 
	\[
	B:=\ykh{\frac{M}{A}\cdot\log W+\sqrt{\ykh{\frac{M}{A}\cdot\log W}^2+8n\Gamma\log W}}^{\theta} =4^{\theta}\cdot \ykh{n\Gamma\log W}^{\theta/2},\]
	we derive \label{23072603}
\begin{align*}
	\mbe\zkh{T^\theta}&=\int_{0}^\infty\mb{P}\ykh{T\geq t^{1/\theta}}\mr{d}t\leq B+\int_{B}^\infty\mb{P}\ykh{T\geq t^{1/\theta}}\mr{d}t\\
	&\leq B+\int_{B}^\infty\ykh{2W\exp\left(-\frac{t^{2/\theta}}{2Mt^{1/\theta}/A+8n\Gamma}}\right)\mr{d}t\\
	&\leq B+\int_{B}^\infty\ykh{2W\exp\left(-\frac{B^{1/\theta}\cdot t^{1/\theta}}{2MB^{1/\theta}/A+8n\Gamma}}\right)\mr{d}t\\
	&=B+2WB\theta \cdot\ykh{\log W}^{-\theta}\int_{\log W}^\infty\mr{e}^{-u}u^{\theta-1}\mr{d}u\\
	&\leq B+2WB\theta\cdot \ykh{\log W}^{-\theta}\cdot \theta\cdot\mr{e}^{-\log W}\ykh{\log W}^{\theta-1}\\
	&\leq 5\theta B\leq 5\theta\cdot 4^{\theta}\cdot \ykh{n\Gamma\log W}^{\theta/2}.  
\end{align*}
	Plugging the inequality above and (\ref{defA}) into (\ref{ineq 5.12}), we obtain
	\begin{align*}
	&\mbe\zkh{\mc{R}_P^\phi\ykh{\hat{f}_n}-\Psi}-\hat R\leq\abs{\hat{R}-\mbe\zkh{\mc{R}_P^\phi\ykh{\hat{f}_n}-\Psi}}\\
	&\leq \frac{\e \mbe\zkh{\mc{R}_P^\phi\ykh{\hat{f}_n}-\Psi}}{2+2\e}+\frac{{(1+\e)}{\mbe\zkh{T^2}}}{2\e\cdot n^2}+\frac{A+\sqrt{\gamma}}{n}\mbe\zkh{T}+2\gamma\\
	&\leq \frac{\e}{1+\e}\mbe\zkh{\mc{R}_P^\phi\ykh{\hat{f}_n}-\Psi}+20\cdot{\sqrt{\gamma}}\cdot \sqrt{\frac{\Gamma\log W}{n}}\\
	&\;\;\;\;+20M\cdot {\frac{\log W}{n}}+80\cdot {\frac{\Gamma\log W}{n}}\cdot{\frac{1+\e}{\e}}+2\gamma,\;\forall\;\e\in(0,1). 
	\end{align*} Multiplying    the above inequality by $(1+\e)$ and then  rearranging,  we obtain that
	\beq\label{ineq 5.13}
	& \mbe\zkh{\mc{R}_P^\phi\ykh{\hat{f}_n}-\Psi}-(1+\e)\cdot\hat R\leq 20\cdot(1+\e)\cdot{\sqrt{\gamma}}\cdot \sqrt{\frac{\Gamma\log W}{n}}\\
	&\;\;\;\;+20\cdot (1+\e) \cdot M\cdot {\frac{\log W}{n}}+80\cdot {\frac{\Gamma\log W}{n}}\cdot{\frac{(1+\e)^2}{\e}}+(2+2\e)\cdot\gamma,\;\forall\;\e\in(0,1). 
	\eeq Combining (\ref{ineq 5.13}) and (\ref{ineq 5.11}), we deduce that 
	\begin{align*}
	& \mbe\zkh{\mc{R}_P^\phi\ykh{\hat{f}_n}-\Psi}\leq (1+\e)\cdot \inf_{g\in\mc{F}}\ykh{\mc{R}_P^\phi(g)-\Psi}+ 20\cdot(1+\e)\cdot{\sqrt{\gamma}}\cdot \sqrt{\frac{\Gamma\log W}{n}}\\
	&\;\;\;\;+20\cdot (1+\e) \cdot M\cdot {\frac{\log W}{n}}+80\cdot {\frac{\Gamma\log W}{n}}\cdot{\frac{(1+\e)^2}{\e}}+(2+2\e)\cdot\gamma,\;\forall\;\e\in(0,1). 
	\end{align*} This proves the desired inequality (\ref{bound 2.10}) and completes the proof of Theorem \ref{thm2.1}.
\end{proof}

\subsection{Proof of Theorem \ref{thm2.3}}\label{section: proof of thm2.3}
To prove Theorem \ref{thm2.3}, we need the following Lemma \ref{lemma5.1} and Lemma \ref{lemma5.2}. 

Lemma \ref{lemma5.1}, which describes neural networks that approximate the multiplication operator,  can be derived directly  from Lemma A.2 of \cite{schmidt2020nonparametric}. Thus we omit its proof.  One can also find a similar result to Lemma \ref{lemma5.1} in the earlier paper \cite{yarotsky2017error} (cf. Proposition 3 therein). 

\begin{lem}\label{lemma5.1}
	For any $\e\in(0,1/2]$, there exists a neural network
	\[
	\mr{M} \in\fdnn_2\ykh{15\log\frac{1}{\e},6,900\log\frac{1}{\e},1,1}
	\] such that for any $t,t'\in[0,1]$, there hold $\mr{M}(t,t')\in[0,1]$, $\mr{M}(t,0)=\mr{M}(0,t')=0$ and 
	\[ 
	\abs{\mr{M}(t,t')-t \cdot t'}\leq \e.
	\]
\end{lem}

In Lemma \ref{lemma5.2}, we construct a neural network which performs the operation of multiplying the inputs by $2^k$. 

\begin{lem} \label{lemma5.2}
	Let $k$ be a positive integer and $f$ be a univariate function given by $f(x)=2^k\cdot\max\hkh{x,0}$. Then 
	\[
	f\in\fdnn_1\ykh{k,2,4k,1,\infty}.
	\]
\end{lem}
\begin{proof}
	For any $1 \leq i \leq k-1$, let $v_i=(0,0)^\top$ and
	\[
	{\bm W}_i=\begin{pmatrix}
	1&1\\
	1&1
	\end{pmatrix}.
	\]
	In addition, take
	\[
	{\bm W}_0=(1,1)^\top, {\bm W}_k=(1,1), \text{ and }v_k=(0,0)^\top.
	\]
	Then we have
	\[
	f=x\mapsto {\bm W}_k\sigma_{v_k}{\bm W}_{k-1}\sigma_{v_{k-1}}\cdots {\bm W}_1\sigma_{v_1}{\bm W}_0x\in\fdnn_1\ykh{k,2,4k,1,\infty},
	\]
	which proves this lemma.
\end{proof} Now we are in the position to prove Theorem \ref{thm2.3}.

\begin{proof}[Proof of Theorem \ref{thm2.3}]
	Given $a\in(0,1/2]$, let $I:=\ceil{-\log_2  a}$ and $J_k:=\zkh{\frac{1}{3\cdot 2^k},\frac{1}{2^k}}$ for $k=0,1,2,\cdots$. Then $1\leq I\leq 1-\log_2 a\leq 4\log\frac{1}{a}$. The idea of proof is to construct neural networks $\hkh{\tilde{h}_k}_k$ which satisfy that $0\leq \tilde{h}_k\ykh{t}\leq 1$ and $(8\log a)\cdot\tilde{h}_k$ approximates the natural logarithm function on $J_k$.  Then the function \[x\mapsto(8\log a)\cdot\sum_{k}\mr{M}\ykh{\tilde{h}_k(x),\tilde{f}_k(x)}\] is the desired neural network in Theorem \ref{thm2.3}, where $\mr{M}$ is the neural network that approximates multiplication operators given in Lemma \ref{lemma5.1} and $\{\tilde{f}_k\}_k$ are neural networks representing piecewise linear function supported on $J_k$ which constitutes a partition of unity. 
	
	Specifically, given $\alpha \in (0,\infty)$, there exists some $r_{\alpha}>0$ only depending on $\alpha$ such that
	\[
	x\mapsto\log\ykh{\frac{2x}{3}+\frac{1}{3}}\in \mc{B}^{\alpha}_{r_{\alpha}}\ykh{[0,1]}. 
	\]
	Hence it follows from Corollary \ref{corollaryA2} that there exists 
	\[\ba
	\tilde{g}_1&\in\fdnn_1\ykh{C_{\alpha}\log\frac{2}{\e}, C_{\alpha}\ykh{\frac{2}{\e}}^{1/\alpha}, C_{\alpha}\ykh{\frac{2}{\e}}^{1/\alpha}\log\frac{2}{\e},1,\infty}\\
	&\subset \fdnn_1\ykh{C_{\alpha}\log\frac{1}{\e}, C_{\alpha}\ykh{\frac{1}{\e}}^{1/\alpha}, C_{\alpha}\ykh{\frac{1}{\e}}^{1/\alpha}\log\frac{1}{\e},1,\infty}
	\ea\]
	such that
	\[
	\sup_{x\in[0,1]}\abs{\tilde{g}_1(x)-\log\ykh{\frac{2x}{3}+\frac{1}{3}}}\leq\e/2.
	\]
	Recall that the ReLU function is given by $\sigma(t)=\max\hkh{t,0}$. Let
	\[
	\tilde{g}_2:\mbR\to\mbR,\quad x\mapsto-\sigma\ykh{-\sigma\ykh{\tilde{g}_1(x)+\log 3}+\log 3}. 
	\]
	Then
	\beq\label{tildeg2}
	\tilde{g}_2\in \fdnn_1\ykh{C_{\alpha}\log\frac{1}{\e}, C_{\alpha}\ykh{\frac{1}{\e}}^{1/\alpha}, C_{\alpha}\ykh{\frac{1}{\e}}^{1/\alpha}\log\frac{1}{\e},1,\infty},
	\eeq
	and for $x\in \mathbb{R}$, there holds
	\[
	-\log 3 \leq \tilde{g}_2(x)=\fltl
	{
		&-\log 3,&&\text{ if }\tilde{g}_1(x)<-\log 3,\\
		&\tilde{g}_1(x),&&\text{ if }-\log 3\leq\tilde{g}_1(x)\leq 0,\\ 
		&0,&&\text{ if }\tilde{g}_1(x)>0.
	}
	\] Moreover, since $-\log 3\leq \log\ykh{\frac{2x}{3}+\frac{1}{3}}\leq 0$ whenever $x\in[0,1]$, we have
	\[
	\sup_{x\in[0,1]}\abs{\tilde{g}_2(x)-\log\ykh{\frac{2x}{3}+\frac{1}{3}}}\leq	\sup_{x\in[0,1]}\abs{\tilde{g}_1(x)-\log\ykh{\frac{2x}{3}+\frac{1}{3}}}\leq\e/2.
	\] Let $x=\frac{3 \cdot 2^k\cdot t-1}{2}$ in the above inequality, we obtain
	\beq\ba\label{ineq 5.21}
	\sup_{t\in J_k}\abs{\tilde{g}_2\ykh{\frac{3 \cdot 2^k\cdot t-1}{2}}-k\log 2-\log t}\leq \e/2,\;\quad \forall\;k = 0,1,2,\cdots.
	\ea\eeq
	
	For any $0 \leq k\leq I$, define
	\[
	\tilde{h}_k:\mbR \to \mbR,\quad t\mapsto\sigma\ykh{\frac{\sigma\ykh{-\tilde{g}_2\ykh{\sigma\ykh{\frac{3}{4\cdot 2^{I-k}} \cdot 2^{I+1}\cdot \sigma(t)-\frac{1}{2}}}}}{8\log\frac{1}{a}}+\frac{k\log 2}{8\log\frac{1}{a}}}.
	\]
	Then we have
	\beq\ba\label{ineq 5.22}
	0&\leq{\tilde{h}_k(t)}\leq \abs{\frac{\sigma\ykh{-\tilde{g}_2\ykh{\sigma\ykh{\frac{3}{4\cdot 2^{I-k}} \cdot 2^{I+1}\cdot \sigma(t)-\frac{1}{2}}}}}{8\log\frac{1}{a}}+\frac{k\log 2}{8\log\frac{1}{a}}}\\
	&\leq\abs{\frac{-\tilde{g}_2\ykh{\sigma\ykh{\frac{3}{4\cdot 2^{I-k}} \cdot 2^{I+1}\cdot \sigma(t)-\frac{1}{2}}}}{8\log\frac{1}{a}}}+\frac{k\log 2}{8\log\frac{1}{a}}\\
	&\leq \frac{\sup_{x\in\mbR}\abs{\tilde{g}_2(x)}}{8\log\frac{1}{a}}+\frac{I}{8\log\frac{1}{a}}\leq\frac{\log 3+4\log\frac{1}{a}}{8\log\frac{1}{a}}\leq 1,\; \forall\;t\in\mbR.
	\ea\eeq

	\begin{figure}[h]
		\centering
		\betikz
		\tikzset{ %
			mypt/.style ={
				circle, %
				minimum width =0pt, %
				minimum height =0pt, %
				inner sep=0pt, %
				draw=blue, %
			}
		}
		\tikzset{
			mybox/.style ={
				rectangle, %
				rounded corners =5pt, %
				minimum width =30pt, %
				minimum height =30pt, %
				inner sep=5pt, %
				draw=blue, %
				fill=cyan
		}}
		\tikzset{
			bigbox/.style ={
				rectangle, %
				rounded corners =5pt, %
				minimum width =377pt, %
				minimum height =309pt, %
				inner sep=0pt, %
				draw=black, %
				fill=none
		}}
		\tikzset{
			tinycircle/.style ={
				circle, %
				minimum width =6pt, %
				minimum height =6pt, %
				inner sep=0pt, %
				draw=red, %
				fill=red %
			}
		}
		
		\node[mybox] at (0,2) {$\ba& I+1\;\textrm{layers sub-network equipped with the architecture described}\\&\textrm{in Lemma {\@refstar{lemma5.2}} and representing the function $t\mapsto 2^{ I+1}\sigma(t)$ }\ea$};
		\node[tinycircle] (1) at(0,0) {};
		\node[right] at (0,0) {$\;\;t\in\mbR$};
		\node[below] at (0,-0.3) {Input};
		\node[tinycircle] (3) at (0,4) {};
		\node[mypt] (21) at (-2,1.33) {};
		\node[mypt] (22) at (2,1.33) {};
		\filldraw[->, blue] (1)--(21);
		\filldraw[->, blue] (1)--(22);
		\node[mypt] (23) at (-2,2.67) {};
		\node[mypt] (24) at (2,2.67) {};
		\filldraw[->, blue] (23)--(3);
		\filldraw[->, blue] (24)--(3);
		\node[right] at (0,4) {$\;\;2^{ I+1}\sigma(t)$};
		\node[tinycircle] (4) at (0,4.8) {};
		\filldraw[->, blue] (3)--(4);
		\node[right] at (0,4.8) {$\;\;\sigma\ykh{\frac{3}{4\cdot 2^{ I-k}}\cdot 2^{ I+1}\sigma(t)-\frac{1}{2}}$};
		\node[mybox] (5) at (0,5.8) {$-\tilde{g}_2$};
		\node[below] (51) at (0,5.55){};
		\node[above] (52) at (0,6.05){};
		\filldraw[->, blue] (4)--(51);
		\node[tinycircle] (6) at (0,6.8){};
		\node[right] at(0,6.8) {$\;\;\sigma\ykh{-\tilde{g}_2\ykh{\sigma\ykh{\frac{3}{4\cdot 2^{I-k}}\cdot 2^{I+1}\sigma(t)-\frac{1}{2}}}}$};
		\node[tinycircle] (7) at (0,7.7){};
		\node[tinycircle] (8) at (0,8.9){};
		\node[right] at (0,7.7) {
		$\;\;\sigma\ykh{\frac{\sigma\ykh{-\tilde{g}_2\ykh{\sigma\ykh{\frac{3}{4\cdot 2^{ I-k}}\cdot 2^{ I+1}\sigma(t)-\frac{1}{2}}}}}{8\log\frac{1}{a}}+\frac{k\log 2}{8\log\frac{1}{a}}}$};
		\filldraw[->, blue] (52)--(6);
		\filldraw[->, blue] (6)--(7);
		\filldraw[->, blue] (7)--(8);
		\node[above] at (0,9.3) {Output};
		\node[right] at (0,8.9) {$\;\;\tilde{h}_k(t)$};
		\node[bigbox] at (0.7,4.6) {};
		\eetikz	
		\caption{Networks representing functions $\tilde{h}_k$. }
			\label{fig1}
	\end{figure}
	\noindent Therefore, it follows from (\ref{tildeg2}), the definition of $\tilde{h}_k$, and Lemma \ref{lemma5.2} that (cf. Figure \ref{fig1})
	\beq\ba\label{ineq 5.23}
	\tilde{h}_k &\in \fdnn_1\ykh{C_{\alpha}\log\frac{1}{\e}+I, C_{\alpha}\ykh{\frac{1}{\e}}^{\frac{1}{\alpha}}, C_{\alpha}\ykh{\frac{1}{\e}}^{\frac{1}{\alpha}}\log\frac{1}{\e}+4I,1,1}\\
	&\subset\fdnn_1\ykh{C_{\alpha}\log\frac{1}{\e}+4\log\frac{1}{a}, C_{\alpha}\ykh{\frac{1}{\e}}^{\frac{1}{\alpha}}, C_{\alpha}\ykh{\frac{1}{\e}}^{\frac{1}{\alpha}}\log\frac{1}{\e}+16\log\frac{1}{a},1,1}
	\ea\eeq for all $0 \leq k \leq I$.   Besides, according to (\ref{ineq 5.21}), it is easy to verify that for $0 \leq k \leq I$, there holds
	\[
	\abs{\ykh{8\log{a}}\cdot\tilde{h}_k(t)-\log t}=\abs{\tilde{g}_2\ykh{{\frac{3}{2} \cdot 2^{k}\cdot t-1/2}}-k\log 2-\log t} \leq \e/2,\;\quad \forall\;t\in J_k.
	\]
	
	Define	
	\[\ba
	\tilde{f}_0:\mbR\to [0,1],\quad x\mapsto\fltl
	{
		&0,&&\text{ if }x\in(-\infty,1/3),\\
		&6\cdot\ykh{x-\frac{1}{3}},&&\text{ if }x\in[1/3,1/2],\\
		&1, &&\text{ if }x\in (1/2,\infty),\\
	}
	\ea\] and for $k \in \mathbb{N}$,
	\[\ba
	\tilde{f}_k:\mbR\to [0,1],\quad
	x \mapsto\fltl
	{
		&0,&&\text{ if }x\in\mbR\setminus J_k,\\
		& 6\cdot 2^k\cdot\ykh{x-\frac{1}{3\cdot 2^k}},&&\text{ if }x\in\left[\frac{1}{3\cdot 2^k},\frac{1}{2^{k+1}}\right),\\
		&1,&&\text{ if }x\in\zkh{\frac{1}{2^{k+1}},\frac{1}{3\cdot 2^{k-1}}},\\
		&-3\cdot 2^k\cdot\ykh{x-\frac{1}{2^k}},&&\text{ if }x\in\left(\frac{1}{3\cdot 2^{k-1}},\frac{1}{2^k}\right].\\
	} \ea\]

	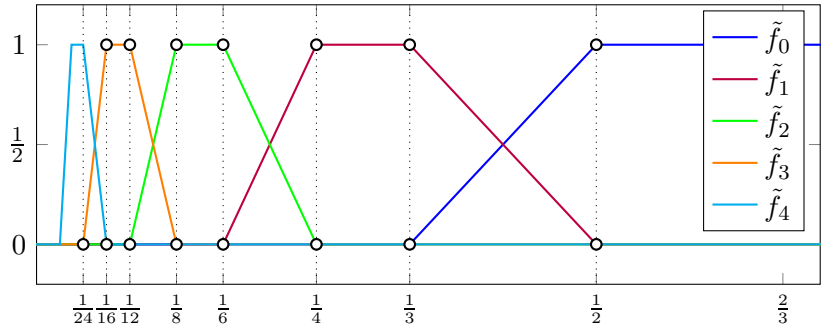
\begin{figure}[htbp]
		\centering
		\betikz
		\pgfplotsset{width=12cm, height=5.3cm}
		\begin{axis}[xmin=0,ymin=-0.2,xmax=0.7, ymax=1.2,xtick={1/24,1/16,1/12,1/8,1/6,1/4,1/3,1/2,2/3},ytick={0,0.5,1}, yticklabels={$0$,$\frac{1}{2}$,$1$}, xticklabels={{\fontsize{9}{11}\selectfont$\frac{1}{24}$},{\fontsize{9}{11}\selectfont
		{$\frac{1}{16}$}},{\fontsize{9}{11}\selectfont$\frac{1}{12}$}, {\fontsize{9}{11}\selectfont$\frac{1}{8}$}, {\fontsize{9}{11}\selectfont$\frac{1}{6}$},{\fontsize{9}{11}\selectfont$\frac{1}{4}$},{\fontsize{9}{11}\selectfont$\frac{1}{3}$},{\fontsize{9}{11}\selectfont$\frac{1}{2}$},{\fontsize{9}{11}\selectfont$\frac{2}{3}$}}],enlarge y limits={upper,value=0.4}]
		
		\addplot[color=black,solid,thick,mark=*, mark options={fill=white},forget plot, only marks]  %
		coordinates {
			(1/3,0)
			(1/2,1)
			(1/6,0)
			(1/4,1)
			(1/3,1)
			(1/2,0)
			(1/4,0)
			(1/6,1)
			(1/12,1)
			(1/12,0)
			(1/8,0)
			(1/8,1)
			(1/16,1)
			(1/16,0)
			(1/24,0)
		}; %
		
		\addplot[color=blue,solid,thick,mark=none, mark options={fill=white}] 
		coordinates {
			(-1,0)
			(1/3,0)
			(1/2,1)
			(2,1)
		}; %
		
		\addlegendentry{$\tilde{f}_0$};
		
		\addplot[color=purple,solid,thick,mark=none, mark options={fill=white}] 
		coordinates {
			(-1,0)
			(1/6,0)
			(1/4,1)
			(1/3,1)
			(1/2,0)
			(2,0)
		}; %
		
		\addlegendentry{$\tilde{f}_1$};
		
		\addplot[color=green,solid,thick,mark=none, mark options={fill=white}] 
		coordinates {
			(-1,0)
			(1/12,0)
			(1/8,1)
			(1/6,1)
			(1/4,0)
			(2,0)
		}; %
		
		\addlegendentry{$\tilde{f}_2$};
		
		\addplot[color=orange,solid,thick,mark=none, mark options={fill=white}] 
		coordinates {
			(-1,0)
			(1/24,0)
			(1/16,1)
			(1/12,1)
			(1/8,0)
			(2,0)
		}; %
		
		\addlegendentry{$\tilde{f}_3$};
		
		\addplot[color=cyan,solid,thick,mark=none, mark options={fill=white}] 
		coordinates {
			(-1,0)
			(1/48,0)
			(1/32,1)
			(1/24,1)
			(1/16,0)
			(2,0)
		}; %
		
		\addlegendentry{$\tilde{f}_4$};
		
		\newcommand{\dashplot}[1]{\addplot[color=black,dotted,thin,mark=none, mark options={fill=white},forget plot] 
			coordinates {
				(#1,-2)
				(#1,2)
		}} %
		
		\dashplot{1/3};
		\dashplot{1/2};
		\dashplot{1/4};
		\dashplot{1/6};
		\dashplot{1/8};
		\dashplot{1/12};
		\dashplot{1/16};
		\dashplot{1/24};

		\end{axis}
		\eetikz
		\caption{Graphs of functions $\tilde{f}_k$.}
	\end{figure}
	
	\noindent Then it is easy to show that for any $x\in\mbR$ and $k\in \mathbb{N}$, there hold
	\[\ba
	\tilde{f}_k(x)&=\frac{6}{2^{I-k+3}}\cdot 2^{I+3}\cdot\sigma\ykh{x-\frac{1}{3\cdot 2^k}}-\frac{6}{2^{I-k+3}}\cdot 2^{I+3}\cdot\sigma\ykh{x-\frac{1}{2^{k+1}}}\\
	&\;\;\;\;+\frac{6}{2^{I-k+4}}\cdot 2^{I+3}\cdot\sigma\ykh{x-\frac{1}{2^k}}-\frac{6}{2^{I-k+3}}\cdot 2^{I+3}\cdot\sigma\ykh{x-\frac{1}{3\cdot 2^{k-1}}},
	\ea\]
	and
	\[\ba
	\tilde{f}_0(x)=\frac{6}{2^{I+3}}\cdot 2^{I+3}\cdot\sigma(x-1/3)-\frac{6}{2^{I+3}}\cdot 2^{I+3}\cdot\sigma(x-1/2).
	\ea\]
	Hence it follows from Lemma \ref{lemma5.2} that (cf. Figure \ref{fig3})
	\beq\label{ineq 5.24n}\ba
	\tilde{f}_k&\in\fdnn_1\ykh{I+5,8,16I+60,1,\infty}\\
	&\subset\fdnn_1\ykh{12\log\frac{1}{a},8,152\log\frac{1}{a},1,\infty},\;\quad \forall\;0 \leq k \leq I.
	\ea\eeq

	\begin{figure}[htbp]
		\centering
		\betikz
		\tikzset{ %
			mypt/.style ={
				circle, %
				minimum width =0pt, %
				minimum height =0pt, %
				inner sep=0pt, %
				draw=blue, %
			}
		}
		\tikzset{
			mybox/.style ={
				rectangle, %
				rounded corners =5pt, %
				minimum width =40pt, %
				minimum height =100pt, %
				inner sep=5pt, %
				draw=blue, %
				fill=cyan
		}}
		\tikzset{
			bigbox/.style ={
				rectangle, %
				rounded corners =5pt, %
				minimum width =359pt, %
				minimum height =276pt, %
				inner sep=0pt, %
				draw=black, %
				fill=none
		}}
		\tikzset{
			tinycircle/.style ={
				circle, %
				minimum width =6pt, %
				minimum height =6pt, %
				inner sep=0pt, %
				draw=red, %
				fill=red %
			}
		}
		
		\node[mybox] (1) at(0,-1) {\fontsize{9}{9}\selectfont{\begin{minipage}{0.125\textwidth} 
		\centering
		$I+3$ layers sub-network \\ equipped \\with the\\  architecture described in \\   Lemma {\@refstar{lemma5.2}} \\ and\\ representing\\ the function\\ $t\mapsto2^{I+3}\sigma(t)$
		\end{minipage}}};
		\node[mybox] (2) at(3,-1) {\fontsize{9}{9}\selectfont{\begin{minipage}{0.125\textwidth}
				\centering
				$I+3$ layers sub-network \\ equipped \\with the\\  architecture described in \\   Lemma \@refstar{lemma5.2} \\ and\\ representing\\ the function\\ $t\mapsto2^{I+3}\sigma(t)$
				\end{minipage}}};
		\node[mybox] (3) at(6,-1) {\fontsize{9}{9}\selectfont{\begin{minipage}{0.125\textwidth}
				\centering
				$I+3$ layers sub-network \\ equipped \\with the\\  architecture described in \\   Lemma \@refstar{lemma5.2} \\ and\\ representing\\ the function\\ $t\mapsto2^{I+3}\sigma(t)$
				\end{minipage}}};
		\node[mybox] (4) at(9,-1) {\fontsize{9}{9}\selectfont{\begin{minipage}{0.125\textwidth} 
				\centering
				$I+3$ layers sub-network \\ equipped \\with the\\  architecture described in \\   Lemma \@refstar{lemma5.2} \\ and\\ representing\\ the function\\ $t\mapsto2^{I+3}\sigma(t)$
				\end{minipage}}};
		\node[tinycircle] (x0) at(4.5,-5.5) {};
		\node[tinycircle]  (x1) at(0,-3.5) {};
		\node[tinycircle]  (x2) at(3,-3.5) {};
		\node[tinycircle]  (x3) at(6,-3.5) {};
		\node[tinycircle]  (x4) at(9,-3.5) {};
		\node[tinycircle]  (s1) at(0,1.5) {};
		\node[tinycircle]  (s2) at(3,1.5) {};
		\node[tinycircle] (s3) at(6,1.5) {};
		\node[tinycircle] (s4)at(9,1.5) {};
		\node[tinycircle] (s0) at(4.5,2.7) {};
		\node[above] at (4.5,-6.15) {Input};
		\node[below] at (4.5,3.7) {Output};
		\node[right] at (4.5,-5.5) {$\;x\in\mbR$};
		\node[right] at (0,-3.5) {$\;\sigma\ykh{x-\frac{1}{3\cdot{2^k}}}$};
		\node[right] at (3,-3.5) {$\;\sigma\ykh{x-\frac{1}{{2^{k+1}}}}$};
		\node[right] at (6,-3.5) {$\;\sigma\ykh{x-\frac{1}{{2^{k}}}}$};
		\node[right] at (9,-3.5) {$\;\sigma\ykh{x-\frac{2}{3\cdot{2^{k}}}}$};
		\filldraw[->, blue] (x0)--(x1);
		\filldraw[->, blue] (x0)--(x2);
		\filldraw[->, blue] (x0)--(x3);
		\filldraw[->, blue] (x0)--(x4);
		\filldraw[->, blue] (s1)--(s0);
		\filldraw[->, blue] (s2)--(s0);
		\filldraw[->, blue] (s4)--(s0);
		\filldraw[->, blue] (s3)--(s0);
		\node[mypt] (zx1) at (-0.5,-2.89) {};
		\node[mypt] (zx3) at (2.5,-2.89) {};
		\node[mypt] (zx4) at (3.5,-2.89) {};
		\filldraw[->, blue] (x1)--(zx1); 
		\filldraw[->, blue] (x2)--(zx3);
		\node[mypt] (zx2) at (0.5,-2.89) {};
		\filldraw[->, blue] (x1)--(zx2);
		\filldraw[->, blue] (x2)--(zx4);
		\node[mypt] (zx5) at (5.5,-2.89) {};
		\node[mypt] (zx6) at (6.5,-2.89) {};
		\node[mypt] (zx7) at (8.5,-2.89) {};
		\node[mypt] (zx8) at (9.5,-2.89) {};
		\filldraw[->, blue] (x3)--(zx5);
		\filldraw[->, blue] (x3)--(zx6);
		\filldraw[->, blue] (x4)--(zx7);
		\filldraw[->, blue] (x4)--(zx8);
		
		\node[mypt] (zs1) at (-0.5,0.9) {};
		\node[mypt] (zs3) at (2.5,0.9) {};
		\node[mypt] (zs4) at (3.5,0.9) {};
		\node[mypt] (zs2) at (0.5,0.9) {};
		\node[mypt] (zs5) at (5.5,0.9) {};
		\node[mypt] (zs6) at (6.5,0.9) {};
		\node[mypt] (zs7) at (8.5,0.9) {};
		\node[mypt] (zs8) at (9.5,0.9) {};
		\filldraw[->, blue] (zs1)--(s1); 
		\filldraw[->, blue] (zs2)--(s1);
		\filldraw[->, blue] (zs3)--(s2);
		\filldraw[->, blue] (zs4)--(s2);
		\filldraw[->, blue] (zs5)--(s3);
		\filldraw[->, blue] (zs6)--(s3);
		\filldraw[->, blue] (zs7)--(s4);
		\filldraw[->, blue] (zs8)--(s4);
		\node[right] at (-0.05,1.64) {%
		{\fontsize{9}{11}\selectfont
		$2^{{I}+3}\sigma\ykh{x-\frac{1}{3\cdot 2^k}}$}};
		\node[right] at (2.95,1.64) {%
		{\fontsize{9}{11}\selectfont
		$2^{ I+3}\sigma\ykh{x-\frac{1}{ 2^{k+1}}}$}};
		\node[right] at (5.95,1.64) {%
		{\fontsize{9}{11}\selectfont
		$2^{ I+3}\sigma\ykh{x-\frac{1}{2^k}}$}};
		\node[right] at (8.95,1.64) {%
		{\fontsize{9}{11}\selectfont
		$2^{ I+3}\sigma\ykh{x-\frac{2}{3\cdot 2^k}}$}};
		\node[right] at (4.5,2.9) {$\;\tilde{f}_k(x)$};
		\node[bigbox] at (5.1,-1.2) {};
		\eetikz
		\caption{Networks representing functions $\tilde{f}_k$.}
		\label{fig3}
	\end{figure}

	Next, we show that
	\beq\label{ineq 5.24}\ba
	\sup_{t\in[a,1]}\abs{\log(t)+8\log\ykh{\frac{1}{a}}\sum_{k=0}^I\tilde{h}_k(t)\tilde{f}_k(t)}\leq\e/2. 
	\ea\eeq
	Indeed, we have the following inequalities:
	\beq\label{ineq 5.25}\ba
	\abs{\log(t)+8\log\ykh{\frac{1}{a}}\sum_{k=0}^{I}\tilde{h}_k(t)\tilde{f}_k(t)}&=\abs{\log t+8\log\ykh{\frac{1}{a}}\tilde{h}_0(t)\tilde{f}_0(t)}\\
	&=\abs{\log t+8\log\ykh{\frac{1}{a}}\tilde{h}_0(t)}\leq \e/2,\;\forall\;t\in[1/2,1];
	\ea\eeq
	\beq\label{ineq 5.27}\ba
	\abs{\log(t)+8\log\ykh{\frac{1}{a}}\sum_{k=0}^{I}\tilde{h}_k(t)\tilde{f}_k(t)}
	&=\abs{\log(t)+8\log\ykh{\frac{1}{a}}\tilde{h}_{m-1}(t)}
	\leq\e/2,\\
	& \forall t\in\zkh{\frac{1}{2^m},\frac{1}{3\cdot 2^{m-2}}}\cap [a,1]
	\text{ with } 2 \leq m \leq I;
	\ea\eeq
	and
	\beq\label{ineq 5.26}\ba
	&\abs{\log(t)+8\log\ykh{\frac{1}{a}}\sum_{k=0}^{I}\tilde{h}_k(t)\tilde{f}_k(t)}\\
	&=\abs{\log\ykh{t}(\tilde{f}_m(t)+\tilde{f}_{m-1}(t))-8\log\ykh{{a}}\ykh{\tilde{h}_{m}(t)\tilde{f}_m(t)+\tilde{h}_{m-1}(t)\tilde{f}_{m-1}(t)}}\\
	&\leq \tilde{f}_m(t)\abs{\log(t)-8\log\ykh{{a}}\tilde{h}_m(t)}+\tilde{f}_{m-1}(t)\abs{\log(t)-8\log\ykh{a}\tilde{h}_{m-1}(t)}\\
	&\leq \tilde{f}_m(t)\cdot\frac{\e}{2}+\tilde{f}_{m-1}(t)\cdot\frac{\e}{2}=\frac{\e}{2},\;\quad \forall \;t\in\zkh{\frac{1}{3\cdot 2^{m-1}},\frac{1}{2^m}}\cap [a,1] \text{ with } 1 \leq m \leq I.
	\ea\eeq
	Note that
	\[
	[a,1]\subset[1/2,1]\cup\ykh{\bigcup_{m=1}^{I}\zkh{\frac{1}{3\cdot 2^{m-1}},\frac{1}{2^m}}}\cup\ykh{\bigcup_{m=2}^{I}\zkh{\frac{1}{2^m},\frac{1}{3\cdot 2^{m-2}}}}.
	\]
	Consequently, (\ref{ineq 5.24}) follows immediately from (\ref{ineq 5.25}), (\ref{ineq 5.27}) and (\ref{ineq 5.26}). 
	
	From Lemma \ref{lemma5.1} we know that there exists 
	\beq\label{multiply}
	\mr{M}\in\fdnn_2\ykh{15\log\frac{96\ykh{\log a}^2}{\e},6,900\log\frac{96\ykh{\log a}^2}{\e},1,1}
	\eeq
	such that for any $t,t'\in[0,1]$, there hold $\mr{M}(t,t')\in[0,1]$, $\mr{M}(t,0)=\mr{M}(0,t')=0$ and 
	\beq\label{ineq 5.28}
	\abs{\mr{M}(t,t')-t\cdot t'}\leq\frac{\e}{96\ykh{\log a}^2}.
	\eeq Define
	\[\ba
	\tilde{g}_3:\mbR \to\mbR, \quad x\mapsto \sum_{k=0}^{I}\mr{M}\ykh{\tilde{h}_k(x),\tilde{f}_k(x)},
	\ea\] and
	\[\ba
	\tilde{f}:\mbR &\to\mbR,\\ x&\mapsto \sum_{k=1}^{8I}\zkh{\frac{\log (a)}{I}\cdot\sigma\ykh{\frac{\log b}{8\log a}+\sigma\ykh{\sigma\ykh{\tilde{g}_3(x)}-\frac{\log b}{8\log a}}-\sigma\ykh{\sigma\ykh{\tilde{g}_3(x)}-\frac{1}{8}}}}.
	\ea\] \begin{figure}[H]
		\centering
		\betikz
		\tikzset{
			mybox/.style ={
				rectangle, %
				rounded corners =5pt, %
				minimum width =10pt, %
				minimum height =70pt, %
				inner sep=0.6pt, %
				draw=blue, %
				fill=cyan
		}}
		\tikzset{
			ttinybox/.style ={
				rectangle, %
				rounded corners =3pt, %
				minimum width =20pt, %
				minimum height =20pt, %
				inner sep=2pt, %
				draw=blue, %
				fill=cyan
		}}
		\tikzset{
			tinybox/.style ={
				rectangle, %
				rounded corners =5pt, %
				minimum width =20pt, %
				minimum height =40pt, %
				inner sep=5pt, %
				draw=blue, %
				fill=cyan
		}}
		\tikzset{
			bigbox/.style ={
				rectangle, %
				rounded corners =5pt, %
				minimum width =277pt, %
				minimum height =279pt, %
				inner sep=0pt, %
				draw=black, %
				fill=none
		}}
		\tikzset{
			tinycircle/.style ={
				circle, %
				minimum width =3pt, %
				minimum height =3pt, %
				inner sep=0pt, %
				draw=red, %
				fill=red %
			}
		}
		\node[tinycircle] (1) at (0,0) {};
		\node[right] at (0,0){$\;\;x\in\mbR$};
		\node[below] at (0,-0.3) {Input};
		\node[mybox] at (-4,2) {$\tilde{h}_0$};
		\node[tinybox] at (-3,1.48) {${\tilde f}_0$};
		\node[mybox] at (-2,2) {$\tilde{h}_1$};
		\node[tinybox] at (-1,1.48) {${\tilde f}_1$};
		\node[mybox] at (3,2) {$\tilde{h}_{I}$};
		\node[tinybox] at (4,1.48) {${\tilde f}_{ I}$};
		\node[below] at (0.5,1.7) {$\cdots$};%
		\node[below] at (1.5,1.7) {$\cdots$};%
		\node[below] at (0.8,2.2) {$\cdots$};
		\node[below] at (1.8,2.2) {$\cdots$};
		\node[below] at (0.8,2.7) {$\cdots$};
		\node[below] at (1.8,2.7) {$\cdots$};
		\node[below] at (0.8,4.2) {$\cdots$};
		\node[below] at (1.8,4.2) {$\cdots$};
		\node[below] at (0.8,3.7) {$\cdots$};
		\node[below] at (1.8,3.7) {$\cdots$};
		\node[below] at (0.8,5.0) {$\cdots$};
		\node[below] at (1.8,5.0) {$\cdots$};
		\node[below] at (1.8,5.9) {$\cdots$};
		\node[below] at (0.8,5.9) {$\cdots$};
		\node[below] at (-0.2, 5.9) {$\cdots$};
		\node[below] at (2.1,6.6) {$\cdots$};
		\node[below] at (0.9,7.3) {$\cdots$};
		\node[below] at (-0.1,7.3) {$\cdots$};

		\node[below]  (11) at (-4.2,0.93) {};
		\filldraw[->,blue] (1)--(11);
		\node[below]  (12) at (-3.1,0.93) {};
		\filldraw[->,blue] (1)--(12);
		\node[below]  (13) at (-2.2,0.93) {};
		\filldraw[->,blue] (1)--(13);
		\node[below]  (14) at (-1.15,1.03) {};
		\filldraw[->,blue] (1)--(14);
		\node[below]  (15) at (3.15,0.93) {};
		\filldraw[->,blue] (1)--(15);
		\node[below]  (16) at (4.2,0.93) {};
		\filldraw[->,blue] (1)--(16);
		\node[below] at (-0.1,0.7) {$\cdots$};%
		\node[below] at (0.9,0.7) {$\cdots$};%
		\node[below] at (0.2,1.2) {$\cdots$};%
		\node[below] at (1.2,1.2) {$\cdots$};%
		\node[tinycircle] (21) at (-4,3.6) {};
		\node[tinycircle] (31) at (-4,4.1) {};
		\node[tinycircle] (41)at (-4,4.8) {};
		\node[tinycircle]  (22)at (-2,3.6) {};
		\node[tinycircle] (32)at (-2,4.1) {};
		\node[tinycircle] (42)at (-2,4.8) {};
		\node[tinycircle]  (23)at (3,3.6) {};
		\node[tinycircle] (33)at (3,4.1) {};
		\node[tinycircle] (43)at (3,4.8) {};
		
		\node[tinycircle] (24)at (-3,2.6) {};
		\node[tinycircle] (34)at (-3,3.1) {};
		\node[tinycircle] (44)at (-3,4.8) {};
		
		\node[tinycircle] (25)at (-1,2.6) {};
		\node[tinycircle] (35)at (-1,3.1) {};
		\node[tinycircle] (45)at (-1,4.8) {};
		
		\node[tinycircle] (26)at (4,2.6) {};
		\node[tinycircle] (36)at (4,3.1) {};
		\node[tinycircle] (46)at (4,4.8) {};
		
		\node (z1) at(-4,3.1) {};
		\filldraw[->,blue] (z1)-- (21);
		\filldraw[->,blue] (21)-- (31);
		\node[below] at (-4,4.95) {$\vdots$};
		
		\node (z2) at(-2,3.1) {};
		\filldraw[->,blue] (z2)-- (22);
		\filldraw[->,blue] (22)-- (32);
		\node[below] at (-2,4.95) {$\vdots$};
		
		\node (z3) at(3,3.1) {};
		\filldraw[->,blue] (z3)-- (23);
		\filldraw[->,blue] (23)-- (33);
		\node[below] at (3,4.95) {$\vdots$};
		
		\node (z4) at(-3,2.03) {};
		\filldraw[->,blue] (z4)-- (24);
		\filldraw[->,blue] (24)-- (34);
		\node[below] at (-3,4.1) {$\vdots$};
		\node[below] at (-3,4.75) {$\vdots$};
		
		\node (z5) at(-1,2.03) {};
		\filldraw[->,blue] (z5)-- (25);
		\filldraw[->,blue] (25)-- (35);
		\node[below] at (-1,4.1) {$\vdots$};
		\node[below] at (-1,4.75) {$\vdots$};
		
		\node (z6) at(4,2.03) {};
		\filldraw[->,blue] (z6)-- (26);
		\filldraw[->,blue] (26)-- (36);
		\node[below] at (4,4.1) {$\vdots$};
		\node[below] at (4,4.75) {$\vdots$};
		
	\node[tinycircle] (6) at  (0,7.9) {};
		\node[right] at (0,8.0) {$\;\tilde{g}_3(x)$};
		\node[above] at (0,8.3) {Output};
		
		\node[ttinybox] at (-3.5,5.7) {$\mr{M}$};		
		
				\node[below] (51) at (-3.6,5.64) {};
				\node[below]  (52)at (-3.4,5.64) {};
				\filldraw[->,blue] (41)--(51);
				\filldraw[->,blue] (44)--(52);
				\node[above] (61) at (-3.5,5.77) {};
				\node[tinycircle](71) at (-3.5, 6.5) {};
				\filldraw[->,blue] (61)--(71);
				\filldraw[->,blue] (71)--(6);

\node[ttinybox] at (-1.5,5.7) {$\mr{M}$};
		\node[below] (53) at (-1.6,5.64) {};
		\node[below]  (54)at (-1.4,5.64) {};
		\filldraw[->,blue] (42)--(53);
		\filldraw[->,blue] (45)--(54);
		\node[above] (62) at (-1.5,5.77) {};
		\node[tinycircle](72) at (-1.5, 6.5) {};
		\filldraw[->,blue] (62)--(72);
		\filldraw[->,blue] (72)--(6);
		
	\node[ttinybox] at (3.5,5.7) {$\mr{M}$};
		\node[below] (55) at (3.6,5.64) {};
		\node[below]  (56)at (3.4,5.64) {};
		\filldraw[->,blue] (43)--(55);
		\filldraw[->,blue] (46)--(56);
		\node[above] (63) at (3.5,5.77) {};
		\node[tinycircle](73) at (3.5, 6.5) {};
		\filldraw[->,blue] (63)--(73);
		\filldraw[->,blue] (73)--(6);

		\node[right]  at (-4.1,3.6) {%
		{\fontsize{9}{11}\selectfont
		$\tilde{h}_0(x)$}};
		\node[right]  at (-4.1,4.1) {%
		{\fontsize{9}{11}\selectfont
		$\tilde{h}_0(x)$}};
		\node[right] at (-4.1,4.8) {%
		{\fontsize{9}{11}\selectfont
		$\tilde{h}_0(x)$}};
		
		\node[right]  at (-2.1,3.6) {%
		{\fontsize{9}{11}\selectfont
		$\tilde{h}_1(x)$}};
		\node[right]  at (-2.1,4.1) {%
		{\fontsize{9}{11}\selectfont
		$\tilde{h}_1(x)$}};
		\node[right] at (-2.1,4.8) {%
		{\fontsize{9}{11}\selectfont
		$\tilde{h}_1(x)$}};
		
		\node[right]  at (2.9,3.6) {%
		{\fontsize{9}{11}\selectfont
		$\tilde{h}_{ I}(x)$}};
		\node[right]  at (2.9,4.1) {%
		{\fontsize{9}{11}\selectfont
		$\tilde{h}_{ I}(x)$}};
		\node[right] at (2.9,4.8) {%
		{\fontsize{9}{11}\selectfont
		$\tilde{h}_{ I}(x)$}};
		
		\node[right] at (-3.1,2.6) {%
		{\fontsize{9}{11}\selectfont
		${\tilde f}_0(x)$}};
		\node[right] at (-3.1,3.1) {%
		{\fontsize{9}{11}\selectfont
		${\tilde f}_0(x)$}};
		\node[right] at (-3.1,4.8) {%
		{\fontsize{9}{11}\selectfont
		${\tilde f}_0(x)$}};
		\node[right] at (-1.1,2.6) {%
		{\fontsize{9}{11}\selectfont
		${\tilde f}_1(x)$}};
		\node[right] at (-1.1,3.1) {%
		{\fontsize{9}{11}\selectfont
		${\tilde f}_1(x)$}};
		\node[right] at (-1.1,4.8) {%
		{\fontsize{9}{11}\selectfont
		${\tilde f}_1(x)$}};
		\node[right] at (3.9,2.6) {%
		{\fontsize{9}{11}\selectfont
		${\tilde f}_{ I}(x)$}};
		\node[right] at (3.9,3.1) {%
		{\fontsize{9}{11}\selectfont
		${\tilde f}_{ I}(x)$}};
		\node[right] at (3.9,4.8) {%
		{\fontsize{9}{11}\selectfont
		${\tilde f}_{ {I}}(x)$}};

		\node[right] at (-4.3,6.7) {%
		{\fontsize{9}{11}\selectfont
		$\mr{M}\ykh{\tilde{h}_0(x),{\tilde f}_0(x)}$}};

		\node[right] at (-1.5,6.4) {%
			{\fontsize{9}{11}\selectfont
			$\mr{M}\ykh{\tilde{h}_1(x),{\tilde f}_1(x)}$}};

		\node[right] at (2.7,6.7) {%
		{\fontsize{9}{11}\selectfont
		$\mr{M}\ykh{\tilde{h}_{ I}(x),{\tilde{f}}_{ I}(x)}$}};
		
		\node[bigbox] at (0.5,4.1) {};
		
		\eetikz
		\caption{The network representing the function $\tilde{g}_3$. }
	\end{figure}
	\noindent Then it follows from (\ref{ineq 5.22}),(\ref{ineq 5.28}), (\ref{ineq 5.24}), the definitions of $\tilde{f}_k$ and $\tilde{g}_3$ that
	\beq\label{ineq 5.29}\ba
	&\abs{\log t-8\log (a)\cdot
		\tilde{g}_3(t)}\\&\leq 8\log\ykh{\frac{1}{a}}\cdot\abs{\tilde{g}_3(t)-\sum_{k=0}^{I}\tilde{h}_k(t)\tilde{f}_k(t)}+\abs{\log t+8\log\ykh{\frac{1}{a}}\sum_{k=0}^{I}\tilde{h}_k(t)\tilde{f}_k(t)}\\
	&\leq 8\log\ykh{\frac{1}{a}}\cdot\abs{\tilde{g}_3(t)-\sum_{k=0}^{I}\tilde{h}_k(t)\tilde{f}_k(t)}+\e/2\\
	&\leq \e/2+\abs{8\log a}\cdot\sum_{k=0}^{I}\abs{\mr{M}\ykh{\tilde{h}_k(t),\tilde{f}_k(t)}-\tilde{h}_k(t)\tilde{f}_k(t)}\\
	&\leq \e/2+\abs{8\log a}\cdot(I+1)\cdot  \frac{\e}{96\ykh{\log a}^2} \leq \e,\;\forall\;t\in[a,1].
	\ea\eeq
	However, for any $t\in \mathbb{R}$, by the definition of $\tilde{f}$, we have  
	\beq\label{ineq 5.30}\ba
	&\tilde{f}(t)=\fltl{
		&8\log (a)\cdot\tilde{g}_3(t),&&\text{ if }8\log (a)\cdot\tilde{g}_3(t)\in[\log a,\log b],\\
		&\log a,&&\text{ if }8\log (a)\cdot\tilde{g}_3(t)<\log a,\\
		&\log b,&&\text{ if }8\log (a)\cdot\tilde{g}_3(t)>\log b,}\\& \text{ satisfying }\ \log a \leq \tilde{f}(t)\leq \log b\leq 0.
	\ea\eeq
	Then by (\ref{ineq 5.29}), (\ref{ineq 5.30}) and the fact that $\log t\in[\log a,\log b],\;\forall\;t\in[a,b]$, we obtain
	\[
	\abs{\log t-\tilde{f}(t)}\leq\abs{\log t-8\log(a)\cdot\tilde{g}_3(t)}\leq\e,\;\forall\;t\in[a,b].
	\]
	That is, 
	\beq\label{ineq 5.31}
	\sup_{t\in[a,b]}\abs{\log t-\tilde{f}(t)}\leq\e.
	\eeq
	
	\begin{figure}[H]
		\centering
		\betikz
		\tikzset{
			mybox/.style ={
				rectangle, %
				rounded corners =5pt, %
				minimum width =10pt, %
				minimum height =70pt, %
				inner sep=0.6pt, %
				draw=blue, %
				fill=cyan
		}}
		\tikzset{
			ttinybox/.style ={
				rectangle, %
				rounded corners =3pt, %
				minimum width =25pt, %
				minimum height =25pt, %
				inner sep=2pt, %
				draw=blue, %
				fill=cyan
		}}
		\tikzset{
			tinybox/.style ={
				rectangle, %
				rounded corners =5pt, %
				minimum width =20pt, %
				minimum height =40pt, %
				inner sep=5pt, %
				draw=blue, %
				fill=cyan
		}}
		\tikzset{
			bigbox/.style ={
				rectangle, %
				rounded corners =5pt, %
				minimum width =347pt, %
				minimum height =289pt, %
				inner sep=0pt, %
				draw=black, %
				fill=none
		}}
		\tikzset{
			tinycircle/.style ={
				circle, %
				minimum width =6pt, %
				minimum height =6pt, %
				inner sep=0pt, %
				draw=red, %
				fill=red %
			}
		}
		\node[tinycircle] (1) at (0,0) {};
		\node[right] at (0,0){$\;\;x\in\mbR$};
		\node[below] at (-0.0,-0.3) {Input};
		
		\node[ttinybox] at (0,1) {$\tilde{g}_3$};
		\node[above] (11) at (0,1.15) {};
		\node[below] (x11) at (0,0.85) {};
		\node[tinycircle] (2) at (0,2) {};
		\filldraw[->,blue] (11)--(2);
		\filldraw[->,blue] (1)--(x11);
		\node[right] at (-0.1,1.85) {$\;\;\sigma(\tilde{g}_3(x))$};
		\node[tinycircle] (31) at (-3,3) {};
		\node[tinycircle] (32) at (0.5,3) {};
		\node[tinycircle] (33) at (5,3) {};
		\filldraw[->,blue] (2)--(31) {};
				\filldraw[->,blue] (2)--(32) {};
				\filldraw[->,blue] (2)--(33) {};
	\node[tinycircle] (3) at (-1.4,4) {};
		\filldraw[->,blue] (31)--(3) {};
				\filldraw[->,blue] (32)--(3) {};
				\filldraw[->,blue] (33)--(3) {};
		\node[right] at (-3.8,3.3) {%
		{\fontsize{9}{11}\selectfont
		$\sigma(\sigma(\tilde{g}_3(x))-\frac{\log b}{8\log a})$}};
		\node[right] at (0.4,3.2) {%
		{\fontsize{9}{11}\selectfont
		$\sigma(\sigma(\tilde{g}_3(x))-\frac{1}{8})$}};
		\node[right] at (3.6,2.6) {%
		{\fontsize{9}{11}\selectfont
		$\sigma(0\cdot\sigma(\tilde{g}_3(x))+\frac{\log b}{8\log a})=\frac{\log b}{8\log a}$}};

		\node[tinycircle] (51) at (-3.5,5) {};
				\node[tinycircle] (52) at (-2,5) {};
				\node[tinycircle] (53) at (2,5) {};
				\node[tinycircle] (54) at (3.5,5) {};
					\filldraw[->,blue] (3)--(51) {};
						\filldraw[->,blue] (3)--(52) {};
						\filldraw[->,blue] (3)--(53) {};
						\filldraw[->,blue] (3)--(54) {};

		\node[right] at (-1.1,4) {%
		{\fontsize{9}{11}\selectfont
		$\;\;\;\;\sigma\ykh{\sigma\ykh{\sigma(\tilde{g}_3(x))-\frac{\log b}{8\log a}}-\sigma(\sigma(\tilde{g}_3(x))-\frac{1}{8})+\frac{\log b}{8\log a}}=\frac{\tilde{f}(x)}{8\log a}$}};

		\node[right] at (-3.47,5.16) {%
		{\fontsize{9}{11}\selectfont
		$\;\frac{\tilde{f}(x)}{8\log a}$}};
		
		\node[right] at (-2,5.16) {%
		{\fontsize{9}{11}\selectfont
		$\;\frac{\tilde{f}(x)}{8\log a}$}};
		\node[right] at (1.9,5.16) {%
		{\fontsize{9}{11}\selectfont
		$\;\frac{\tilde{f}(x)}{8\log a}$}};
		\node[right] at (3.4,5.16) {%
		{\fontsize{9}{11}\selectfont
		$\;\frac{\tilde{f}(x)}{8\log a}$}};
	
		\node[right] at (-0.8,5) {$\cdots$};
		\node[right] at (0.4,5) {$\cdots$};
		\node[right] at (4.7,5) {\textbf{($8{I}$ neurons)}};
		
		\node[right] at (-0.8,6.5) {$\cdots$};
		\node[right] at (0.,6.5) {$\cdots$};
		
		\node[tinycircle] (8) at (0,8.3) {};
		\filldraw[->,blue] (51)--(8) {};
		\filldraw[->,blue] (52)--(8) {};
		\filldraw[->,blue] (53)--(8) {};
		\filldraw[->,blue] (54)--(8) {};
		\node[right] (8) at (0,8.3) {$\;\tilde{f}(x)=\sum_{k=1}^{8{I}}\frac{\log a}{{I}}\cdot \frac{\tilde{f}(x)}{8\log a}$};
		
		\node[above] (81) at (0,8.5) {Output};
		\node[bigbox] at(2.1,4.2) {};
		\eetikz
		\caption{The network representing the function $\tilde{f}$. }
		\label{fig5}
	\end{figure}
	
	On the other hand, it follows from (\ref{ineq 5.23}), (\ref{ineq 5.24n}), (\ref{multiply}), the definition of $\tilde{g}_3$, and $1\leq I\leq 4\log\frac{1}{a}$ that
	\[\ba
	&\tilde{g}_3\in\fdnn_1\left(C_\alpha\log\frac{1}{\e}+I+15\log\ykh{96\ykh{\log{a}}^2},C_{\alpha}\ykh{\frac{1}{\e}}^{\frac{1}{\alpha}}I,\right.\\
	&\;\;\;\;\;\;\;\;\;\left.(I+1)\cdot\ykh{20I+C_{\alpha}\ykh{\frac{1}{\e}}^{\frac{1}{\alpha}}\cdot\log\frac{1}{\e}+900\log\ykh{96\ykh{\log a}^2}},1,\infty\right)\\
	&\;\;\;\subset\fdnn_1\left(C_\alpha\log\frac{1}{\e}+139\log\frac{1}{a},C_{\alpha}\ykh{\frac{1}{\e}}^{\frac{1}{\alpha}}\log\frac{1}{a},\right.\\
	&\;\;\;\;\;\;\;\;\;\left.C_{\alpha}\ykh{\frac{1}{\e}}^{\frac{1}{\alpha}}\cdot\ykh{\log\frac{1}{\e}}\cdot\ykh{\log\frac{1}{a}}+65440\ykh{\log{a}}^2,1,\infty\right).\\
	\ea\]
	Then by the definition of $\tilde{f}$ we obtain (cf. Figure \ref{fig5})
	\[\ba
	&\tilde{f}\in\fdnn_1\left(C_\alpha\log\frac{1}{\e}+139\log\frac{1}{a},C_{\alpha}\ykh{\frac{1}{\e}}^{\frac{1}{\alpha}}\log\frac{1}{a},\right.\\
	&\;\;\;\;\;\;\;\;\;\left.C_{\alpha}\ykh{\frac{1}{\e}}^{\frac{1}{\alpha}}\cdot\ykh{\log\frac{1}{\e}}\cdot\ykh{\log\frac{1}{a}}+65440\ykh{\log{a}}^2,1,\infty\right).\\
	\ea\]
	This, together with (\ref{ineq 5.30}) and (\ref{ineq 5.31}), completes the proof of Theorem \ref{thm2.3}.\end{proof}

\subsection{Proof of Theorem \ref{thm2.2} and Theorem \ref{thm2.2p}}\label{section: proof of thm2.2}

Appendix \ref{section: proof of thm2.2} is devoted to the proof of Theorem \ref{thm2.2} and Theorem \ref{thm2.2p}. We will first establish several lemmas. We then use these lemmas to prove Theorem \ref{thm2.2p}. Finally, we derive Theorem  \ref{thm2.2} by applying Theorem \ref{thm2.2p} with $q=0$,  $d_*=d$ and $d_\star=K=1$.

\begin{lem}\label{a.10}
	Let $\phi(t)=\log(1+\me^{-t})$ be the logistic loss. Suppose real numbers $a,f,A,B$ satisfy that $0<a<1$ and $A\leq\min\hkh{f,\log\frac{a}{1-a}}\leq \max\hkh{f,\log\frac{a}{1-a}}\leq B$.  Then there holds
	\[\ba
	&\min\hkh{\frac{1}{4+2\me^A+2\me^{-A}}, \frac{1}{4+2\me^B+2\me^{-B}}}\cdot\abs{f-\log\frac{a}{1-a}}^2\\
	&\leq a\phi(f)+(1-a)\phi(-f)-a\log\frac{1}{a}-(1-a)\log\frac{1}{1-a}\\
	&\leq  \sup\hkh{\frac{1}{4+2\me^z+2\me^{-z}}\Big|z\in[A,B]}\cdot\abs{f-\log\frac{a}{1-a}}^2\leq  \frac{1}{8}\cdot\abs{f-\log\frac{a}{1-a}}^2.
	\ea\] 
\end{lem}
\begin{proof}
	Consider the map $G:\mbR\to[0,\infty),z\mapsto a\phi(z)+(1-a)\phi(-z)$. Obviously $G$ is twice continuously differentiable on $\mbR$ with $G'\ykh{\log\frac{a}{1-a}}=0$ and  $G''(z)=\frac{1}{2+\me^z+\me^{-z}}$ for any real number $z$. Then it follows from Taylor's theorem that there exists a real number $\xi$ between $\log\frac{a}{1-a}$ and $f$, such that
	\beq\label{a83}
	&a\phi(f)+(1-a)\phi(-f)-a\log\frac{1}{a}-(1-a)\log\frac{1}{1-a}=G(f)-G\ykh{\log\frac{a}{1-a}}\\
	&=\ykh{f-\log\frac{a}{1-a}}\cdot G'\ykh{\log\frac{a}{1-a}}+\frac{G''(\xi)}{2}\cdot\abs{f-{\log\frac{a}{1-a}}}^2\\&=\frac{G''(\xi)}{2}\cdot\abs{f-{\log\frac{a}{1-a}}}^2=\frac{\abs{f-{\log\frac{a}{1-a}}}^2}{4+2\me^\xi+2\me^{-\xi}}. 
	\eeq  Since $A\leq\min\hkh{f,\log\frac{a}{1-a}}\leq \max\hkh{f,\log\frac{a}{1-a}}\leq B$, we must have $\xi\in[A,B]$, which, together with (\ref{a83}), yields
	\beq\label{a84}
	&\min\hkh{\frac{1}{4+2\me^A+2\me^{-A}}, \frac{1}{4+2\me^B+2\me^{-B}}}\cdot\abs{f-\log\frac{a}{1-a}}^2\\
	&= \ykh{\inf_{t\in[A,B]}\frac{1}{4+2\me^t+\me^{-t}}}\cdot \abs{f-{\log\frac{a}{1-a}}}^2\leq \frac{\abs{f-{\log\frac{a}{1-a}}}^2}{4+2\me^\xi+2\me^{-\xi}}\\
	&=a\phi(f)+(1-a)\phi(-f)-a\log\frac{1}{a}-(1-a)\log\frac{1}{1-a}=\frac{\abs{f-{\log\frac{a}{1-a}}}^2}{4+2\me^\xi+2\me^{-\xi}}\\
	&\leq  \sup\hkh{\frac{1}{4+2\me^z+2\me^{-z}}\Big|z\in[A,B]}\cdot\abs{f-\log\frac{a}{1-a}}^2\leq  \frac{1}{8}\cdot\abs{f-\log\frac{a}{1-a}}^2.
	\eeq This completes the proof. 
\end{proof}

\begin{lem}\label{lemma5.3}
	Let $\phi(t)=\log\ykh{1+\me^{-t}}$ be the logistic loss,  $f$ be a real number, $d\in\mb N$, and $P$ be a Borel probability measure on $[0,1]^d\times\hkh{-1,1}$ of which the conditional probability function $[0,1]^d\ni z\mapsto P(\hkh{1}|z)\in[0,1]$ is denoted by $\eta$. Then for $x\in[0,1]^d$ such that $\eta(x)\notin\{0,1\}$, there holds
	\[\ba
	&\abs{\inf_{t\in \zkh{f\qx\log\frac{\eta(x)}{1-\eta(x)},f \qd\log\frac{\eta(x)}{1-\eta(x)}}}\frac{1}{2(2+\me^t+\me^{-t})}}\cdot \abs{f-\log\frac{\eta(x)}{1-\eta(x)}}^2\\
	&\leq\int_{\{-1,1\}}\ykh{\phi\ykh{yf}-\phi\ykh{y\log\frac{\eta(x)}{1-\eta(x)}}}\mr{d}P(y|x)\\&\leq\abs{\sup_{t\in \zkh{f\qx\log\frac{\eta(x)}{1-\eta(x)},f \qd\log\frac{\eta(x)}{1-\eta(x)}}}\frac{1}{2(2+\me^t+\me^{-t})}}\cdot \abs{f-\log\frac{\eta(x)}{1-\eta(x)}}^2\leq\frac{1}{4}\abs{f-\log\frac{\eta(x)}{1-\eta(x)}}^2.
	\ea\]	
\end{lem}
\begin{proof}
	Given $x\in[0,1]^d$ such that $\eta(x)\notin\{0,1\}$, define
	\[V_x:\mbR\to (0,\infty), \quad t\mapsto\eta(x)\phi(t)+(1-\eta(x))\phi(-t). 
	\] Then it is easy to verify that
	\[\int_{\{-1,1\}}\phi\ykh{yt}\mr{d}P(y|x)=\phi(t)P(Y=1|X=x)+\phi(-t)P(Y=-1|X=x)=V_x(t) \] for all $t\in\mbR$. Consequently, 
	\[\ba
	&\int_{\{-1,1\}}\ykh{\phi\ykh{yf}-\phi\ykh{y\log\frac{\eta(x)}{1-\eta(x)}}}\mr{d}P(y|x)=V_x(f)-V_x\ykh{\log\frac{\eta(x)}{1-\eta(x)}}\\
	&=\eta(x)\phi(f)+(1-\eta(x))\phi(-f)-\eta(x)\log\frac{1}{\eta(x)}-(1-\eta(x))\log\frac{1}{1-\eta(x)}. 
	\ea\] The desired inequalities then follow immediately by applying Lemma \ref{a.10}.  	
\end{proof}

\begin{lem}\label{lem5.4} Let $\phi(t)=\log\ykh{1+\me^{-t}}$ be the logistic loss, $d\in\mb N$, $f:[0,1]^d\to\mbR$ be a measurable function, and $P$ be a Borel probability measure on $[0,1]^d\times\hkh{-1,1}$ of which the conditional probability function $[0,1]^d\ni z\mapsto P(\hkh{1}|z)\in[0,1]$ is denoted by $\eta$.  Assume that there exist constants $(a,b)\in\mbR^2$, $\delta\in (0,1/2)$, and a measurable function $\hat{\eta}:[0,1]^d\to\mbR$,  such that $\hat\eta=\eta$, $P_X$-a.s.,
	\[
	\log\frac{\delta}{1-\delta}\leq f(x)\leq -a,\;\forall\;x\in[0,1]^d \textrm{ satisfying } 0 \leq \hat\eta(x)=\eta(x)<\delta,
	\]
	and
	\[
	b\leq f(x)\leq \log \frac{1-\delta}{\delta},\;\forall\;x\in[0,1]^d \textrm{ satisfying } 1-\delta<\hat\eta(x)=\eta(x) \leq 1.
	\]
	Then 
	\[\ba
	&\mc{E}_P^{\phi}\ykh{f}-\phi(a)P_X(\Omega_2)-\phi(b)P_X(\Omega_3)\\&\leq \int_{\Omega_1}{\sup\setl{\frac{\abs{f(x)-\log\frac{\eta(x)}{1-\eta(x)}}^2}{2(2+\me^t+\me^{-t})}}{{t\in \zkh{f(x)\qx\log\frac{\eta(x)}{1-\eta(x)},f(x) \qd\log\frac{\eta(x)}{1-\eta(x)}}}}}\mr{d}P_X(x)\\
	&\leq \int_{\Omega_1}\abs{f(x)-\log\frac{\eta(x)}{1-\eta(x)}}^2\mr{d}P_X(x),
	\ea\]
	where 
	\beq\label{20221019232201}
	&\Omega_1:=\setl{x\in[0,1]^d}{\delta\leq\hat\eta(x)=\eta(x)\leq 1-\delta}, \\
	&\Omega_2:=\setl{x\in[0,1]^d}{0 \leq \hat\eta(x)=\eta(x)<\delta}, \\
	&\Omega_3:=\setl{x\in[0,1]^d}{1-\delta <\hat\eta(x)=\eta(x)\leq 1}. \\
	\eeq
\end{lem}
\begin{proof}
	Define \[\ba
	\psi:[0,1]^d\times\{-1,1\}&\to[0,\infty),\\(x,y)&\mapsto\left\{
	\ba
	&\phi\ykh{y\log\frac{\eta(x)}{1-\eta(x)}},&&\text{if }\eta(x)\in [\delta,1-\delta],\\
	&0,&&\text{if }\eta(x)\in\{0,1\},\\
	&\eta(x)\log\frac{1}{\eta(x)}+(1-\eta(x))\log\frac{1}{1-\eta(x)},&&\text{if }\eta(x)\in (0, \delta)\cup(1-\delta,1).
	\ea
	\right.
	\ea\]	
	Since $\hat\eta=\eta\in[0,1]$, $P_X$-a.s., we have that $P_X([0,1]^d\setminus(\Omega_1\cup\Omega_2\cup\Omega_3))=0$. Then it follows from lemma \ref{23022804} that 
\beq\label{23030220}
&\mc{E}_P^\phi(f)=\mc R^\phi_P(f)-\inf \setl{\mc{R}^\phi_P(g)}{ \textrm{$g:[0,1]^d\to\mbR$ is measurable}}\\
&=\int_{[0,1]^d\times\hkh{-1,1}}\phi(yf(x))\mr{d}P(x,y)-\int_{[0,1]^d\times\hkh{-1,1}}\psi(x,y)\mr{d}P(x,y)=I_1+I_2+I_3,
\eeq where
	\[
	I_i:=\int_{\Omega_i\times\{-1,1\}}\ykh{\phi\ykh{yf(x)}-\psi(x,y)}\mr{d}P(x,y),\; i=1,2,3.
	\] 
	According to Lemma \ref{lemma5.3}, we have 
	\beq\ba\label{ineq 5.41}
	I_1&=\int_{\Omega_1}\int_{\{-1,1\}}\ykh{\phi\ykh{yf(x)}-\phi\ykh{y\log\frac{\eta(x)}{1-\eta(x)}}}\mr{d}P(y|x)\mr{d}P_X(x)\\
	&\leq  \int_{\Omega_1}{\sup\setl{\frac{\abs{f(x)-\log\frac{\eta(x)}{1-\eta(x)}}^2}{2(2+\me^t+\me^{-t})}}{\ba &t\in \left[f(x)\qx\log\frac{\eta(x)}{1-\eta(x)},\infty\right)\text{ and}\\&\;\;t\in\left(-\infty,f(x) \qd\log\frac{\eta(x)}{1-\eta(x)}\right]\ea}}\mr{d}P_X(x).
	\ea\eeq Then it remains to bound $I_2$ and $I_3$. 
	
	Indeed, for any $x\in \Omega_2$, if $\eta(x)=0$, then
	\[
	\int_{\{-1,1\}}\ykh{\phi(yf(x))-\psi(x,y)}\mr{d}P(y|x)=\phi(-f(x))\leq \phi(a).
	\]
	Otherwise, we have 
	\begin{align*}
	&\int_{\{-1,1\}}\ykh{\phi(yf(x))-\psi(x,y)}\mr{d}P(y|x)\\&=\ykh{\phi(f(x))-\log\frac{1}{\eta(x)}}\eta(x)+\ykh{\phi(-f(x))-\log\frac{1}{1-\eta(x)}}(1-\eta(x))\\
&= \ykh{\phi\ykh{f(x)}-\phi\ykh{\log\frac{\eta(x)}{1-\eta(x)}}}\eta(x)+\ykh{\phi\ykh{-f(x)}-\phi\ykh{-\log\frac{\eta(x)}{1-\eta(x)}}}(1-\eta(x))\\
	&\leq \ykh{\phi\ykh{\log\frac{\delta}{1-\delta}}-\phi\ykh{\log\frac{\eta(x)}{1-\eta(x)}}}\eta(x)+\phi(-f(x))(1-\eta(x))\\
	&\leq {\phi(-f(x))}(1-\eta(x))\leq {\phi(-f(x))}\leq \phi(a). 
	\end{align*}
	Therefore, no matter whether $\eta(x)=0$ or $\eta(x)\neq 0$, there always holds \[\int_{\{-1,1\}}\ykh{\phi(yf(x))-\psi(x,y)}\mr{d}P(y|x)\leq\phi(a),\] which means that
	\beq\label{ineq 5.42}
	I_2&=\int_{\Omega_2}\int_{\{-1,1\}}\ykh{\phi(yf(x))-\psi(x,y)}\mr{d}P(y|x)\mr{d}P_X(x)\\&\leq \int_{\Omega_2}\phi(a)\mr{d}P_X(x)=\phi(a)P_X(\Omega_2).
	\eeq
	
	Similarly, for any $x\in \Omega_3$, if $\eta(x)=1$, then
	\[
	\int_{\{-1,1\}}\ykh{\phi(yf(x))-\psi(x,y)}\mr{d}P(y|x)=\phi(f(x))\leq \phi(b).
	\]
	Otherwise, we have 
	\begin{align*}
	&\int_{\{-1,1\}}\ykh{\phi(yf(x))-\psi(x,y)}\mr{d}P(y|x)\\&=\ykh{\phi(f(x))-\log\frac{1}{\eta(x)}}\eta(x)+\ykh{\phi(-f(x))-\log\frac{1}{1-\eta(x)}}(1-\eta(x))\\
	&= \ykh{\phi\ykh{f(x)}-\phi\ykh{\log\frac{\eta(x)}{1-\eta(x)}}}\eta(x)+\ykh{\phi\ykh{-f(x)}-\phi\ykh{-\log\frac{\eta(x)}{1-\eta(x)}}}(1-\eta(x))\\
	&\leq \phi(f(x))\eta(x)+ \ykh{\phi\ykh{\log\frac{\delta}{1-\delta}}-\phi\ykh{\log\frac{1-\eta(x)}{\eta(x)}}}(1-\eta(x))\\
	&\leq \phi(f(x))\eta(x)\leq\phi(f(x))\leq\phi(b). 
	\end{align*}
	Therefore, no matter whether $\eta(x)=1$ or $\eta(x)\neq 1$, we have \[\int_{\{-1,1\}}\ykh{\phi(yf(x))-\psi(x,y)}\mr{d}P(y|x)\leq\phi(b),\] which means that
	\beq\label{ineq 5.43}
	I_3&=\int_{\Omega_3}\int_{\{-1,1\}}\ykh{\phi(yf(x))-\psi(x,y)}\mr{d}P(y|x)\mr{d}P_X(x)\\&\leq \int_{\Omega_3}\phi(b)\mr{d}P_X(x)=\phi(b)P_X(\Omega_3).
	\eeq
	The desired inequality then follows immediately from (\ref{ineq 5.41}), (\ref{ineq 5.42}), (\ref{ineq 5.43}) and \eqref{23030220}. Thus we complete the proof.
\end{proof}

\begin{lem}\label{a.11}Let $\delta\in(0,1/2)$, $a\in[\delta,1-\delta]$,  $f\in\zkh{-\log\frac{1-\delta}{\delta},\log\frac{1-\delta}{\delta}}$, and $\phi(t)=\log(1+\me^{-t})$ be the logistic loss. Then there hold
	\[\ba
	H(a,f)\leq \Gamma \cdot G(a,f) 
	\ea\] with $\Gamma= 5000\abs{\log\delta}^2$, 
	\[
	H(a,f):=a\cdot\abs{\phi(f)-\phi\ykh{\log\frac{a}{1-a}}}^2+(1-a)\cdot \abs{\phi(-f)-\phi\ykh{-\log\frac{a}{1-a}}}^2,
	\] and 
	\[\ba
	G(a,f)&:=a\phi(f)+(1-a)\phi(-f)-a\phi\ykh{\log\frac{a}{1-a}}-(1-a)\phi\ykh{-\log\frac{a}{1-a}}\\
	&=a\phi(f)+(1-a)\phi(-f)-a\log\frac{1}{a}-(1-a)\log\frac{1}{1-a}. 
	\ea\]
\end{lem}
\begin{proof}
	In this proof, we will frequently use elementary inequalities 
	\beq\label{a85}
	x\log\frac{1}{x}\leq\min\hkh{1-x,(1-x)\cdot\log\frac{1}{1-x}},\;\forall\;x\in[1/2,1), 
	\eeq and
	\beq\label{a86}
	&-\log\frac{1}{1-x}-2<-\log7\leq-\log\ykh{\exp\ykh{\frac{3-3x}{x}\log\frac{1}{1-x}}-1}\\&<\log\frac{x}{1-x}<2+\log\frac{1}{1-x},\;\forall\;x\in[1/2,1). 
	\eeq 
	
	We first show that 
	\beq\label{a87}
	&G(a,f)\geq \frac{a\phi(f)}{3}\;\\&\;\;\;\;\;\text{provided $\frac{1}{2}\leq a\leq1-\delta$ and $f\leq-\log\ykh{\exp\ykh{\frac{3-3a}{a}\log\frac{1}{1-a}}-1}$. } 
	\eeq Indeed, if $1/2\leq a\leq1-\delta$ and $f\leq-\log\ykh{\exp\ykh{\frac{3-3a}{a}\log\frac{1}{1-a}}-1}$, then
	\[\ba
	&\frac{2}{3}\cdot a\phi(f)\geq\frac{2}{3}\cdot a\phi\ykh{-\log\ykh{\exp\ykh{\frac{3-3a}{a}\log\frac{1}{1-a}}-1}}=(2-2a)\cdot\log\frac{1}{1-a}\\
	&\geq a\log\frac{1}{a}+(1-a)\log\frac{1}{1-a},
	\ea\] which means that
	\[
	G(a,f)\geq a\phi(f)-a\log\frac{1}{a}-(1-a)\log\frac{1}{1-a}\geq\frac{a\phi(f)}{3}. 
	\] This proves (\ref{a87}). 
	
	We next show that
	\beq\label{a88}
	&G(a,f)\geq \frac{1-a}{18}\abs{f-\log\frac{a}{1-a}}^2 \\&\;\;\;\;\;\text{provided $\frac{1}{2}\leq a\leq 1-\delta$ and $-2-\log\frac{1}{1-a}\leq f\leq  2+\log\frac{1}{1-a}$. }
	\eeq Indeed, if $1/2\leq a\leq 1-\delta$ and $-2-\log\frac{1}{1-a}\leq f\leq 2+\log\frac{1}{1-a}$, then it follows from Lemma \ref{a.10} that
	\[\ba
	&G(a,f)\geq \frac{\abs{f-\log\frac{a}{1-a}}^2}{4+2\exp\ykh{2+\log\frac{1}{1-a}}+2\exp\ykh{-2-\log\frac{1}{1-a}}}\\
	&\geq \frac{\abs{f-\log\frac{a}{1-a}}^2}{5+15\cdot\frac{1}{1-a}}\geq \frac{(1-a)\cdot\abs{f-\log\frac{a}{1-a}}^2}{5-5a+15}\geq \frac{(1-a)\cdot\abs{f-\log\frac{a}{1-a}}^2}{18}, 
	\ea\] which proves (\ref{a88}). 
	
	We then show
	\beq\label{a89}
	H(a,f)\leq\Gamma\cdot G(a,f) \text{ provided $1/2\leq a\leq1-\delta$ and $-\log\frac{1-\delta}{\delta}\leq f\leq\log\frac{1-\delta}{\delta}$}
	\eeq by considering the following four cases. 
	
	\textbf{Case \uppercase\expandafter{\romannumeral1}.} $1/2\leq a\leq 1-\delta$ and $2+\log\frac{1}{1-a}\leq f\leq \log\frac{1-\delta}{\delta}$. In this case we have 
	\beq\label{a90p}
	&\log\frac{1}{\delta}=\phi\ykh{\log\frac{\delta}{1-\delta}}\geq\phi(-f)=\log(1+\me^f)\geq f\geq2+\log\frac{1}{1-a}\\&>\phi\ykh{-\log\frac{a}{1-a}}=\log\frac{1}{1-a}\geq\log\frac{1}{a}>0,
	\eeq which, together with (\ref{a85}), yields
	\[
	a\log\frac{1}{a}+(1-a)\log\frac{1}{1-a}\leq(1-a)\cdot\ykh{1+\log\frac{1}{1-a}}\leq(1-a)\cdot\frac{1+\log\frac{1}{1-a}}{2+\log\frac{1}{1-a}}\cdot\phi(-f). 
	\] Consequently, 
	\beq\label{a90}
	&G(a,f)\geq (1-a)\cdot\phi(-f)-a\log\frac{1}{a}-(1-a)\log\frac{1}{1-a}
	\\&\geq (1-a)\cdot\phi(-f)- (1-a)\cdot\frac{1+\log\frac{1}{1-a}}{2+\log\frac{1}{1-a}}\cdot\phi(-f) \\&= \frac{(1-a)\cdot\phi(-f)}{2+\log\frac{1}{1-a}}\geq \frac{(1-a)\cdot\phi(-f)}{4\log\frac{1}{\delta}}. 
	\eeq On the other hand, it follows from $f\geq 2+\log\frac{1}{1-a}>\log\frac{a}{1-a}$ that \[
	0\leq \phi\ykh{\log\frac{a}{1-a}}-\phi(f)<\phi\ykh{\log\frac{a}{1-a}},
	\] which, together with (\ref{a85}) and (\ref{a90p}), yields
	\beq\label{a92}
	&a\cdot\abs{\phi(f)-\phi\ykh{\log\frac{a}{1-a}}}^2\leq a\cdot\abs{\phi\ykh{\log\frac{a}{1-a}}}^2\\&=a\cdot\abs{\log\frac{1}{a}}^2\leq (1-a)\cdot \log\frac{1}{a}\leq(1-a)\cdot\phi(-f). 
	\eeq Besides, it follows from (\ref{a90}) that $0\leq\phi(-f)-\phi\ykh{-\log\frac{a}{1-a}}\leq\phi(-f)$. Consequently, 
	\beq\label{a93}
	(1-a)\cdot\abs{\phi(-f)-\phi\ykh{-\log\frac{a}{1-a}}}^2\leq (1-a)\cdot\phi(-f)^2\leq (1-a)\cdot\phi(-f)\cdot\log\frac{1}{\delta}.
	\eeq Combining (\ref{a90}), (\ref{a92}) and (\ref{a93}), we deduce that
	\[
	H(a,f)\leq (1-a)\cdot\phi(-f)\cdot\abs{1+\log\frac{1}{\delta}}\leq (1-a)\cdot\phi(-f)\cdot\frac{\Gamma}{4\log\frac{1}{\delta}}\leq \Gamma\cdot G(a,f),
	\] which proves the desired inequality. 
	
	\textbf{Case \uppercase\expandafter{\romannumeral2}. } $1/2\leq a\leq 1-\delta$ and $-\log\ykh{\exp\ykh{\frac{3-3a}{a}\log\frac{1}{1-a}}-1}\leq f<2+\log\frac{1}{1-a}$. In this case, we have $-2-\log\frac{1}{1-a}\leq f\leq 2+\log\frac{1}{1-a}$, where we have used (\ref{a86}). Therefore, it follows from (\ref{a88}) that $G(a,f)\geq\frac{1-a}{18}\abs{f-\log\frac{a}{1-a}}^2$. On the other hand, it follow from (\ref{a86}) and Taylor's Theorem that there exists
	\begin{align*}
	&-\log 7\leq -\log\ykh{\exp\ykh{\frac{3-3a}{a}\log\frac{1}{1-a}}-1}\\&\leq f\qx\log\frac{a}{1-a}\leq \xi\leq f\qd\log\frac{a}{1-a}\leq 2+\log\frac{1}{1-a}, 
	\end{align*}  such that
	\beq\label{a94}
	&a\cdot\abs{\phi(f)-\phi\ykh{\log\frac{a}{1-a}}}^2\\&=a\cdot\abs{\phi'(\xi)}^2\cdot\abs{f-\log\frac{a}{1-a}}^2\leq a\cdot\me^{-2\xi}\cdot \abs{f-\log\frac{a}{1-a}}^2\\
	&\leq a\cdot\exp(\log 7)\cdot\exp\ykh{\log\ykh{\exp\ykh{\frac{3-3a}{a}\log\frac{1}{1-a}}-1}}\cdot\abs{f-\log\frac{a}{1-a}}^2\\
	&=7a\cdot\int_{0}^{\frac{3-3a}{a}\log\frac{1}{1-a}}\me^t\mr{d}t\cdot \abs{f-\log\frac{a}{1-a}}^2\\&\leq 7a\cdot \abs{\frac{3-3a}{a}\log\frac{1}{1-a}}\cdot\exp\ykh{\frac{3-3a}{a}\log\frac{1}{1-a}}\cdot\abs{f-\log\frac{a}{1-a}}^2\\
	&\leq 7a\cdot\abs{\frac{3-3a}{a}\log\frac{1}{1-a}}\cdot\ykh{1+\exp\ykh{\log 7}}\cdot\abs{f-\log\frac{a}{1-a}}^2\\
	&\leq 168\cdot \abs{(1-a)\cdot\log\frac{1}{1-a}}\cdot \abs{f-\log\frac{a}{1-a}}^2\\&\leq 168\cdot \abs{(1-a)\cdot\log\frac{1}{\delta}}\cdot \abs{f-\log\frac{a}{1-a}}^2. 
	\eeq Besides, we have
	\beq\label{a95}
	&(1-a)\cdot\abs{\phi(-f)-\phi\ykh{-\log\frac{a}{1-a}}}^2\\&\leq \abs{1-a}\cdot\norm{\phi'}_\mbR\cdot\abs{f-\log\frac{a}{1-a}}^2\leq \abs{1-a}\cdot\abs{f-\log\frac{a}{1-a}}^2. 
	\eeq Combining (\ref{a94}), (\ref{a95}) and the fact that $G(a,f)\geq\frac{1-a}{18}\abs{f-\log\frac{a}{1-a}}^2$, we deduce that
	\[\ba
	&H(a,f)\leq 168\cdot \abs{(1-a)\cdot\log\frac{1}{\delta}}\cdot \abs{f-\log\frac{a}{1-a}}^2+ \abs{1-a}\cdot \abs{f-\log\frac{a}{1-a}}^2\\
	&\leq 170\cdot \abs{(1-a)\cdot\log\frac{1}{\delta}}\cdot \abs{f-\log\frac{a}{1-a}}^2\leq\Gamma\cdot \frac{1-a}{18}\cdot\abs{f-\log\frac{a}{1-a}}^2\leq \Gamma\cdot G(a,f), 
	\ea\] which proves the desired inequality. 
	
	\textbf{Case \uppercase\expandafter{\romannumeral3}.} $1/2\leq a\leq 1-\delta$ and $-\log\frac{a}{1-a}\leq f<-\log\ykh{\exp\ykh{\frac{3-3a}{a}\log\frac{1}{1-a}}-1}$. In this case, we still have (\ref{a95}). Besides, it follows from (\ref{a87}) that  $G(a,f)\geq \frac{a\phi(f)}{3}$. Moreover, by (\ref{a86}) we obtain  $-2-\log\frac{1}{1-a}<f<2+\log\frac{1}{1-a}$, which, together with (\ref{a88}), yields $G(a,f)\geq\frac{1-a}{18}\abs{f-\log\frac{a}{1-a}}^2$. In addition, since $f<-\log\ykh{\exp\ykh{\frac{3-3a}{a}\log\frac{1}{1-a}}-1}\leq \log\frac{a}{1-a}$, we have that $0<\phi(f)-\phi\ykh{\log\frac{a}{1-a}}<\phi(f)$, which means that
	\beq\label{a96}
	&a\cdot\abs{\phi(f)-\phi\ykh{\log\frac{a}{1-a}}}^2\leq a\cdot\abs{\phi(f)}^2\\&\leq a\phi(f)\phi\ykh{-\log\frac{a}{1-a}}= a\phi(f)\log\frac{1}{1-a}\leq a\phi(f)\log\frac{1}{\delta}. 
	\eeq Combining all these inequalities, we obtain
	\[\ba
	&H(a,f)\leq  a\phi(f)\cdot\abs{\log\frac{1}{\delta}}+\abs{1-a}\cdot\abs{f-\log\frac{a}{1-a}}^2\\&\leq\frac{\Gamma a \phi(f)}{6}+\Gamma\cdot\frac{1-a}{36}\cdot \abs{f-\log\frac{a}{1-a}}^2\\
	&\leq \frac{\Gamma \cdot G(a,f)}{2}+\frac{\Gamma \cdot G(a,f)}{2}=\Gamma\cdot G(a,f), 
	\ea\] which proves the desired inequality. 
	
	\textbf{Case \uppercase\expandafter{\romannumeral4}.}  $-\log\frac{1-\delta}{\delta}\leq f<\min\hkh{-\log\frac{a}{1-a}, -\log\ykh{\exp\ykh{\frac{3-3a}{a}\log\frac{1}{1-a}}-1}}$ and $1/2\leq a\leq 1-\delta$. In this case, we still have $G(a,f)\geq \frac{a\phi(f)}{3}$ according to (\ref{a87}). Besides, it follows from
	\[
	f<\min\hkh{-\log\frac{a}{1-a}, -\log\ykh{\exp\ykh{\frac{3-3a}{a}\log\frac{1}{1-a}}-1}}\leq -\log\frac{a}{1-a}\leq \log\frac{a}{1-a}
	\] that 
	\beq\label{a97}
	0&\leq\min\hkh{ \phi\ykh{-\log\frac{a}{1-a}}-\phi(-f),\phi(f)-\phi\ykh{\log\frac{a}{1-a}}}\\
	&\leq \max\hkh{ \phi\ykh{-\log\frac{a}{1-a}}-\phi(-f),\phi(f)-\phi\ykh{\log\frac{a}{1-a}}}\\&\leq \max\hkh{\phi\ykh{-\log\frac{a}{1-a}},\phi(f)}=\phi(f). 
	\eeq Combining (\ref{a97}) and the fact that $G(a,f)\geq\frac{a\phi(f)}{3}$, we deduce that
	\begin{align*}
	&H(a,f)\leq a\cdot\abs{\phi(f)}^2+ (1-a)\cdot\abs{\phi(f)}^2\leq \phi(f)\phi\ykh{-\log\frac{1-\delta}{\delta}}\\&=\phi(f)\log\frac{1}{\delta}\leq \frac{\Gamma a\phi(f)}{3}\leq \Gamma\cdot G(a,f),
	\end{align*} which proves the desired inequality. 
	
	Combining all these four cases, we conclude that (\ref{a89}) has been proved. Furthermore, (\ref{a89}) yields that
	\[
	H(a,f)=H(1-a,-f)\leq \Gamma\cdot G(1-a,-f)=\Gamma\cdot G(a,f)
	\] provided $\delta\leq a\leq 1/2$ and $-\log\frac{1-\delta}{\delta}\leq f\leq \log\frac{1-\delta}{\delta}$, which, together with (\ref{a89}), proves this lemma. 
\end{proof}

\begin{lem}\label{lem5.5} Let $\phi(t)=\log\ykh{1+\mr{e}^{-t}}$ be the logistic loss,  $\delta_0\in(0,1/3)$, $d\in\mb N$ and $P$ be a Borel probability measure on $[0,1]^d\times\hkh{-1,1}$ of which the conditional probability function $[0,1]^d\ni z\mapsto P(\hkh{1}|z)\in[0,1]$ is denoted by $\eta$.  Then there exists a measurable function
	\[
	\psi:[0,1]^d\times\{-1,1\}\to\zkh{ 0,\log\frac{10\log(1/\delta_0)}{\delta_0}}
	\]
	such that
	\beq\label{eq 5.51}
	\int_{[0,1]^d\times\{-1,1\}}{\psi\ykh{x,y}}\mr{d}P(x,y)=\inf \setl{\mc{R}_P^\phi(g) }{ \textrm{$g:[0,1]^d\to\mbR$ is measurable}} 
	\eeq
	and  
	\beq\ba\label{ineq 5.51}
	&\int_{[0,1]^d\times\{-1,1\}}{\ykh{\phi\ykh{yf(x)}-\psi(x,y)}^2}\mr{d}P(x,y)\\
	& \leq {125000\abs{\log\delta_0}^2}\cdot\int_{[0,1]^d\times\{-1,1\}}\ykh{\phi\ykh{yf(x)}-\psi(x,y)}\mr{d}P(x,y)
	&
	\ea\eeq for any measurable $f:[0,1]^d\to\zkh{\log\frac{\delta_0}{1-\delta_0},\log\frac{1-\delta_0}{\delta_0}}$.
\end{lem}
\begin{proof}
	Let
	\[\ba
	H:[0,1]\to [0,\infty),\quad t \mapsto\left\{\ba & t\log\ykh{\frac{1}{t}}+(1-t)\log\ykh{\frac{1}{1-t}},&&\text{ if }\in(0,1),\\
	&0,&&\text{ if }t\in\{0,1\}.
	\ea\right.
	\ea\]
Then it is easy to show that $ H\ykh{	\frac{\delta_0}{10\log\ykh{1/\delta_0}}}\leq\frac{4}{5}\log\ykh{\frac{1}{1-\delta_0}}\leq H\ykh{\frac{\delta_0}{\log(1/\delta_0)}}$.  Thus there exists $\delta_1\in\left(0,\frac{1}{3}\right)$ such that
	\[
	H(\delta_1)\leq \frac{4}{5}\log\ykh{\frac{1}{1-\delta_0}}
	\]
	and 
	\[
	0<\frac{\delta_0}{10\log\ykh{1/\delta_0}}\leq \delta_1\leq\frac{\delta_0}{\log(1/\delta_0)}\leq\delta_0<1/3.\]
	Take
	\[\ba
	\psi:[0,1]^d\times\{-1,1\}\to\mbR,\quad (x,y)\mapsto \left\{
	\ba
	&\phi\ykh{y\log\frac{\eta(x)}{1-\eta(x)}},&&\textrm{if }\eta(x)\in [\delta_1,1-\delta_1],\\
	&H(\eta(x)),&&\textrm{if }\eta(x)\notin [\delta_1,1-\delta_1],
	\ea
	\right.
	\ea\] which can be further expressed as
	\begin{align*}
	&\psi:[0,1]^d\times\{-1,1\}\to\mbR,\\ &\;\;\;\;\;\;\;\;\;(x,y)\mapsto \left\{
	\ba
	&\phi\ykh{y\log\frac{\eta(x)}{1-\eta(x)}},&&\text{if }\eta(x)\in [\delta_1,1-\delta_1],\\
	&0,&&\text{if }\eta(x)\in\{0,1\},\\
	&\eta(x)\log\frac{1}{\eta(x)}+(1-\eta(x))\log\frac{1}{1-\eta(x)},&&\text{if }\eta(x)\in (0, \delta_1)\cup(1-\delta_1,1).
	\ea
	\right.
	\end{align*}
	Obviously, $\psi$ is a measurable function such that   
	\[\ba
	0\leq\psi(x,y)\leq \log\frac{1}{\delta_1}\leq\log\frac{10\log(1/\delta_0)}{\delta_0},\;\quad \forall\;(x,y)\in[0,1]^d\times \{-1,1\}, 
	\ea\] and it follows immediately from Lemma \ref{23022804} that (\ref{eq 5.51}) holds. We next show (\ref{ineq 5.51}). 
	
	For any measurable function  $f:[0,1]^d\to\zkh{\log\frac{\delta_0}{1-\delta_0},\log\frac{1-\delta_0}{\delta_0}}$ and any  $x\in[0,1]^d$, if $\eta(x)\notin[\delta_1,1-\delta_1]$, then we have 
	\[\ba
	0&\leq\psi(x,y)=H(\eta(x))\leq H(\delta_1)\leq \frac{4}{5}\log\frac{1}{1-\delta_0}\\
	&=\frac{4}{5}\phi\ykh{\log\frac{1-\delta_0}{\delta_0}}\leq\frac{4}{5}\phi(yf(x))\leq\phi(yf(x)), \;\quad \forall\;y\in\{-1,1\}.
	\ea\]
	Hence $0\leq\frac{1}{5}\phi(yf(x))\leq\phi(yf(x))-\psi(x,y)\leq \phi(yf(x)),\;\forall\;y\in\{-1,1\}$, which means that 
	\[\ba
	&\ykh{\phi(yf(x))-\psi(x,y)}^2\leq \phi(yf(x))^2\leq\phi(yf(x))\phi\ykh{-\log\frac{1-\delta_0}{\delta_0}}\\
	&= \frac{1}{5}\phi(yf(x))\cdot 5\log\frac{1}{\delta_0}\leq \ykh{\phi(yf(x))-\psi(x,y)}\cdot5000\abs{\log{\delta_1}}^2,\;\quad \forall\;y\in\{-1,1\}.
	\ea\]
	Integrating both sides with respect to $y$, we obtain
	\beq\label{ineq 5.53}\ba
	&\int_{\{-1,1\}}\ykh{\phi(yf(x))-\psi(x,y)}^2\mr{d}P(y|x)\\&\leq5000\abs{\log{\delta_1}}^2\cdot\int_{\{-1,1\}} \ykh{\phi(yf(x))-\psi(x,y)}\mr{d}P(y|x).
	\ea\eeq
	\newcommand{\tema}{\log\frac{\eta(x)}{1-\eta(x)}}
	If $\eta(x)\in [\delta_1,1-\delta_1]$, then it follows from Lemma \ref{a.11} that  
	\begin{align*}
	&\int_{\{-1,1\}}\ykh{\phi(yf(x))-\psi(x,y)}^2\mr{d}P(y|x)\\&=\eta(x)\abs{\phi(f(x))-\phi\ykh{\log\frac{\eta(x)}{1-\eta(x)}}}^2+(1-\eta(x))\abs{\phi(-f(x))-\phi\ykh{-\log\frac{\eta(x)}{1-\eta(x)}}}^2\\
	&\leq5000\abs{\log\delta_1}^2\cdot\Bigg(\eta(x)\phi(f(x))+(1-\eta(x))\phi(-f(x))\\&\;\;\;\;\;\;\;\;\;\;\;\;\;\;\;\;\;\;\;\;\;\;\;\;\;\;\;\;\;\;\;\;\;\;\;\;\;-\eta(x)\phi\Big(\log\frac{\eta(x)}{1-\eta(x)}\Big)-(1-\eta(x))\phi\Big(-\log\frac{\eta(x)}{1-\eta(x)}\Big)\Bigg)\\
	&=5000\abs{\log\delta_1}^2\int_{\{-1,1\}}\ykh{\phi(yf(x))-\psi(x,y)}\mr{d}P(y|x),
	\end{align*}
	which means that (\ref{ineq 5.53}) still holds. Therefore, (\ref{ineq 5.53}) holds for all $x\in[0,1]^d$.  We then integrate both sides of  (\ref{ineq 5.53}) with respect to $x$ and obtain
	\begin{align*}
	&\int_{[0,1]^d\times\{-1,1\}}{\ykh{\phi(yf(x))-\psi(x,y)}^2}\mr{d}P(x,y)\\&\leq 5000\abs{\log{\delta_1}}^2 \int_{[0,1]^d\times\{-1,1\}}\ykh{{\phi(yf(x))-\psi(x,y)}}\mr{d}P(x,y)\\
	&\leq125000\abs{\log\delta_0}^2   \int_{[0,1]^d\times\{-1,1\}}\ykh{{\phi(yf(x))-\psi(x,y)}}\mr{d}P(x,y),
	\end{align*}
	which yields \eqref{ineq 5.51}. In conclusion, the function $\psi$ defined above has all the desired properties. Thus we complete the proof. 
\end{proof}

The following Lemma \ref{23051404} is similar to Lemma 3 of \cite{schmidt2020nonparametric}. 

\begin{lem}\label{23051404}Let  $(d,d_\star,d_*,K)\in\mb N^4$, %
 $\beta\in(0,\infty)$, $r\in[1,\infty)$,  and $q\in \mb N\cup\hkh{0}$. Suppose $h_0,h_1,\ldots,h_q, \tilde h_0,\tilde h_1,\ldots,\tilde h_q$ are functions satisfying that 
	\begin{align*}
	\begin{minipage}{0.65\textwidth}
	\begin{enumerate}[(i)]
	\item   ${\mathbf{{dom}}}(h_i)={\mathbf{{dom}}}(\tilde h_i)=[0,1]^K$ for $0<i\leq q$ and ${\mathbf{{dom}}}( h_0)={\mathbf{{dom}}}(\tilde h_0)=[0,1]^d$;
	\item $\mathbf{ran}( h_i)\cup\mathbf{ran}(\tilde h_i)\subset[0,1]^K$ for $0 \leq i<q$ and $\mathbf{ran}( h_q)\cup\mathbf{ran}(\tilde h_q)\subset\mbR$;
\item $h_q\in \mc G_\infty^{\mathbf{H}}(d_*, \beta,r)\cup \mc G_\infty^{\mathbf{M}}(d_\star)$;
\item For $0\leq i< q$ and $1\leq j\leq K$, the $j$-th coordinate function  of $h_i$ given by  $\mathbf{dom}(h_i)\ni x\mapsto (h_i(x))_j\in\mbR$ belongs to $\mc G_\infty^{\mathbf{H}}(d_*, \beta,r)\cup \mc G_\infty^{\mathbf{M}}(d_\star)$.
\end{enumerate}
	\end{minipage}
\end{align*} Then there holds
\beq
&\norm{h_{q}\circ h_{q-1}\circ\cdots \circ h_1\circ h_0-\tilde h_q\circ\tilde h_{q-1}\circ\cdots\circ \tilde h_1\circ \tilde h_0}_{[0,1]^d}\\
&\leq \abs{r\cdot d_*^{{1\qx\beta}}}^{\sum_{k=0}^{q-1}(1\qx\beta)^k}\cdot\sum_{k=0}^q\norm{\tilde h_k-h_k}_{\mathbf{dom}( h_k)}^{(1\qx\beta)^{q-k}}.
\eeq 
\end{lem}
\begin{proof} We will prove this lemma by induction on $q$. The case  $q=0$ is trivial. Now assume that $q>0$ and that the desired result holds for $q-1$. Consider the case $q$.   For each $0\leq i<q$ and $1\leq j\leq K$, denote
\[
\tilde h_{i,j}:\mathbf{dom}(\tilde h_i)\to \mbR,\;\;x\mapsto \big(\tilde h_i(x)\big)_j,  
\] and
\[
h_{i,j}:\mathbf{dom}(h_i)\to \mbR,\;\;x\mapsto \big({ h_i(x)}\big)_j. 
\] Obviously, $\mathbf{ran}(\tilde h_{i,j})\cup\mathbf{ran}(h_{i,j})\subset[0,1]$. By induction hypothesis (that is, the case $q-1$ of this lemma), we have that 
\begin{align*}
&\norm{
	h_{q-1,j}\circ h_{q-2}\circ h_{q-3}\circ\cdots\circ h_0-\tilde h_{q-1,j}\circ \tilde h_{q-2}\circ\tilde h_{q-3}\circ\cdots\circ \tilde h_0
}_{[0,1]^d}\\
&\leq \abs{r\cdot d_*^{{1\qx\beta}}}^{\sum_{k=0}^{q-2}(1\qx\beta)^k}\cdot\ykh{\norm{\tilde h_{q-1,j}-h_{q-1,j}}_{\mathbf{dom}( h_{q-1,j})}+\sum_{k=0}^{q-2}\norm{\tilde h_k-h_k}_{\mathbf{dom}( h_k)}^{(1\qx\beta)^{q-1-k}}}\\
&\leq \abs{r\cdot d_*^{{1\qx\beta}{}}}^{\sum_{k=0}^{q-2}(1\qx\beta)^k}\cdot{\sum_{k=0}^{q-1}\norm{\tilde h_k-h_k}_{\mathbf{dom}( h_k)}^{(1\qx\beta)^{q-1-k}}},\;\forall\;j\in\mb Z\cap(0,K]. 
\end{align*} Therefore, 
\beq\label{23051403}
&\norm{
	h_{q-1}\circ h_{q-2}\circ h_{q-3}\circ\cdots\circ h_0-\tilde h_{q-1}\circ \tilde h_{q-2}\circ\tilde h_{q-3}\circ\cdots\circ \tilde h_0
}_{[0,1]^d}\\
&=\sup_{j\in\mb Z\cap(0,K]}\norm{
	h_{q-1,j}\circ h_{q-2}\circ h_{q-3}\circ\cdots\circ h_0-\tilde h_{q-1,j}\circ \tilde h_{q-2}\circ\tilde h_{q-3}\circ\cdots\circ \tilde h_0
}_{[0,1]^d}\\
&\leq \abs{r\cdot d_*^{{1\qx\beta}{}}}^{\sum_{k=0}^{q-2}(1\qx\beta)^k}\cdot{\sum_{k=0}^{q-1}\norm{\tilde h_k-h_k}_{\mathbf{dom}( h_k)}^{(1\qx\beta)^{q-1-k}}}. 
\eeq We next show that 
\beq\label{23051402}
\abs{h_q(x)-h_q(x')}\leq r\cdot d_*^{{1\qx\beta}{}}\cdot\norm{x-x'}_\infty^{1\qx\beta},\;\forall\;x,x'\in[0,1]^K
\eeq by considering three cases. 

\textbf{Case \uppercase\expandafter{\romannumeral1}: } $h_q\in \mc G_\infty^{\mathbf{H}}(d_*, \beta,r)$ and $\beta>1$. In this case, we must have that  $h_q\in \mc G_K^{\mathbf{H}}(d_*, \beta,r)$ since $\mathbf{dom}(h_q)=[0,1]^K$. Therefore, 
there exist  $I\subset \hkh{1,2,\ldots,K}$ and  $g\in \mc{B}^{\beta}_r\ykh{[0,1]^{d_*}}$ such that $\#(I)=d_*$ and $h_q(x)=g((x)_I)$ for all $x\in[0,1]^K$.  Denote $\lambda:=\beta+1-\ceil{\beta}$.  We then use Taylor's formula to deduce that 
\begin{align*}
&\abs{h_q(x)-h_q(x')}=\abs{g((x)_I)-g((x')_I)}\xlongequal{\exists\,\xi\in[0,1]^{d_*}}\abs{\nabla g(\xi)\cdot \ykh{(x)_I-(x')_I}}\\&\leq \norm{\nabla g(\xi)}_\infty\cdot \norm{(x)_I-(x')_I}_1\leq \norm{\nabla g}_{[0,1]^d}\cdot{d_*}\cdot{\norm{(x)_I-(x')_I}_\infty} \\
&\leq \norm{g}_{\mc C^{\beta-\lambda,\lambda}([0,1]^d)}\cdot {d_*}\cdot\norm{(x)_I-(x')_I}_\infty\leq  r\cdot {d_*}\cdot\norm{(x)_I-(x')_I}_\infty\\
&\leq r\cdot d_*^{{1\qx\beta}{}}\cdot\norm{x-x'}_\infty^{1\qx\beta},\;\forall\;x,x'\in[0,1]^K,
\end{align*} which yields \eqref{23051402}. 

\textbf{Case \uppercase\expandafter{\romannumeral2}: } $h_q\in \mc G_\infty^{\mathbf{H}}(d_*, \beta,r)$ and $\beta\leq 1$.  In this case, we still have that  $h_q\in \mc G_K^{\mathbf{H}}(d_*, \beta,r)$. Therefore, 
there exist  $I\subset \hkh{1,2,\ldots,K}$ and  $g\in \mc{B}^{\beta}_r\ykh{[0,1]^{d_*}}$ such that $\#(I)=d_*$ and $h_q(x)=g((x)_I)$ for all $x\in[0,1]^K$. Consequently, 
\begin{align*}
&\abs{h_q(x)-h_q(x')}=\abs{g((x)_I)-g((x')_I)}\leq \norm{(x)_I-(x')_I}_2^{\beta}\cdot\sup_{[0,1]^{d_*}\ni z\neq z'\in[0,1]^{d_*}}\frac{\abs{g(z)-g(z')}}{\norm{z-z'}_2^\beta}\\
&\leq \norm{(x)_I-(x')_I}_2^{\beta}\cdot \norm{g}_{\mc C^{0,\beta}([0,1]^d)}\leq  \norm{(x)_I-(x')_I}_2^{\beta}\cdot r\leq r\cdot \abs{\sqrt{d_*}\cdot\norm{x-x'}_\infty}^\beta\\
&\leq r\cdot d_*^{1\qx\beta}\cdot \norm{x-x'}_\infty^{1\qx\beta},\;\forall\;x,x'\in[0,1]^K,
\end{align*} which yields \eqref{23051402}. 

\textbf{Case \uppercase\expandafter{\romannumeral3}}: $h_q\in\mc G_\infty^{\mathbf{M}}(d_\star)$.  In this case, we have that there exists $I\subset\hkh{1,2,\ldots,K}$ such that $1\leq \#(I)\leq d_\star$ and  $h_q(x)=\max\hkh{(x)_i\big|i\in I}$ for all $x\in[0,1]^K$. Consequently, 
\begin{align*}
&\abs{h_q(x)-h_q(x')}=\abs{\max\hkh{(x)_i\big|i\in I}-\max\hkh{(x')_i\big|i\in I}}\leq \norm{(x)_I-(x')_I}_\infty\\
&\leq r\cdot d_*^{1\qx\beta}\cdot \norm{x-x'}_\infty\leq r\cdot d_*^{1\qx\beta}\cdot \norm{x-x'}_\infty^{1\qx\beta},\;\forall\;x,x'\in[0,1]^K, 
\end{align*}  which yields \eqref{23051402}. 

Combining the above three cases, we deduce that  \eqref{23051402}   always holds true.  From  \eqref{23051402} and \eqref{23051403} we obtain that
\begin{align*}
&\abs{h_{q}\circ h_{q-1}\circ\cdots \circ h_0(x)-\tilde h_q\circ\tilde h_{q-1}\circ\cdots\circ \tilde h_0(x)}\\
&\leq \abs{h_{q}\circ h_{q-1}\circ\cdots \circ h_0(x)-h_q\circ\tilde h_{q-1}\circ\cdots\circ \tilde h_0(x)}\\&\;\;\;\;\;\;\;\;+\abs{h_{q}\circ \tilde h_{q-1}\circ\cdots \circ \tilde h_0(x)-\tilde h_q\circ\tilde h_{q-1}\circ\cdots\circ \tilde h_0(x)}\\
&\leq r\cdot d_*^{{1\qx\beta}{}}\cdot\norm{ h_{q-1}\circ\cdots \circ h_0(x)-\tilde h_{q-1}\circ\cdots\circ \tilde h_0(x)}_\infty^{1\qx\beta}+\norm{h_q-\tilde h_q}_{\mathbf{dom}(h_q)}\\
&\leq r\cdot d_*^{{1\qx\beta}{}}\cdot\norm{ h_{q-1}\circ\cdots \circ h_0-\tilde h_{q-1}\circ\cdots\circ \tilde h_0}_{[0,1]^d}^{1\qx\beta}+\norm{h_q-\tilde h_q}_{\mathbf{dom}(h_q)}\\
&\leq r\cdot d_*^{{1\qx\beta}{}}\cdot\abs{ \abs{r\cdot d_*^{{1\qx\beta}{}}}^{\sum_{k=0}^{q-2}(1\qx\beta)^k}\cdot{\sum_{k=0}^{q-1}\norm{\tilde h_k-h_k}_{\mathbf{dom}( h_k)}^{(1\qx\beta)^{q-1-k}}}}^{1\qx\beta}+\norm{h_q-\tilde h_q}_{\mathbf{dom}(h_q)}\\
&= r\cdot d_*^{{1\qx\beta}{}}\cdot \abs{r\cdot d_*^{{1\qx\beta}{}}}^{\sum_{k=0}^{q-2}(1\qx\beta)^{k+1}}\cdot\abs{{\sum_{k=0}^{q-1}\norm{\tilde h_k-h_k}_{\mathbf{dom}( h_k)}^{(1\qx\beta)^{q-1-k}}}}^{1\qx\beta}+\norm{h_q-\tilde h_q}_{\mathbf{dom}(h_q)}\\
&\leq r\cdot d_*^{{1\qx\beta}{}}\cdot \abs{r\cdot d_*^{{1\qx\beta}{}}}^{\sum_{k=0}^{q-2}(1\qx\beta)^{k+1}}\cdot{\sum_{k=0}^{q-1}\abs{\norm{\tilde h_k-h_k}_{\mathbf{dom}( h_k)}^{(1\qx\beta)^{q-1-k}}}}^{1\qx\beta}+\norm{h_q-\tilde h_q}_{\mathbf{dom}(h_q)}\\
&\leq r\cdot d_*^{{1\qx\beta}{}}\cdot \abs{r\cdot d_*^{{1\qx\beta}{}}}^{\sum_{k=0}^{q-2}(1\qx\beta)^{k+1}}\cdot\abs{{\sum_{k=0}^{q-1}\abs{\norm{\tilde h_k-h_k}_{\mathbf{dom}( h_k)}^{(1\qx\beta)^{q-1-k}}}}^{1\qx\beta}+\norm{h_q-\tilde h_q}_{\mathbf{dom}(h_q)}}\\
&= \abs{r\cdot d_*^{{1\qx\beta}{}}}^{\sum_{k=0}^{q-1}(1\qx\beta)^{k}}\cdot{\sum_{k=0}^{q}\norm{\tilde h_k-h_k}_{\mathbf{dom}( h_k)}^{(1\qx\beta)^{q-k}}},\;\forall\;x\in[0,1]^d. 
\end{align*} Therefore, 
\begin{align*}
	&\norm{h_{q}\circ h_{q-1}\circ\cdots \circ h_1\circ h_0-\tilde h_q\circ\tilde h_{q-1}\circ\cdots\circ \tilde h_1\circ \tilde h_0}_{[0,1]^d}\\
&=\sup_{x\in[0,1]^d}\abs{h_{q}\circ h_{q-1}\circ\cdots \circ h_1\circ h_0(x)-\tilde h_q\circ\tilde h_{q-1}\circ\cdots\circ \tilde h_1\circ \tilde h_0(x)}\\
	&\leq \abs{r\cdot d_*^{{1\qx\beta}}}^{\sum_{k=0}^{q-1}(1\qx\beta)^k}\cdot\sum_{k=0}^q\norm{\tilde h_k-h_k}_{\mathbf{dom}( h_k)}^{(1\qx\beta)^{q-k}},
\end{align*} meaning that the desired result holds for $q$. 

In conclusion, according to mathematical induction, we have that the desired result holds for all $q\in\mb N\cup\hkh{0}$. This completes the proof. \end{proof}

\begin{lem}\label{23051601}Let $k$ be an positive integer. Then there exists a neural network \[\tilde f\in \fdnn_k\ykh{1+2\cdot\ceil{\frac{\log k}{\log 2}},2k, 26\cdot 2^{{}^{\ceil{\frac{\log k}{\log 2}}}} -20-2\cdot\ceil{\frac{\log k}{\log 2}},1,1}\] such that
\[
\tilde f(x)=\norm{x}_\infty,\;\forall\;x\in\mbR^k. 
\] 
\end{lem}
\begin{proof}
We argue by induction. 

Firstly, consider the case $k=1$. Define 

\begin{align*}
	\tilde f_1:\mbR\to\mbR,  x\mapsto \sigma(x)+\sigma(-x). 
\end{align*} Obviously, 
\begin{align*}
\tilde f_1&\in\fdnn_1(1,2,6,1,1)\\&\subset \fdnn_1\ykh{1+2\cdot\ceil{\frac{\log 1}{\log 2}},2\cdot 1, 26\cdot 2^{{}^{\ceil{\frac{\log 1}{\log 2}}}} -20-2\cdot\ceil{\frac{\log 1}{\log 2}},1,1}
\end{align*} and $\tilde f(x)=\sigma(x)+\sigma(-x)=\abs{x}=\norm{x}_\infty$ for all $x\in\mbR=\mbR^1$. This proves the $k=1$ case.

Now assume that the desired result holds for $k=1,2,3,\ldots,m-1$ ($m\geq 2$), and consider the case $k=m$. Define
\begin{align*}
\tilde g_1:\mbR^m&\to\mbR^{{}^{\flr{\frac{m}{2}}}},\\
x&\mapsto  \begin{pmatrix}
(x)_1,
(x)_2,
\cdots,
(x)_{{}_{\flr{\frac{m}{2}}-1}},
(x)_{{}_{\flr{\frac{m}{2}}}}
\end{pmatrix},
\end{align*} 
\begin{align*}
	\tilde g_2:\mbR^m&\to\mbR^{{}^{\ceil{\frac{m}{2}}}},\\
x&\mapsto  \ykh{
		(x)_{{}_{\flr{\frac{m}{2}}+1}},
		(x)_{{}_{\flr{\frac{m}{2}}+2}},
		\cdots,
		(x)_{m-1},
		(x)_{m}
	},
\end{align*}
and 
\beq\label{230510}
	\tilde{f}_m:\mbR^m&\to\mbR,\\
	x&\mapsto\sigma\ykh{\frac{1}{2}\cdot\sigma\Big(\tilde f_{{}_{\flr{\frac{m}{2}}}}(\tilde g_1(x))\Big)-\frac{1}{2}\cdot\sigma\Big(\tilde f_{{}_{\ceil{\frac{m}{2}}}}(\tilde g_2(x))\Big)}\\&\;\;\;\;\;\;\;\;\;\;\;\;\;+\sigma\ykh{\frac{1}{2}\cdot\sigma\Big(\tilde f_{{}_{\ceil{\frac{m}{2}}}}(\tilde g_2(x))\Big)-\frac{1}{2}\cdot\sigma\Big(\tilde f_{{}_{\flr{\frac{m}{2}}}}(\tilde g_1(x))\Big)}\\
	&\;\;\;\;\;\;\;\;\;\;\;\;\;+\frac{1}{2}\cdot\sigma\Big(\tilde f_{{}_{\flr{\frac{m}{2}}}}(\tilde g_1(x))\Big)+\frac{1}{2}\cdot\sigma\Big(\tilde f_{{}_{\ceil{\frac{m}{2}}}}(\tilde g_2(x))\Big). 
\eeq  It follows from the induction hypothesis that 
\begin{align*}
{\tilde f_{{}_{{\ceil{\frac{m}{2}}}}}\circ\tilde g_2}&\in\fdnn_m\ykh{1+2\ceil{\frac{\log \ceil{\frac{m}{2}}}{\log 2}},2\ceil{\frac{m}{2}}, 26\cdot 2^{{}^{\ceil{\frac{\log \ceil{\frac{m}{2}}}{\log 2}}}} -20-2\ceil{\frac{\log \ceil{\frac{m}{2}}}{\log 2}},1,1}\\
&=\fdnn_m\ykh{-1+2\ceil{\frac{\log {{m}}}{\log 2}},2\ceil{\frac{m}{2}}, 13\cdot 2^{{}^{\ceil{\frac{\log {{m}}}{\log 2}}}} -18-2\ceil{\frac{\log {{m}}}{\log 2}},1,1}
\end{align*} and
\begin{align*}
	{\tilde f_{{}_{{\flr{\frac{m}{2}}}}}\circ\tilde g_1}&\in\fdnn_m\ykh{1+2\ceil{\frac{\log \flr{\frac{m}{2}}}{\log 2}},2\flr{\frac{m}{2}}, 26\cdot 2^{{}^{\ceil{\frac{\log \flr{\frac{m}{2}}}{\log 2}}}} -20-2\ceil{\frac{\log \flr{\frac{m}{2}}}{\log 2}},1,1}\\
	&\subset\fdnn_m\ykh{-1+2\ceil{\frac{\log {{m}}}{\log 2}},2\flr{\frac{m}{2}}, 13\cdot 2^{{}^{\ceil{\frac{\log {{m}}}{\log 2}}}} -18-2\ceil{\frac{\log {{m}}}{\log 2}},1,1},
\end{align*}  %
which, together with \eqref{230510}, 
 yield
\beq\label{2305102}
\tilde f_m&\in \fdnn_m\left(2-1+2\ceil{\frac{\log {{m}}}{\log 2}},2\ceil{\frac{m}{2}}+2\flr{\frac{m}{2}}, \right.\\&\;\;\;\;\;\;\;\;\;\;\;\;\;\;\;\;\;\;\;\;\;\;\;\;\;\;\left.2\cdot\abs{13\cdot 2^{{}^{\ceil{\frac{\log {{m}}}{\log 2}}}} -18-2\ceil{\frac{\log {{m}}}{\log 2}}}+2\ceil{\frac{\log m}{\log 2}}+16,1,\infty\right)\\
&=\fdnn_m\ykh{1+\ceil{\frac{\log {{m}}}{\log 2}},2m, 26\cdot 2^{{}^{\ceil{\frac{\log {{m}}}{\log 2}}}} -20-2\ceil{\frac{\log {{m}}}{\log 2}},1,\infty}
\eeq
(cf. Figure \ref{fig6x}). Besides, it is easy to verify that

\begin{figure}[H]	\centering
	\betikz
	\tikzset{
		mybox/.style ={
			rectangle, %
			rounded corners =5pt, %
			minimum width =90pt, %
			minimum height =20pt, %
			inner sep=0.6pt, %
			draw=blue, %
			fill=cyan
	}}
	\tikzset{
		ttinybox/.style ={
			rectangle, %
			rounded corners =3pt, %
			minimum width =25pt, %
			minimum height =25pt, %
			inner sep=2pt, %
			draw=blue, %
			fill=cyan
	}}
	\tikzset{
		tinybox/.style ={
			rectangle, %
			rounded corners =5pt, %
			minimum width =20pt, %
			minimum height =20pt, %
			inner sep=5pt, %
			draw=blue, %
			fill=cyan
	}}
	\tikzset{
		gaobox/.style ={
			rectangle, %
			rounded corners =5pt, %
			minimum width =100pt, %
			minimum height =100pt, %
			inner sep=5pt, %
			draw=blue, %
			fill=cyan
	}}
	\tikzset{
		aibox/.style ={
			rectangle, %
			rounded corners =5pt, %
			minimum width =100pt, %
			minimum height =40pt, %
			inner sep=5pt, %
			draw=blue, %
			fill=cyan
	}}
	\tikzset{
		bigbox/.style ={
			rectangle, %
			rounded corners =5pt, %
			minimum width =379.9pt, %
			minimum height =343pt, %
			inner sep=0pt, %
			draw=black, %
			fill=none
	}}
	\tikzset{
		tinycircle/.style ={
			circle, %
			minimum width =6pt, %
			minimum height =6pt, %
			inner sep=0pt, %
			draw=red, %
			fill=red %
		}
	}
	\tikzset{ %
		mypt/.style ={
			circle, %
			minimum width =0pt, %
			minimum height =0pt, %
			inner sep=0pt, %
			draw=none, %
		}
	}
	
	\node[bigbox] at (-0.53,4.4) {};
	
	\node[tinycircle] (01) at (-5,0) {};
	\node[tinycircle] (02) at (-4.5,0) {};
	\node[tinycircle] (03) at (-3,0) {};
	\node[tinycircle] (01x) at (-1,0) {};
	\node[tinycircle] (02x) at (-0.5,0) {};
	\node[tinycircle] (03x) at (1,0) {};
	\node[gaobox] at (-4,2.4) {$\tilde{f}_{{}_{\ceil{\frac{m}{2}}}}$};
	\node[aibox] at (-0,1.35) {$\tilde{f}_{{}_{\flr{\frac{m}{2}}}}$};
	\node[tinycircle] (04) at (-1,0) {};
	\node[mypt] (04x) at (-3.75,0) {$\cdots$};
	\node[mypt] (06x) at (0.25,0) {$\cdots$};
	
	\node[mypt] (11) at (-5.2,0.65) {};
	\node[mypt] (12) at (-4.6,0.65) {};
	\node[mypt] (13) at (-3.4,0.65) {};
	\node[mypt] (14) at (-2.8,0.65) {};

	\node[mypt] (11x) at (-1.2,0.65) {};
	\node[mypt] (12x) at (-0.6,0.65) {};
	\node[mypt] (13x) at (0.6,0.65) {};
	\node[mypt] (14x) at (1.2,0.65) {};

	\filldraw[->,blue] (01x)--(11x){};
	\filldraw[->,blue] (01x)--(12x){};
	\filldraw[->,blue] (01x)--(13x){};
	\filldraw[->,blue] (01x)--(14x){};
	\filldraw[->,blue] (02x)--(11x){};
	\filldraw[->,blue] (02x)--(12x){};
	\filldraw[->,blue] (02x)--(13x){};
	\filldraw[->,blue] (02x)--(14x){};
	\filldraw[->,blue] (03x)--(11x){};
	\filldraw[->,blue] (03x)--(12x){};
	\filldraw[->,blue] (03x)--(13x){};
	\filldraw[->,blue] (03x)--(14x){};
	
	\filldraw[->,blue] (01)--(11){};
	\filldraw[->,blue] (01)--(12){};
	\filldraw[->,blue] (01)--(13){};
	\filldraw[->,blue] (01)--(14){};
	\filldraw[->,blue] (02)--(11){};
	\filldraw[->,blue] (02)--(12){};
	\filldraw[->,blue] (02)--(13){};
	\filldraw[->,blue] (02)--(14){};
	\filldraw[->,blue] (03)--(11){};
	\filldraw[->,blue] (03)--(12){};
	\filldraw[->,blue] (03)--(13){};
	\filldraw[->,blue] (03)--(14){};
	
	\node[mypt] (-10) at (-4,-0.35) {$x''$};
	\node[mypt] (-10) at (0,-0.35) {$x'$};
	\node[below] (-12) at (-2,-0.8) {Input: $x=(x',x'')$ with $x'\in\mbR^{{}^{\flr{\frac{m}{2}}}}$ and $x''\in\mbR^{{}^{\ceil{\frac{m}{2}}}}$};
	
	\node[tinycircle] (32x) at (0,2.6) {};
	\node[tinycircle] (33x) at (0,3.1) {};
	\node[tinycircle] (34x) at (0,3.6) {};
	\node[tinycircle] (35x) at (0,4.8) {};
	
	\node[mypt] (36x) at (0,4.3) {$\vdots$};
	\node[mypt] (36x) at (1.2,4.1) {$\vdots$};
	\node[mypt] (36x) at (1.2,3.45) {$\vdots$};

	\filldraw[->,blue] (32x)--(33x){};
	\filldraw[->,blue] (33x)--(34x){};

	\node[mypt] (21x) at (-1.2,2.05) {};
	\node[mypt] (22x) at (-0.6,2.05) {};
	\node[mypt] (23x) at (0.6,2.05) {};
	\node[mypt] (24x) at (1.2,2.05) {};

	\filldraw[->,blue] (21x)--(32x){};
	\filldraw[->,blue] (22x)--(32x){};
	\filldraw[->,blue] (23x)--(32x){};
	\filldraw[->,blue] (24x)--(32x){};
	
	\node[mypt] (21) at (-5.2,4.15) {};
	\node[mypt] (22) at (-4.6,4.15) {};
	\node[mypt] (23) at (-3.5,4.15) {};
	\node[mypt] (24) at (-2.8,4.15) {};	
	
	\node[tinycircle] (35) at (-4,4.8) {};
	\filldraw[->,blue] (21)--(35){};
	\filldraw[->,blue] (22)--(35){};
	\filldraw[->,blue] (23)--(35){};
	\filldraw[->,blue] (24)--(35){};
	
	\node[right] (36) at (-3.85,4.8) {%
	{\fontsize{9}{11}\selectfont
	$\sigma\Big(\tilde f_{{}_{\ceil{\frac{m}{2}}}}(x'')\Big)$}};
	\node[right] (36x) at (0.15,4.8) {%
	{\fontsize{9}{11}\selectfont
	$\sigma\Big(\tilde f_{{}_{\flr{\frac{m}{2}}}}(x')\Big)$}};

			\node[right] (36x3) at (0.15,2.6) {%
			{\fontsize{9}{11}\selectfont
			$\sigma\Big(\tilde f_{{}_{\flr{\frac{m}{2}}}}(x')\Big)$}};
	
	\node[tinycircle] (41) at (4.6,7.2) {};
		\node[tinycircle] (42x) at (-5,7.2) {};
		\node[tinycircle] (42) at (-7,7.2) {};
		\node[tinycircle] (41x) at (0.,7.2) {};
		\filldraw[->,blue] (35x)--(42x){};
		\filldraw[->,blue] (35x)--(41){};
		\filldraw[->,blue] (35)--(41x){};
		\filldraw[->,blue] (35x)--(41x){};
		\filldraw[->,blue] (35)--(42){};
		\filldraw[->,blue] (35)--(42x){};
		
		\node[tinycircle] (5) at (-3,9) {};
		\filldraw[->,blue] (41)--(5){};
		\filldraw[->,blue] (41x)--(5){};
		\filldraw[->,blue] (42)--(5){};
		\filldraw[->,blue] (42x)--(5){};

	\node[mypt] () at (-3.4,7.59) {%
	{\fontsize{9}{11}\selectfont
	$\sigma\Big({\frac{1}{2}\sigma\Big(\tilde f_{{}_{\ceil{\frac{m}{2}}}}(x'')\Big)-\frac{1}{2}\sigma\Big(\tilde f_{{}_{\flr{\frac{m}{2}}}}(x')\Big)}\Big)$}};
	\node[mypt] () at (5.1,7.65) {%
	{\fontsize{9}{11}\selectfont
	$\sigma\Big(\tilde f_{{}_{\flr{\frac{m}{2}}}}(x')\Big)$}};
	\node[mypt] () at (-6.2,6.78) {%
	{\fontsize{9}{11}\selectfont
	$\sigma\Big(\tilde f_{{}_{\ceil{\frac{m}{2}}}}(x'')\Big)$}};
	\node[mypt] () at (1.4,6.7) {%
	{\fontsize{9}{11}\selectfont
	$\sigma\Big({\frac{1}{2}\sigma\Big(\tilde f_{{}_{\flr{\frac{m}{2}}}}(x')\Big)-\frac{1}{2}\sigma\Big(\tilde f_{{}_{\ceil{\frac{m}{2}}}}(x'')\Big)}\Big)$}};

	\node[above] () at (-2,9.75) {Output};
	\node[above] () at (-2.11,8.99) {%
	{\fontsize{11}{11}\selectfont
	$\tilde f_m(x)$}};
	\eetikz
	\caption{The network $\tilde{f}_m$. }
	\label{fig6x}
\end{figure} 
\beq\label{2305101}
&\tilde f_m(x)=\max\hkh{\sigma\Big(\tilde f_{{}_{\flr{\frac{m}{2}}}}(\tilde g_1(x))\Big),\sigma\Big(\tilde f_{{}_{\ceil{\frac{m}{2}}}}(\tilde g_2(x))\Big)}\\
&=\max\hkh{\sigma\Bigg(\norm{\Big((x)_1,\ldots,(x)_{{}_{\flr{\frac{m}{2}}}}\Big)}_\infty\Bigg),\sigma\Bigg(\norm{\Big((x)_{{}_{\flr{\frac{m}{2}}+1}},\ldots,(x)_{{m}}\Big)}_\infty\Bigg)}\\
&=\max\hkh{\norm{\Big((x)_1,\ldots,(x)_{{}_{\flr{\frac{m}{2}}}}\Big)}_\infty,\norm{\Big((x)_{{}_{\flr{\frac{m}{2}}+1}},\ldots,(x)_{{m}}\Big)}_\infty}\\
&=\max\hkh{\max_{1\leq i\leq \flr{\frac{m}{2}}}\abs{(x)_i},\max_{\flr{\frac{m}{2}}+1\leq i\leq m}\abs{(x)_i}}=\max_{1\leq i\leq m}\abs{(x)_i}=\norm{x}_\infty,\;\forall\;x\in\mbR^m.
\eeq 
Combining \eqref{2305102} and \eqref{2305101}, we deduce that the desired result holds for $k=m$. Therefore, according to mathematical induction, we have that the desired result hold for all positive integer $k$. This completes the proof. 
\end{proof}

\begin{lem}\label{23051701}Let $(\e,d,d_\star,d_*,\beta,r)\in(0,1/2]\times\mb N\times\mb N\times\mb N\times(0,\infty)\times(0,\infty)$ and $f$ be a function from $[0,1]^d$ to $\mbR$. Suppose  $f\in \mc G_\infty^{\mathbf{H}}(d_*, \beta,r\qd 1)\cup \mc G_\infty^{\mathbf{M}}(d_\star)$. Then there exist constants $E_1,E_2,E_3\in(0,\infty)$ only depending on $(d_*,\beta,r)$ and a neural network
	\[
	\tilde{f}\in\fdnn_{d}\ykh{3\log d_\star+E_1\log\frac{1}{\e},2d_\star+E_2\e^{-\frac{d_*}{\beta}},52d_\star+E_3\e^{-\frac{d_*}{\beta}}\log\frac{1}{\e},1,\infty}
\] such that
\[
\sup_{x\in[0,1]^d}\abs{\tilde f(x)-f(x)}<2 \e. 
\]\end{lem}
\begin{proof} According to Corollary \ref{corollaryA2}, there exist constants $E_1, E_2, E_3\in(6,\infty)$ only depending on $(d_*,\beta,r)$, such that 
	\beq\label{23051603}
	&\inf\setr{\sup_{x\in[0,1]^{d_*}}\abs{g(x)-\tilde g(x)}}{\tilde{g}\in\fdnn_{d_*}\ykh{E_1\log\frac{1}{t},E_2t^{-\frac{d_*}{\beta}},E_3t^{-\frac{d_*}{\beta}}\log\frac{1}{t},1,\infty}}\\
	&\leq t,\;\forall\;g\in \mc{B}^{\beta}_{r\qd 1}\ykh{[0,1]^{d_*}},\;\forall\;t\in(0,1/2]. 
	\eeq We next consider two cases.
	
\textbf{Case \uppercase\expandafter{\romannumeral1}: } $f\in \mc G_\infty^{\mathbf{M}}(d_\star)$.  In this case, we must have $f\in\mc G_d^{\mathbf{M}}(d_\star)$, since $\mathbf{dom}(f)=[0,1]^d$. Therefore, there exists $I\subset\hkh{1,2,\ldots,d}$, such that $1\leq \#(I)\leq d_\star$ and 
\[
f(x)=\max\hkh{(x)_i\big|i\in I},\;\forall\;x\in[0,1]^d. 
\] According to Lemma \ref{23051601}, there exists
\beq\label{23051602}
\tilde g&\in \fdnn_{\#(I)}\ykh{1+2\cdot\ceil{\frac{\log \#(I)}{\log 2}},2\cdot\#(I), 26\cdot 2^{{}^{\ceil{\frac{\log \#(I)}{\log 2}}}} -20-2\cdot\ceil{\frac{\log \#(I)}{\log 2}},1,1}\\
&\subset \fdnn_{\#(I)}\ykh{1+2\cdot\ceil{\frac{\log d_\star}{\log 2}},2d_\star, 26\cdot 2^{{}^{\ceil{\frac{\log d_\star}{\log 2}}}} ,1,1}\\&\subset \fdnn_{\#(I)}\ykh{3+3\log d_\star,2d_\star, 52d_\star ,1,1}\\
\eeq such that
\[
\tilde g(x)=\norm{x}_\infty,\;\forall\;x\in\mbR^{\#(I)}. 
\] Define $\tilde f:\mbR^d\to\mbR,\;x\mapsto \tilde g((x)_I)$. Then it follows from \eqref{23051602}  that 
\begin{align*}
\tilde f&\in \fdnn_{d}\ykh{3+3\log d_\star,2d_\star, 52d_\star ,1,1}\\
&\subset \fdnn_{d}\ykh{3\log d_\star+E_1\log\frac{1}{\e},2d_\star+E_2\e^{-\frac{d_*}{\beta}},52d_\star+E_3\e^{-\frac{d_*}{\beta}}\log\frac{1}{\e},1,\infty}
\end{align*} and 
\begin{align*}
&\sup_{x\in[0,1]^d}\abs{f(x)-\tilde f(x)}=\sup_{x\in[0,1]^d}\abs{\max\hkh{(x)_i\big|i\in I}-\tilde g((x)_I)}\\
&=\sup_{x\in[0,1]^d}\abs{\max\hkh{\abs{(x)_i}\big|i\in I}-\norm{(x)_I}_\infty}=0<2\e,
\end{align*}  which yield the desired result. 

\textbf{Case \uppercase\expandafter{\romannumeral2}: } $f\in \mc G_\infty^{\mathbf{H}}(d_*, \beta,r\qd 1)$.  In this case, we must have $f\in\mc G_d^{\mathbf{H}}(d_*, \beta,r\qd 1)$, since $\mathbf{dom}(f)=[0,1]^d$. By definition, there exist $I\subset\hkh{1,2,\ldots,d}$ and $g\in \mc{B}^{\beta}_{r\qd 1}\ykh{[0,1]^{d_*}}$ such that $\#(I)=d_*$ and $f(x)=g\ykh{(x)_I}$ for all $x\in[0,1]^d$.  Then it follows from \eqref{23051603} that there exists $\tilde{g}\in\fdnn_{d_*}\ykh{E_1\log\frac{1}{\e},E_2\e^{-\frac{d_*}{\beta}},E_3\e^{-\frac{d_*}{\beta}}\log\frac{1}{\e},1,\infty}$ such that 
\[
\sup_{x\in[0,1]^{d_*}}\abs{g(x)-\tilde g(x)}<2\e. 
\] Define $\tilde f:\mbR^d\to\mbR,\;x\mapsto \tilde g((x)_I)$. Then we have that \begin{align*}
\tilde{f}&\in\fdnn_{d}\ykh{E_1\log\frac{1}{\e},E_2\e^{-\frac{d_*}{\beta}},E_3\e^{-\frac{d_*}{\beta}}\log\frac{1}{\e},1,\infty}\\
&\subset \fdnn_{d}\ykh{3\log d_\star+E_1\log\frac{1}{\e},2d_\star+E_2\e^{-\frac{d_*}{\beta}},52d_\star+E_3\e^{-\frac{d_*}{\beta}}\log\frac{1}{\e},1,\infty}
\end{align*} and
\begin{align*}
\sup_{x\in[0,1]^d}\abs{f(x)-\tilde f(x)}=\sup_{x\in[0,1]^d}\abs{g((x)_I)-\tilde g((x)_I)}=\sup_{x\in[0,1]^{d_*}}\abs{g(x)-\tilde g(x)}<2\e. 
\end{align*} These yield the desired result again. 

In conclusion, the desired result always holds. Thus we completes the proof of this lemma. \end{proof}

\begin{lem} \label{23051904} Let $\beta\in(0,\infty)$, $r\in(0,\infty)$,  $q\in \mb N\cup\hkh{0}$, and   $(d,d_\star, d_*, K)\in\mb N^4$ with  $d_*\leq \min\hkh{d,K+\idf_{\hkh{0}}(q)\cdot(d-K)}$. Suppose $f\in \mc G_d^{\mathbf{CHOM}}(q, K,d_\star, d_*, \beta,r)$ and $\e\in(0,1/2]$. Then  there exist  $E_7\in (0,\infty)$ only depending on $(d_*,\beta,r,q)$ and 
		\beq\label{23051901}
\tilde f&\in\fdnn_{d}\left((q+1)\cdot\abs{3\log d_{\star}+E_7\log\frac{1}{\e}},2K d_{\star}+KE_7\e^{-\frac{d_*}{\beta\cdot(1\qx\beta)^q}},\right.\\&\;\;\;\;\;\;\;\;\;\;\;\;\;\;\;\;\;\;\;\;\;\;\;\;\;\;\;\;\;\;\;\;\;\;\;\;\;\;\;\;\left.(Kq+1)\cdot\abs{63 d_{\star}+E_7\e^{-\frac{d_*}{\beta\cdot(1\qx\beta)^q}}\log\frac{1}{\e}},1,\infty\right)
	\eeq such that \beq\label{23051902}
	\sup_{x\in[0,1]^d}\abs{f(x)-\tilde{f}(x)}\leq\frac{\e}{8}.
	\eeq
\end{lem}
\begin{proof}  
By the definition of $\mc G_d^{\mathbf{CHOM}}(q, K, d_\star,d_*, \beta,r)$, there exist functions $h_0, h_1,\ldots,h_q$ such that \begin{align*}
	\begin{minipage}{0.81\textwidth}
		\begin{enumerate}[(i)]
	\item   ${\mathbf{{dom}}}(h_i)=[0,1]^K$ for $0<i\leq q$ and ${\mathbf{{dom}}}( h_0)=[0,1]^d$;
			\item $\mathbf{ran}( h_i)\subset[0,1]^K$ for $0 \leq i<q$ and $\mathbf{ran}( h_q)\subset\mbR$;
			\item $h_q\in \mc G_\infty^{\mathbf{H}}(d_*, \beta,r\qd 1)\cup \mc G_\infty^{\mathbf{M}}(d_\star)$;
			\item For $0\leq i< q$ and $1\leq j\leq K$, the $j$-th coordinate function  of $h_i$ given by  $\mathbf{dom}(h_i)\ni x\mapsto (h_i(x))_j\in\mbR$ belongs to $\mc G_\infty^{\mathbf{H}}(d_*, \beta,r\qd 1)\cup \mc G_\infty^{\mathbf{M}}(d_\star)$;
\item $f=h_q\circ h_{q-1}\circ\cdots\circ h_2\circ h_1\circ h_0$. 
		\end{enumerate}
	\end{minipage}
\end{align*} Define $\Omega:=\setr{(i,j)\in\mb Z^2}{0\leq i\leq q, 1\leq j\leq K, \idf_{\hkh{q}}(i)\leq \idf_{\hkh{1}}(j)}$.  For each $(i,j)\in\Omega$, denote $d_{i,j}:=K+\idf_{\hkh{0}}(i)\cdot (d-K)$ and 
\[
h_{i,j}:\mathbf{dom}(h_i)\to \mbR,\;\;x\mapsto \big({ h_i(x)}\big)_j. 
\]  Then it is easy to verify that, 
\beq\label{23051604}
\mathbf{dom}(h_{i,j})=[0,1]^{d_{i,j}}\text{ and }h_{i,j}\in\mc G_\infty^{\mathbf{H}}(d_*, \beta,r\qd 1)\cup \mc G_\infty^{\mathbf{M}}(d_\star),\;\forall\;(i,j)\in\Omega, 
\eeq and \beq\label{23051801}
\mathbf{ran}\ykh{h_{i,j}}\subset[0, 1],\;\forall\;(i,j)\in\Omega\setminus \hkh{(q,1)}. 
\eeq %

Fix $\e\in(0,1/2]$. Take
\begin{align*}
	\delta:=\frac{1}{2}\cdot\abs{\frac{\e}{8\cdot\abs{(1\qd r)\cdot d_*}^q\cdot (q+1)}}^{\frac{1}{\ykh{1\qx\beta}^q}}\leq \frac{\e/2}{8\cdot\abs{(1\qd r)\cdot d_*}^q\cdot (q+1)}\leq \frac{\e}{8}\leq \frac{1}{16}.
\end{align*} According to \eqref{23051604} and  Lemma \ref{23051701}, there exists a constant $E_1\in (6,\infty)$ only depending on $(d_*,\beta,r)$ and a set of  functions $\big\{\tilde g_{i,j}: \mbR^{d_{i,j}}\to\mbR\big\}_{(i,j)\in\Omega}$, such that 
\beq\label{23051702}
\tilde{g}_{i,j}&\in\fdnn_{d_{i,j}}\left(3\log d_{\star}+E_1\log\frac{1}{\delta},2 d_{\star}+E_1\delta^{-\frac{d_*}{\beta}},\right.\\&\;\;\;\;\;\;\;\;\;\;\;\;\;\;\;\;\;\;\;\;\;\;\;\;\;\;\;\;\;\;\;\;\;\;\;\;\;\left.52 d_{\star}+E_1\delta^{-\frac{d_*}{\beta}}\log\frac{1}{\delta},1,\infty\right),\forall(i,j)\in\Omega
\eeq and 
\beq\label{23051703}
\sup\setr{\abs{\tilde g_{i,j}(x)-h_{i,j}(x)}}{x\in[0,1]^{d_{i,j}}}\leq 2 \delta,\;\forall\;(i,j)\in\Omega. 
\eeq Define
\begin{align*}
&E_4:=8\cdot\abs{(1\qd r)\cdot d_*}^q\cdot (q+1), \\
&E_5:=2^{\frac{d_*}{\beta}}\cdot E_4^{\frac{d_*}{\beta\cdot(1\qx\beta)^q}},\\
&E_6:=\frac{1}{(1\qx\beta)^q}+\frac{2\log E_4}{(1\qx\beta)^q}+2\log 2,\\
&E_7:=E_1 E_6+E_1 E_5+2E_1E_5E_6+6,
\end{align*} Obviously, $E_4,E_5, E_6, E_7$ are constants only depending on $(d_*,\beta,r,q)$. Next, define \[\tilde h_{i,j}:\mbR^{d_{i,j}}\to\mbR,\;x\mapsto \sigma\big(\sigma\ykh{\tilde g_{i,j}(x)}\big)-\sigma\big(\sigma\ykh{\tilde g_{i,j}(x)}-1\big)\]  for each  $(i,j)\in\Omega\setminus \hkh{(q,1)}$,  and define $\tilde h_{q,1}:=\tilde g_{q,1}$.  It follows from  the fact 
\[
\sigma\big(\sigma\ykh{z}\big)-\sigma\big(\sigma\ykh{z}-1\big)\in[0,1],\;
\forall\;z\in\mbR
\] and \eqref{23051702} that
\beq\label{23051808}
\mathbf{ran}(\tilde h_{i,j})\subset[0,1],\;\forall\;(i,j)\in\Omega\setminus(q,1)
\eeq and
\beq\label{23051809}
\tilde{h}_{i,j}&\in\fdnn_{d_{i,j}}\left(2+3\log d_{\star}+E_1\log\frac{1}{\delta},2 d_{\star}+E_1\delta^{-\frac{d_*}{\beta}},\right.\\&\;\;\;\;\;\;\;\;\;\;\;\;\;\;\;\;\;\;\;\;\;\;\;\;\;\;\;\;\;\;\;\;\;\;\;\;\;\left.58 d_{\star}+E_1\delta^{-\frac{d_*}{\beta}}\log\frac{1}{\delta},1,\infty\right),\forall(i,j)\in\Omega.
\eeq Besides, it follows from the fact that
\[
\abs{\sigma\big(\sigma\ykh{z}\big)-\sigma\big(\sigma\ykh{z}-1\big)-w}\leq\abs{w-z},\;\forall\;z\in\mbR,\;\forall\;w\in[0,1]
\] and \eqref{23051703} that
\beq\label{23051903}
&\sup\setr{\big|\tilde h_{i,j}(x)-h_{i,j}(x)\big|}{x\in[0,1]^{d_{i,j}}}\\
&\leq\sup\setr{\abs{\tilde g_{i,j}(x)-h_{i,j}(x)}}{x\in[0,1]^{d_{i,j}}}\leq 2 \delta. 
\eeq We then define 
\[\tilde h_i:\mbR^{d_{i,1}}\to\mbR^{K}, x\mapsto  \ykh{\tilde h_{i,1}(x),\tilde h_{i,2}(x),\ldots,\tilde h_{i,K}(x)}^\top
\] for each $i\in\hkh{0,1,\ldots,q-1}$, and $\tilde h_{q}:=\tilde h_{q,1}$. From \eqref{23051808}  we obtain \beq\label{23051807}
\mathbf{ran}(\tilde h_{i})\subset[0,1]^K\subset\mathbf{dom}(\tilde h_{i+1}),\;\forall\;i\in\hkh{0,1,\ldots,q-1}. 
\eeq Thus we can well define the function $\tilde f:=\tilde h_q\circ\tilde h_{q-1}\circ\cdots \circ\tilde h_1\circ\tilde h_{0}$, which is from $\mbR^d$ to $\mbR$. Since all the functions $\tilde h_{i,j}$ $((i,j)\in\Omega)$ are neural networks satisfying \eqref{23051809}, we deduce that $\tilde f$ is also a neural network, which  is comprised of all those networks $\tilde h_{i,j}$ through series and parallel connection. Obviously, the depth of $\tilde f$ is less than or equal to \[\sum_{i}\ykh{1+\max_j\ykh{\text{the depth of } \tilde h_{i,j}}},\] the width of $\tilde f$ is less than or equal to
\[
\max_i{\sum_j\ykh{\text{the width of $\tilde h_{i,j}$}}}, 
\] the number of nonzero parameters of $\tilde f$ is less than or equal to
\[
\sum_{i,j}\ykh{\ykh{\text{the number of nonzero parameters $\tilde h_{i,j}$}}+\max_k\ykh{\text{the depth of } \tilde h_{i,k}}},
\] and the parameters of $\tilde f$ is bounded by $1$ in absolute value. Thus we have that
\begin{align*}
\tilde{f}&\in\fdnn_{d}\left((q+1)\cdot\abs{3+3\log d_{\star}+E_1\log\frac{1}{\delta}},2K d_{\star}+KE_1\delta^{-\frac{d_*}{\beta}},\right.\\&\;\;\;\;\;\;\;\;\;\;\;\;\;\;\;\;\;\;\;\;\;\;\;\;\;\left.(Kq+1)\cdot\abs{63 d_{\star}+2E_1\delta^{-\frac{d_*}{\beta}}\log\frac{1}{\delta}},1,\infty\right)\\
&=\fdnn_{d}\left((q+1)\cdot\abs{3+3\log d_{\star}+E_1\cdot\Big(\log 2+\frac{\log\frac{E_4}{\e}}{(1\qx\beta)^q}\Big)},2K d_{\star}+KE_1E_5\e^{-\frac{d_*}{\beta\cdot(1\qx\beta)^q}},\right.\\&\;\;\;\;\;\;\;\;\;\;\;\;\;\;\;\;\;\;\;\;\;\;\;\;\;\left.(Kq+1)\cdot\abs{63 d_{\star}+2E_1E_5\e^{-\frac{d_*}{\beta\cdot(1\qx\beta)^q}}\cdot\Big(\log 2+\frac{\log\frac{E_4}{\e}}{(1\qx\beta)^q}\Big)},1,\infty\right)\\
&\subset\fdnn_{d}\left((q+1)\cdot\abs{3+3\log d_{\star}+E_1E_6\log\frac{1}{\e}},2K d_{\star}+KE_1E_5\e^{-\frac{d_*}{\beta\cdot(1\qx\beta)^q}},\right.\\&\;\;\;\;\;\;\;\;\;\;\;\;\;\;\;\;\;\;\;\;\;\;\;\;\;\left.(Kq+1)\cdot\abs{63 d_{\star}+2E_1E_5\e^{-\frac{d_*}{\beta\cdot(1\qx\beta)^q}}E_6\log\frac{1}{\e}},1,\infty\right)\\
&\subset\fdnn_{d}\left((q+1)\cdot\abs{3\log d_{\star}+E_7\log\frac{1}{\e}},2K d_{\star}+KE_7\e^{-\frac{d_*}{\beta\cdot(1\qx\beta)^q}},\right.\\&\;\;\;\;\;\;\;\;\;\;\;\;\;\;\;\;\;\;\;\;\;\;\;\;\;\left.(Kq+1)\cdot\abs{63 d_{\star}+E_7\e^{-\frac{d_*}{\beta\cdot(1\qx\beta)^q}}\log\frac{1}{\e}},1,\infty\right),
\end{align*} leading to \eqref{23051901}. Moreover, it follows from \eqref{23051903} and Lemma \ref{23051404} that 
\begin{align*}
&\sup_{x\in[0,1]^d}\abs{\tilde f(x)-f(x)}=\sup_{x\in[0,1]^d}\abs{\tilde h_q\circ\cdots\circ \tilde h_0(x)-h_q\circ\cdots\circ h_0(x)}\\
&\leq \abs{(1\qd r)\cdot d_*^{{1\qx\beta}}}^{\sum_{i=0}^{q-1}(1\qx\beta)^i}\cdot\sum_{i=0}^q\abs{\sup_{x\in[0,1]^{d_{i,1}}}\norm{\tilde h_i(x)-h_i(x)}_\infty}^{(1\qx\beta)^{q-i}}\\
&\leq \abs{(1\qd r)\cdot d_*}^{q}\cdot\sum_{i=0}^q\abs{2\delta}^{(1\qx\beta)^{q-i}}\leq \abs{(1\qd r)\cdot d_*}^{q}\cdot\sum_{i=0}^q\abs{2\delta}^{(1\qx\beta)^{q}}=\frac{\e}{8},
\end{align*} which yields \eqref{23051902}. 

In conclusion, the constant $E_7$ and the neural network $\tilde f$ have all the desired properties. The proof of this lemma is then  completed. 

\end{proof}

The next lemma aims to estimate the approximation error. 
\begin{lem}\label{lemma5.6}  Let $\phi(t)=\log\ykh{1+\me^{-t}}$ be the logistic loss,  $q\in\mb N\cup\hkh{0}$, $(\beta,r)\in(0,\infty)^2$,    $(d,d_\star, d_*, K)\in\mb N^4$ with  $d_*\leq \min\hkh{d,K+\idf_{\hkh{0}}(q)\cdot(d-K)}$, and $P$ be a Borel probability measure on $[0,1]^d\times\hkh{-1,1}$.  Suppose that there exists an $\hat\eta\in  \mc G_d^{\mathbf{CHOM}}(q, K,d_\star, d_*, \beta,r)$  such that $P_X(\setl{x\in[0,1]^d}{\hat\eta(x)=P(\hkh{1}|x))}=1$.  Then there exist constants $D_1,D_2,D_3$ only depending on $(d_\star,d_*,\beta,r,q)$ such that for any $\delta\in\ykh{0,1/3}$,
\beq\label{approximationerror1}\ba
		&\inf\hkh{\mc{E}_P^\phi\ykh{f}\left| f\in\fdnn_d\ykh{D_1\log\frac{1}{\delta},KD_2{\delta}^{\frac{-d_*/\beta}{(1\qx\beta)^q}},KD_3{\delta}^{\frac{-d_*/\beta}{(1\qx\beta)^q}}\cdot\log\frac{1}{\delta},1,\log\frac{1-\delta}{\delta}}\right.}
		\\&\leq 8\delta.
		\ea\eeq
\end{lem}
\begin{proof}
Denote by $\eta$ the conditional probability function $[0,1]^d\ni x\mapsto P(\hkh{1}|x)\in[0,1]$. 	Fix $\delta\in(0,1/3)$. Then it follows from Lemma \ref{23051904}  that there exists
	\beq\label{tildeeta}\ba
	\tilde{\eta}
	&\in \fdnn_d\left( C_{d_\star,d_*,\beta,r,q}\log\frac{1}{\delta},KC_{d_\star,d_*,\beta,r,q}{\delta}^{-\frac{d_*}{\beta\cdot(1\qx\beta)^q}},\right.\\&\;\;\;\;\;\;\;\;\;\;\;\;\;\;\;\;\;\;\;\;\;\;\;\;\;\;\;\;\;\;\;\;\;\;\;\;\;\left.KC_{d_\star,d_*,\beta,r,q}{\delta}^{-\frac{d_*}{\beta\cdot(1\qx\beta)^q}}\log\frac{1}{\delta},1,\infty\right)
	\ea\eeq
	such that
	\beq\label{ineq 5.61}
	\sup_{x\in[0,1]^d}\abs{\tilde{\eta}(x)-\hat\eta(x)}\leq{\delta}/8. 
	\eeq
	Also, by Theorem \ref{thm2.3} with $a=\e=\delta$, $b=1-\delta$ and  $\alpha=\frac{2\beta}{d_*}$,
	there exists 
	\beq\ba\label{ineq 5.62}
	\tilde{l}&\in\fdnn_1\left(C_{d_*,\beta}\log\frac{1}{\delta}+139\log\frac{1}{\delta},C_{d_*,\beta}\cdot\ykh{\frac{1}{\delta}}^{\frac{1}{2\beta/d_*}}\log\frac{1}{\delta},\right.\\
	&\;\;\;\;\;\;\;\left.C_{d_*,\beta}\cdot\ykh{\frac{1}{\delta}}^{\frac{1}{2\beta/d_*}}\cdot\ykh{\log\frac{1}{\delta}}\cdot\ykh{\log\frac{1}{\delta}}+65440\ykh{\log{\delta}}^2,1,\infty\right)\\
	&\subset\fdnn_1\ykh{C_{d_*,\beta}\log\frac{1}{\delta},C_{d_*,\beta}{\delta}^{-\frac{d_*}{\beta}},C_{d_*,\beta}{\delta}^{-\frac{d_*}{\beta}}\log\frac{1}{\delta},1,\infty}\\
	&\subset\fdnn_1\ykh{C_{d_*,\beta}\log\frac{1}{\delta},C_{d_*,\beta}{\delta}^{-\frac{d_*}{\beta\cdot(1\qx\beta)^q}},C_{d_*,\beta}{\delta}^{-\frac{d_*}{\beta\cdot(1\qx\beta)^{q}}}\log\frac{1}{\delta},1,\infty}\\
	\ea\eeq
	such that
	\beq\label{ineq 5.63}
	\sup_{t\in[\delta,1-\delta]}\abs{\tilde{l}(t)-\log t}\leq \delta
	\eeq
	and 
	\beq\label{n52}
	\log\delta\leq\tilde{l}(t)\leq\log\ykh{1-\delta}<0,\;\forall\;t\in\mbR.
	\eeq
	
	Recall that the clipping function $\Pi_{\delta}$ is given by 
	\[
	\Pi_{\delta}:\mbR\to[\delta,1-\delta],\quad t\mapsto\left\{
	\ba
	&1-\delta,&&\text{ if }t>1-\delta,\\
	&\delta,&&\text{ if }t<\delta,\\
	&t,&&\text{ otherwise}.\\
	\ea
	\right.
	\]
	Define $\tilde{f}:\mbR\to\mbR,x\mapsto \tilde{l}\ykh{\Pi_{\delta}\ykh{\tilde{\eta}(x)}}-\tilde{l}\ykh{1-\Pi_{\delta}\ykh{\tilde{\eta}(x)}}$. Consequently, we know from (\ref{tildeeta}), (\ref{ineq 5.62}) and (\ref{n52})  that (cf. Figure \ref{fig6})
	\begin{align*}
	\tilde{f}
&\in \fdnn_d\left( C_{d_\star,d_*,\beta,r,q}\log\frac{1}{\delta},KC_{d_\star,d_*,\beta,r,q}{\delta}^{-\frac{d_*}{\beta\cdot(1\qx\beta)^q}},\right.\\&\;\;\;\;\;\;\;\;\;\;\;\;\;\;\;\;\;\;\;\;\;\;\;\;\;\;\;\;\;\;\;\;\;\;\;\;\;\left.KC_{d_\star,d_*,\beta,r,q}{\delta}^{-\frac{d_*}{\beta\cdot(1\qx\beta)^q}}\log\frac{1}{\delta},1,\log\frac{1-\delta}{\delta}\right).
	\end{align*} Let $\Omega_1,\Omega_2,\Omega_3$ be defined in \eqref{20221019232201}. Then it follows from (\ref{ineq 5.61}) that
	\begin{align*}
	&\abs{\Pi_\delta(\tilde{\eta}(x))-\eta(x)}=	\abs{\Pi_\delta(\tilde{\eta}(x))-\Pi_\delta\ykh{\hat\eta(x)}}\leq \abs{\tilde{\eta}(x)-\hat\eta(x)}\\&\leq \frac{\delta}{8}\leq\frac{\min\hkh{\eta(x),1-\eta(x)}}{8},\;\forall\;x\in\Omega_1, 
	\end{align*} which means that
	\beq\label{20220317214101}
	\min\hkh{\frac{\Pi_\delta\ykh{\tilde{\eta}(x)}}{\eta(x)},\frac{1-\Pi_\delta\ykh{\tilde{\eta}(x)}}{1-\eta(x)}}\geq 7/8,\;\forall\;x\in\Omega_1. 
	\eeq Combining  (\ref{ineq 5.63}) and (\ref{20220317214101}), we obtain that
	\begin{align*}
	&\abs{\tf(x)-\log\frac{\eta(x)}{1-\eta(x)}}\\&\leq \abs{\tilde{l}\ykh{\Pi_{\delta}\ykh{\tilde{\eta}(x)}}-\log\ykh{{\eta(x)}}}+\abs{\tilde{l}\ykh{1-\Pi_{\delta}\ykh{\tilde{\eta}(x)}}-\log\ykh{1-{\eta(x)}}}\\
	&\leq\abs{\tilde{l}\ykh{\Pi_{\delta}\ykh{\tilde{\eta}(x)}}-\log\ykh{\Pi_{\delta}\ykh{\te(x)}}}+\left|\log\ykh{\Pi_{\delta}\ykh{\te(x)}}-\log\ykh{{\eta(x)}}\right|\\
	&\;\;\;+\abs{\tilde{l}\ykh{1-\Pi_{\delta}\ykh{\tilde{\eta}(x)}}-\log\ykh{1-\Pi_{\delta}\ykh{\te(x)}}}+\abs{\log\ykh{1-\Pi_{\delta}\ykh{\te(x)}}-\log\ykh{1-{\eta(x)}}}\\
	&\leq \delta+\sup_{t\in \left[\Pi_\delta\ykh{\te(x)}\qx\eta(x),\infty\right)}\abs{\log'(t)}\cdot\abs{\Pi_{\delta}\ykh{\te(x)}-{\eta(x)}}\\&\;\;\;+\delta+\sup_{t\in \left[\min\hkh{1-\Pi_\delta\ykh{\te(x)},1-\eta(x)},\infty\right)}\abs{\log'(t)}\cdot\abs{\Pi_{\delta}\ykh{\te(x)}-{\eta(x)}}\\
	&\leq \delta+\sup_{t\in \left[7\eta(x)/8,\infty\right)}\abs{\log'(t)}\cdot\abs{\Pi_{\delta}\ykh{\te(x)}-{\eta(x)}}\\&\;\;\;+\delta+\sup_{t\in \left[\frac{7-7\eta(x)}{8},\infty\right)}\abs{\log'(t)}\cdot\abs{\Pi_{\delta}\ykh{\te(x)}-{\eta(x)}}\\
	&\leq 2\delta+\frac{8}{7\eta(x)}\cdot\frac{\delta}{8}+\frac{8}{7-7\eta(x)}\cdot\frac{\delta}{8},\;\forall\;x\in\Omega_1,
	\end{align*} meaning that
	\beq\label{ineq 5.64}
	&\abs{\tf(x)-\log\frac{\eta(x)}{1-\eta(x)}}\leq  2\delta+\frac{8}{7\eta(x)}\cdot\frac{\delta}{8}+\frac{8}{7-7\eta(x)}\cdot\frac{\delta}{8}\\
	&=2\delta+\frac{\delta}{7\eta(x)(1-\eta(x))}\leq\frac{2}{3}+\frac{2}{7}<1,\;\forall\;x\in\Omega_1.
	\eeq
	
	\begin{figure}[H]	\centering
		\betikz
		\tikzset{
			mybox/.style ={
				rectangle, %
				rounded corners =5pt, %
				minimum width =90pt, %
				minimum height =20pt, %
				inner sep=0.6pt, %
				draw=blue, %
				fill=cyan
		}}
		\tikzset{
			ttinybox/.style ={
				rectangle, %
				rounded corners =3pt, %
				minimum width =25pt, %
				minimum height =25pt, %
				inner sep=2pt, %
				draw=blue, %
				fill=cyan
		}}
		\tikzset{
			tinybox/.style ={
				rectangle, %
				rounded corners =5pt, %
				minimum width =20pt, %
				minimum height =20pt, %
				inner sep=5pt, %
				draw=blue, %
				fill=cyan
		}}
		\tikzset{
			bigbox/.style ={
				rectangle, %
				rounded corners =5pt, %
				minimum width =396.9pt, %
				minimum height =289pt, %
				inner sep=0pt, %
				draw=black, %
				fill=none
		}}
		\tikzset{
			tinycircle/.style ={
				circle, %
				minimum width =6pt, %
				minimum height =6pt, %
				inner sep=0pt, %
				draw=red, %
				fill=red %
			}
		}
		\tikzset{ %
			mypt/.style ={
				circle, %
				minimum width =0pt, %
				minimum height =0pt, %
				inner sep=0pt, %
				draw=blue, %
			}
		}
		\node[tinycircle] (11) at (-1,0) {};
		\node[tinycircle] (12) at (-0.5,0) {};
		\node[below] at (0.3061,0.2) {$\cdots\cdots$};
		\node[tinycircle] (14) at (1,0) {};
		\node[right] at (1,0){$\;\;x\in\mbR^d$};
		\node[below] at (0,-0.3) {Input};
		\node[mybox] at (0,1) {$\tilde{\eta}$};
		\node[below] (1121)at (-1.3,0.93) {};
		\node[below] (1122)at (1.3,0.82) {};
		\node[below] (1221)at (-1.34,0.91) {};
		\node[below] (1222)at (1.33,0.83) {};
		\node[below] (1321)at (-1.3,0.82) {};
		\node[below] (1322)at (1.3,0.93) {};
		
		\node[mypt] (22)at (-1.3,0.69) {};
		\node[mypt] (21)at (1.3,0.69) {};

		\filldraw[->,blue] (11)--(22){};
		\filldraw[->,blue] (11)--(21){};
		\filldraw[->,blue] (12)--(21){};
		\filldraw[->,blue] (12)--(22){};
		\filldraw[->,blue] (14)--(21){};
		\filldraw[->,blue] (14)--(22){};

		\node[below] at (-0.46,0.52) {$\cdots$};
		\node[below] at (0.84,0.52) {$\cdots$};
		\node[below] at (0.1,0.69) {$\cdots$};
		\node[tinycircle] (41) at (-5,3.5) {};
		\node[tinycircle] (42) at (-1,3.5) {};
		\node[tinycircle] (43) at (2,3.5) {};
		\node[tinycircle] (44) at (5,3.5) {};

				\node[mypt] (x3) at(0,1.33){};
				
				\filldraw[->,blue] (x3)--(44){};
				\filldraw[->,blue] (x3)--(43){};
				\filldraw[->,blue] (x3)--(42){};
				\filldraw[->,blue] (x3)--(41){};

		\node[right]  at (-4.98,3.6) {%
		{\fontsize{9}{11}\selectfont
		$\ba&\sigma\ykh{0\cdot\tilde{\eta}(x)+\delta}=\delta\ea$}};
		\node[right]  at (-1.1,3.5) 
		{%
		{\fontsize{9}{11}\selectfont
		$\;\sigma\ykh{\tilde{\eta}(x)-\delta}$}};
		\node[right]  at (1.7,3.3) {%
		{\fontsize{9}{11}\selectfont
		$\;\sigma\ykh{\tilde{\eta}(x)-1+\delta}$}};
		\node[right]  at (4.5,3.2) {%
		{\fontsize{9}{11}\selectfont
		$\ba&\sigma\ykh{0\cdot\tilde{\eta}(x)+1-\delta}=1-\delta\ea$}};
		
		\node[above] (324) at (-0.15,1.13) {};
		\node[above] (323) at (-0.04,1.09) {};
		\node[above] (322) at (0.04,1.09) {};
		\node[above] (321) at (0.15,1.13) {};
		
				\node[tinybox] (71) at (-3,6.5){$\tilde{l}$};
				
				\node[tinybox] (72) at (1,6.5){$\tilde{l}$};

		\node[tinycircle] (51) at(-4.3,5){};
		\node[tinycircle] (52) at(3.,5){};
		
			\node[tinycircle] (8) at (0,8) {};
		
		\node[mypt] (7x1)at (-3,6.16) {};
		\node[mypt] (7x2)at (1,6.16) {};
		\filldraw[->,blue] (51)--(7x1);

		\filldraw[->,blue] (52)--(7x2);
		
		\node[mypt] (7s1)at (-3,6.84) {};
		\node[mypt] (7s2)at (1,6.84) {};
		
		\filldraw[->,blue] (7s1)--(8);
		\filldraw[->,blue] (7s2)--(8);
		
		\filldraw[->,blue] (41)--(51);
		\filldraw[->,blue] (42)--(51);
		\filldraw[->,blue] (43)--(51);
		
		\filldraw[->,blue] (44)--(52);
		\filldraw[->,blue] (42)--(52);
		\filldraw[->,blue] (43)--(52);

		\node[right]  at(-5.2,5.35){%
		{\fontsize{9}{11}\selectfont
		$\sigma\ykh{\delta+\sigma\ykh{\tilde{\eta}(x)-\delta}-\sigma\ykh{\tilde{\eta}(x)-1+\delta}}=\Pi_\delta\ykh{\tilde{\eta}(x)}$}};
		
		\node[right]  at(0.2,4.57){%
		{\fontsize{9}{11}\selectfont
		$\sigma\ykh{1-\delta-\sigma\ykh{\tilde{\eta}(x)-\delta}+\sigma\ykh{\tilde{\eta}(x)-1+\delta}}=1-\Pi_\delta\ykh{\tilde{\eta}(x)}$}};

		\node[right] at (0,8) {{$\;\tilde{l}\ykh{\Pi_\delta\ykh{\tilde{\eta}(x)}}-\tilde{l}\ykh{1-\Pi_\delta\ykh{\tilde{\eta}(x)}}=\tilde{f}(x)$}};

		\node[bigbox] at (1.658,4) {};
		\node[above] at (0,8.2) {Output};

		\eetikz
		\caption{The network representing the function $\tilde{f}$. }
		\label{fig6}
	\end{figure}

	\noindent Besides, note that
\begin{align*}
	&x\in\Omega_2\Rightarrow\te(x)\in[-\xi_1,\delta+\xi_1]\Rightarrow\Pi_{\delta}\ykh{\te(x)}\in[\delta,\delta+\xi_1]\\
	&\Rightarrow\tll\ykh{\Pi_{\delta}\ykh{\te(x)}}\in\zkh{\log\delta,\delta+\log\ykh{\delta+\xi_1}}\\
	&\;\;\;\;\;\textrm
	{as well as }\tll\ykh{1-\Pi_{\delta}\ykh{\te(x)}}\in[-\delta+\log(1-\delta-\xi_1),\log(1-\delta)]\\
	&\Rightarrow \tf(x)\leq 2\delta+\log\frac{\xi_1+\delta}{1-\xi_1-\delta}\leq \log 2+\log\frac{2{\delta}}{1-2{\delta}}=\log\frac{4{\delta}}{1-2{\delta}}.
	\end{align*}
	Therefore, by (\ref{n52}) and the definition of $\tilde{f}$, we have 
	\beq\label{o63}
	\log\frac{\delta}{1-\delta}\leq \tilde{f}(x)\leq \log\frac{4\delta}{1-2\delta}=-\log\frac{1-2\delta}{4\delta},\;\forall\;x\in\Omega_2.
	\eeq
	Similarly, we can show that 
	\beq\label{on64}
	\log\frac{1-2\delta}{4\delta}\leq \tilde{f}(x)\leq \log\frac{1-\delta}{\delta},\;\forall\;x\in\Omega_3.
	\eeq Then it follows from (\ref{ineq 5.64}), (\ref{o63}), (\ref{on64}) and Lemma \ref{lem5.4} that
	\begin{small}
		\begin{align*}
		&\inf\hkh{\mc{E}_P^\phi\ykh{f}\left|\ba
				{f}
			&\in \fdnn_d\left( C_{d_\star,d_*,\beta,r,q}\log\frac{1}{\delta},KC_{d_\star,d_*,\beta,r,q}{\delta}^{-\frac{d_*}{\beta\cdot(1\qx\beta)^q}},\right.\\&\;\;\;\;\;\;\;\;\left.KC_{d_\star,d_*,\beta,r,q}{\delta}^{-\frac{d_*}{\beta\cdot(1\qx\beta)^q}}\log\frac{1}{\delta},1,\log\frac{1-\delta}{\delta}\right)
			\ea\right.}\\
		&\leq \mc{E}_P^\phi\ykh{\tilde{f}}\leq \phi\ykh{\log\frac{1-2\delta}{4\delta}}P_X(\Omega_2)+\phi\ykh{\log\frac{1-2\delta}{4\delta}}P_X(\Omega_3)\\
		&\;\;\;\;\;\;\;\;+\int_{\Omega_1}{\sup\setl{\frac{\abs{\tf(x)-\log\frac{\eta(x)}{1-\eta(x)}}^2}{2(2+\me^t+\me^{-t})}}{{t\in \zkh{\tf(x)\qx\log\frac{\eta(x)}{1-\eta(x)},\tf(x) \qd\log\frac{\eta(x)}{1-\eta(x)}}}}}\mr{d}P_X(x)\\
		&\leq \int_{\Omega_1}{\sup\setl{\frac{\abs{\tf(x)-\log\frac{\eta(x)}{1-\eta(x)}}^2}{2(2+\me^t+\me^{-t})}}{{t\in \zkh{-1+\log\frac{\eta(x)}{1-\eta(x)},1+\log\frac{\eta(x)}{1-\eta(x)}}}}}\mr{d}P_X(x)\\&\;\;\;\;\;\;\;\;+ P_X(\Omega_2\cup\Omega_3)\cdot\log\frac{1+2\delta}{1-2\delta}\\&\leq\int_{\Omega_1}\abs{\tf\ykh{x}-\log\frac{\eta(x)}{1-\eta(x)}}^2\cdot 2\cdot (1-\eta(x))\eta(x)\mr dP_X(x)+6\delta\\
		&\leq\int_{\Omega_1}\abs{2\delta+\frac{\delta}{7\eta(x)(1-\eta(x))}}^2\cdot 2\cdot (1-\eta(x))\eta(x)\mr dP_X(x)+6\delta\\&\leq\int_{\Omega_1} \frac{\delta^2}{ (1-\eta(x))\eta(x)}\mr dP_X(x)+6\delta\leq\frac{\delta^2}{\delta(1-\delta)}+6\delta<8\delta,\\
		\end{align*}\end{small}
	which proves this lemma. 
\end{proof}
Now we are in the position to prove Theorem \ref{thm2.2} and Theorem \ref{thm2.2p}. 
\begin{proof}[Proof of Theorem \ref{thm2.2} and  Theorem \ref{thm2.2p}] We first prove Theorem \ref{thm2.2p}.  	According to Lemma \ref{lemma5.6}, there exist $(D_1,D_2,D_3)\in (0,\infty)^3$ only depending on $(d_\star,d_*,\beta,r,q)$ such that (\ref{approximationerror1}) holds for any $\delta\in(0,1/3)$ and any $P\in	\mc{H}^{d,\beta,r}_{4,q,K,d_\star,d_*}$. Take $E_1=1+ D_1$, then $E_1>0$ only depends on $(d_\star,d_*,\beta,r,q)$.  
	 We next show that for any constants ${\bm a}:=(a_2,a_3)\in (0,\infty)^2$  and  ${\bm b}:=(b_1,b_2,b_3,b_4,b_5)\in (0,\infty)^5$, there exist constants $E_2\in(3,\infty)$ only depends on $(\bm a, d_\star,d_*,\beta,r,q,K)$  and  
	 $E_3\in(0,\infty)$ only depending on $(\bm a,\bm b,\nu, d,d_\star,d_*,\beta,r,q,K)$ such that  when $n \geq E_2$, the $\phi$-ERM  $\hat{f}^{\FNN}_n$  defined by \eqref{FCNNestimator} with 
	 \beq\ba\label{bounds1}
	 &E_1\cdot\log n\leq G\leq b_1\cdot\log n, \\& a_2\cdot\ykh{\frac{(\log n)^5}{n}}^{\frac{-d_*}{d_*+\beta\cdot(1\qx\beta)^q}}\leq N\leq b_2\cdot\ykh{\frac{(\log n)^5}{n}}^{\frac{-d_*}{d_*+\beta\cdot(1\qx\beta)^q}},\\
	 &  a_3\cdot \ykh{\frac{(\log n)^5}{n}}^{\frac{-d_*}{d_*+\beta\cdot(1\qx\beta)^q}} \cdot\log n\leq S\leq b_3\cdot \ykh{\frac{(\log n)^5}{n}}^{\frac{-d_*}{d_*+\beta\cdot(1\qx\beta)^q}}\cdot\log n,
	 \ \\
	 & \frac{\beta\cdot(1\qx\beta)^q}{d_*+\beta\cdot(1\qx\beta)^q}\cdot\log n\leq F\leq b_4\log n,\text{ and }1\leq B\leq b_5\cdot n^{\nu}
	 \ea\eeq 
	 must satisfy 
	 \beq\ba\label{bounds2}
	 &\sup_{P\in	\mc{H}^{d,\beta,r}_{4,q,K,d_\star,d_*}}\bm E_{P^{\otimes n}}\zkh{\mc{E}_P^\phi\ykh{\hat{f}^{\FNN}_n}}\leq E_3 \cdot\ykh{\frac{(\log n)^5}{n}}^{\frac{\beta\cdot(1\qx\beta)^q}{d_*+\beta\cdot(1\qx\beta)^q}}\\&\text{ and }\sup_{P\in	\mc{H}^{d,\beta,r}_{4,q,K,d_\star,d_*}}\bm E_{P^{\otimes n}}\zkh{\mc{E}_P\ykh{\hat{f}^{\FNN}_n}}\leq E_3\cdot \ykh{\frac{(\log n)^5}{n}}^{\frac{\beta\cdot(1\qx\beta)^q}{2d_*+2\beta\cdot(1\qx\beta)^q}},
	 \ea\eeq which will lead to the results of Theorem \ref{thm2.2p}.
	 
	Let ${\bm a}:=(a_2,a_3)\in (0,\infty)^2$ and  ${\bm b}:=(b_1,b_2,b_3,b_4,b_5)\in (0,\infty)^5$ be arbitrary and fixed.  Take
	 \[
	 D_4={1\qd{ \ykh{\frac{D_2K}{a_2}}^{\frac{\beta\cdot(1\qx\beta)^q}{d_*}}}\qd{ \ykh{\frac{D_3E_1K}{D_1a_3}}^{\frac{\beta\cdot(1\qx\beta)^q}{d_*}}}},
	 \]   then  $D_4>0$ only depends on $(\bm a, d_\star,d_*,\beta,r,q,K)$. Hence there exists $E_2\in (3,\infty)$ only depending on $(\bm a, d_\star,d_*,\beta,r,q,K)$ such that
	 \beq\ba\label{68}
	 0&<\frac{\ykh{\log t}^5}{t}<D_4\cdot\ykh{\frac{(\log t)^5}{t}}^{\frac{\beta\cdot(1\qx\beta)^q}{d_*+\beta\cdot(1\qx\beta)^q}}
	 <1/4\\&<1< {\log t},\;\forall\;t\in[E_2,\infty).
	 \ea\eeq From now on we assume that $n\geq E_2$,  and (\ref{bounds1}) holds. We have to show that there exists  $E_3\in(0,\infty)$ only depending on $(\bm a,\bm b,\nu, d,d_\star,d_*,\beta,r,q,K)$ such that \eqref{bounds2} holds.
	
	Let $P$ be an arbitrary probability in $	\mc{H}^{d,\beta,r}_{4,q,K,d_\star,d_*}$. Denote by $\eta$ the conditional probability function $x\mapsto P(\hkh{1}|x)$ of $P$. Then there exists an $\hat\eta\in\mc G_d^{\mathbf{CHOM}}(q, K,d_\star, d_*, \beta,r)$ such that $\hat\eta=\eta$, $P_X$-a.s.. Define 
	\beq\label{nnn69}
	\zeta:=D_4\cdot\ykh{\frac{(\log n)^5}{n}}^{\frac{\beta\cdot(1\qx\beta)^q}{d_*+\beta\cdot(1\qx\beta)^q}}.
	\eeq
	By (\ref{68}), $0<n^{\frac{-\beta\cdot(1\qx\beta)^q}{d_*+\beta\cdot(1\qx\beta)^q}}\leq\zeta<\frac{1}{4}$ and there hold inequalities 	\beq\ba\label{n70}
	\log 2<\log\frac{1-\zeta}{\zeta}\leq\log\frac{1}{\zeta}\leq \log\ykh{n^{\frac{\beta\cdot(1\qx\beta)^q}{d_*+\beta\cdot(1\qx\beta)^q}}}\leq F,
	\ea\eeq 
	\beq\ba\label{nn69}
	&D_1\log\frac{1}{\zeta}\leq D_1\log\ykh{n^{\frac{\beta\cdot(1\qx\beta)^q}{d_*+\beta\cdot(1\qx\beta)^q}}}\\
	&\leq D_1\log n\leq \max\hkh{1,D_1\log n}\leq E_1\log n\leq G,
	\ea\eeq
		and
	\beq\ba\label{71}
	&KD_2{\zeta}^{\frac{-d_*/\beta}{(1\qx\beta)^q}}= KD_2\cdot{D_4}^{\frac{-d_*/\beta}{(1\qx\beta)^q}}\cdot \ykh{\frac{(\log n)^5}{n}}^{\frac{-d_*}{d_*+\beta\cdot(1\qx\beta)^q}}\\&\leq KD_2\cdot\abs{{ \ykh{\frac{D_2K}{a_2}}^{\frac{\beta\cdot(1\qx\beta)^q}{d_*}}}}^{\frac{-d_*/\beta}{(1\qx\beta)^q}}\cdot \ykh{\frac{(\log n)^5}{n}}^{\frac{-d_*}{d_*+\beta\cdot(1\qx\beta)^q}}\\
	&=a_2\cdot \ykh{\frac{(\log n)^5}{n}}^{\frac{-d_*}{d_*+\beta\cdot(1\qx\beta)^q}}\leq N.
	\ea\eeq
	Consequently,
	\beq\ba\label{n72}
	&KD_3{\zeta}^{\frac{-d_*/\beta}{(1\qx\beta)^q}}\cdot\log\frac{1}{\zeta}= KD_3\cdot{D_4}^{\frac{-d_*/\beta}{(1\qx\beta)^q}}\cdot \ykh{\frac{(\log n)^5}{n}}^{\frac{-d_*}{d_*+\beta\cdot(1\qx\beta)^q}}\cdot\log\frac{1}{\zeta}\\
	&\leq KD_3\cdot\abs{{ \ykh{\frac{D_3E_1K}{D_1a_3}}^{\frac{\beta\cdot(1\qx\beta)^q}{d_*}}}}^{\frac{-d_*/\beta}{(1\qx\beta)^q}}\cdot \ykh{\frac{(\log n)^5}{n}}^{\frac{-d_*}{d_*+\beta\cdot(1\qx\beta)^q}}\cdot\log\frac{1}{\zeta}\\
	&= a_3\cdot \ykh{\frac{(\log n)^5}{n}}^{\frac{-d_*}{d_*+\beta\cdot(1\qx\beta)^q}}\cdot\frac{D_1\cdot\log\frac{1}{\zeta}}{E_1}
		\leq a_3\cdot \ykh{\frac{(\log n)^5}{n}}^{\frac{-d_*}{d_*+\beta\cdot(1\qx\beta)^q}}\cdot \log n\leq S.
	\ea\eeq
	Then it follows  from (\ref{approximationerror1}), (\ref{nnn69}), (\ref{nn69}), (\ref{n70}), (\ref{71}), and (\ref{n72}) that
	\beq\label{n74}
	&\inf\hkh{\mc{E}_P^\phi\ykh{f}\left|f\in\fdnn_d\ykh{G,N,S,B,F}\right.}\\
	&\leq \inf\hkh{\mc{E}_P^\phi\ykh{f}\left| f\in\fdnn_d\ykh{D_1\log\frac{1}{\zeta},\frac{KD_2}{{\zeta}^{\frac{d_*/\beta}{(1\qx\beta)^q}}},\frac{KD_3}{{\zeta}^{\frac{d_*/\beta}{(1\qx\beta)^q}}}\cdot\log\frac{1}{\zeta},1,\log\frac{1-\zeta}{\zeta}}\right.}\\
	&\leq 8\zeta=8D_4\cdot\ykh{\frac{(\log n)^5}{n}}^{\frac{\beta\cdot(1\qx\beta)^q}{d_*+\beta\cdot(1\qx\beta)^q}}.
	\eeq 
	Besides, from (\ref{n70}) we know $\me^F>2$. Hence by taking $\delta_0=\frac{1}{\me^F+1}$ in  Lemma \ref{lem5.5}, we obtain immediately that there exists
	\beq\label{n75}
	\psi:[0,1]^d\times\{-1,1\}\to\zkh{0,\log\ykh{(10\me^F+10)\cdot\log\ykh{\me^F+1}}},
	\eeq
	such that
	\beq\label{n76}
	\int_{[0,1]^d\times\{-1,1\}}{\psi\ykh{x,y}}\mr{d}P(x,y)=\inf \hkh{\mc{R}_P^\phi(f) \mid \textrm{$f:[0,1]^d\to\mbR$ is measurable}}, 
	\eeq
	and for any measurable $f:[0,1]^d\to\zkh{-F,F}$,
	\beq \label{n77}
	&\int_{[0,1]^d\times\{-1,1\}}{\ykh{\phi\ykh{yf(x)}-\psi(x,y)}^2}\mr{d}P(x,y)\\
	&\leq 125000 \abs{\log\ykh{1+\me^F}}^2\cdot\int_{[0,1]^d\times\{-1,1\}}\ykh{\phi\ykh{yf(x)}-\psi(x,y)}\mr{d}P(x,y)\\
	&\leq 500000 F^2\cdot\int_{[0,1]^d\times\{-1,1\}}\ykh{\phi\ykh{yf(x)}-\psi(x,y)}\mr{d}P(x,y).
	\eeq
	
	Moreover,  it follows from Corollary \ref{corollaryA1} with $\gamma=\frac{1}{n}$ that
	\beq\ba\label{78}
	&\log W\leq (S+Gd+1)(2G+5)\log\ykh{(\max\hkh{N,d}+1)(2nG+2n)B}\\
	&\leq C_{\bm b,d}\cdot(\log n)^2\cdot \ykh{\frac{(\log n)^5}{n}}^{\frac{-d_*}{d_*+\beta\cdot(1\qx\beta)^q}}\cdot\log\ykh{(\max\hkh{N,d}+1)(2nG+2n)b_5n^{\nu}}\\
	&\leq C_{\bm b,d,\nu}\cdot (\log n)^3 \cdot\ykh{\frac{(\log n)^5}{n}}^{\frac{-d_*}{d_*+\beta\cdot(1\qx\beta)^q}}=E_4\cdot(\log n)^3 \cdot\ykh{\frac{(\log n)^5}{n}}^{\frac{-d_*}{d_*+\beta\cdot(1\qx\beta)^q}}
	\ea\eeq for some constant $E_4\in(0,\infty)$ only depending on$(\bm b,d,\nu)$, 
	where \[W=3\qd\mc{N}\ykh{\hkh{\left.f|_{[0,1]^d}\right|f\in\fdnn_d\ykh{G,N,S,B,F}},\frac{1}{n}}.\]
	Also, note that 
	\beq\label{o78}
	\sup_{t\in \zkh{-F,F}}\phi(t) &=\log\ykh{1+\me^F}\leq \log\ykh{(10\me^F+10)\cdot\log\ykh{\me^F+1}}\leq 7F. 
	\eeq
	Therefore, by taking $\epsilon=\frac{1}{2}$, $\gamma=\frac{1}{n}$, $\Gamma=500000F^2$, $M=7F$, and $$\mc{F}=\hkh{\left.f|_{[0,1]^d}\right|f\in\fdnn_d\ykh{G,N,S,B,F}}$$ in Theorem  \ref{thm2.1} and combining  (\ref{n75}), (\ref{n76}), (\ref{n77}), (\ref{78}), (\ref{n74}), we obtain
	\begin{align*}
	&\bm E_{P^{\otimes n}}\zkh{\mc{E}_P^\phi\ykh{\hat{f}^{\FNN}_n}}=\bm E_{P^{\otimes n}}\zkh{\mc{R}_P^\phi\ykh{\hat{f}^{\FNN}_n}-\int_{[0,1]^d\times\{-1,1\}}{\psi(x,y)}\mr{d}P(x,y)}\\
	&\leq 360\cdot\frac{\Gamma\log W}{n}+\frac{4}{n}+\frac{30M\log W}{n}+30\cdot\sqrt{\frac{\Gamma\log W}{n^2}}+2\inf_{f\in\mc{F}}\ykh{\mc{R}_P^\phi(f)-\int{\psi}\mr{d}P}\\
	&\leq \frac{360\Gamma\log W}{n}+\frac{\Gamma\log W}{n}+\frac{\Gamma\log W}{n}+{\frac{\Gamma\log W}{n}}+2\inf_{f\in\mc{F}}\mc{E}_P^\phi(f)\\
	&\leq {\frac{2\cdot 10^{8}\cdot F^2\cdot\log W}{n}}+16D_4\cdot\ykh{\frac{(\log n)^5}{n}}^{\frac{\beta\cdot(1\qx\beta)^q}{d_*+\beta\cdot(1\qx\beta)^q}}\\
	&\leq {\frac{10^9\cdot |b_4\log n|^2\cdot E_4\cdot(\log n)^3 \cdot\ykh{\frac{(\log n)^5}{n}}^{\frac{-d_*}{d_*+\beta\cdot(1\qx\beta)^q}}}{n}}+16D_4\cdot\ykh{\frac{(\log n)^5}{n}}^{\frac{\beta\cdot(1\qx\beta)^q}{d_*+\beta\cdot(1\qx\beta)^q}}\\
&=\ykh{16D_4+10^9\cdot \abs{b_4}^2\cdot E_4}\cdot  \ykh{\frac{(\log n)^5}{n}}^{\frac{\beta\cdot(1\qx\beta)^q}{d_*+\beta\cdot(1\qx\beta)^q}}\leq E_3\cdot  \ykh{\frac{(\log n)^5}{n}}^{\frac{\beta\cdot(1\qx\beta)^q}{d_*+\beta\cdot(1\qx\beta)^q}}
	\end{align*} with
\[
E_3:=4\cdot\ykh{16D_4+10^9\cdot \abs{b_4}^2\cdot E_4}+4
\] only depending on $(\bm a,\bm b,\nu, d,d_\star,d_*,\beta,r,q,K)$. 
	We then apply the calibration inequality \eqref{calibrationineq} and conclude that
	\beq\label{23032701}
	&\bm E_{P^{\otimes n}}\zkh{\mc{E}_P\ykh{\hat{f}^{\FNN}_n}}\leq 2\sqrt{2}\cdot\bm E_{P^{\otimes n}}\zkh{\sqrt{\mc{E}_P^\phi\ykh{\hat{f}^{\FNN}_n}}}\leq 4\cdot\sqrt{\bm E_{P^{\otimes n}}\zkh{\mc{E}_P^\phi\ykh{\hat{f}^{\FNN}_n}}}\\&
	\leq 4\cdot\sqrt{\ykh{16D_4+10^9\cdot \abs{b_4}^2\cdot E_4}\cdot  \ykh{\frac{(\log n)^5}{n}}^{\frac{\beta\cdot(1\qx\beta)^q}{d_*+\beta\cdot(1\qx\beta)^q}}}\\&\leq E_3\cdot \ykh{\frac{(\log n)^5}{n}}^{\frac{\beta\cdot(1\qx\beta)^q}{2d_*+2\beta\cdot(1\qx\beta)^q}}.
	\eeq Since $P$ is arbitrary, the desired bound \eqref{bounds2}  follows.  Setting $\mr c=E_1$ completes the proof of Theorem \ref{thm2.2p}.

Now it remains to show Theorem \ref{thm2.2}. Indeed, it follows from \eqref{23052201} that
\[
\mc {H}_1^{d,\beta,r}\subset 	\mc{H}^{d,\beta,r}_{4,0,1,1,d}.
\] Then by taking $q=0$, $d_*=d$ and $d_\star=K=1$ in Theorem \ref{thm2.2p}, we obtain that there exists a constant $\mr c\in(0,\infty)$ only depending on $(d,\beta,r)$ such that the estimator $\hat{f}^{\FNN}_n$ defined by \eqref{FCNNestimator} with 
\begin{align*}
		&\mr c\log n\leq G \lesssim \log n, \ N \asymp \ykh{\frac{(\log n)^5}{n}}^{\frac{-d}{d+\beta\cdot(1\qx\beta)^0}}=\ykh{\frac{(\log n)^5}{n}}^{\frac{-d}{d+\beta}}, 
	\\& S \asymp \ykh{\frac{(\log n)^5}{n}}^{\frac{-d}{d+\beta\cdot(1\qx\beta)^0}} \cdot\log n= \ykh{\frac{(\log n)^5}{n}}^{\frac{-d}{d+\beta}} \cdot\log n,\\
	& 1\leq B \lesssim  n^{\nu}, \textrm{ and } \ \frac{\beta}{d+\beta}\cdot\log n=\frac{\beta\cdot(1\qx\beta)^0}{d+\beta\cdot(1\qx\beta)^0}\cdot\log n\leq F\lesssim\log n
\end{align*} must satisfy
	\begin{align*}
&\sup_{P\in\mc{H}^{d,\beta,r}_{1}}\bm E_{P^{\otimes n}}\zkh{\mc{E}_P^\phi\ykh{\hat{f}^{\FNN}_n}}\leq\sup_{P\in\mc{H}^{d,\beta,r}_{4,0,1,1,d}}\bm E_{P^{\otimes n}}\zkh{\mc{E}_P^\phi\ykh{\hat{f}^{\FNN}_n}}\\&\lesssim \ykh{\frac{(\log n)^5}{n}}^{\frac{\beta\cdot(1\qx\beta)^0}{d+\beta\cdot(1\qx\beta)^0}}=\ykh{\frac{(\log n)^5}{n}}^{\frac{\beta}{d+\beta}}
\end{align*} and 
\begin{align*}
&\sup_{P\in\mc{H}^{d,\beta,r}_{1}}\bm E_{P^{\otimes n}}\zkh{\mc{E}_{P}\ykh{\hat{f}^{\FNN}_n}}\leq\sup_{P\in\mc{H}^{d,\beta,r}_{4,0,1,1,d}}\bm E_{P^{\otimes n}}\zkh{\mc{E}_{P}\ykh{\hat{f}^{\FNN}_n}}\\&\lesssim \ykh{\frac{(\log n)^5}{n}}^{\frac{\beta\cdot(1\qx\beta)^0}{2d+2\beta\cdot(1\qx\beta)^0}}=\ykh{\frac{(\log n)^5}{n}}^{\frac{\beta}{2d+2\beta}}.
\end{align*} This completes the proof of Theorem \ref{thm2.2}. 
\end{proof}

\subsection{Proof of Theorem \ref{thm2.4}} \label{section: proof of thm2.4}

Appendix \ref{section: proof of thm2.4} is devoted to the proof of Theorem \ref{thm2.4}. To this end, we need the following lemmas. Note  that the logistic loss is given by $\phi(t)=\log(1+\me^{-t})$ with $\phi'(t)=-\frac{1}{1+\me^{t}}\in (-1,0)$ and $\phi''(t)=\frac{\me^t}{(1+\me^t)^2}=\frac{1}{\me^{t}+\me^{-t}+2}\in (0,\frac{1}{4}]$ for all $t\in\mbR$.

\begin{lem}\label{lemma5.7}
	Let $\eta_0 \in (0,1)$, $F_0 \in \ykh{0, \log \frac{1+\eta_0}{1-\eta_0}}$, $a \in \zkh{-F_0,F_0}$, $\phi(t)=\log(1+\me^{-t})$ be the logistic loss, $d\in\mb N$, and $P$ be a Borel probability measure on $[0,1]^d\times\hkh{-1,1}$ of which the conditional probability function $[0,1]^d\ni z\mapsto P(\hkh{1}|z)\in[0,1]$ is denoted by $\eta$.  Then for any $x\in [0,1]^d$ such that  $\abs{2\eta(x)-1}>\eta_0$, there holds
	\beq\ba\label{ineq 5.70}
	0&\leq \abs{a-F_0 \mr{sgn}(2\eta(x)-1)} \cdot \ykh{\frac{1-\eta_0}{2} \phi'(-F_0) - \frac{{\eta_0}+1}{2} \phi'(F_0)}\\
	&\leq \abs{a-F_0 \mr{sgn}(2\eta(x)-1)} \cdot \ykh{\frac{1-\eta_0}{2} \phi'(-F_0) - \frac{{\eta_0}+1}{2} \phi'(F_0)}\\
	&\qquad + \frac{1}{2\ykh{\me^{-F_0}+\me^{F_0}+2}} \abs{a-F_0 \mr{sgn}(2\eta(x)-1)}^2\\ 
	&\leq \int_{\{-1,1\}}\ykh{\phi(ya)-\phi(yF_0 \mr{sgn}(2\eta(x)-1))}\mr{d}P(y|x)\\
	&\leq \abs{a-F_0 \mr{sgn}(2\eta(x)-1)} + F^2_0.
	\ea\eeq
\end{lem}

\begin{proof} Given $x\in[0,1]^d$, recall the function $V_x$ defined in the proof of Lemma \ref{lemma5.3}. By Taylor expansion, there exists $\xi$ between $a$ and $F_0 \mr{sgn}(2\eta(x)-1)$ such that
	\beq\ba\label{ineq 5.71}
	&\int_{\{-1,1\}}\ykh{\phi(ya)-\phi(yF_0 \mr{sgn}(2\eta(x)-1))}\mr{d}P(y|x)\\
	&= V_x(a)-V_x(F_0 \mr{sgn}(2\eta(x)-1))\\
	&= \ykh{a-F_0 \mr{sgn}(2\eta(x)-1)}\cdot V'_x(F_0 \mr{sgn}(2\eta(x)-1))+\frac{1}{2}\abs{a-F_0 \mr{sgn}(2\eta(x)-1)}^2\cdot V''_x(\xi).
	\ea \eeq Since $\xi\in [-F_0,F_0]$, we have
	\[\ba
	0 &\leq \frac{1}{\me^{-F_0}+\me^{F_0}+2} = \inf\hkh{\phi''(t)\mid t\in [-F_0,F_0]} \\
	&\leq V''_x(\xi)=\eta(x)\phi''(\xi)+(1-\eta(x))\phi''(-\xi)\leq \frac{1}{4}
	\ea\] and then 
	\beq\ba\label{ineq 5.72}
	0 &\leq \frac{1}{2} \abs{a- F_0 \mr{sgn}(2\eta(x)-1)}^2 \cdot \frac{1}{\me^{-F_0}+\me^{F_0}+2}\\ 
	&\leq \frac{1}{2} \abs{a- F_0 \mr{sgn}(2\eta(x)-1)}^2 \cdot V''_x(\xi)\\
	&\leq \frac{1}{2}\ykh{\abs{a} + F_0}^2\cdot\frac{1}{4} \leq \frac{1}{2} F^2_0.
	\ea \eeq
	
	On the other hand, if $2 \eta(x) -1 > \eta_0$, then
	\begin{align*}
	&\ykh{a-F_0 \mr{sgn}(2\eta(x)-1)}\cdot V'_x(F_0 \mr{sgn}(2\eta(x)-1))\\ 
	&= \ykh{a-F_0} \ykh{\eta(x)\phi'(F_0)-(1-\eta(x))\phi'(-F_0)}\\
	&= \abs{a-F_0} \ykh{(1-\eta(x))\phi'(-F_0)-\eta(x)\phi'(F_0))}\\
	&\geq \abs{a-F_0} \ykh{\ykh{1-\frac{1+\eta_0}{2}} \phi'(-F_0)-\frac{1+\eta_0}{2} \phi'(F_0)}\\
	&=\abs{a-F_0 \mr{sgn}(2\eta(x)-1)}\cdot \ykh{\frac{1-\eta_0}{2} \phi'(-F_0)-\frac{1+\eta_0}{2} \phi'(F_0)}.
\end{align*} Similarly, if $2 \eta(x) -1 <- \eta_0$, then 
\begin{align*}
	&\ykh{a-F_0 \mr{sgn}(2\eta(x)-1)}\cdot V'_x(F_0 \mr{sgn}(2\eta(x)-1))\\ 
	&= \ykh{a+F_0} \ykh{\eta(x)\phi'(-F_0)-(1-\eta(x))\phi'(F_0)}\\
	&= \abs{a+F_0} \ykh{\eta(x)\phi'(-F_0)-(1-\eta(x)) \phi'(F_0))}\\
	&\geq \abs{a+F_0} \ykh{\frac{1-\eta_0}{2} \phi'(-F_0)-\ykh{1-\frac{1-\eta_0}{2}} \phi'(F_0)}\\
	&=\abs{a-F_0 \mr{sgn}(2\eta(x)-1)}\cdot \ykh{\frac{1-\eta_0}{2} \phi'(-F_0)-\frac{1+\eta_0}{2} \phi'(F_0)}.
	\end{align*} Therefore, for given $x\in [0,1]^d$ satisfying $\abs{2\eta(x)-1}>\eta_0$, there always holds
	\beq\ba\label{ineq 5.73}
	& \ykh{a-F_0 \mr{sgn}(2\eta(x)-1)}\cdot V'_x(F_0 \mr{sgn}(2\eta(x)-1))\\
	&\quad \geq \abs{a-F_0 \mr{sgn}(2\eta(x)-1)}\cdot \ykh{\frac{1-\eta_0}{2} \phi'(-F_0)-\frac{1+\eta_0}{2} \phi'(F_0)}.
	\ea \eeq
	
	We next show that $\frac{1-\eta_0}{2} \phi'(-F_0)-\frac{1+\eta_0}{2} \phi'(F_0)>0$. Indeed, let $g(t)=\frac{1-\eta_0}{2} \phi'(-t)-\frac{1+\eta_0}{2} \phi'(t)$. Then $g'(t)=-\frac{1-\eta_0}{2} \phi''(-t)-\frac{1+\eta_0}{2} \phi''(t)<0$, i.e., $g$ is strictly decreasing, and thus 
	\beq\ba\label{ineq 5.74}
	\frac{1-\eta_0}{2} \phi'(-F_0)-\frac{1+\eta_0}{2} \phi'(F_0)=g(F_0)>g\ykh{\log \frac{1+\eta_0}{1-\eta_0}}=0.
	\ea \eeq
	
	Moreover, we also have 
	\beq\ba\label{ineq 5.75}
	& \ykh{a-F_0 \mr{sgn}(2\eta(x)-1)}\cdot V'_x(F_0 \mr{sgn}(2\eta(x)-1))\\
	&\leq \abs{a-F_0 \mr{sgn}(2\eta(x)-1)}\cdot \abs{V'_x(F_0 \mr{sgn}(2\eta(x)-1))}\\
	&= \abs{a-F_0 \mr{sgn}(2\eta(x)-1)} \cdot \abs{\eta(x) \phi'(F_0 \mr{sgn}(2\eta(x)-1))-(1-\eta(x))\phi'(-F_0\mr{sgn}(2\eta(x)-1))}\\
	&\leq \abs{a-F_0 \mr{sgn}(2\eta(x)-1)} \abs{\eta(x)+(1-\eta(x))}=\abs{a-F_0 \mr{sgn}(2\eta(x)-1)}.
	\ea \eeq
	
	Then the first inequality of \eqref{ineq 5.70} is from \eqref{ineq 5.74}, the third inequality of \eqref{ineq 5.70} is due to \eqref{ineq 5.71}, \eqref{ineq 5.72} and \eqref{ineq 5.73}, and the last inequality of \eqref{ineq 5.70} is from \eqref{ineq 5.71}, \eqref{ineq 5.72} and \eqref{ineq 5.75}. Thus we complete the proof.	
\end{proof}

\begin{lem}\label{lemma5.8}
	Let $\eta_0 \in (0,1)$, $F_0 \in \ykh{0, \log \frac{1+\eta_0}{1-\eta_0}}$, $d\in\mb N$, and $P$ be a Borel probability measure on $[0,1]^d\times\hkh{-1,1}$ of which the conditional probability function $[0,1]^d\ni z\mapsto P(\hkh{1}|z)\in[0,1]$ is denoted by $\eta$. Define 
	\beq \label{def 5.80}
	\psi:[0,1]^d\times\{-1,1\}&\to\mbR,\\ (x,y)&\mapsto \left\{
	\ba
	&\phi\ykh{yF_0\mr{sgn}(2\eta(x)-1)},&&\textrm{ if } \abs{2\eta(x)-1}>\eta_0,\\
	&\phi\ykh{y\log \frac{\eta(x)}{1-\eta(x)}},&&\textrm{ if } \abs{2\eta(x)-1}\leq \eta_0.
	\ea
	\right.
	\eeq Then there hold
	\beq\ba\label{ineq 5.80}
	&\int_{[0,1]^d\times\{-1,1\}}{\ykh{\phi\ykh{yf(x)}-\psi(x,y)}^2}\mr{d}P(x,y)\\
	& \leq\frac{8}{1-\eta^2_0}\cdot\int_{[0,1]^d\times\{-1,1\}}\ykh{\phi\ykh{yf(x)}-\psi(x,y)}\mr{d}P(x,y)
	\ea\eeq for any measurable $f:[0,1]^d\to\zkh{-F_0,F_0},$ and 
	\beq\ba\label{ineq 5.81}
	0 \leq \psi(x,y)\leq \log \frac{2}{1-\eta_0}, \quad \forall (x,y) \in [0,1]^d \times \left\{-1,1\right\}.
	\ea \eeq
\end{lem}

\begin{proof}
	Recall that given $x\in [0,1]^d$, $V_x(t)=\eta(x)\phi(t)+(1-\eta(x))\phi(-t), \forall t \in \mathbb{R}$. Due to inequality \eqref{ineq 5.70} and Lemma \ref{lemma5.3}, for any measurable $f:[0,1]^d \to [-F_0,F_0]$, we have 
\begin{align*}
	& \int_{[0,1]^d\times\{-1,1\}} \ykh{\phi\ykh{yf(x)}-\psi(x,y)}\mr{d}P(x,y)\\
	&= \int_{\abs{2\eta(x)-1}> \eta_0} \int_{\hkh{-1,1}} \phi\ykh{yf(x)}-\phi\ykh{yF_0\mr{sgn}(2\eta(x)-1)} \mr{d}P(y|x) \mr{d}P_X(x)\\
	&\quad + \int_{\abs{2\eta(x)-1}\leq \eta_0} \int_{\hkh{-1,1}} \phi\ykh{yf(x)}-\phi\ykh{y\log \frac{\eta(x)}{1-\eta(x)}} \mr{d}P(y|x) \mr{d}P_X(x)\\
	&\geq  \int_{\abs{2\eta(x)-1}> \eta_0}  \frac{1}{2\ykh{\me^{F_0}+\me^{-F_0}+2}}\abs{f(x)-F_0 \mr{sgn}(2\eta(x)-1)}^2\mr{d}P_X(x)\\
	&\quad + \int_{\abs{2\eta(x)-1} \leq \eta_0}   \left[\inf_{t\in \zkh{\log\frac{1-{\eta_0}}{1+{\eta_0}},\log\frac{1+{\eta_0}}{1-{\eta_0}}}}\frac{1}{2(\me^t+\me^{-t}+2)}\right]\abs{f(x)-\log\frac{\eta(x)}{1-\eta(x)}}^2\mr{d}P_X(x)\\
	&\geq \frac{1}{2}\frac{1}{\frac{1+\eta_0}{1-\eta_0}+\frac{1-\eta_0}{1+\eta_0}+2}\int_{\left\{\abs{2\eta(x)-1}> \eta_0\right\}\times \hkh{-1,1}}\abs{ \phi\ykh{yf(x)}-\phi\ykh{yF_0\mr{sgn}(2\eta(x)-1)}}^2 \mr{d}P(x,y)\\
	&\quad + \frac{1}{2}\frac{1}{\frac{1+\eta_0}{1-\eta_0}+\frac{1-\eta_0}{1+\eta_0}+2}\int_{\left\{\abs{2\eta(x)-1}\leq \eta_0\right\}\times \hkh{-1,1}} \abs{\phi\ykh{yf(x)}-\phi\ykh{{y}\log\frac{\eta(x)}{1-\eta(x)}}}^2 \mr{d}P(x,y)\\
	&=\frac{1-\eta^2_0}{8}\cdot\int_{[0,1]^d\times\{-1,1\}}\ykh{\phi\ykh{yf(x)}-\psi(x,y)}^{{2}}\mr{d}P(x,y),
\end{align*} where the second inequality is from \eqref{ineq 36} and the fact that $F_0 \in \ykh{0, \log \frac{1+\eta_0}{1-\eta_0}}$. Thus we have proved the inequality \eqref{ineq 5.80}.  
	
	On the other hand, from the definition of $\psi$ as well as $F_0 \in \ykh{0, \log \frac{1+\eta_0}{1-\eta_0}}$, we also have
	\[\ba
	0&\leq \psi(x,y) \leq \max\left\{\phi(-F_0),\phi\ykh{-\log\frac{1+\eta_0}{1-\eta_0}}\right\}\leq \phi\ykh{-\log\frac{1+\eta_0}{1-\eta_0}}=\log\frac{2}{1-\eta_0},
	\ea\] which gives the inequality \eqref{ineq 5.81}. The proof is completed.	 
\end{proof}

Now we are in the position to prove Theorem \ref{thm2.4}.

\begin{proof}[Proof of Theorem \ref{thm2.4}] Let  $\eta_0\in (0,1)\cap [0,t_1]$, $F_0 \in (0,\log \frac{1+\eta_0}{1-\eta_0}) \cap [0,1]$ , $\xi\in(0, \frac{1}{2} \qx t_2]$ and $P\in\mc H^{d,\beta,r,I,\Theta,s_1,s_2}_{6,t_1,c_1,t_2,c_2}$ be arbitrary. Denote by $\eta$ the conditional probability function $P(\hkh{1}|\cdot)$ of $P$. By definition, there exists a classifier  $\texttt C\in\mc{C}^{d,\beta,r,I,\Theta}$ such that
	\eqref{Tsybakovnoisecondition}, \eqref{230225012} and \eqref{margincondition} hold.    According to Proposition A.4 and the proof of Theorem 3.4 in \cite{kim2021fast}, there exist positive constants $G_0, N_0,S_0,B_0$ only depending on $d,\beta, r, I, \Theta$ and $\tilde{f}_0 \in \fdnn_d(G_{\xi},N_{\xi},S_{\xi},B_{\xi},1)$ such that $\tilde{f}_0 (x) = \texttt C(x)$ for $x\in [0,1]^d$ with $\Delta_{\texttt C}(x)>\xi$, where
	\beq \label{eq 2.40}
	G_{\xi} = G_0 \log \frac{1}{\xi}, \ N_{\xi} = N_0 \left(\frac{1}{\xi}\right)^{\frac{d-1}{\beta}},
	\ S_{\xi} = S_0\left(\frac{1}{\xi}\right)^{\frac{d-1}{\beta}}\log\left(\frac{1}{\xi}\right),
	\ B_{\xi} = \left(\frac{B_0}{\xi}\right).
	\eeq
	
	Define $\psi:[0,1]^d\times\{-1,1\}\to\mbR$ by \eqref{def 5.80}. Then for any measurable function $f:[0,1]^d \to [-F_0,F_0]$, there holds
	\beq\ba\label{ineq 2.41}
	&\mc{E}_P(f)=\mc{E}_P\ykh{\frac{f}{F_0}}\leq \int_{[0,1]^d}\abs{\frac{f(x)}{F_0}-\mr{sgn}(2\eta(x)-1)}\abs{2\eta(x)-1}\mr{d}P_X(x)\\
	&\leq 2 P\ykh{\abs{2\eta(x)-1}\leq \eta_0} + \int_{\abs{2\eta(x)-1}>\eta_0} \abs{\frac{f(x)}{F_0}-\mr{sgn}(2\eta(x)-1)} \mr{d}P_X(x)\\
	&\leq 2c_1 \eta_0^{s_1} +\frac{1}{F_0} \int_{\abs{2\eta(x)-1}>\eta_0} \abs{f(x)-F_0\mr{sgn}(2\eta(x)-1)} \mr{d}P_X(x)\\
	&\leq 2c_1 \eta_0^{s_1} +  \int_{\abs{2\eta(x)-1}>\eta_0} \frac{\int_{}\ykh{\phi(yf(x))-\phi(yF_0 \mr{sgn}(2\eta(x)-1))}\mr{d}P(y|x)}{F_0\cdot\ykh{\frac{1-\eta_0}{2} \phi'(-F_0) - \frac{{\eta_0}+1}{2} \phi'(F_0)}}\mr{d}P_X(x)\\
	&\leq 2c_1 \eta_0^{s_1} +  \frac{\int_{[0,1]^d \times \{-1,1\}}\ykh{\phi(yf(x))-\psi(x,y)}\mr{d}P({x,y})}{F_0\cdot\ykh{\frac{1-\eta_0}{2} \phi'(-F_0) - \frac{{\eta_0}+1}{2} \phi'(F_0)}},
	\ea\eeq where the first inequality is from Theorem 2.31 of \cite{steinwart2008support}, the third inequality is due to the noise condition \eqref{Tsybakovnoisecondition}, and the fourth inequality is from \eqref{ineq 5.70} in Lemma \ref{lemma5.7}.
	
	Let $\mc{F}=\fdnn_d(G_{\xi},N_{\xi},S_{\xi},B_{\xi},F_0)$ with $(G_{\xi},N_{\xi},S_{\xi},B_{\xi})$ given by \eqref{eq 2.40}, $\Gamma=\frac{8}{1-\eta^2_0}$ and $M=\frac{2}{1-\eta_0}$ in Theorem \ref{thm2.1}. Then we will use this theorem to derive the desired  generalization bounds for the $\phi$-ERM  $\hat{f}_n:=\hat{f}^{\FNN}_n$ over $\fdnn_d(G_{\xi},N_{\xi},S_{\xi},B_{\xi},F_0)$. Indeed, Lemma \ref{lemma5.8} guarantees that the conditions \eqref{ineq 2.10}, \eqref{ineq 2.11} and \eqref{ineq 2.12} of Theorem \ref{thm2.1} are satisfied. Moreover, take $\gamma=\frac{1}{n}$. Then $W=\max\hkh{3,\;\mc{N}\ykh{\mc{F},\gamma}}$ satisfies
	\[\log W \leq C_{d,\beta, r, I, \Theta} \xi^{-\frac{d-1}{\beta}}\left(\log\frac{1}{\xi}\right)^2\left(\log \frac{1}{\xi}+ \log n\right).\] Thus the expectation of $\int_{[0,1]^d \times \{-1,1\}}\ykh{\phi(y\hat{f}_n(x))-\psi(x,y)}\mr{d}P(x,y)$ can be bounded by inequality \eqref{bound 2.10} in Theorem \ref{thm2.1} as \beq\label{ineq 93}
	&\bm{E}_{P^{\otimes n}}\zkh{\int_{[0,1]^d \times \{-1,1\}}\ykh{\phi(y\hat{f}_n(x))-\psi(x,y)}\mr{d}P(x,y)}\\&\leq \frac{4000C_{d,\beta, r, I, \Theta} \xi^{-\frac{d-1}{\beta}}\left(\log\frac{1}{\xi}\right)^2\left(\log \frac{1}{\xi}+ \log n\right)}{n(1-\eta_0^2)}\\&\;\;\;\;\;\;\;\;\;\;+2\inf_{f\in\mc{F}}\ykh{\mc{R}_P^{\phi}(f)-\int_{[0,1]^d\times\{-1,1\}}{\psi(x,y)}\mr{d}P(x,y)}.
	\eeq
	
	We next estimate the approximation error, i.e., the second term on the right hand side of \eqref{ineq 93}. Take $f_0 =F_0 \tilde{f}_0 \in \mc{F}$ where $\tilde{f}_0 \in \fdnn_d(G_{\xi},N_{\xi},S_{\xi},B_{\xi},1)$ satisfying $\tilde{f}_0 (x) =\texttt C(x)$ for $x\in [0,1]^d$ with $\Delta_{\texttt C}(x)>\xi$. Then one can bound the approximation error as
	\beq\label{23063002}
	&\inf_{f\in\mc{F}}\ykh{\mc{R}_P^{\phi}(f)-\int_{[0,1]^d\times\{-1,1\}}{\psi(x,y)}\mr{d}P(x,y)}\\
	&\leq \mc{R}_P^{\phi}(f_0) - \int_{[0,1]^d\times\{-1,1\}}{\psi(x,y)}\mr{d}P(x,y)= I_1 + I_2 + I_3,
	\eeq where
\begin{align*}
	&I_1:= \int_{\hkh{\abs{2\eta(x)-1}> \eta_0, \Delta_{\texttt C}(x)>\xi} \times \{-1,1\}} \phi(yf_0(x))-\phi(yF_0 \mr{sgn}(2\eta(x)-1))\mr{d}P(x,y),\\
	&I_2:=\int_{\hkh{\abs{2\eta(x)-1}\leq \eta_0} \times \{-1,1\}} \phi(yf_0(x))-\phi\ykh{y\log\frac{\eta(x)}{1-\eta(x)}}\mr{d}P(x,y),\\
	&I_3:= \int_{\hkh{\abs{2\eta(x)-1}> \eta_0, \Delta_{\texttt C}(x)\leq \xi} \times \{-1,1\}} \phi(yf_0(x))-\phi(yF_0 \mr{sgn}(2\eta(x)-1))\mr{d}P(x,y).
\end{align*} Note that $f_0(x)=F_0\tilde{f}_0(x)=F_0\texttt C(x)=F_0\mr{sgn}(2\eta(x)-1)$ for $P_X$-almost all $x\in[0,1]^d$ with $\Delta_{\texttt C}(x)>\xi$. Thus it follows that $I_1=0$. On the other hand, from Lemma \ref{lemma5.3} and the noise condition \eqref{Tsybakovnoisecondition}, we see that
	\beq\label{23063003}
	&I_2 \leq \int_{\hkh{\abs{2\eta(x)-1}\leq \eta_0}\times \{-1,1\}} \abs{f_0(x)-\log\frac{\eta(x)}{1-\eta(x)}}^2\mr{d}P(x,y)\\
	&\leq \int_{\hkh{\abs{2\eta_(x)-1}\leq \eta_0}\times \{-1,1\}} \ykh{F_0+\log \frac{1+\eta_0}{1-\eta_0}}^2 \mr{d}P(x,y)\leq 4 \ykh{\log\frac{1+\eta_0}{1-\eta_0}}^2 c_1\cdot \eta_0^{s_1}.
	\eeq  Moreover, due to Lemma \ref{lemma5.7} and the margin condition \eqref{margincondition}, we have
\beq\label{23063004}
	I_3 &\leq \int_{\hkh{\abs{2\eta(x)-1}> \eta_0,\Delta_{\texttt C}(x)\leq \xi}} \left(2F_0+F^2_0\right) \mr{d}P_X(x)\\
	&\leq 3F_0\cdot P_X(\setl{x\in[0,1]^d}{\Delta_{\texttt C}(x)\leq \xi})\leq 3F_0 \cdot c_2 \cdot \xi^{s_2}.
	\eeq The estimates above together with \eqref{ineq 2.41} and \eqref{ineq 93} give
	\beq\ba\label{ineq 2.42x}
	&\bm{E}_{P^{\otimes n}}\zkh{\mc{E}_P(\hat{f}_n)}\\&\leq 2c_1 \eta_0^{s_1} + \frac{1}{F_0} \cdot \frac{\bm{E}_{P^{\otimes n}}\zkh{\int_{[0,1]^d \times \{-1,1\}}\ykh{\phi(y\hat{f}_n(x))-\psi(x,y)}\mr{d}P(x,y)}}{\frac{1-\eta_0}{2} \phi'(-F_0) - \frac{{\eta_0}+1}{2} \phi'(F_0)}\\
	&\leq 2c_1 \eta_0^{s_1} + \frac{{8} \abs{\log\frac{1+\eta_0}{1-\eta_0}}^2 c_1 \eta_0^{s_1}+{6}F_0  c_2  \xi^{s_2}+\frac{{4000}C_{d,\beta, r, I, \Theta} \xi^{-\frac{d-1}{\beta}}\left(\log\frac{1}{\xi}\right)^2\left(\log \frac{1}{\xi}+ \log n\right)}{n(1-\eta^2_0)}}{F_0\cdot \ykh{\frac{1-\eta_0}{2} \phi'(-F_0) - \frac{{\eta_0}+1}{2} \phi'(F_0)}}.
	\ea\eeq Since $P$ is arbitrary, we can take the  supremum over all $P\in \mc H^{d,\beta,r,I,\Theta,s_1,s_2}_{6,t_1,c_1,t_2,c_2}$ to obtain from \eqref{ineq 2.42x} that \beq\ba
\label{ineq 2.42}
	&\sup_{P\in\mc H^{d,\beta,r,I,\Theta,s_1,s_2}_{6,t_1,c_1,t_2,c_2}}\bm{E}_{P^{\otimes n}}\zkh{\mc{E}_P(\hat{f}^{\FNN}_n)}\\
	&\leq 2c_1 \eta_0^{s_1} + \frac{{8} \abs{\log\frac{1+\eta_0}{1-\eta_0}}^2 c_1 \eta_0^{s_1}+{6}F_0  c_2  \xi^{s_2}+\frac{{4000}C_{d,\beta, r, I, \Theta} \xi^{-\frac{d-1}{\beta}}\left(\log\frac{1}{\xi}\right)^2\left(\log \frac{1}{\xi}+ \log n\right)}{n(1-\eta^2_0)}}{F_0\cdot \ykh{\frac{1-\eta_0}{2} \phi'(-F_0) - \frac{{\eta_0}+1}{2} \phi'(F_0)}}.
	\ea\eeq \eqref{ineq 2.42} holds for all $\eta_0\in (0,1)\cap [0,t_1]$, $F_0 \in (0,\log \frac{1+\eta_0}{1-\eta_0}) \cap [0,1]$ , $\xi\in(0, \frac{1}{2} \qx t_2]$.  We then take suitable $\eta_0$, $F_0$,  and $\xi$ in \eqref{ineq 2.42} to derive the convergence rates stated in Theorem \ref{thm2.4}.

	\beq\ba
	&\sup_{P\in\mc H^{d,\beta,r,I,\Theta,s_1,s_2}_{6,t_1,c_1,t_2,c_2}}\bm{E}_{P^{\otimes n}}\zkh{\mc{E}_P(\hat{f}^{\FNN}_n)}\\
	&\leq 2c_1 \eta_0^{s_1} + \frac{{8} \abs{\log\frac{1+\eta_0}{1-\eta_0}}^2 c_1\cdot \eta_0^{s_1}+{6}F_0  c_2  \xi^{s_2}+\frac{{4000}C_{d,\beta, r, I, \Theta} \xi^{\frac{d-1}{\beta}}\left(\log\frac{1}{\xi}\right)^2\left(\log \frac{1}{\xi}+ \log n\right)}{n(1-\eta^2_0)}}{F_0\cdot \ykh{\frac{1-\eta_0}{2} \phi'(-F_0) - \frac{{\eta_0}+1}{2} \phi'(F_0)}}.
	\ea\eeq
		
		\textbf{Case \uppercase\expandafter{\romannumeral1}.} When $s_1=s_2=\infty$, taking $\eta_0=F_0=t_1\qx\frac{1}{2}$ and $\xi=t_2\qx\frac{1}{2}$ in \eqref{ineq 2.42} yields   
		\[\sup_{P\in\mc H^{d,\beta,r,I,\Theta,s_1,s_2}_{6,t_1,c_1,t_2,c_2}}\bm{E}_{P^{\otimes n}}\zkh{\mc{E}_P\ykh{\hat{f}^{\FNN}_n}}\lesssim  \frac{\log n}{n}.\]
		
		\textbf{Case \uppercase\expandafter{\romannumeral2}. } When $s_1=\infty$ and $s_2<\infty$, taking $\eta_0=F_0=t_1\qx\frac{1}{2}$ and $\xi\asymp \left(\frac{(\log n)^3}{n}\right)^{\frac{1}{s_2+\frac{d-1}{\beta}}}$ in \eqref{ineq 2.42} yields 
		\[\sup_{P\in\mc H^{d,\beta,r,I,\Theta,s_1,s_2}_{6,t_1,c_1,t_2,c_2}}\bm{E}_{P^{\otimes n}}\zkh{\mc{E}_P\ykh{\hat{f}^{\FNN}_n}}\lesssim \ykh{\frac{\left(\log n\right)^3}{n}}^{\frac{1}{1+\frac{d-1}{\beta s_2}}}.\]
		
		\textbf{Case \uppercase\expandafter{\romannumeral3}. } When $s_1<\infty$ and $s_2=\infty$, take $\eta_0=F_0 \asymp \ykh{\frac{\log n}{n}}^{\frac{1}{s_1+{2}}}$ and $\xi=t_2\qx\frac{1}{2}$ in \eqref{ineq 2.42}. From the fact that $\frac{\eta_0}{4} \leq \frac{1-\eta_0}{2} \phi'(-\eta_0) - \frac{\eta_0+1}{2} \phi'(\eta_0)\leq \eta_0,\forall 0 \leq \eta_0 \leq 1$, the item in the denominator of the second term on the right hand side of \eqref{ineq 2.42} is larger than $\frac{1}{4}\eta^2_0$. Then we have
		\[\sup_{P\in\mc H^{d,\beta,r,I,\Theta,s_1,s_2}_{6,t_1,c_1,t_2,c_2}}\bm{E}_{P^{\otimes n}}\zkh{\mc{E}_P\ykh{\hat{f}^{\FNN}_n}}\lesssim  \ykh{\frac{\log n}{n}}^{\frac{s_1}{s_1+2}}.\]
		
		\textbf{Case \uppercase\expandafter{\romannumeral4}. } When $s_1<\infty$ and $s_2<\infty$, taking 
		\[\eta_0=F_0 \asymp \ykh{\frac{(\log n)^3}{n}}^{\frac{s_2}{s_2+(s_1+1)\ykh{s_2+\frac{d-1}{\beta}}}} \text{ and } \xi \asymp \ykh{\frac{(\log n)^3}{n}}^{\frac{s_1+1}{s_2+(s_1+1)\ykh{s_2+\frac{d-1}{\beta}}}}\] in
		\eqref{ineq 2.42} yields
		\[\sup_{P\in\mc H^{d,\beta,r,I,\Theta,s_1,s_2}_{6,t_1,c_1,t_2,c_2}}\bm{E}_{P^{\otimes n}}\zkh{\mc{E}_P\ykh{\hat{f}^{\FNN}_n}}\lesssim \ykh{\frac{\left(\log n\right)^3}{n}}^{\frac{s_1}{1+(s_1+1)\left(1+\frac{d-1}{\beta{s_2}}\right)}}.\]  
	 Combining above cases, we obtain the desired results. The proof of Theorem \ref{thm2.4} is completed.
\end{proof}

\subsection{Proof of Theorem \ref{thm2.6p} and Corollary  \ref{thm2.6}}\label{section: proof of thm2.6}

In Appendix \ref{section: proof of thm2.6}, we provide the proof of Theorem \ref{thm2.6p} and Corollary  \ref{thm2.6}. 
Hereinafter, for $a\in \mb R^d$ and $R\in\mbR$, we define $\mathscr{B}(a,R):=\setl{x\in\mb R^d}{\norm{x-a}_2\leq R}$. 
\begin{lem}\label{xxlem3.1}Let $d\in\mb N$, $\beta\in(0,\infty)$, $r\in(0,\infty)$, $Q\in\mb N\cap(1,\infty)$, \[
	G_{Q,d}:=\setl{(\frac{k_1}{2Q},\ldots,\frac{k_d}{2Q})^\top}{k_1,\ldots,k_d\text{ are odd integers}}\cap[0,1]^d,
	\]  and $T:G_{Q,d}\to\hkh{-1,1}$ be a map. Then there exist a constant $\mr c_1\in(0,\frac{1}{9999})$ only depending on $(d,\beta,r)$, and an $f\in \mc{B}^{\beta}_r\ykh{[0,1]^d}$ depending on $(d,\beta,r,Q,T)$, such that $\norm{f}_{[0,1]^d}= \frac{\mr c_1}{Q^\beta}$, and
	\[
	f(x)=\norm{f}_{[0,1]^d}\cdot T(a)=\frac{\mr c_1}{Q^\beta}\cdot T(a),\;\forall\; a\in G_{Q,d},\;x\in \mathscr{B}(a,{\frac{1}{5Q}})\cap[0,1]^d. 
	\]\end{lem}
\begin{proof}
	Let \[\kappa:\mbR\to[0,1], t\mapsto\frac{\int_{t}^\infty\exp\ykh{-1/(x-1/9)}\cdot\exp\ykh{-1/(1/8-x)}\cdot\idf_{(1/9,1/8)}(x)\mr{d}x}{\int_{1/9}^{1/8}\exp\ykh{-1/(x-1/9)}\cdot\exp\ykh{-1/(1/8-x)}\mr{d}x}\] be a well defined infinitely differentiable decreasing function on $\mbR$ with $\kappa(t)=1$ for $t\leq 1/9$ and $\kappa(t)=0$ for $t\geq 1/8$. Then define  $b:=\ceil{\beta}-1$, $\lambda:=\beta-b$, 
	\[
	u:\mbR^d\to [0,1], x\mapsto \kappa(\norm{x}_2^2),
	\] and $\mr c_2:=\norm{u|_{[-2,2]^d}}_{\mc{C}^{b,\lambda}([-2,2]^d)}$. Obviously, $u$ only depends on $d$, and $\mr c_2$ only depends on $(d,\beta)$.  Since $u$ is infinitely differentiable and supported in $\mathscr{B}(\bm 0, {\sqrt{\frac{1}{8}}})$, we have $0<\mr c_2<\infty$. Take $\mr c_1:=\frac{r}{4\mr c_2}\qx\frac{1}{10000}$. Then $\mr c_1$ only depends on $(d,\beta,r)$, and $0<\mr c_1<\frac{1}{9999}$. Define
	\[
	f:[0,1]^d\to\mbR, x\mapsto \sum_{a\in G_{Q,d}} T(a)\cdot \frac{\mr c_1}{Q^\beta}\cdot u(Q\cdot(x-a)).
	\] We then show that these $\mr c_1$ and $f$ defined above have the desired properties. 
	
	For any $\bm m\in\mb (\mb N\cup \hkh{0})^d$, we write $u_{\bm m}$ for $\mr D^{\bm m}u$, i.e., the partial derivative of $u$ with respect to the multi-index $\bm m$. An elementary calculation yields
	\beq\label{23021917}
	\mr{D}^{\bm m}f(x)=\sum_{a\in G_{Q,d}} T(a)\cdot \frac{\mr c_1}{Q^{\beta-\norm{\bm m}_1}}\cdot u_{\bm m}(Q\cdot(x-a)),\;\forall\;\bm m\in\mb (\mb N\cup \hkh{0})^d,\;x\in[0,1]^d. 
	\eeq Note that the supports of the functions $T(a)\cdot \frac{\mr c_1}{Q^{\beta-\norm{\bm m}_1}}\cdot u_{\bm m}(Q\cdot(x-a))$ ($a\in G_{Q,d}$) in \eqref{23021917} are disjoint. Indeed, we have \beq\label{22082301}
	&\setr{x\in\mbR^d}{T(a)\cdot \frac{\mr c_1}{Q^{\beta-{\norm{\bm m}_1}}}\cdot u_{\bm m}(Q\cdot(x-a))\neq 0}\\&\subset\mathscr{B}(a,{\frac{\sqrt{1/8}}{Q}})\subset \setr{a+v}{v\in (\frac{-1}{2Q},\frac{1}{2Q})^d}\\&\subset [0,1]^d\setminus \setr{z+v}{v\in [\frac{-1}{2Q},\frac{1}{2Q}]^d},\,\forall\,\bm m\in\mb (\mb N\cup \hkh{0})^d,\,a\in G_{Q,d},\,z\in G_{Q,d}\setminus\hkh{a},\eeq and sets ${\mathscr{B}(a,{\frac{\sqrt{1/8}}{Q}})}\;\ykh{a\in G_{Q,d}}$ are disjoint.  Therefore, 
	\beq\label{22082302}
	&\norm{\mr D^{\bm m}f}_{[0,1]^d}=\sup_{a\in G_{Q,d}}\sup_{x\in[0,1]^d}\abs{T(a)\cdot \frac{\mr c_1}{Q^{\beta-{\norm{\bm m}_1}}}\cdot u_{\bm m}(Q\cdot(x-a))}\\
	&=\sup_{a\in G_{Q,d}}\sup_{x\in\mathscr{B}(a,{\frac{\sqrt{1/8}}{Q}})}\abs{T(a)\cdot \frac{\mr c_1}{Q^{\beta-{\norm{\bm m}_1}}}\cdot u_{\bm m}(Q\cdot(x-a))}\\&=\sup_{a\in G_{Q,d}}\sup_{x\in\mathscr{B}(\bm 0,\sqrt{1/8})}\abs{ \frac{\mr c_1}{Q^{\beta-{\norm{\bm m}_1}}}\cdot u_{\bm m}(x)}
	\leq \sup_{x\in[-2,2]^d}\abs{ \frac{\mr c_1}{Q^{\beta-{\norm{\bm m}_1}}}\cdot u_{\bm m}(x)}\\&\leq \sup_{x\in[-2,2]^d}\abs{ {\mr c_1}\cdot u_{\bm m}(x)}\leq \mr c_1\mr c_2,\;\forall\;\bm m\in (\mb N\cup \hkh{0})^d\text{ with }\norm{\bm m}_1\leq b. 
	\eeq In particular, we have that
	\beq
	\norm{f}_{[0,1]^d}=\sup_{a\in G_{Q,d}}\sup_{x\in\mathscr{B}(\bm 0,\sqrt{1/8})}\abs{ \frac{\mr c_1}{Q^{\beta}}\cdot u(x)}=\frac{\mr c_1}{Q^\beta}. 
	\eeq Besides, for any $a\in G_{Q,d}$, any $x\in \mathscr{B}(a,\frac{1}{5Q})\cap[0,1]^d$, and any $z\in G_{Q,d}\setminus\hkh{a}$, we have \[\norm{Q\cdot(x-z)}_2\geq Q\norm{a-z}_2-Q\norm{x-a}_2\geq1-\frac{1}{5}>{\sqrt{1/8}}>\sqrt{1/9}>\norm{Q\cdot (x-a)}_2,\]which means that $u(Q\cdot(x-z))=0$ and $u(Q\cdot(x-a))=1$ . Thus
	\beq
	&f(x)=T(a)\cdot \frac{\mr c_1}{Q^\beta}\cdot u(Q\cdot(x-a))+\sum_{z\in G_{Q,d}\setminus\hkh{a}} T(z)\cdot \frac{\mr c_1}{Q^\beta}\cdot u(Q\cdot(x-z))\\
	&=T(a)\cdot \frac{\mr c_1}{Q^\beta}\cdot 1+\sum_{z\in G_{Q,d}\setminus\hkh{a}} T(z)\cdot \frac{\mr c_1}{Q^\beta}\cdot 0\\&=T(a)\cdot \frac{\mr c_1}{Q^\beta},\;\forall\;a\in G_{Q,d},\;x\in \mathscr{B}(a,\frac{1}{5Q})\cap[0,1]^d. 
	\eeq
	
	Now it remains to show that $f\in \mc{B}^{\beta}_r\ykh{[0,1]^d}$. Let $\bm m\in (\mb N\cup\hkh{0})^d$ be an arbitrary multi-index with $\norm{\bm m}_1=b$, and $x,y$ be arbitrary points in $\bigcup_{a\in G_{Q,d}}\setr{a+v}{v\in(-\frac{1}{2Q},\frac{1}{2Q})^d}$. Then there exist $a_x,a_y\in G_{Q,d}$, such that $x-a_x\in(-\frac{1}{2Q},\frac{1}{2Q})^d$ and $y-a_y\in(-\frac{1}{2Q},\frac{1}{2Q})^d$. If $a_x=a_y$, then it follows from (\ref{22082301}) that 
	\[
	u_{\bm m}(Q\cdot (x-z))=u_{\bm m}(Q\cdot (y-z))=0,\;\forall\;z\in G_{Q,d}\setminus\hkh{a_x},
	\] which, together with the fact that $\hkh{Q\cdot (x-a_x),Q\cdot (y-a_y)}\subset (-\frac{1}{2},\frac{1}{2})^d$, yields
	\begin{align*}
		&\abs{\mr{D}^{\bm m}f(x)-\mr{D}^{\bm m}f(y)}\\&=\abs{T(a_x)\cdot \frac{\mr c_1}{Q^{\beta-{\norm{\bm m}_1}}}\cdot u_{\bm m}(Q\cdot(x-a_x))-T(a_y)\cdot \frac{\mr c_1}{Q^{\beta-{\norm{\bm m}_1}}}\cdot u_{\bm m}(Q\cdot(y-a_y))}\\
		&=\mr c_1\cdot\abs{ \frac{u_{\bm m}(Q\cdot(x-a_x))-u_{\bm m}(Q\cdot(y-a_y))}{Q^{\lambda}} }\\&\leq \frac{\mr c_1}{Q^\lambda}\cdot\norm{Q\cdot (x-a_x)-Q\cdot (y-a_y)}_2^\lambda\cdot\sup_{ z,z'\in(-\frac{1}{2},\frac{1}{2})^d,z\neq z',}\abs{\frac{u_{\bm m}(z)-u_{\bm m}(z')}{\norm{z-z'}_2^\lambda}}\\
		&\leq \frac{\mr c_1}{Q^\lambda}\cdot\norm{Q\cdot (x-a_x)-Q\cdot (y-a_y)}_2^\lambda\cdot \mr c_2=\mr c_1\mr c_2\cdot\norm{x-y}_2^\lambda.
	\end{align*} If, otherwise, $a_x\neq a_y$, then it is easy to show that 
	\begin{align*}
		&\hkh{t\cdot x+(1-t)\cdot y|t\in[0,1]}\cap \setr{a_x+v}{v\in[-\frac{1}{2Q},\frac{1}{2Q}]^d\setminus(-\frac{1}{2Q},\frac{1}{2Q})^d}\neq \varnothing,\\
		&\hkh{t\cdot x+(1-t)\cdot y|t\in[0,1]}\cap \setr{a_y+v}{v\in[-\frac{1}{2Q},\frac{1}{2Q}]^d\setminus(-\frac{1}{2Q},\frac{1}{2Q})^d}\neq \varnothing. 
	\end{align*} In other words, the line segment joining points $x$ and $y$ intersects boundaries of rectangles $\setr{a_x+v}{v\in(-\frac{1}{2Q},\frac{1}{2Q})^d}$ and $\setr{a_y+v}{v\in(-\frac{1}{2Q},\frac{1}{2Q})^d}$. Take \[x'\in\hkh{t\cdot x+(1-t)\cdot y|t\in[0,1]}\cap \setr{a_x+v}{v\in[-\frac{1}{2Q},\frac{1}{2Q}]^d\setminus(-\frac{1}{2Q},\frac{1}{2Q})^d}\] and \[y'\in\hkh{t\cdot x+(1-t)\cdot y|t\in[0,1]}\cap \setr{a_y+v}{v\in[-\frac{1}{2Q},\frac{1}{2Q}]^d\setminus(-\frac{1}{2Q},\frac{1}{2Q})^d}\] (cf. Figure \ref{fig7}).\begin{figure}[H]
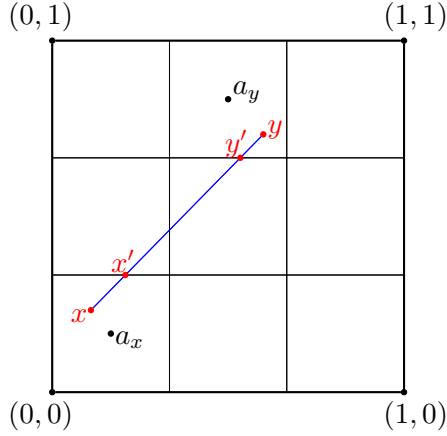
	\centering
		\betikz
		\tikzset{
			mybox/.style ={
				rectangle, %
				rounded corners =5pt, %
				minimum width =90pt, %
				minimum height =20pt, %
				inner sep=0.6pt, %
				draw=blue, %
				fill=cyan
		}}
		\tikzset{
			ttinybox/.style ={
				rectangle, %
				rounded corners =3pt, %
				minimum width =25pt, %
				minimum height =25pt, %
				inner sep=2pt, %
				draw=blue, %
				fill=cyan
		}}
		\tikzset{
			tinybox/.style ={
				rectangle, %
				rounded corners =5pt, %
				minimum width =20pt, %
				minimum height =20pt, %
				inner sep=5pt, %
				draw=blue, %
				fill=cyan
		}}
		\tikzset{
			bigbox/.style ={
				rectangle, %
				rounded corners =5pt, %
				minimum width =346.9pt, %
				minimum height =289pt, %
				inner sep=0pt, %
				draw=black, %
				fill=none
		}}
		\tikzset{
			tinycircle/.style ={
				circle, %
				minimum width =2pt, %
				minimum height =2pt, %
				inner sep=0pt, %
				draw=black, %
				fill=black %
			}
		}
		\tikzset{
			tinyr/.style ={
				circle, %
				minimum width =2pt, %
				minimum height =2pt, %
				inner sep=0pt, %
				draw=red, %
				fill=red %
			}
		}
		\tikzset{ %
			mypt/.style ={
				circle, %
				minimum width =0pt, %
				minimum height =0pt, %
				inner sep=0pt, %
			}
		}

		\node[tinycircle] (00)at (0,0) {};
		\node[mypt] (00label)at (-0.01\textwidth,-0.02\textwidth) {$(0,0)$};
		\node[mypt] (01)at (0,.1\textwidth) {};
		\node[mypt] (02)at (0,.2\textwidth) {};
		\node[tinycircle] (03)at (0,.3\textwidth) {};
		\node[mypt] (01label)at (-0.01\textwidth,0.32\textwidth) {$(0,1)$};
		\node[mypt] (10)at (.1\textwidth,0) {};
		\node[mypt] (11)at (.1\textwidth,.1\textwidth) {};
		\node[mypt] (12)at (.1\textwidth,.2\textwidth) {};
		\node[mypt] (13)at (.1\textwidth,.3\textwidth) {};
		\node[mypt] (20)at (.2\textwidth,0) {};
		\node[mypt] (21)at (.2\textwidth,.1\textwidth) {};
		\node[mypt] (22)at (.2\textwidth,.2\textwidth) {};
		\node[mypt] (23)at (.2\textwidth,.3\textwidth) {};
		\node[tinycircle] (30)at (.3\textwidth,.0) {};
		\node[mypt] (10label)at (0.31\textwidth,-0.02\textwidth) {$(1,0)$};
		\node[mypt] (31)at (.3\textwidth,.1\textwidth) {};
		\node[mypt] (32)at (.3\textwidth,.2\textwidth) {};
		\node[tinycircle] (33)at (.3\textwidth,.3\textwidth) {};
		\node[mypt] (11label)at (0.31\textwidth,0.32\textwidth) {$(1,1)$};

		\filldraw[-,black,thick] (00)--(03){};
		\filldraw[-,black,thick] (00)--(30){};
		\filldraw[-,black,thick] (33)--(03){};
		\filldraw[-,black,thick] (33)--(30){};
		\filldraw[-,black] (10)--(13){};
		\filldraw[-,black] (01)--(31){};
		\filldraw[-,black] (20)--(23){};
		\filldraw[-,black] (02)--(32){};
		
		\node[tinycircle](ax) at (0.05\textwidth,0.05\textwidth){};
		\node[mypt](ax) at (0.066\textwidth,0.045\textwidth){$a_x$};
		\node[tinycircle](ay) at (0.15\textwidth,0.25\textwidth){};
		\node[mypt](ay) at (0.166\textwidth,0.256\textwidth){$a_y$};
		\node[tinyr](y) at (0.18\textwidth,0.22\textwidth){};
		\node[mypt](ylabel) at (0.19\textwidth,0.224\textwidth){$\hz{y}$};
		\node[tinyr](x) at (0.033\textwidth,0.07\textwidth){};
		\node[mypt](xlabel) at (0.023\textwidth,0.065\textwidth){$\hz{x}$};
		\filldraw[-,blue] (x)--(y){};
		\node[tinyr](xp) at (0.0624\textwidth,0.1\textwidth){};
		\node[mypt](xplabel) at (0.0604\textwidth,0.112\textwidth){$\hz{x'}$};
		\node[tinyr](yp) at (0.1604\textwidth,0.2\textwidth){};
		\node[mypt](yplabel) at (0.157\textwidth,0.212\textwidth){$\hz{y'}$};
		\eetikz
		\caption{Illustration of the points $x$, $y$, $a_x$, $a_y$, $x'$, $y'$ when $Q=3$ and $d=2$.}
		\label{fig7}
	\end{figure}
	
	\noindent Obviously, we have that    \[
	\hkh{Q\cdot(x-a_x),Q\cdot(x'-a_x),Q\cdot(y-a_y),Q\cdot(y'-a_y)}\subset [-\frac{1}{2},\frac{1}{2}]^d. 
	\] By (\ref{22082301}), we have that 
	\begin{align*}
		&u_{\bm m}(Q\cdot (x-z))\cdot (1-\idf_{\hkh{a_x}}(z))=u_{\bm m}(Q\cdot (x'-z))\\&=u_{\bm m}(Q\cdot (y'-z))=u_{\bm m}(Q\cdot (y-z))\cdot(1-\idf_{\hkh{a_y}}(z))=0,\;\forall\;z\in G_{Q,d}.
	\end{align*} Consequently, 
	\begin{align*}
		&\abs{\mr{D}^{\bm m}f(x)-\mr{D}^{\bm m}f(y)}\leq \abs{\mr{D}^{\bm m}f(x)}+\abs{\mr{D}^{\bm m}f(y)}\\
		&=\abs{T(a_x)\cdot \frac{\mr c_1}{Q^{\beta-{\norm{\bm m}_1}}}\cdot u_{\bm m}(Q\cdot(x-a_x))}+\abs{T(a_y)\cdot \frac{\mr c_1}{Q^{\beta-{\norm{\bm m}_1}}}\cdot u_{\bm m}(Q\cdot(y-a_y))}\\
		&= \frac{\mr c_1}{Q^{\lambda}}\cdot \abs{u_{\bm m}(Q\cdot(x-a_x))}+\frac{\mr c_1}{Q^{\lambda}}\cdot \abs{u_{\bm m}(Q\cdot(y-a_y))}\\
		&=\frac{\mr c_1}{Q^{\lambda}}\cdot \abs{u_{\bm m}(Q\cdot(x-a_x))-u_{\bm m}(Q\cdot(x'-a_x))}\\&\;\;\;\;\;\;\;\;+\frac{\mr c_1}{Q^{\lambda}}\cdot \abs{u_{\bm m}(Q\cdot(y-a_y))-u_{\bm m}(Q\cdot(y'-a_y))}\\
		&\leq\frac{\mr c_1}{Q^\lambda}\cdot{\norm{Q\cdot (x-a_x)-Q\cdot (x'-a_x)}_2^\lambda}\cdot\sup_{ z,z'\in[-\frac{1}{2},\frac{1}{2}]^d,z\neq z',}\abs{\frac{u_{\bm m}(z)-u_{\bm m}(z')}{\norm{z-z'}_2^\lambda}}\\
		&\;\;\;\;\;\;\;\;+\frac{\mr c_1}{Q^\lambda}\cdot{\norm{Q\cdot (y-a_y)-Q\cdot (y'-a_y)}_2^\lambda}\cdot\sup_{ z,z'\in[-\frac{1}{2},\frac{1}{2}]^d,z\neq z',}\abs{\frac{u_{\bm m}(z)-u_{\bm m}(z')}{\norm{z-z'}_2^\lambda}}\\
		&\leq\frac{\mr c_1}{Q^\lambda}\cdot\abs{\norm{Q\cdot (x-a_x)-Q\cdot (x'-a_x)}_2^\lambda+\norm{Q\cdot (y-a_y)-Q\cdot (y'-a_y)}_2^\lambda}\cdot \mr c_2\\
		&=\mr c_1 \mr c_2\cdot\abs{\norm{x-x'}_2^\lambda+\norm{y-y'}_2^\lambda}\leq 2\mr c_1 \mr c_2\cdot\norm{x-y}_2^\lambda.
	\end{align*} Therefore, no matter whether $a_x=a_y$ or not, we always have that \[\abs{\mr{D}^{\bm m}f(x)-\mr{D}^{\bm m}f(y)}\leq 2\mr c_1 \mr c_2\cdot\norm{x-y}_2^\lambda.\] Since $\bm m$, $x$, $y$ are arbitrary, we deduce that \[\abs{\mr{D}^{\bm m}f(x)-\mr{D}^{\bm m}f(y)}\leq 2\mr c_1 \mr c_2\cdot\norm{x-y}_2^\lambda\] for any $\bm m\in (\mb N\cup\hkh{0})^d\text{ with }\norm{\bm m}_1=b$ and any $x,y\in \bigcup_{a\in G_{Q,d}}\setr{a+v}{v\in(-\frac{1}{2Q},\frac{1}{2Q})^d}$. Note that $\bigcup_{a\in G_{Q,d}}\setr{a+v}{v\in(-\frac{1}{2Q},\frac{1}{2Q})^d}$ is dense in $[0,1]^d$. Hence, by taking limit, we obtain
	\beq\label{22082303}
	&\abs{\mr{D}^{\bm m}f(x)-\mr{D}^{\bm m}f(y)}\\&\leq 2\mr c_1 \mr c_2\cdot\norm{x-y}_2^\lambda,\;\forall\;\bm m\in (\mb N\cup\hkh{0})^d\text{ with }\norm{\bm m}_1=b,\;\forall\;x,y\in[0,1]^d. 
	\eeq Combining (\ref{22082302}) and (\ref{22082303}), we conclude that $\norm{f}_{\mc{C}^{b,\lambda}([0,1]^d)}\leq \mr c_1\mr c_2+2\mr c_1\mr c_2<r$. Thus $f\in \mc{B}^{\beta}_r\ykh{[0,1]^d}$. Then the proof of this lemma is completed. 
\end{proof}	

Let $\mathscr{P}$ and $\mathscr Q$ be two arbitrary probability measures which have the same domain. We write $\mathscr P<<\mathscr Q$ if $\mathscr P$ is  absolutely continuous with respect to $\mathscr Q$. The \emph{Kullback-Leibler divergence} (\emph{KL divergence}) from $\mathscr Q$ to $\mathscr P$ is given by
\[
\mr{KL}(\mathscr P||\mathscr Q):=\begin{cases}\int \log\ykh{\frac{\mr{d}\mathscr P}{\mr{d}\mathscr Q}}\mr{d}\mathscr P, &\text{if $\mathscr P<<\mathscr Q$},\\
	+\infty,&\text{otherwise},\end{cases}
\] where $\frac{\mr{d}\mathscr P}{\mr{d}\mathscr Q}$ is the \emph{Radon-Nikodym derivative} of $\mathscr P$ with respect to $\mathscr Q$ (cf.  Definition 2.5 of  \cite{tsybakov2009introduction}).

\begin{lem}\label{lem082501}Suppose $\eta_1:[0,1]^d\to[0,1]$ and $\eta_2:[0,1]^d\to(0,1)$ are two Borel measurable functions, and $\mathscr Q$ is a Borel probability measure on $[0,1]^d$. Then $P_{\eta_1,\mathscr Q}<<P_{\eta_2,\mathscr Q}$, and \[\frac{\mr{d}P_{\eta_1,\mathscr Q}}{\mr{d}P_{\eta_2,\mathscr Q}}(x,y)=\begin{cases}\frac{\eta_1(x)}{\eta_2(x)},&\;\text{if }y=1,\\
		\frac{1-\eta_1(x)}{1-\eta_2(x)},&\;\text{if }y=-1.\end{cases}\]\end{lem}

\begin{proof}
	Let $f:[0,1]^d\times\hkh{-1,1}\to[0,\infty),\;(x,y)\mapsto\begin{cases}\frac{\eta_1(x)}{\eta_2(x)},&\;\text{if }y=1,\\
		\frac{1-\eta_1(x)}{1-\eta_2(x)},&\;\text{if }y=-1.\end{cases}$  Then we have that $f$ is well defined and measurable. 
	For any Borel subset $S$ of $[0,1]^d\times\hkh{-1,1}$, let $S_1:=\setl{x\in[0,1]^d}{(x,1)\in S}$, and $S_2:=\setl{x\in[0,1]^d}{(x,-1)\in S}$.  Obvioulsy, $S_1\times\hkh{1}$ and $S_2\times\hkh{-1}$ are measurable and disjoint. Besides, it is easy to verify that $S=(S_1\times\hkh{1})\cup(S_2\times\hkh{-1})$. Therefore, 
	\begin{align*}
		&\int_{S}f(x,y)\mr{d}P_{\eta_2,\mathscr Q}(x,y)\\&=\int_{S_1}\int_{\hkh{1}}f(x,y)\mr{d}\mathscr M_{\eta_2(x)}(y)\mr{d}{\mathscr Q}(x)+\int_{S_2}\int_{\hkh{-1}}f(x,y)\mr{d}\mathscr M_{\eta_2(x)}(y)\mr{d}{\mathscr Q}(x)\\
		&=\int_{S_1}\eta_2(x)f(x,1)\mr{d}{\mathscr Q}(x)+\int_{S_2}(1-\eta_2(x))f(x,-1)\mr{d}{\mathscr Q}(x)\\&=\int_{S_1}\eta_1(x)\mr{d}{\mathscr Q}(x)+\int_{S_2}(1-\eta_1(x))\mr{d}{\mathscr Q}(x)\\
		&=\int_{S_1}\int_{\hkh{1}}\mr{d}\mathscr M_{\eta_1(x)}(y)\mr{d}{\mathscr Q}(x)+\int_{S_2}\int_{\hkh{-1}}\mr{d}\mathscr M_{\eta_1(x)}(y)\mr{d}{\mathscr Q}(x)\\&=P_{\eta_1,\mathscr Q}(S_1\times\hkh{1})+P_{\eta_1,\mathscr Q}(S_2\times\hkh{-1})=P_{\eta_1,\mathscr Q}(S). 
	\end{align*}Since $S$ is arbitrary, we deduce that $P_{\eta_1,\mathscr Q}<<P_{\eta_2,\mathscr Q}$, and $\frac{\mr{d}P_{\eta_1,\mathscr Q}}{\mr{d}P_{\eta_2,\mathscr Q}}=f$. This completes the proof. \end{proof}	

\begin{lem}\label{lem082904}Let $\e\in(0,\frac{1}{5}]$, $\mathscr Q$ be a Borel probability on $[0,1]^d$,  and $\eta_1:[0,1]^d\to [\e,3\e]$, $\eta_2:[0,1]^d\to[\e,3\e]$ be two measurable functions. Then 
	\[
	\mr{KL}(P_{\eta_1,\mathscr Q}||P_{\eta_2,\mathscr Q})\leq  9\e.
	\]
	
\end{lem}
\begin{proof}By Lemma \ref{lem082501}, 
	\begin{align*}
		&\mr{KL}(P_{\eta_1,\mathscr Q}||P_{\eta_2,\mathscr Q})\\&=\int_{[0,1]^d\times\hkh{-1,1}}\log\ykh{\frac{\eta_1(x)}{\eta_2(x)}\cdot\idf_{\hkh{1}}(y)+\frac{1-\eta_1(x)}{1-\eta_2(x)}\cdot\idf_{\hkh{-1}}(y)}\mr{d}P_{\eta_1,\mathscr Q}(x,y)\\
		&=\int_{[0,1]^d}\ykh{\eta_1(x)\log\ykh{\frac{\eta_1(x)}{\eta_2(x)}}+(1-\eta_1(x))\log\ykh{\frac{1-\eta_1(x)}{1-\eta_2(x)}}}\mr{d}\mathscr Q(x)\\
		&\leq\int_{[0,1]^d}\ykh{3\e\cdot\abs{\log\ykh{\frac{\eta_1(x)}{\eta_2(x)}}}+\abs{\log\ykh{\frac{1-\eta_1(x)}{1-\eta_2(x)}}}}\mr{d}\mathscr Q(x)\\
		&\leq\int_{[0,1]^d}\ykh{3\e\cdot{\log\ykh{\frac{3\e}{\e}}}+{\log\ykh{\frac{1-\e}{1-3\e}}}}\mr{d}\mathscr Q(x)\\&=\log\ykh{1+\frac{2\e}{1-3\e}}+3\e\cdot\log 3\leq \frac{2\e}{1-3\e}+4\e\leq 9\e. 
\end{align*}\end{proof}

\begin{lem}\label{vgbound}Let $m\in\mb N\cap(1,\infty)$,  $\Omega$ be a set with $\#(\Omega)=m$, and $\hkh{0,1}^\Omega$ be the set of all functions mapping from $\Omega$ to $\hkh{0,1}$.    Then there exists a subset $E$ of $\hkh{0,1}^\Omega$,  such that $\#(E)\geq 1+2^{m/8}$, and 
	\[
	\#\ykh{\setl{x\in \Omega}{f(x)\neq g(x)}}\geq \frac{m}{8},\;\forall\;f\in E,\;\forall\;g\in E\setminus\hkh{f}. 
	\]\end{lem}
\begin{proof}
	If $m\leq 8$, then  $E=\hkh{0,1}^{\Omega}$ have the desired properties. The proof for the	case $m>8$ can be found in Lemma 2.9 of \cite{tsybakov2009introduction}.\end{proof}

\begin{lem}\label{thm082803}Let  $\phi$ be the logistic loss, 
	\beq\label{082601}
	\mc J:(0,1)^2&\to\mbR\\
	(x,y)&\mapsto (x+y)\log\frac{2}{x+y}+(2-x-y)\log\frac{2}{2-x-y}\\&\;\;\;\;\;\;\;\;\;\;\;\;-\ykh{x\log\frac{1}{x}+(1-x)\log\frac{1}{1-x}+y\log\frac{1}{y}+(1-y)\log\frac{1}{1-y}},
	\eeq $\mathscr Q$ be a Borel probability measure on $[0,1]^d$,  and  $\eta_1:[0,1]^d\to(0,1)$,  $\eta_2:[0,1]^d\to(0,1)$ be two   measurable functions. Then there hold
	\beq\label{082701}
	\mc{J}(x,y)=\mc{J}(y,x)\geq 0,\;\forall\;x\in(0,1),\;y\in(0,1),
	\eeq
	\beq\label{082602}
	\frac{\e}{4}<\mc J(\e,3\e)=\mc J(3\e,\e)<\e,\;\forall\;\e\in(0,\frac{1}{6}], 
	\eeq and
	\beq\label{082603}
	\int_{[0,1]^d}\mc J(\eta_1(x),\eta_2(x))\mr{d}\mathscr Q(x)\leq\inf_{f\in\mc F_d}\abs{\mc{E}^\phi_{P_{\eta_1,\mathscr Q}}(f)+\mc{E}^\phi_{P_{\eta_2,\mathscr Q}}(f)}. 
	\eeq
\end{lem}
\begin{proof}Let $g:(0,1)\to(0,\infty),x\mapsto x\log\frac{1}{x}+(1-x)\log\frac{1}{1-x}$. Then it is easy to verify that $g$ is concave (i.e., $-g$ is convex), and 
	\[
	\mc{J}(x,y)=2g(\frac{x+y}{2})-g(x)-g(y),\;\forall\;x\in(0,1),\;y\in(0,1).
	\]This yields (\ref{082701}).

	An elementary calculation gives
	\begin{align*}
		&\mc{J}(\e,3\e)=\mc J(3\e,\e)\\&=\e\log{\frac{27}{16}}-\log\ykh{\frac{(1-2\e)^2}{(1-\e)(1-3\e)}}+4\e\log(1-2\e)-\e\log(1-\e)-3\e\log(1-3\e)\\
		&\xlongequal{\text{Taylor expansion}}\e\log{\frac{27}{16}}+\sum_{k=2}^\infty\frac{3^k+1-2\cdot2^k}{k\cdot(k-1)}\cdot \e^k,\;\forall\;\e\in(0,1/3). %
	\end{align*} Therefore, 
	\begin{align*}
		&\frac{\e}{4}<\e\log\frac{27}{16}\leq \e\log{\frac{27}{16}}+\sum_{k=2}^\infty\frac{1+\ykh{\ykh{\frac{3}{2}}^k-2}\cdot2^k}{k\cdot(k-1)}\cdot \e^k=\e\log{\frac{27}{16}}+\sum_{k=2}^\infty\frac{3^k+1-2\cdot2^k}{k\cdot(k-1)}\cdot \e^k \\
		&=\mc J(\e,3\e)=\mc{J}(3\e,\e)=\e\log{\frac{27}{16}}+\sum_{k=2}^\infty\frac{3^k+1-2\cdot2^k}{k\cdot(k-1)}\cdot \e^k\leq \e\log{\frac{27}{16}}+\sum_{k=2}^\infty\frac{3^k-7}{k\cdot(k-1)}\cdot \e^k\\
		&=\e\log{\frac{27}{16}}+\e^2+\e\cdot\sum_{k=3}^\infty\frac{3^k-7}{k\cdot(k-1)}\cdot \e^{k-1}\leq \e\log{\frac{27}{16}}+\e/6+\e\cdot\sum_{k=3}^\infty\frac{3^k}{3\cdot(3-1)}\cdot \ykh{\frac{1}{6}}^{k-1}\\
		&=\ykh{\frac{1}{6}+\frac{1}{4}+\log\frac{27}{16}}\cdot\e<\e,\;\forall\;\e\in(0,1/6],
	\end{align*} which proves (\ref{082602}). 
	
	Define $f_1:[0,1]^d\to\mbR, x\mapsto \log\frac{\eta_1(x)}{1-\eta_1(x)}$ and $f_2:[0,1]^d\to\mbR, x\mapsto \log\frac{\eta_2(x)}{1-\eta_2(x)}$.  Then it is easy to verify that
	\[
	\mc{R}^\phi_{P_{\eta_i,\mathscr Q}}(f_i)=\int_{[0,1]^d}g(\eta_i(x))\mr{d}\mathscr Q(x)\in (0,\infty),\;\forall\;i\in\hkh{1,2}, \] and
	\[
	\inf\hkh{a\phi(t)+(1-a)\phi(-t)\big|{t\in\mbR}}=g(a),\;\forall\;a\in(0,1).  
	\]Consequently, for any measurable function $f:[0,1]^d\to\mbR$, there holds
	\begin{align*}
		&\mc{E}^\phi_{P_{\eta_1,\mathscr Q}}(f)+\mc{E}^\phi_{P_{\eta_2,\mathscr Q}}(f)\geq{\mc{R}^\phi_{P_{\eta_1,\mathscr Q}}(f)-\mc{R}^\phi_{P_{\eta_1,\mathscr Q}}(f_1)+\mc{R}^\phi_{P_{\eta_2,\mathscr Q}}(f)}-\mc{R}^\phi_{P_{\eta_2,\mathscr Q}}(f_2)\\
		&=\int_{[0,1]^d}\ykh{(\eta_1(x)+\eta_2(x))\phi(f(x))+(2-\eta_1(x)-\eta_2(x))\phi(-f(x)}\mr{d}\mathscr Q(x)\\&\;\;\;\;\;\;\;\;\;\;\;\;\;\;\;\;\;\;\;\;-\mc{R}^\phi_{P_{\eta_1,\mathscr Q}}(f_1)-\mc{R}^\phi_{P_{\eta_2,\mathscr Q}}(f_2)\\
		&\geq\int_{[0,1]^d}2\cdot\inf\setl{\frac{\eta_1(x)+\eta_2(x)}{2}\phi(t)+(1-\frac{\eta_1(x)+\eta_2(x)}{2})\phi(-t)}{t\in\mbR}\mr{d}\mathscr Q(x)\\&\;\;\;\;\;\;\;\;\;\;\;\;\;\;\;\;\;\;\;\;-\mc{R}^\phi_{P_{\eta_1,\mathscr Q}}(f_1)-\mc{R}^\phi_{P_{\eta_2,\mathscr Q}}(f_2)\\
		&=\int_{[0,1]^d}2g(\frac{\eta_1(x)+\eta_2(x)}{2})\mr{d}\mathscr Q(x)-\mc{R}^\phi_{P_{\eta_1,\mathscr Q}}(f_1)-\mc{R}^\phi_{P_{\eta_2,\mathscr Q}}(f_2)\\
		&=\int_{[0,1]^d}\ykh{2g(\frac{\eta_1(x)+\eta_2(x)}{2})-g(\eta_1(x))-g(\eta_2(x))}\mr{d}\mathscr Q(x)\\&=\int_{[0,1]^d}\mc{J}(\eta_1(x),\eta_2(x))\mr{d}\mathscr Q(x). 
	\end{align*}This proves (\ref{082603}). \end{proof}

\mybookmark{current}{2023?5?26? 15:50:16}
\begin{proof}[Proof of Theorem \ref{thm2.6p} and  Corollary  \ref{thm2.6}] We first prove Theorem \ref{thm2.6p}. Let $n$ be an arbitrary integer greater than $\abs{\frac{7}{1-A}}^{\frac{d_*+\beta\cdot(1\qx\beta)^q}{\beta\cdot(1\qx\beta)^q}}$.  Take $b:=\ceil{\beta}-1$, $\lambda:=\beta+1-\ceil{\beta}$,  $Q:=\flr{n^{\frac{1}{d_*+\beta\cdot(1\qx\beta)^q}}}+1$, $\texttt M:=\ceil{2^{{Q^{d_*}}/{8}}}$,  \[G_{Q,d_*}:=\setl{(\frac{k_1}{2Q},\ldots,\frac{k_{d_*}}{2Q})^\top}{k_1,\ldots,k_{d_*}\text{ are odd integers}}\cap[0,1]^{d_*}, \] and $\mc J$ to be the function defined in (\ref{082601}). Note that $\#\ykh{G_{Q,d_*}}=Q^{d_*}$. Thus it follows from Lemma \ref{vgbound} that there exist functions $T_j:G_{Q,d_*}\to \hkh{-1,1}$, $j=0,1,2,\ldots,\texttt M$, such that
	\beq\label{082804}
	\#\ykh{\setl{a\in G_{Q,d_*}}{T_i(a)\neq T_j(a)}}\geq \frac{Q^{d_*}}{8},\;\forall\;0\leq i<j\leq \texttt M. 
	\eeq According to Lemma \ref{xxlem3.1}, for each $j\in\hkh{0,1,\ldots,\texttt M}$, there exists an $f_j\in \mc{B}^{\beta}_{\frac{r\qx 1}{777}}\ykh{[0,1]^{d_*}}$, such that \beq\label{23052601}
	\frac{\mr c_1}{Q^\beta}=\norm{f_j}_{[0,1]^{d_*}}\leq \norm{f_j}_{\mc C^{b,\lambda}([0,1]^{d_*})}\leq \frac{1\qx r}{777},\eeq and 
	\beq\label{082801}
	f_j(x)=\frac{\mr c_1}{Q^\beta}\cdot T_j(a),\;\forall\; a\in G_{Q,d_*},\;x\in \mathscr{B}(a,{\frac{1}{5Q}})\cap[0,1]^{d_*},
	\eeq where $\mr c_1\in(0,\frac{1}{9999})$ only depends on $(d_*,\beta,r)$. Define \[g_j:[0,1]^{d_*}\to\mbR,x\mapsto\frac{\mr c_1}{Q^\beta}+f_j(x).  
	\] It follows from \eqref{23052601} that \beq\label{23052701}
	\mathbf{ran}(g_j)\subset\zkh{0,\frac{2\mr c_1}{Q^\beta}}\subset\zkh{0,2\cdot\frac{1\qx r}{777}}\subset[0,1]\eeq and \beq\label{23052702}\frac{\mr c_1}{Q^\beta}+\norm{g_j}_{\mc C^{b,\lambda}([0,1]^{d_*})}\leq \frac{2\mr c_1}{Q^\beta}+\norm{f_j}_{\mc C^{b,\lambda}([0,1]^{d_*})}\leq2\cdot\frac{1\qx r}{777}+\frac{1\qx r}{777}<r,\eeq meaning that \beq\label{230527020} g_j\in \mc{B}^{\beta}_{{r}}\ykh{[0,1]^{d_*}}\text{ and }g_j+\frac{\mr c_1}{Q^\beta}\in \mc{B}^{\beta}_{{r}}\ykh{[0,1]^{d_*}}.\eeq We then define 
	\[
	h_{0,j}:[0,1]^d\to [0,1],\; (x_1,\ldots,x_d)^\top\mapsto g_j(x_1,\ldots,x_{d_*})
	\] if $q=0$, and define 	\[
	h_{0,j}:[0,1]^d\to [0,1]^K,\; (x_1,\ldots,x_d)^\top\mapsto (g_j(x_1,\ldots,x_{d_*}),0,0,\ldots,0)^\top
	\] if $q>0$. Note that $h_{0,j}$ is well defined because $d_*\leq d$ and $\mathbf{ran}(g_j)\subset[0,1]$.  Take \[\e=\frac{1}{2}\cdot\abs{\frac{1\qx r}{777}}^{\sum_{k=0}^{q-1}(1\qx\beta)^k}\cdot\abs{\frac{2\mr c_1}{Q^\beta}}^{(1\qx\beta)^q}.\] From \eqref{23052601} we see that   \beq\label{23052703}
	0<\e\leq \frac{1\qx r}{777}. 
	\eeq  For all real number $t$, define the function
	\[
	u_t:[0,1]^{d_*}\to\mbR,\;(x_1,\ldots,x_{d_*})^\top\mapsto t+ \frac{1\qx r}{777}\cdot\abs{x_1}^{(1\qx\beta)}. 
	\]Then it follows from \eqref{23052703} and  the elementary inequality
	\[
	\abs{\abs{z_1}^{w}-\abs{z_2}^{w}}\leq\abs{z_1-z_2}^{w},\;\forall\;z_1\in\mbR,z_2\in\mbR, w\in(0,1]
	\]that 
	\beq
	&\max\hkh{\norm{u_\e}_{[0,1]^{d_*}},\norm{u_0}_{[0,1]^{d_*}}}\leq \max\hkh{\norm{u_\e}_{\mc C^{b,\lambda}([0,1]^{d_*})},\norm{u_0}_{\mc C^{b,\lambda}([0,1]^{d_*})}}\\&\leq \norm{u_0}_{\mc C^{b,\lambda}([0,1]^{d_*})}+\e\leq \frac{1\qx r}{777}\cdot2+\e
	\leq \frac{1\qx r}{777}\cdot2+\frac{1\qx r}{777}<r\qx 1,
	\eeq which means that
	\beq\label{23052705}
	\mathbf{ran}(u_0)\cup\mathbf{ran}(u_\e)\subset[0,1],
	\eeq and 
	\beq\label{23052704}
	\hkh{u_0,u_\e}\subset  \mc{B}^{\beta}_{{r}}\ykh{[0,1]^{d_*}}. 
	\eeq Next, for each $i\in\mb N$, define 
	\begin{align*}
		h_i:[0,1]^K&\to\mbR,\\(x_1,\ldots,x_K)^\top&\mapsto u_0(x_1,\ldots,x_{d_*})
	\end{align*} if $i=q>0$, and define
	\[
	h_i:[0,1]^K\to\mbR^K,\;(x_1,\ldots,x_K)^\top\mapsto \ykh{u_0(x_1,\ldots,x_{d_*}),0,0,\ldots,0}^\top
	\] otherwise.  It follows from \eqref{23052705} that $\mathbf{ran}(h_i)\subset [0,1]$ if $i=q>0$, and $\mathbf{ran}(h_i)\subset [0,1]^K$ otherwise. Thus, for each $j\in\hkh{0,1,\ldots,\texttt M}$, we can well define
	\[
	\eta_j:[0,1]^d\to\mbR,\;x\mapsto \e+h_q\circ h_{q-1}\circ\cdots\circ h_3\circ h_2\circ h_1\circ h_{0,j}(x). 
	\] We then deduce from \eqref{230527020} and \eqref{23052704} that 
	\beq\label{23052708}
	\eta_j\in \mc G_d^{\mathbf{CH}}(q, K, d_*, \beta,r),\;\forall\;j\in\hkh{0,1,\ldots,\texttt M}. 
	\eeq Moreover, an elementary calculation gives
	\beq\label{23052706}
	&\abs{\frac{1\qx r}{777}}^{\sum_{k=0}^{q-1}(1\qx\beta)^k}\cdot \abs{g_j(x_1,\ldots,x_{d_*})}^{(1\qx\beta)^q}+\e\\
	&=\eta_j(x_1,\ldots,x_d),\;\forall\;(x_1,\ldots,x_d)\in[0,1]^d,\;\forall\;j\in\hkh{0,1,\ldots,\texttt M},
	\eeq which, together with \eqref{23052701}, yields
	\begin{align*}
		&0<\e\leq \eta_j(x_1,\ldots,x_d)\leq  \abs{\frac{1\qx r}{777}}^{\sum_{k=0}^{q-1}(1\qx\beta)^k}\cdot \abs{\frac{2\mr c_1}{Q^\beta}}^{(1\qx\beta)^q}+\e=2\e+\e\\
		&=3\e\leq \abs{\frac{3\mr c_1}{Q^\beta}}^{(1\qx\beta)^q}< \frac{1}{Q^{\beta\cdot(1\qx\beta)^q}}\leq \frac{1}{n^{\frac{\beta\cdot(1\qx\beta)^q}{d_*+\beta\cdot(1\qx\beta)^q}}}\leq \frac{1-A}{7}<\frac{1-A}{2}\\&<1,\;\forall\;(x_1,\ldots,x_d)\in[0,1]^d,\;\forall\;j\in\hkh{0,1,\ldots,\texttt M}.
	\end{align*} Consequently, \beq\label{23052602}
	\mathbf{ran}(\eta_j) \subset[\e,3\e]\subset(0,1), \;\forall\;j\in\hkh{0,1,\ldots,\texttt M},\eeq    and \beq\label{23052707}\setl{x\in[0,1]^d}{\abs{2\eta_j(x)-1}\leq A}=\varnothing,\;\forall\;j\in\hkh{0,1,\ldots,\texttt M}.\eeq  Combining \eqref{23052708}, \eqref{23052602}, and \eqref{23052707}, we obtain 
	\beq\label{082802}
	P_j:=P_{\eta_j}\in\mc{H}^{d,\beta,r}_{5,A,q,K,d_*},\;\forall\;j\in\hkh{0,1,2,\ldots,\texttt M}. 
	\eeq By (\ref{082801}) and \eqref{23052706},  for any $0\leq i<j\leq \texttt M$, any $a\in G_{Q,d_*}$ with $T_i(a)\neq T_j(a)$, and any $x\in[0,1]^d$ with $(x)_{\hkh{1,2,\ldots,d_*}}\in\mathscr B(a,\frac{1}{5Q})$, there holds
	\begin{align*}
		&\mc J(\eta_i(x),\eta_j(x))\\&=\mc J\left(\abs{\frac{1\qx r}{777}}^{\sum_{k=0}^{q-1}(1\qx\beta)^k}\cdot \abs{\frac{\mr c_1}{Q^\beta}+T_i(a)\cdot\frac{\mr c_1}{Q^\beta}}^{(1\qx\beta)^q}+\e,\right.\\&\;\;\;\;\;\;\;\;\;\;\;\;\;\;\;\;\;\;\;\;\;\;\;\;\left.\abs{\frac{1\qx r}{777}}^{\sum_{k=0}^{q-1}(1\qx\beta)^k}\cdot \abs{\frac{\mr c_1}{Q^\beta}+T_j(a)\cdot\frac{\mr c_1}{Q^\beta}}^{(1\qx\beta)^q}+\e\right)\\
		&=\mc J\ykh{\abs{\frac{1\qx r}{777}}^{\sum_{k=0}^{q-1}(1\qx\beta)^k}\cdot \abs{\frac{2\mr c_1}{Q^\beta}}^{(1\qx\beta)^q}+\e,\abs{\frac{1\qx r}{777}}^{\sum_{k=0}^{q-1}(1\qx\beta)^k}\cdot \abs{0}^{(1\qx\beta)^q}+\e}\\&=\mc{J}(2\e+\e,\e)=\mc J(\e,3\e).
	\end{align*}Thus it follows from Lemma \ref{thm082803} and (\ref{082804}) that
	\beq\label{082901}
	&\inf_{f\in\mc F_d}\ykh{\mc E^\phi_{P_j}(f)+\mc E^\phi_{P_i}(f)}\geq \int_{[0,1]^d}\mc{J}(\eta_i(x),\eta_j(x))\mr{d}x\\
	&\geq \sum_{a\in G_{Q,d_*}:\;T_j(a)\neq T_i(a)}\int_{[0,1]^d}\mc{J}(\eta_i(x),\eta_j(x))\cdot\idf_{\mathscr B(a,\frac{1}{5Q})}\big((x)_{\hkh{1,\ldots,d_*}}\big)\mr{d}x\\
	&= \sum_{a\in G_{Q,d_*}:\;T_j(a)\neq T_i(a)}\int_{[0,1]^d}\mc{J}(\e,3\e)\cdot\idf_{ \mathscr B(a,\frac{1}{5Q})}\big((x)_{\hkh{1,\ldots,d_*}}\big)\mr{d}x\\
	&\geq \sum_{a\in G_{Q,d_*}:\;T_j(a)\neq T_i(a)}\int_{[0,1]^d}\frac{\e}{4}\cdot\idf_{ \mathscr B(a,\frac{1}{5Q})}\big((x)_{\hkh{1,\ldots,d_*}}\big)\mr{d}x\\
	&=\frac{\#\ykh{\hkh{a\in G_{Q,d_*}\big|T_j(a)\neq T_i(a)}}}{ Q^{d_*}}\cdot\int_{\mathscr B(\bm 0,\frac{1}{5})}\frac{\e}{4}\mr{d}x_1\mr{d}x_2\cdots\mr{d}x_{d_*}\\&\geq \frac{1}{8}\cdot\int_{\mathscr B(\bm 0,\frac{1}{5})}\frac{\e}{4}\mr{d}x_1\mr{d}x_2\cdots\mr dx_{d_*}
	\geq \frac{1}{8}\cdot\int_{[-\frac{1}{\sqrt{25d_*}},\frac{1}{\sqrt{25d_*}}]^{d_*}}\frac{\e}{4}\mr{d}x_1\mr dx_2\cdots\mr d_{x_{d_*}}\\&\geq \abs{\frac{2}{\sqrt{25d_*}}}^{d_*}\cdot\frac{\e}{32}=:s,\;\forall\;0\leq i<j\leq \texttt M. 
	\eeq
	
	Let $\hat{f}_n$ be an arbitrary $\mc F_d$-valued statistic on $([0,1]^d\times\hkh{-1,1})^n$ from the sample $\hkh{(X_i,Y_i)}_{i=1}^n$, and let $\mc T:([0,1]^d\times\hkh{-1,1})^n\to\mc F_d$ be the map associated with $\hat f_n$, i.e., $\hat f_n=\mc T(X_1,Y_1,\ldots, X_n,Y_n)$. Take
	\[
	\mc T_0:\mc F_d\to\hkh{0,1,\ldots,\texttt M}, f\mapsto \inf\mathop{\arg\min}_{j\in\hkh{0,1,\ldots,\texttt M}}\mc E^\phi_{P_j}(f),
	\] i.e., $\mc T_0(f)$ is the smallest integer $j\in\hkh{0,\ldots,\texttt M}$ such that  $\mc{E}^\phi_{P_j}({f})\leq \mc{E}^\phi_{P_i}({f})$ for any $i\in\hkh{0,\ldots,\texttt M}$. Define $g_*=\mc T_0\circ\mc T$. Note that, for any $j\in\hkh{0,1,\ldots,\texttt M}$ and any $f\in\mc F_d$ there holds
	\[
	\mc T_0(f)\neq j\overset{(\ref{082901})}{\Rightarrow}\mc{E}^\phi_{P_{\mc T_0(f)}}({f})+\mc{E}^\phi_{P_{j}}({f})\geq s\Rightarrow \mc{E}^\phi_{P_{j}}({f})+\mc{E}^\phi_{P_{j}}({f})\geq s\Rightarrow \mc{E}^\phi_{P_{j}}({f})\geq\frac{s}{2},
	\] which, together with the fact that the range of $\mc T$ is contained in $\mc F_d$, yields
	\beq\label{082902}
	&\idf_{\mbR\setminus\hkh{j}}(g_*(z))=\idf_{\mbR\setminus\hkh{j}}(\mc T_0(\mc T((z))))\\&\leq\idf_{[\frac{s}{2},\infty]}(\mc E^\phi_{P_j}(\mc T(z))),\;\forall\;z\in([0,1]^d\times\hkh{-1,1})^n,\;\forall\;j\in\hkh{0,1,\ldots,\texttt M}. 
	\eeq  Consequently, 
	\beq\label{082903}
	&\sup_{P\in \mc{H}^{d,\beta,r}_{5,A,q,K,d_*}}\bm{E}_{P^{\otimes n}}\zkh{\mc{E}^\phi_P(\hat{f}_n)}\geq \sup_{j\in\hkh{0,1,\ldots,\texttt M}}\bm{E}_{P_j^{\otimes n}}\zkh{\mc{E}^\phi_{P_j}(\hat{f}_n)}\\&=\sup_{j\in\hkh{0,1,\ldots,\texttt M}}\int\mc{E}^\phi_{P_j}(\mc T(z))\mr d P_j^{\otimes n}(z)\geq \sup_{j\in\hkh{0,1,\ldots,\texttt M}}\int\frac{\idf_{[\frac{s}{2},\infty]}(\mc{E}^\phi_{P_j}(\mc T(z)))}{2/s}\mr d P_j^{\otimes n}(z)\\&\geq \sup_{j\in\hkh{0,1,\ldots,\texttt M}}\int\frac{\idf_{\mbR\setminus\hkh{j}}(g_*(z))}{2/s}\mr d P_j^{\otimes n}(z)=\sup_{j\in\hkh{0,1,\ldots,\texttt M}}\frac{{P_j^{\otimes n}}\ykh{g_*\neq j}}{2/s}\\
	&\geq \frac{s}{2}\cdot\inf\setl{ \sup_{j\in\hkh{0,1,\ldots,\texttt M}}P_j^{\otimes n}\ykh{g\neq j}}{\begin{minipage}{0.37\textwidth}
			$g$ is a measurable function from  $([0,1]^d\times\hkh{-1,1})^n$ to $\hkh{0,1,\ldots,\texttt M}$\end{minipage}},
	\eeq where the first inequality follows from (\ref{082802}) and the third inequality follows from (\ref{082902}).
	
	We then use Proposition 2.3 of \cite{tsybakov2009introduction} to bound the right hand side of (\ref{082903}). By Lemma \ref{lem082904}, we have that 
	\[
	\frac{1}{\texttt M}\cdot\sum_{j=1}^{\texttt M}\mr{KL}(P_j^{\otimes n}||P_0^{\otimes n})=\frac{n}{\texttt M}\cdot\sum_{j=1}^{\texttt M}\mr{KL}(P_j||P_0)\leq \frac{n}{\texttt M}\cdot\sum_{j=1}^{\texttt M}9\e=9n\e,
	\] which, together with Proposition 2.3 of \cite{tsybakov2009introduction}, yields
	\begin{align*}
		&\inf\setl{ \sup_{j\in\hkh{0,1,\ldots,\texttt M}}P_j^{\otimes n}\ykh{g\neq j}}{\begin{minipage}{0.37\textwidth}
				$g$ is a measurable function from  $([0,1]^d\times\hkh{-1,1})^n$ to $\hkh{0,1,\ldots,\texttt M}$\end{minipage}}\\
		&\geq \sup_{\tau\in(0,1)}\ykh{\frac{\tau\texttt M}{1+\tau \texttt M}\cdot \ykh{1+\frac{9n\e+\sqrt{\frac{9n\e}{2}}}{\log \tau}}}\geq \frac{\sqrt{\texttt M}}{1+\sqrt{\texttt M}}\cdot \ykh{1+\frac{9n\e+\sqrt{\frac{9n\e}{2}}}{\log\frac{1}{\sqrt{\texttt M}}}}\\
		&\geq \frac{\sqrt{\texttt M}}{1+\sqrt{\texttt M}}\cdot \ykh{1-\abs{\frac{9n\e+\sqrt{\frac{9n\e}{2}}}{\log\frac{1}{\sqrt{\texttt M}}}}}\geq  \frac{\sqrt{\texttt M}}{1+\sqrt{\texttt M}}\cdot \ykh{1-\abs{\frac{9n\e+\frac{1}{10}+12n\e}{\log{\sqrt{\texttt M}}}}}\\
		&\geq  \frac{\sqrt{\texttt M}}{1+\sqrt{\texttt M}}\cdot \ykh{1-\abs{\frac{21n\e}{\frac{1}{2}\log{\ykh{2^{Q^{d_*}/8}}}}}-\frac{1/10}{\log\sqrt{2}}}\\&= \frac{\sqrt{\texttt M}}{1+\sqrt{\texttt M}}\cdot \ykh{1-\abs{\frac{336n}{{Q^{d_*}}\log{2}}\cdot \frac{1}{2}\cdot\abs{\frac{1\qx r}{777}}^{\sum_{k=0}^{q-1}(1\qx\beta)^k}\cdot\abs{\frac{2\mr c_1}{Q^\beta}}^{(1\qx\beta)^q}}-\frac{1/10}{\log\sqrt{2}}}\\
		&\geq \frac{\sqrt{\texttt M}}{1+\sqrt{\texttt M}}\cdot \ykh{1-\abs{\frac{336n}{{Q^{d_*}}\log{2}}\cdot \frac{1}{2}\cdot\frac{1}{777}\cdot\abs{\frac{1}{Q^\beta}}^{(1\qx\beta)^q}}-\frac{1/10}{\log\sqrt{2}}}\\
		&\geq \frac{\sqrt{\texttt M}}{1+\sqrt{\texttt M}}\cdot \ykh{1-\abs{\frac{336}{{}\log{2}}\cdot \frac{1}{2}\cdot\frac{1}{777}}-\frac{1/10}{\log\sqrt{2}}}\geq \frac{\sqrt{\texttt M}}{1+\sqrt{\texttt M}}\cdot\frac{1}{3}\geq\frac{1}{6}. 
	\end{align*} Combining this with (\ref{082903}), we obtain that
	\begin{align*}
		&\sup_{P\in \mc{H}^{d,\beta,r}_{5,A,q,K,d_*}}\bm{E}_{P^{\otimes n}}\zkh{\mc{E}^\phi_P(\hat{f}_n)}\geq \frac{s}{2}\cdot\frac{1}{6}\\&=\abs{\frac{2}{\sqrt{25d_*}}}^{d_*}\cdot\frac{\abs{{2\mr c_1}}^{(1\qx\beta)^q}}{768}\cdot\abs{\frac{1\qx r}{777}}^{\sum_{k=0}^{q-1}(1\qx\beta)^k}\cdot\abs{\frac{1}{Q^\beta}}^{(1\qx\beta)^q}\\&\geq \abs{\frac{2}{\sqrt{25d_*}}}^{d_*}\cdot\frac{\abs{{2\mr c_1}}^{(1\qx\beta)^q}}{768}\cdot\abs{\frac{1\qx r}{777}}^{\sum_{k=0}^{q-1}(1\qx\beta)^k}\cdot\frac{1}{2^{\beta\cdot(1\qx\beta)^q}}\cdot\frac{1}{n^{\frac{\beta\cdot(1\qx\beta)^q}{d_*+\beta\cdot(1\qx\beta)^q}}}. 
	\end{align*}Since $\hat{f}_n$ is arbitrary, we deduce that
	\[
	\inf_{\hat{f}_n} \sup_{P\in \mc{H}^{d,\beta,r}_{5,A,q,K,d_*}}\bm{E}_{P^{\otimes n}}\zkh{\mc{E}^\phi_P(\hat{f}_n)}\geq \mr c_0 n^{-\frac{\beta\cdot(1\qx\beta)^q}{d_*+\beta\cdot(1\qx\beta)^q}}
	\] with $\mr c_0:=\abs{\frac{2}{\sqrt{25d_*}}}^{d_*}\cdot\frac{\abs{{2\mr c_1}}^{(1\qx\beta)^q}}{768}\cdot\abs{\frac{1\qx r}{777}}^{\sum_{k=0}^{q-1}(1\qx\beta)^k}\cdot\frac{1}{2^{\beta\cdot(1\qx\beta)^q}}$ only depending on $(d_*,\beta,r,q)$. Thus we complete the proof of Theorem \ref{thm2.6p}.

	Now it remains to prove Corollary \ref{thm2.6}.  Indeed, it follows from \eqref{23052201} that
	\[
	\mc{H}^{d,\beta,r}_{3,A}=\mc{H}^{d,\beta,r}_{5,A,0,1,d}. 
	\] Taking $q=0$, $K=1$ and $d_*=d$ in Theorem \ref{thm2.6p}, we obtain that there exists an constant $\mr c_0\in(0,\infty)$ only depending on $(d,\beta,r)$, such that \begin{align*}
		&\inf_{\hat{f}_n} \sup_{P\in \mc{H}^{d,\beta,r}_{3,A}}\bm{E}_{P^{\otimes n}}\zkh{\mc{E}^\phi_P(\hat{f}_n)}=\inf_{\hat{f}_n} \sup_{P\in \mc{H}^{d,\beta,r}_{5,A,0,1,d}}\bm{E}_{P^{\otimes n}}\zkh{\mc{E}^\phi_P(\hat{f}_n)}\geq \mr c_0 n^{-\frac{\beta\cdot(1\qx\beta)^0}{d+\beta\cdot(1\qx\beta)^0}}\\
		&=\mr c_0 n^{-\frac{\beta}{d+\beta}}\text{ provided that }n>\abs{\frac{7}{1-A}}^{\frac{d+\beta\cdot(1\qx\beta)^0}{\beta\cdot(1\qx\beta)^0}}=\abs{\frac{7}{1-A}}^{\frac{d+\beta}{\beta}}.
	\end{align*} This proves Corollary \ref{thm2.6}. 
\end{proof}

\subsection{Proof of (\ref{bound 3.5})}\label{section: proof of bound3.5}

Appendix \ref{section: proof of bound3.5}  is devoted to the proof of the bound \eqref{bound 3.5}.

\begin{proof}[Proof of (\ref{bound 3.5})]Fix $\nu\in[0,\infty)$ and $\mu\in[1,\infty)$.  Let $P$ be an arbitrary probability in $\mc H^{d,\beta}_7$. Denote by $\eta$ the conditional probability function $P(\hkh{1}|\cdot)$ of $P$.  According to Lemma \ref{2302281501} and the definition of $\mc H^{d,\beta}_7$, there exists a function $f^*\in\mc{B}^{\beta}_1\ykh{[0,1]^d}$  such that  
\beq\label{230303020}
f^*_{\phi,P}\xlongequal{\text{$P_X$-a.s.}}\log\frac{\eta}{1-\eta}\xlongequal{\text{$P_X$-a.s.}}f^*. 
\eeq Thus there exists a measurable set $\Omega$ contained in $[0,1]^d$ such that $P_X(\Omega)=1$ and 
\beq
\log\frac{\eta(x)}{1-\eta(x)}=f^*(x),\;\forall\;x\in\Omega. 
\eeq Let $\delta$ be an arbitrary number in $(0,1/3)$. Then it follows from Corollary \ref{corollaryA2} that there exists
	\beq\label{2301190138}\tilde{g}\in \fdnn_{d}\ykh{C_{d,\beta}\log\frac{1}{\delta},C_{d,\beta}\delta^{-d/\beta},C_{d,\beta}\delta^{-d/\beta}\log\frac{1}{\delta},1,\infty}
	\eeq
	such that $\sup_{x\in[0,1]^d}\abs{f^*(x)-\tilde{g}(x)}\leq\delta$. 
	Let $T:\mbR\to[-1,1], z\mapsto \min\hkh{\max\hkh{z,-1},1}$ and 
	\[\tilde{f}:\mbR\to[-1, 1],\;x\mapsto T(\tilde g(x))=\left\{
	\ba
	&-1, &&\textrm{if }\tilde{g}(x)<-1,\\
	&\tilde{g}(x), &&\textrm{if }-1\leq\tilde{g}(x)\leq 1,\\
	&1, &&\textrm{if }\tilde{g}(x)>1.\\
	\ea\right.\]
	Obviously, $\abs{T(z)-T(w)}\leq \abs{z-w}$ for any real numbers $z$ and $w$, and 
	\beq\label{230119163}
	&\sup_{x\in[0,1]^d}\abs{f^*(x)-\tilde{f}(x)}\xlongequal{\because \norm{f^*}_{[0,1]^d}\leq 1}	\sup_{x\in[0,1]^d}\abs{T(f^*(x))-T(\tilde{g}(x))}\\&\leq \sup_{x\in[0,1]^d}\abs{f^*(x)-\tilde{g}(x)}\leq\delta. 
	\eeq Besides, it is easy to verify that
\begin{align*}
&\tilde f(x)=\sigma(\tilde g(x)+1)-\sigma(\tilde g(x)-1)-1,\;\forall\;x\in\mbR^d, 
\end{align*} which, together with   \eqref{2301190138}, yields \[\ba\tilde{f}&\in \fdnn_{d}\ykh{1+C_{d,\beta}\log\frac{1}{\delta},1+C_{d,\beta}\delta^{-d/\beta},4+C_{d,\beta}\delta^{-d/\beta}\log\frac{1}{\delta},1,1}\\
&\subset \fdnn_{d}\ykh{C_{d,\beta}\log\frac{1}{\delta},C_{d,\beta}\delta^{-d/\beta},C_{d,\beta}\delta^{-d/\beta}\log\frac{1}{\delta},1,1}.
	\ea\]In addition, it follows from Lemma \ref{lemma5.3} that
	\beq\ba\label{o136}
	&\frac{1}{2(\me^\mu+\me^{-\mu}+2)}\abs{{f}(x)-f^*(x)}^2  \leq\int_{\{-1,1\}}\ykh{\phi(y{f}(x))-\phi(yf^*(x))}\mr{d}P(y|x)\\
	&\leq \frac{1}{4}\abs{f(x)-f^*(x)}^2,\;\forall\;\text{measurable } f:[0,1]^d\to [-\mu,\mu], \;\forall\;x\in\Omega.  
	\ea\eeq
	Take $\widetilde{C}:=2(\me^\mu+\me^{-\mu}+2)$. Integrating both side with respect to $x$ in (\ref{o136}) and using \eqref{230119163},  we obtain
	\beq\label{23030301}
	&\int_{[0,1]^d\times\{-1,1\}}{\ykh{\phi(yf(x))-\phi(yf^*(x))}^2}\mr{d}P(x,y)\\&\leq\int_{[0,1]^d\times\{-1,1\}}{\ykh{f(x)-f^*(x)}^2}\mr{d}P(x,y)
	=\int_{[0,1]^d}{\abs{{f}(x)-f^*(x)}^2 }\mr{d}P_X(x)\\&\xlongequal{\because P_X(\Omega)=1} \int_{\Omega}{\frac{\widetilde C}{2(\me^\mu+\me^{-\mu}+2)}\abs{{f}(x)-f^*(x)}^2 }\mr{d}P_X(x)\\
&\leq \widetilde C\int_{\Omega}{\int_{\{-1,1\}}\ykh{\phi(y{f}(x))-\phi(yf^*(x))}\mr{d}P(y|x) }\mr{d}P_X(x)\\
	&\xlongequal{\because P_X(\Omega)=1} \widetilde{C}\int_{[0,1]^d\times\{-1,1\}}\ykh{\phi(yf(x))-\phi(yf^*(x))}\mr{d}P(x,y)\\&\xlongequal{\text{by Lemma \ref{23022804}}}\widetilde{C}  \mc{E}_P^{\phi}(f),\;\quad \forall\;\text{measurale } f:[0,1]^d\to\zkh{-\mu,\mu}, 
	\eeq
	and
	\beq\label{231019165}
	&\inf\hkh{\mc{E}_P^\phi(f)\left|f\in\fdnn_{d}\ykh{C_{d,\beta}\log\frac{1}{\delta},C_{d,\beta}\delta^{-\frac{d}{\beta}},C_{d,\beta}\delta^{-\frac{d}{\beta}}\log\frac{1}{\delta},1,1}\right.}\\&\leq \mc{E}_P^{\phi}(\tilde{f})\xlongequal{\text{by Lemma \ref{23022804}}}\int_{[0,1]^d}\int_{\{-1,1\}}\ykh{\phi(y\tilde{f}(x))-\phi(yf^*(x))}\mr{d}P(y|x)\mr{d}P_X(x)\\
&\xlongequal{\because P_X(\Omega)=1}\int_{\Omega}\int_{\{-1,1\}}\ykh{\phi(y\tilde{f}(x))-\phi(yf^*(x))}\mr{d}P(y|x)\mr{d}P_X(x)\\
	&\leq \int_{\Omega}\frac{1}{4}\abs{\tilde{f}(x)-f^*(x)}^2\mr{d}P_X(x)\leq  \int_{[0,1]^d}\abs{\tilde{f}(x)-f^*(x)}^2\mr{d}P_X(x)
	\leq\delta^2.\eeq
	Take $\mr c$ to be the maximum of the three constants $C_{d,\beta}$ in \eqref{231019165}.   Hence $\mr c\in(0,\infty)$ only depends on $(d,\beta)$.  Now suppose \eqref{230119171} holds.  Then it follows that there exists $l\in(0,\infty)$ not depending on $n$ and $P$ such that  ${N}\cdot \left(\frac{\log^3 n}{n}\right)^{\frac{d}{d+2\beta}}>l$ and $\frac{S}{\log n}\cdot \left(\frac{\log^3 n}{n}\right)^{\frac{d}{d+2\beta}}>l$ for any $n>1/l$. We then take $\delta=\delta_n:=\ykh{\frac{\mr c}{l}}^{\frac{\beta}{d}}\cdot\ykh{\frac{(\log n)^3}{n}}^{\frac{1}{2+d/\beta}}\asymp \ykh{\frac{(\log n)^3}{n}}^{\frac{1}{2+d/\beta}}$. Thus $\lim_{n\to\infty} \frac{1}{n\cdot\delta_n}=0=\lim_{n\to\infty}\delta_n$, which means that $\frac{1}{n}\leq\delta_n<1/3$ for $n>C_{l,\mr c,d,\beta}$. We then deduce from \eqref{231019165} that
\beq\label{23030302}
&\inf\hkh{\mc{E}_P^\phi(f)\left|f\in\fdnn_{d}\ykh{G,N,S,B,F}\right.}\\
&\leq\inf\hkh{\mc{E}_P^\phi(f)\left|f\in\fdnn_{d}\ykh{\mr c\log n,l \abs{\frac{(\log n)^3}{n}}^{\frac{-d}{2\beta+d}},l \abs{\frac{(\log n)^3}{n}}^{\frac{-d}{2\beta+d}}\log n,B,F}\right.}\\
&=\inf\hkh{\mc{E}_P^\phi(f)\left|f\in\fdnn_{d}\ykh{\mr c\log n,\mr c\delta_n^{-\frac{d}{\beta}},\mr c\delta_n^{-\frac{d}{\beta}}\log n,B,F}\right.}\\
&\leq\inf\hkh{\mc{E}_P^\phi(f)\left|f\in\fdnn_{d}\ykh{\mr c\log\frac{1}{\delta_n},\mr c\delta_n^{-\frac{d}{\beta}},\mr c\delta_n^{-\frac{d}{\beta}}\log\frac{1}{\delta_n},B,F}\right.}\\
&\leq\inf\hkh{\mc{E}_P^\phi(f)\left|f\in\fdnn_{d}\ykh{C_{d,\beta}\log\frac{1}{\delta_n},C_{d,\beta}\delta_n^{-\frac{d}{\beta}},C_{d,\beta}\delta_n^{-\frac{d}{\beta}}\log\frac{1}{\delta_n},1,1}\right.}\\&\leq \delta_n^2,\;\forall\;n>C_{l,\mr c,d,\beta},
\eeq where we use the fact the infimum taken over a larger set is smaller.   Define $W=3\qd\mc{N}\ykh{\fdnn_d\ykh{G,N,S,B,F},\frac{1}{n}}$. 
	Then by taking $\mc{F}=\hkh{\left.f|_{[0,1]^d}\right|f\in\fdnn_d\ykh{G,N,S,B,F}}$, $\psi(x,y)=\phi(yf^*(x))$, $\Gamma=\widetilde C$, $M=2$,  $\gamma=\frac{1}{n}$ in Theorem \ref{thm2.1}, and using \eqref{230303020}, \eqref{23030301}, \eqref{23030302},  we deduce that
	\begin{align*}
	& \bm{E}_{P^{\otimes n}}\zkh{\norm{\hat{f}^{\FNN}_n-f^*_{\phi,P}}_{\mc{L}^2_{P_X}}^2}=\bm{E}_{P^{\otimes n}}\zkh{\norm{\hat{f}^{\FNN}_n-f^*}_{\mc{L}^2_{P_X}}^2}\leq \widetilde{C}\bm{E}_{P^{\otimes n}}\zkh{\mc{E}_P^{\phi}(\hat{f}^{\FNN}_n)}\\&\xlongequal{\text{by Lemma \ref{23022804}}}\widetilde{C}\bm{E}_{P^{\otimes n}}\zkh{\mc{R}^\phi_P\left(\hat{f}^{\FNN}_n\right)-\int_{[0,1]^d\times\{-1,1\}}\psi(x,y)\mr{d}P(x,y)}\\
&\leq  \frac{500\cdot\abs{\widetilde{C}}^2\cdot\log W}{n}+{2}\widetilde{C}\inf_{f \in\mc{F}}\ykh{\mc{R}_P^\phi(f)-\int_{[0,1]^d\times\{-1,1\}}\psi(x,y)\mr{d}P(x,y)}\\
	&\xlongequal{\text{by Lemma \ref{23022804}}} \frac{500\cdot\abs{\widetilde{C}}^2\cdot\log W}{n}+{2}\widetilde{C}\inf_{f \in\mc{F}}\mc{E}_P^\phi(f)\leq \frac{{500\cdot\abs{\widetilde{C}}^2}\cdot\log W}{n}+{2}\widetilde{C}\delta_n^2
	\end{align*} for $n>C_{l,\mr c,d,\beta}$. Taking the supremum, we obtain, 
	\beq\label{102}
	&\sup_{P\in\mc H^{d,\beta}_7}\bm{E}_{P^{\otimes n}}\zkh{\norm{\hat{f}^{\FNN}_n-f^*_{\phi,P}}_{\mc{L}^2_{P_X}}^2}
	\\&\leq \frac{{500\cdot\abs{\widetilde{C}}^2}\cdot\log W}{n}+{2}\widetilde{C}\delta_n^2,\;\forall\;n>C_{l,\mr c,d,\beta}.
	\eeq  Moreover, it follows from \eqref{230119171} and Corollary \ref{corollaryA1} that
	\begin{align*}
	&\log W\leq (S+Gd+1)(2G+5)\log\ykh{(\max\hkh{N,d}+1) \cdot B\cdot(2nG+2n)}\lesssim(G+S)G\log n\\
	&\lesssim \ykh{\left(\frac{n}{\log^3 n}\right)^{\frac{d}{d+2\beta}}\log {n}+\log n}\cdot(\log n)\cdot(\log n)\lesssim n\cdot\ykh{\frac{(\log n)^3}{n}}^{\frac{2\beta}{d+2\beta}}.
	\end{align*}
	Plugging this into (\ref{102}), we obtain
	\begin{align*}
		&\sup_{P\in\mc H^{d,\beta}_7}\bm{E}_{P^{\otimes n}}\zkh{\norm{\hat{f}^{\FNN}_n-f^*_{\phi,P}}_{\mc{L}^2_{P_X}}^2}\lesssim \frac{\log W}{n}+\delta_n^2\\&\lesssim \abs{\frac{(\log n)^3}{n}}^{\frac{2\beta}{d+2\beta}}+\abs{\ykh{\frac{(\log n)^3}{n}}^{\frac{1}{2+d/\beta}}}^2\lesssim \abs{\frac{(\log n)^3}{n}}^{\frac{2\beta}{d+2\beta}}, 
	\end{align*}
	which proves the desired result.
\end{proof}

\vskip 0.33in
\bibliographystyle{plain}
\bibliography{ref}

\begin{thebibliography}{10}

\bibitem{bartlett2009}
Martin Anthony and Peter~L Bartlett.
\newblock {\em Neural Network Learning: Theoretical Foundations}.
\newblock Cambridge University Press, New York, NY, 2009.

\bibitem{atkinson2009theoretical}
Kendall Atkinson and Weimin Han.
\newblock {\em Theoretical Numerical Analysis: A Functional Analysis
  Framework}.
\newblock Springer, New York, NY, third edition, 2009.

\bibitem{audibert2007fast}
Jean-Yves Audibert and Alexandre~B Tsybakov.
\newblock Fast learning rates for plug-in classifiers.
\newblock {\em The Annals of Statistics}, 35(2):608--633, 2007.

\bibitem{LocalBartlett}
Peter~L. Bartlett, Olivier Bousquet, and Shahar Mendelson.
\newblock Local {R}ademacher complexities.
\newblock {\em The Annals of Statistics}, 33(4):1497--1537, 2005.

\bibitem{bartlett2019nearly}
Peter~L Bartlett, Nick Harvey, Christopher Liaw, and Abbas Mehrabian.
\newblock Nearly-tight {{{VC}-dimension}} and pseudodimension bounds for
  piecewise linear neural networks.
\newblock {\em The Journal of Machine Learning Research}, 20(1):2285--2301,
  2019.

\bibitem{cucker2007learning}
Felipe Cucker and Ding~Xuan Zhou.
\newblock {\em Learning Theory: An Approximation Theory Viewpoint}.
\newblock Cambridge University Press, New York, NY, 2007.

\bibitem{evans1998partial}
Lawrence~C Evans.
\newblock {\em Partial Differential Equations}.
\newblock American Mathematical Society, Providence, RI, second edition, 2010.

\bibitem{fang2020theory}
Zhiying Fang, Han Feng, Shuo Huang, and Ding-Xuan Zhou.
\newblock Theory of deep convolutional neural networks
  \uppercase\expandafter{\romannumeral2}: Spherical analysis.
\newblock {\em Neural Networks}, 131:154--162, 2020.

\bibitem{farrell2021}
Max~H Farrell, Tengyuan Liang, and Sanjog Misra.
\newblock Deep neural networks for estimation and inference.
\newblock {\em Econometrica}, 89(1):181--213, 2021.

\bibitem{feng2021generalization}
Han Feng, Shuo Huang, and Ding-Xuan Zhou.
\newblock Generalization analysis of {CNNs} for classification on spheres.
\newblock {\em IEEE Transactions on Neural Networks and Learning Systems},
  34(9):6200--6213, 2023.

\bibitem{goodfellow2016deep}
Ian Goodfellow, Yoshua Bengio, and Aaron Courville.
\newblock {\em Deep Learning}.
\newblock MIT press, Cambridge, MA, 2016.

\bibitem{guo2020modeling}
Xin Guo, Lexin Li, and Qiang Wu.
\newblock Modeling interactive components by coordinate kernel polynomial
  models.
\newblock {\em Mathematical Foundations of Computing}, 3(4):263--277, 2020.

\bibitem{guo2017thresholded}
Zheng-Chu Guo, Dao-Hong Xiang, Xin Guo, and Ding-Xuan Zhou.
\newblock Thresholded spectral algorithms for sparse approximations.
\newblock {\em Analysis and Applications}, 15(3):433--455, 2017.

\bibitem{gyorfi2002distribution}
L{\'a}szl{\'o} Gy{\"o}rfi, Michael Kohler, Adam Krzy{\.z}ak, and Harro Walk.
\newblock {\em A Distribution-free Theory of Nonparametric Regression}.
\newblock Springer, New York, NY, 2002.

\bibitem{han2015learning}
Song Han, Jeff Pool, John Tran, and William~J Dally.
\newblock Learning both weights and connections for efficient neural networks.
\newblock In {\em Proceedings of the Advances in Neural Information Processing
  Systems 28}, pages 1135--1143, Montreal, Canada, 2015.

\bibitem{he2022approximation}
Juncai He, Lin Li, and Jinchao Xu.
\newblock Approximation properties of deep {ReLU} {CNN}s.
\newblock {\em Research in the Mathematical Sciences}, 9(3):1--24, 2022.

\bibitem{hinton2012deep}
Geoffrey Hinton, Li~Deng, Dong Yu, George~E Dahl, Abdel-rahman Mohamed, Navdeep
  Jaitly, Andrew Senior, Vincent Vanhoucke, Patrick Nguyen, Tara~N Sainath,
  et~al.
\newblock Deep neural networks for acoustic modeling in speech recognition: The
  shared views of four research groups.
\newblock {\em IEEE Signal Processing Magazine}, 29(6):82--97, 2012.

\bibitem{hu2022minimax}
Tianyang Hu, Ruiqi Liu, Zuofeng Shang, and Guang Cheng.
\newblock Minimax optimal deep neural network classifiers under smooth decision
  boundary.
\newblock {\em arXiv preprint arXiv:2207.01602}, 2022.

\bibitem{hu2021understanding}
Tianyang Hu, Jun Wang, Wenjia Wang, and Zhenguo Li.
\newblock Understanding square loss in training overparametrized neural network
  classifiers.
\newblock In {\em Proceedings of the Advances in Neural Information Processing
  Systems 35}, pages 16495--16508, New Orleans, LA, United States, 2022.

\bibitem{hui2020evaluation}
Like Hui and Mikhail Belkin.
\newblock Evaluation of neural architectures trained with square loss vs
  cross-entropy in classification tasks.
\newblock In {\em Proceedings of the Ninth International Conference on Learning
  Representations}, pages 1--17, Virtual Event, 2021.

\bibitem{imaizumi2019deep}
Masaaki Imaizumi and Kenji Fukumizu.
\newblock Deep neural networks learn non-smooth functions effectively.
\newblock In {\em Proceedings of the Twenty-second International Conference on
  Artificial Intelligence and Statistics}, pages 869--878, Naha, Okinawa,
  Japan, 2019.

\bibitem{iyyer2015deep}
Mohit Iyyer, Varun Manjunatha, Jordan Boyd-Graber, and Hal Daum{\'e}~III.
\newblock Deep unordered composition rivals syntactic methods for text
  classification.
\newblock In {\em Proceedings of the ${\text{Fifty-third}}$ Annual Meeting of
  the Association for Computational Linguistics and the ${\text{Seventh}}$
  International Joint Conference on Natural Language Processing (Volume 1: Long
  Papers)}, pages 1681--1691, Beijing, China, 2015.

\bibitem{janocha2016loss}
Katarzyna Janocha and Wojciech Czarnecki.
\newblock On loss functions for deep neural networks in classification.
\newblock {\em Schedae Informaticae}, 25:49--59, 2016.

\bibitem{ji2021early}
Ziwei Ji, Justin Li, and Matus Telgarsky.
\newblock Early-stopped neural networks are consistent.
\newblock In {\em Proceedings of the Advances in Neural Information Processing
  Systems 34}, pages 1805--1817, Virtual Event, 2021.

\bibitem{johnstone1998oracle}
Iain~M. Johnstone.
\newblock Oracle inequalities and nonparametric function estimation.
\newblock In {\em Proceedings of the Twenty-third {I}nternational {C}ongress of
  {M}athematicians (Volume {\uppercase\expandafter{\romannumeral3}})}, pages
  267--278, {B}erlin, Germany, 1998.

\bibitem{kim2021fast}
Yongdai Kim, Ilsang Ohn, and Dongha Kim.
\newblock Fast convergence rates of deep neural networks for classification.
\newblock {\em Neural Networks}, 138:179--197, 2021.

\bibitem{kingma2014adam}
Diederik~P Kingma and Jimmy Ba.
\newblock {ADAM}: A method for stochastic optimization.
\newblock In {\em Proceedings of the Third International Conference on Learning
  Representations}, pages 1--15, San Diego, CA, United States, 2015.

\bibitem{kohler2022rate}
Michael Kohler, Adam Krzy{\.z}ak, and Benjamin Walter.
\newblock On the rate of convergence of image classifiers based on
  convolutional neural networks.
\newblock {\em Annals of the Institute of Statistical Mathematics},
  74:1085--1108, 2022.

\bibitem{kohler2020statistical}
Michael Kohler and Sophie Langer.
\newblock Statistical theory for image classification using deep convolutional
  neural networks with cross-entropy loss.
\newblock {\em arXiv preprint arXiv:2011.13602}, 2020.

\bibitem{2006Local}
Vladimir Koltchinskii.
\newblock Local {R}ademacher complexities and oracle inequalities in risk
  minimization.
\newblock {\em The Annals of Statistics}, 34(6):2593--2656, 2006.

\bibitem{krizhevsky2012imagenet}
Alex Krizhevsky, Ilya Sutskever, and Geoffrey~E Hinton.
\newblock Imagenet classification with deep convolutional neural networks.
\newblock In {\em Proceedings of the Advances in Neural Information Processing
  Systems 25}, pages 1097--1105, Lake Tahoe, NV, United States, 2012.

\bibitem{lin2018distributed}
Shao-Bo Lin and Ding-Xuan Zhou.
\newblock Distributed kernel-based gradient descent algorithms.
\newblock {\em Constructive Approximation}, 47(2):249--276, 2018.

\bibitem{liu2011hard}
Yufeng Liu, Hao~Helen Zhang, and Yichao Wu.
\newblock Hard or soft classification? {L}arge-margin unified machines.
\newblock {\em Journal of the American Statistical Association},
  106(493):166--177, 2011.

\bibitem{mao2021theory}
Tong Mao, Zhongjie Shi, and Ding-Xuan Zhou.
\newblock Theory of deep convolutional neural networks
  \uppercase\expandafter{\romannumeral3}: Approximating radial functions.
\newblock {\em Neural Networks}, 144:778--790, 2021.

\bibitem{FML}
Mehryar Mohri, Afshin Rostamizadeh, and Ameet Talwalkar.
\newblock {\em Foundations of Machine Learning}.
\newblock MIT press, Cambridge, MA, second edition, 2018.

\bibitem{murphy2012machine}
Kevin~P Murphy.
\newblock {\em Machine Learning: A Probabilistic Perspective}.
\newblock MIT press, Cambridge, MA, 2012.

\bibitem{petersen2018optimal}
Philipp Petersen and Felix Voigtlaender.
\newblock Optimal approximation of piecewise smooth functions using deep {ReLU}
  neural networks.
\newblock {\em Neural Networks}, 108:296--330, 2018.

\bibitem{poggio2015theory}
Tomaso Poggio, Fabio Anselmi, and Lorenzo Rosasco.
\newblock I-theory on depth vs width: Hierarchical function composition.
\newblock Technical report, Center for Brains, Minds and Machines, MIT,
  Cambridge, MA, 2015.

\bibitem{poggio2016and}
Tomaso Poggio, Hrushikesh Mhaskar, Lorenzo Rosasco, Brando Miranda, and Qianli
  Liao.
\newblock Why and when can deep--but not shallow--networks avoid the curse of
  dimensionality: {A} review.
\newblock {\em arXiv preprint arXiv:1611.00740}, 2016.

\bibitem{poggio2017and}
Tomaso Poggio, Hrushikesh Mhaskar, Lorenzo Rosasco, Brando Miranda, and Qianli
  Liao.
\newblock Why and when can deep-but not shallow-networks avoid the curse of
  dimensionality: {A} review.
\newblock {\em International Journal of Automation and Computing},
  14(5):503--519, 2017.

\bibitem{schmidt2020nonparametric}
Johannes Schmidt-Hieber.
\newblock Nonparametric regression using deep neural networks with {ReLU}
  activation function.
\newblock {\em The Annals of Statistics}, 48(4):1875--1897, 2020.

\bibitem{shen2021non}
Guohao Shen, Yuling Jiao, Yuanyuan Lin, and Jian Huang.
\newblock Non-asymptotic excess risk bounds for classification with deep
  convolutional neural networks.
\newblock {\em arXiv preprint arXiv:2105.00292}, 2021.

\bibitem{simonyan2014very}
Karen Simonyan and Andrew Zisserman.
\newblock Very deep convolutional networks for large-scale image recognition.
\newblock In {\em Proceedings of the Third International Conference on Learning
  Representations}, pages 1--14, San Diego, CA, United States, 2015.

\bibitem{steinwart2008support}
Ingo Steinwart and Andreas Christmann.
\newblock {\em Support Vector Machines}.
\newblock Springer, New York, NY, 2008.

\bibitem{stone1982}
Charles~J. Stone.
\newblock Optimal global rates of convergence for nonparametric regression.
\newblock {\em The Annals of Statistics}, 10(4):1040--1053, 1982.

\bibitem{suzuki2018adaptivity}
Taiji Suzuki.
\newblock Adaptivity of deep {ReLU} network for learning in {B}esov and mixed
  smooth {B}esov spaces: {O}ptimal rate and curse of dimensionality.
\newblock In {\em Proceedings of the Seventh International Conference on
  Learning Representations}, pages 1--25, New Orleans, LA, United States, 2019.

\bibitem{tsybakov2004optimal}
Alexander~B Tsybakov.
\newblock Optimal aggregation of classifiers in statistical learning.
\newblock {\em The Annals of Statistics}, 32(1):135--166, 2004.

\bibitem{tsybakov2009introduction}
Alexandre~B Tsybakov.
\newblock {\em Introduction to Nonparametric Estimation}.
\newblock Springer, New York, NY, 2009.

\bibitem{wainwright2019high}
Martin~J Wainwright.
\newblock {\em High-dimensional Statistics: A Non-asymptotic Viewpoint}.
\newblock Cambridge University Press, Cambridge, 2019.

\bibitem{xiang2011classification}
Dao-Hong Xiang.
\newblock Classification with {G}aussians and convex loss
  \uppercase\expandafter{\romannumeral2}: {I}mproving error bounds by noise
  conditions.
\newblock {\em Science China Mathematics}, 54(1):165--171, 2011.

\bibitem{yarotsky2017error}
Dmitry Yarotsky.
\newblock Error bounds for approximations with deep {ReLU} networks.
\newblock {\em Neural Networks}, 94:103--114, 2017.

\bibitem{zhang202304cnn}
Zihan Zhang, Lei Shi, and Ding-Xuan Zhou.
\newblock Convolutional neural networks for spherical data classification.
\newblock {\em preprint}, 2024.

\bibitem{zdx2018}
Ding-Xuan Zhou.
\newblock Deep distributed convolutional neural networks: Universality.
\newblock {\em Analysis and Applications}, 16(6):895--919, 2018.

\bibitem{zdxdown}
Ding-Xuan Zhou.
\newblock Theory of deep convolutional neural networks: Downsampling.
\newblock {\em Neural Networks}, 124:319--327, 2020.

\bibitem{zhou2020universality}
Ding-Xuan Zhou.
\newblock Universality of deep convolutional neural networks.
\newblock {\em Applied and Computational Harmonic Analysis}, 48(2):787--794,
  2020.

\end{thebibliography}

\end{document}